\documentclass[mnsc,nonblindrev]{informs3}
\OneAndAHalfSpacedXI 

\usepackage[hidelinks]{hyperref}
\usepackage{amsmath,amssymb,mathtools}
\usepackage[size=small]{subcaption}
\usepackage{bm,bbm,booktabs}
\usepackage{algorithm,algpseudocode}
\usepackage{enumerate}
\usepackage{xcolor}
\usepackage{comment}
\usepackage{multirow}
\usepackage{anyfontsize}

\newcommand{\vX}{\mathbf{X}}
\newcommand{\vY}{\mathbf{Y}}
\newcommand{\vE}{\bm{\epsilon}}
\newcommand{\vI}{\mathbf{I}}
\newcommand{\E}{\mathbb{E}}
\newcommand{\R}{\mathbb{R}}

\newcommand{\N}{\mathbb{N}}
\renewcommand{\P}{\mathbb{P}}
\newcommand{\cO}{\mathcal{O}}
\newcommand{\cX}{\mathcal{X}}
\newcommand{\cN}{\mathcal{N}}
\newcommand{\cP}{\mathcal{P}}
\newcommand{\cK}{\mathcal{K}}
\newcommand{\cC}{\mathcal{C}}
\newcommand{\cG}{\mathcal{G}}
\newcommand{\cI}{\mathcal{I}}
\newcommand{\cS}{\mathcal{S}}
\newcommand{\cB}{\mathcal{B}}
\newcommand{\cF}{\mathcal{F}}
\newcommand{\cH}{\mathcal{H}}
\newcommand{\cJ}{\mathcal{J}}
\newcommand{\cQ}{\mathcal{Q}}
\newcommand{\cW}{\mathcal{W}}
\newcommand{\cT}{\mathcal{T}}
\newcommand{\cU}{\mathcal{U}}
\newcommand{\cV}{\mathcal{V}}
\newcommand{\cD}{\mathcal{D}}
\newcommand{\cE}{\mathcal{E}}
\newcommand{\cL}{\mathcal{L}}
\newcommand{\cM}{\mathcal{M}}
\newcommand{\cA}{\mathcal{A}}
\DeclareMathOperator{\tr}{tr}
\DeclareMathOperator{\var}{Var}
\newcommand{\zero}{{\bf 0}}

\newcommand{\bb}[1]{\left[#1\right]}
\newcommand{\bp}[1]{\left(#1\right)}
\newcommand{\bc}[1]{\left\{#1\right\}}

\usepackage{natbib}
 \bibpunct[, ]{(}{)}{,}{a}{}{,}%
 \def\bibfont{\small}%
 \def\bibsep{\smallskipamount}%
 \def\bibhang{24pt}%
\TheoremsNumberedThrough     
\ECRepeatTheorems

\EquationsNumberedThrough    

\MANUSCRIPTNO{MS-0001-1922.65}

\begin{document}



\RUNTITLE{Multitask Learning and Bandits}

\TITLE{Multitask Learning and Bandits\\ via Robust Statistics}

\ARTICLEAUTHORS{%
\AUTHOR{Kan Xu}
\AFF{W. P. Carey School of Business, Arizona State University, \EMAIL{kanxu1@asu.edu}} 
\AUTHOR{Hamsa Bastani}
\AFF{Wharton School, University of Pennsylvania, \EMAIL{hamsab@wharton.upenn.edu}} 
} 

\ABSTRACT{%
Decision-makers often simultaneously face many related but heterogeneous learning problems. For instance, a large retailer may wish to learn product demand at different stores to solve pricing or inventory problems, making it desirable to learn jointly for stores serving similar customers; alternatively, a hospital network may wish to learn patient risk at different providers to allocate personalized interventions, making it desirable to learn jointly for hospitals serving similar patient populations. Motivated by real datasets, we study a natural setting where the unknown parameter in each learning instance can be decomposed into a shared global parameter plus a sparse instance-specific term. We propose a novel two-stage multitask learning estimator that exploits this structure in a sample-efficient way, using a unique combination of \emph{robust statistics} (to learn across similar instances) and \emph{LASSO regression} (to debias the results). Our estimator yields improved sample complexity bounds in the feature dimension $d$ relative to commonly-employed estimators; this improvement is exponential for ``data-poor'' instances, which benefit the most from multitask learning. We illustrate the utility of these results for online learning by embedding our multitask estimator within simultaneous contextual bandit algorithms. We specify a dynamic calibration of our estimator to appropriately balance the bias-variance tradeoff over time, improving the resulting regret bounds in the context dimension $d$. Finally, we illustrate the value of our approach on synthetic and real datasets.
}%


\maketitle

%


\section{Introduction}

Predictive analytics powers data-driven decision-making across many domains. However, many problems in practice suffer from ``small data'' --- i.e., only a very limited quantity of labeled data is available from the target predictive task, hindering training of highly accurate predictive models. As a consequence, a common solution is to leverage training data from related (but different) predictive tasks to reduce variance. In other words, we have an opportunity to not only learn \textit{within} each predictive task, but also \textit{across} similar tasks. To illustrate, consider the following two examples from healthcare and revenue management respectively:

\begin{example}[Medical Risk Scoring] \label{ex:hosp}
Health providers seek to predict patient-specific risk for adverse events (e.g., diabetes) in order to target preventative interventions. To this end, in our experiments in \S\ref{sec:experiments}, we use electronic medical record data to predict the likelihood of an upcoming Type II diabetes diagnosis for patients. Learning this risk score primarily from patient data collected at the \textit{target} hospital (where the patient is being seen and treatment decisions will be made) is important to account for idiosyncrasies that are specific to the hospital and the patient population it serves. Indeed, we find that a predictive model trained using electronic medical record data from one hospital performs quite poorly when evaluated on patients from other hospitals (see Figure~\ref{fig:datashift_test} below), with the out-of-sample AUC degrading significantly from 0.8 at the target hospital to 0.5-0.65 at other hospitals. This is due to a well-known phenomenon called \textit{dataset shift} \citep{quinonero2008dataset}; in the medical context, this can arise due to systematic differences across hospitals in diagnosis/treatment behavior, healthcare utilization, or medical coding \citep[see, e.g.,][]{subbaswamy2020development, bastani2020predicting, mullainathan2017does}. Therefore, to obtain good performance for all patients, \textit{each} hospital faces a distinct learning problem. Yet, we may expect hospitals that serve similar patient populations to have similar underlying predictive models, creating an important opportunity to transfer knowledge across problem instances.
\end{example}

\begin{figure}[htbp]
\centering
\includegraphics[width=.55\textwidth]{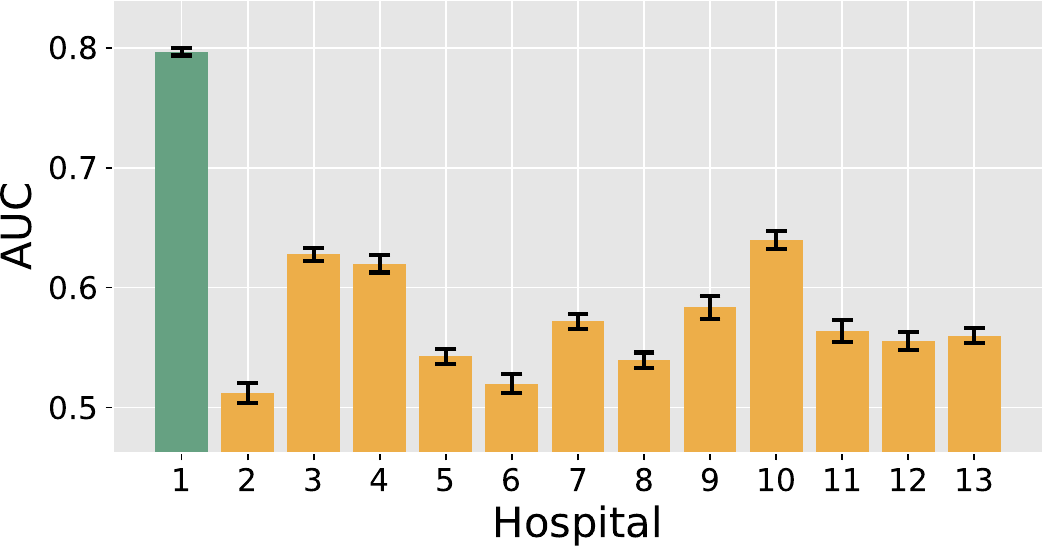}
\caption{Out-of-sample performance (measured by AUC) of a predictive model trained on data from Hospital 1 evaluated on patients from Hospital 1 (green) and Hospitals 2 - 13 (yellow). Point estimates and 95\% confidence intervals are based on 1,000 random draws. We observe a significant degradation in predictive performance in non-target hospitals due to dataset shift.} 
\label{fig:datashift_test}
\end{figure}

\begin{example}[Demand Prediction] \label{ex:store}
Large retailers need to predict store-specific demand for their various products to inform dynamic pricing or inventory management decisions. Again, to protect from dataset shift, it is important to learn this demand model primarily from sales data collected at the \textit{target} store (where sales occurred and decisions will be taken); this will account for idiosyncrasies that are specific to the store and the customer population it serves, including systematic differences in customer trends/preferences, in-store product placement, or promotion decisions \citep[see, e.g.,][]{baardman2020detecting,cohen2018promotion,van2012relationship}. As a result, \textit{each} store faces a distinct learning problem. Yet, we may expect stores that serve similar customer populations to have similar underlying demand models, creating an important opportunity to transfer knowledge across problem instances.
\end{example}

There are numerous other examples where we wish to learn predictive models across related tasks to inform targeted decision-making policies --- e.g., customer promotion targeting for many different promotions, A/B testing on platforms for many candidate interventions, and clinical trials for many promising combination therapies. While the data-driven decision-making literature typically considers a single decision-maker solving an isolated problem instance, we focus on developing algorithms for the setting with many (potentially simultaneous) related learning tasks. We also extend our approach to online learning via simultaneous contextual bandits --- predictive algorithms that are effective with ``small data'' are especially useful here because bandits are largely used in problems where there is little historical data available, e.g., due to the novelty or nonstationarity of the learning problem, or the limited population size relative to the feature dimension. 

We build on the transfer and multitask learning literature~\citep{caruana1997multitask, pan2010survey}, which proposes general algorithms to transfer knowledge across problem instances to improve learning. Unfortunately, these algorithms typically do not improve parameter recovery bounds (ignoring constants) --- i.e., they do not significantly improve predictive performance compared to treating each learning task as its own independent problem. Indeed, in general, transfer or multitask learning cannot improve predictive accuracy without assuming some form of shared structure connecting the different problem instances --- intuitively, if the predictive tasks are unrelated, then learning in one task cannot significantly improve learning in others \citep{hanneke2020no}. Our work bridges this gap by imposing a natural shared structure --- \textit{sparse heterogeneity} --- motivated via real datasets. By designing a multitask estimator that efficiently exploits this structure, we obtain improved performance bounds in the context dimension $d$ for offline and online learning.

\paragraph{Sparse Heterogeneity.} Each problem instance (or, task) $j$ is parameterized with a predictive parameter vector $\beta^j$ --- e.g., the parameters of a linear regression model predicting the reward of each decision as a function of the observed features. Without loss of generality, we can write
\begin{align*}
\beta^j = \beta^\dagger + \delta^j,
\end{align*}
where $\beta^\dagger$ represents the portion of the parameter vector that is ``shared'' across similar tasks, and $\delta^j$ is the task-specific portion that represents idiosyncratic biases specific to task $j$. Sparse heterogeneity imposes that the task-specific bias $\delta^j$ is \textit{sparse} --- i.e., only a few of its components are nonzero~\citep[see, e.g.,][]{bastani2020predicting, xu2021group, tian2022transfer, li2023estimation}. Prior work has argued that this is the case when some (unknown) mechanism systematically affects a subset of the features, e.g., some hospitals under-diagnose certain conditions~\citep{bastani2020predicting}, or a domain change affects the meaning of a subset of words in natural language~\citep{xu2021group}.

We empirically examine this assumption in the context of our previous Example~\ref{ex:hosp}. Specifically, we train separate linear models $\{\widehat{\beta}^j\}_{j=1}^{13}$ for predicting diabetes risk at each of the 13 hospitals using hospital-specific electronic medical record data. Then, we use the trimmed mean to estimate the shared model (for reasons explained in \S \ref{sec:robmulti_est}), and compute the resulting task-specific parameters $\{\widehat{\delta}^j\}_{j=1}^{13}$ for each hospital by subtracting the estimated shared parameter from $\{\widehat{\beta}^j\}_{j=1}^{13}$. If there was no idiosyncratic task-specific bias for each hospital, these parameters would be statistically indistinguishable from zero; on the other hand, if the predictive tasks for each hospital had no shared structure, these parameters would be large and non-sparse. Each row of Figure~\ref{fig:heatmap_deltas} below shows a heatmap of the nonzero coefficients of the task-specific parameter $|\widehat{\delta}^j|$ across 77 features used for prediction. We find that each task exhibits statistically distinguishable hospital-specific idiosyncrasies in the underlying predictive model. Furthermore, in support of our hypothesis of sparse heterogeneity, each $\|\widehat{\delta}^j\|_0 \leq 8 \ll 77$, i.e., task-specific parameters are $s$-sparse with $s/d \lesssim 0.1$ (see Appendix~\ref{app:sprs_heatmap} for more details).

\begin{figure}[htbp]
\centering
\includegraphics[width=.7\textwidth]{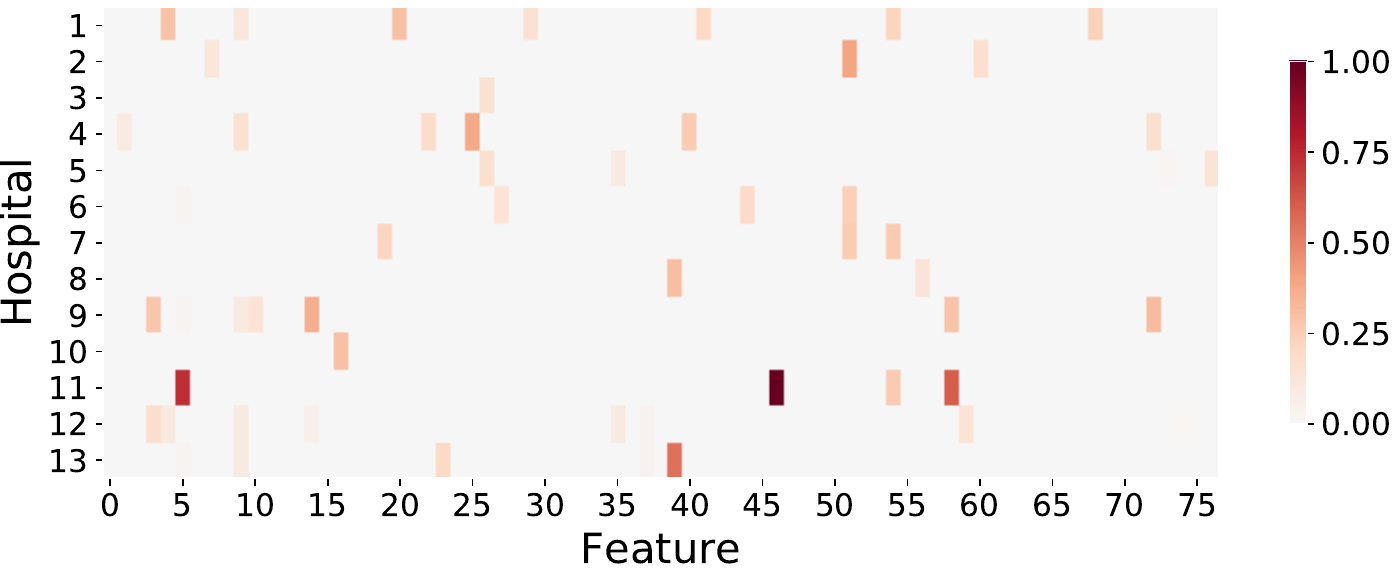}
\caption{Heatmap of nonzero coefficients given by the estimated task-specific parameters $\{|\delta^j|\}_{j=1}^{13}$ for each hospital. Each row represents one of 13 hospitals and each column represents one of 77 features extracted from the data. Nonzero coefficients are determined by a bootstrap hypothesis test across 500 random draws of the data (see Appendix~\ref{app:sprs_heatmap}); we set coefficient $i$ of row $j$ to be zero if the null hypothesis ($\delta_{(i)}^j=0$) is not rejected at a 5\% significance level. Each $\|\widehat{\delta}^j\|_0 \ll d$, lending support for our hypothesis of sparse heterogeneity.} 
\label{fig:heatmap_deltas}
\end{figure}

Existing multitask learning algorithms (e.g., pooling data or regularizing estimates across problem instances) are not designed to leverage this structure (see \S\ref{ssec:related-lit} for an overview of current methods). Thus, we first propose a novel two-stage multitask estimator, \textsf{RMEstimator}, that exploits this structure in the supervised learning setting. In the first stage, it leverages the trimmed mean from robust statistics \citep{rousseeuw1991tutorial, lugosi2021robust} to estimate a ``shared'' model $\widehat\beta^\dagger$ across data collected from similar tasks.\footnote{Note that we do not attempt to estimate the original shared parameter $\beta^\dagger$, since it is not identifiable; rather, as discussed in \S\ref{sec:robmulti_est}, it suffices to estimate some $\widetilde\beta^\dagger$ that lies in an $\ell_0$ ball of radius $\cO(s)$ around $\beta^\dagger$.} Then, in the second stage, it uses LASSO regression \citep{donoho, tibshirani} to efficiently learn the task-specific bias $\delta^j$, which can be combined with our estimate of $\beta^\dagger$ to obtain the task-specific parameter $\beta^j$. We prove finite-sample generalization bounds that show favorable performance (compared to popular baselines), especially in terms of the feature dimension $d$; importantly, this error bound improvement is \textit{exponential} for ``data-poor'' learning tasks, which stand to benefit the most from shared learning.

As noted earlier, we believe such a data-efficient approach is especially useful for bandit problems, since they often operate with limited historical data. To this end, we extend our results to online learning by embedding the \textsf{RMEstimator} within simultaneous linear contextual bandit algorithms running at each problem instance. We specify a dynamic calibration of our estimator within the \textsf{RMBandit} algorithm to appropriately balance the bias-variance tradeoff arising from incorporating auxiliary data from similar bandit instances (multitask learning) in conjunction with the classical exploration-exploitation tradeoff (bandit learning). We derive upper bounds for the cumulative regret of the \textsf{RMBandit}, demonstrating improvements in the context dimension $d$; analogous to the offline setting, this regret improvement is exponential for data-poor bandit instances.

Finally, we empirically evaluate our approach on both synthetic and real datasets in healthcare and pricing. We find that our multitask learning strategy based on the \textsf{RMEstimator} substantially speeds up learning and improves overall performance in both offline and online settings.

\subsection{Related Literature} \label{ssec:related-lit}

Our work relates to the literature on multitask learning and contextual bandits; we contribute on both fronts. Our approach builds on the literature on robust and high-dimensional statistics. 

There has been significant interest from the machine learning community on developing methods that combine data from multiple learning problems (typically referred to as tasks). These can be broadly classified into three categories: (i) multitask learning \citep{caruana1997multitask}, where one aims to learn jointly across a fixed set of similar tasks, (ii) transfer learning \citep{pan2010survey}, a special case of multitask learning, where the goal is to maximize performance on a distinguished ``target'' task, and (iii) meta-learning \citep{finn2017model}, where one aims to learn from historical tasks to improve learning in similar future tasks. Our problem is an instance of multitask learning, since our goal is to learn across a fixed set of problem instances with related unknown parameters.

\textbf{Multitask Learning.} Naturally, if the tasks are sufficiently different, then learning in one task cannot substantially improve learning in other tasks \citep{hanneke2020no}. Thus, a common approach in machine learning is to assume that the underlying parameters across tasks are close in $\ell_2$ norm. Joint learning can then be operationalized by regularizing the estimated parameters together, e.g., through ridge \citep{evgeniou2004regularized} or kernel ridge \citep{evgeniou2005learning} regularization. Alternatively, one can employ a shared Bayesian prior across tasks \citep{raina2006constructing, gupta2021data} or simply pool data from nearby tasks \citep{ben2010theory, crammer2008learning}. However, these approaches do not improve performance bounds beyond constants; in general, one must impose (and exploit) additional structure to obtain nontrivial theoretical improvements. \cite{bastani2020predicting} uses real datasets to motivate the assumption that the parameters across tasks are close in $\ell_0$ norm. This structure motivates a two-step estimator of transfer learning using LASSO regression, yielding improved bounds in the feature dimension $d$ for supervised learning \citep{bastani2020predicting, li2020transfer, tian2022transfer, li2023estimation} and unsupervised learning \citep{xu2021group}. One can further impose that the underlying parameters for each task are sparse, sharing the same support \citep{lounici2009taking, li2023estimation} or similar covariance matrices \citep{li2020transfer, tian2022transfer} across tasks; we do not make these assumptions since the applications we consider often have dense underlying parameters \citep[see, e.g.,][]{bastani2020predicting} and the covariance matrices vary widely across tasks due to covariate shifts \citep[e.g., due to different customer populations at hospitals or stores, see][]{subbaswamy2020development}.

We build on the last stream of two-step estimators for the multitask learning problem. However, we need a fundamentally different algorithmic approach; as we discuss in \S\ref{sec:robmulti_est}, the challenge is that the sparse bias terms can be poorly aligned across tasks, and thus classical estimates of the shared model (e.g., via data pooling or model averaging) destroy task-specific sparse structure and therefore cannot be debiased using LASSO (as was the case in prior work). Instead, we take the view that each component where the bias terms align poorly suffers ``corruptions'' to the shared model; we use a counting argument to show that either the number of corruptions must be small, or the component is one of a small number of well-aligned components. We use robust statistics to overcome corruptions for poorly-aligned components and LASSO to debias well-aligned components. To the best of our knowledge, our work proposes the first such combination of robust statistics and high-dimensional regression, yielding improved bounds for multitask learning. 

The first step of our approach (using robust statistics) relates to recent robust machine learning methods that can handle adversarial corruptions to a small fraction of the data \citep{yin2018byzantine, konstantinov2019robust}. These approaches do not apply to our setting --- as a consequence of our sparse differences assumption, we show that only a few similar \textit{tasks} (as opposed to observations or features) have unknown parameters that are ``corrupted'' in most dimensions. Rather, we build on the classical trimmed mean estimator \citep{rousseeuw1991tutorial, lugosi2021robust}. The second step (using LASSO) builds on the high-dimensional statistics literature \citep{tibshirani, candes, bickel, buhlmann2011statistics}.

\textbf{Multitask Bandits.} A few recent papers have studied multitask learning across contextual bandit instances; however, to the best of our knowledge, a key drawback of these algorithms is that none of them ultimately improve the regret bounds for any bandit instance beyond constants. Similar to the multitask learning literature discussed above, one strategy is to regularize the learned parameters for a given bandit instance towards parameters for similar bandit instances \citep{soare2014multi}. For example, \cite{cesa2013gang} and \cite{deshmukh2017multi} leverage parameter updates that are similar to kernel ridge regularization, and \cite{gentile2014online} additionally perform a pre-processing step clustering bandit instances prior to such regularization; however, the resulting regret bound for a single bandit instance may actually \textit{increase} in the number of instances $N$. Another popular approach is to impose a shared Bayesian prior across bandit instances \citep{cella2020meta, bastani2021meta, kveton2021meta}, but they also obtain similar results; furthermore, these algorithms require the more restrictive assumption that bandit instances appear sequentially (rather than simultaneously) in order to learn the prior.
We embed our robust multitask estimator across $N$ linear contextual bandit instances; the specific setting and assumptions we consider are based on \cite{goldenshluger2013linear,bastani2020online}. We demonstrate that, unlike prior work, we obtain improved regret bounds for each bandit instance in the context dimension $d$ under the practically-motivated sparse heterogeneity assumption; the improvement we obtain is exponential for data-poor instances where shared learning is most helpful. 

We empirically illustrate the value of our approach on well-studied data-driven decision-making problems that leverage contextual bandit algorithms, such as personalized healthcare \citep{bastani2020online,zhalechian2022online} and dynamic pricing \citep{ban2021personalized,wang2021multimodal}.

\subsection{Contributions}

We highlight our main technical contributions below:
\begin{enumerate}
    \item In \S\ref{sec:robmulti_est}, we introduce the \textsf{RMEstimator} for multitask learning, which leverages a unique combination of robust statistics (for learning a shared model across tasks) and LASSO (for debiasing this shared model for a specific task). In \S\ref{ssec:estimator-theory}, we prove upper and lower bounds demonstrating that our estimator outperforms intuitive baselines, and improves existing error bounds in the feature dimension $d$; notably, this improvement is exponential for data-poor tasks.
    \item We extend the \textsf{RMEstimator} to several settings of interest, including robustly learning in the presence of some ``outlier'' tasks (\S\ref{sec:rbs_outlier}), under generalized linear models (\S\ref{ssec:glm}), and when we must choose the subset of similar tasks to learn from (\S\ref{sec:netstructure}). Our results generalize naturally.
    \item In \S\ref{sec:rmbandit}, we embed our estimator in a multitask contextual bandit framework and propose the \textsf{RMBandit} algorithm, with a suitable dynamic calibration of the \textsf{RMEstimator}. We introduce a new batching strategy to ensure conditional independence of our parameter estimates \textit{across} bandit instances for the multitask setting. The resulting regret bounds analogously exhibit up to exponentially improved scaling in the context dimension $d$. 
\end{enumerate}
Finally, we conclude with numerical experiments on both synthetic and real datasets.

\section{Problem Formulation} \label{sec:formulation}

This section formulates multitask learning under sparse heterogeneity in the offline supervised learning setting; we extend our results to the online contextual bandit setting in \S\ref{sec:rmbandit}.

\paragraph{Notation.} Let $[n]$ denote the index set $\{1, 2, \cdots, n\}$. For any vector $\beta \in \R^d$ and $i \in [d]$, let $\beta_{(i)}$ be the $i^{\text{th}}$ element of $\beta$; for any index set $\cI \subseteq [d]$, let $\beta_\cI$ denote the vector obtained by replacing the elements of $\beta$ that are not in $\cI$ with 0's. We use superscripts to index the task, e.g., the design matrix $\vX^j$ represents the covariates observed at task $j$. 
Further, for any square matrix $\Sigma \in \R^{d \times d}$, let $\lambda_{\min}(\Sigma)$ and $\lambda_{\max}(\Sigma)$ denote its minimum and maximum eigenvalues respectively. 
We use the subscript $(i, \cdot)$ to index the $i^{\text{th}}$ row of a matrix, $(\cdot, j)$ to index the $j^{\text{th}}$ column, and $(i, j)$ to index the $(i, j)^{\text{th}}$ element, e.g., $\vX_{(i, \cdot)}$ is the $i^{\text{th}}$ row of matrix $\vX$. 

\paragraph{Model.} We consider $N$ distinct problem instances, each facing a linear learning task, e.g., $N$ service providers such as hospitals in Example~\ref{ex:hosp} and stores in Example~\ref{ex:store}; each task $j\in[N]$ has $n_j$ observations (e.g., patients or customers). An observation $i$ is associated with a $d$-dimensional feature vector $X_i\in\R^d$. The response $Y_i$ of an individual $i$ from task $j$ has
\begin{align*}
Y_i = X_i^\top\beta^j+\epsilon_i,
\end{align*}
where the noise $\epsilon_i$ is an independent $\sigma_j$-subgaussian random variable (see Definition \ref{def:subgaussian}). Let the vector $\vY^j\in\R^{n_j}$ encode all observed responses in task $j$, and the vector $\vE^j\in\R^{n_j}$ encode the corresponding noise terms.\footnote{Note that the subgaussian parameter $\sigma_j$ for the noise is task-dependent; this is because different providers serve potentially very different populations, which is reflected in their feature/noise distributions.}
\begin{definition}\label{def:subgaussian}
A random variable $Z \in \R$ with mean $\mu=\E[Z]$ is $\sigma$-subgaussian if, for any $\lambda \in \R$, $\E\bb{\exp\bp{\lambda(Z-\mu)}} \le \exp\bp{\sigma^2 \lambda^2/2}$.
\end{definition}

The formulation above captures any $N$ linear instances; we now impose our assumption on sparse heterogeneity. As discussed in the introduction, we impose that each task's predictive parameter can be decomposed into a shared parameter $\beta^\dagger$ (that captures the similarity across all $N$ instances) and a task-specific parameter $\delta^j$ (that captures idiosyncratic behavior inherent to task $j$):
\begin{align*}
\beta^j = \beta^\dagger + \delta^j,
\end{align*}
where $\delta^j$ is sparse (i.e., $\|\delta^j\|_0 \le s$ for some $s\in\N$) for all $j\in[N]$. This key assumption enables us to learn across instances efficiently. Note that we do \textit{not} assume that the individual parameters $\{\beta^j\}_{j\in[N]}$ or the shared models $\beta^\dagger$ are themselves sparse, since the responses can often depend on the entire set of observed covariates (see, e.g., discussion in \citealp{bastani2020predicting}).

\begin{remark}
Note that the choice of the shared vector $\beta^\dagger$ here is not unique --- e.g., changes up to $\cO(s)$ components of $\beta^\dagger$ preserve the sparsity of $\delta^j$ up to constant factors --- and therefore is not identifiable. As we describe in \S\ref{sec:robmulti_est}, it suffices for our purposes to estimate any vector $\widetilde\beta^\dagger$ that lies in an $\cO(s)$ ball in $\ell_0$ norm centered around an admissible choice of $\beta^\dagger$.
\end{remark}

For each task $j$, we construct the usual design matrix $\vX^j\in\mathbb{R}^{n_j\times d}$ , where the $i^{\text{th}}$ row $\vX_{(i,\cdot)}^j = X_i^\top$. Following standard practice for regularized regression~\citep{tibshirani1996regression,hastie2009elements}, we standardize each feature such that each column $i$ of the design matrices satisfies
\begin{align}\label{eq:x_stand}
\frac{1}{n_j}\|\vX^j_{(\cdot, i)}\|_2^2 = 1\,.
\end{align}
We further define the corresponding sample covariance matrices as
\begin{align*}
\widehat{\Sigma}^j=\frac{\vX^{j\top}\vX^j}{n_j}\,.
\end{align*}
Due to our normalization, every entry on the diagonal of $\widehat{\Sigma}^j$ is 1.

\paragraph{Performance.} Our goal is to use the observed data $\{(\vX^j, \vY^j)\}_{j \in [N]}$ to estimate the unknown parameter vectors $\{\beta^j\}_{j\in[N]}$ for all tasks. We measure the performance of an estimator $\widehat{\beta}^j$ by its $\ell_1$ error, i.e., $\|\widehat{\beta}^j - \beta^j\|_1$; a good estimator $\widehat{\beta}^j$ has a small estimator error, and hence a small prediction error by noting that $|X_i^\top(\widehat{\beta}^j - \beta^j)|\le\|X_i\|_{\infty}\|\widehat{\beta}^j - \beta^j\|_1$.

We first analyze the fixed design setting where the observations $\vX^j$ are treated as given and normalized as in \eqref{eq:x_stand} (\S\ref{ssec:rme_theory}). We then show how our results straightforwardly extend to the random design setting where $X_i$ is drawn i.i.d. from some (potentially unknown) distribution $\cP_X^j$ (\S\ref{sec:robmulti_est_random}). Note that the feature and noise distributions can vary as a function of the task $j$, since different providers serve different populations.

\paragraph{Regimes of Interest.} While our primary result on the performance of the \textsf{RMEstimator} (Theorem~\ref{thm:tmean_reg_hpb} in \S\ref{ssec:rme_theory}) is general with respect to the sample sizes $\{n_j\}_{j=1}^N$, we will find it useful to interpret the implications under two intuitive regimes. The first is the ``standard'' regime where all tasks have roughly similar numbers of observations (i.e., $n_j = \Theta(n_{j'})=\Theta(n/N)$ for any $j, j'\in[N]$, where $n=\sum_{j\in[N]}n_j$). However, in some settings, some service providers may receive substantially less traffic than others (e.g., a rural hospital in Example~\ref{ex:hosp} or a relatively small store in Example~\ref{ex:store}). Thus, we also consider the limit where some task $j\in[N]$ is relatively ``data-poor'', receiving far less traffic than other tasks (i.e., $n_j = \Theta(n_j'/d^2)$ for any $j' \neq j$); we refer to this as the ``data-poor'' setting (Theorem~\ref{thm:tmean_reg_hpb_dp} in \S\ref{sec:robmulti_est_datapoor}). We focus on a single data-poor task for simplicity; our results generalize straightforwardly to the case where there are a constant number of data-poor tasks. 

\section{Robust Multitask Learning}\label{sec:robmulti_est}

In this section, we overview our robust multitask estimator \textsf{RMEstimator} (\S\ref{sec:robmulti_est_overview}) and provide intuition for its design (\S\ref{sec:RM-intuition}). \S\ref{sec:comp_baselines} illustrates the statistical benefits of our approach relative to intuitive baselines.

\subsection{Preliminaries} \label{ssec:estimator-prelim}

We define and briefly review the \emph{trimmed mean} estimator from the classical robust statistics literature \citep{rousseeuw1991tutorial, lugosi2021robust}, which computes the mean of a distribution $\cP$ given samples $\{Z_j\}_{j\in[N]}$. A typical setting is as follows: most of the samples are i.i.d. (i.e., $Z_j\sim\cP$), but a small fraction (indexed by an unknown set $\cJ\subseteq[N]$) are ``corrupted'' and can be arbitrary. Here, the traditional mean can be arbitrarily biased, but the trimmed mean obtains strong guarantees given a bound on the number of corrupted samples $|\cJ|<\zeta N$ for some $\zeta < 1/2$. The trimmed mean estimator first sorts the samples in increasing order to obtain $Z_{j_1}\le \cdots \le Z_{j_N}$, where the subscript $j_\iota$ is the index of the $\iota^{\text{th}}$ smallest sample. Then, given a hyperparameter $\omega>\zeta$, it removes the top and bottom $\omega$ quantiles or $N\omega$ values and takes the mean of the remaining ones --- i.e.,
\begin{align*}
\texttt{TrimmedMean}\left(\{Z_j\}_{j\in[N]},\;\omega\right)=\frac{1}{N(1-2\omega)}\sum_{\iota=N\omega+1}^{N(1-\omega)}Z_{j_\iota}.
\end{align*}
Intuitively, this estimator is robust since either the corruptions are among the deleted values, or they are sufficiently close to the true mean that they do not significantly affect the estimate.

\subsection{Algorithm Description}\label{sec:robmulti_est_overview}

Our robust multitask estimator is summarized in Algorithm~\ref{alg:tmean_reg}. At a high level, the first step combines high-variance ordinary least squares (OLS) estimators across instances using robust statistics to estimate the shared parameter $\beta^\dagger$ (up to $\cO(s)$ deviations in $\ell_0$ norm); then, the second step uses LASSO regression to debias this estimate for each task $j \in [N]$.

\begin{algorithm}
\begin{algorithmic}
\State \textbf{Inputs:} $\lambda, \omega$
\State Initialize $\lambda_j = \lambda/\sqrt{n_j}$
\For {$j \in [N]$}
\State Let $\widehat{\beta}_{\text{ind}}^j=(\vX^{j\top}\vX^j)^{-1}\vX^\top \vY^j$ be the OLS estimator for task $j$
\EndFor
\For{$i\in[d]$}
\State Let $\widehat{\beta}_{\text{RM},(i)}^\dagger=\texttt{TrimmedMean}(\{\widehat{\beta}_{\text{ind},(i)}^j\}_{j \in [N]},\;\omega)$ be the element-wise trimmed mean
\EndFor
\For {$j \in [N]$}
\State Compute $\widehat{\beta}_{\text{RM}}^j = \argmin_{\beta} \left\{\frac{1}{n_j}\|\vX^j\beta - \vY^j\|_2^2 + \lambda_j \|\beta - \widehat\beta_{\text{RM}}^\dagger\|_1\right\}$
\EndFor
\State \textbf{Outputs:} $\{\widehat\beta_{\text{RM}}^j\}_{j \in [N]}$
\end{algorithmic}
\caption{Robust Multitask Estimator (\textsf{RMEstimator})}
\label{alg:tmean_reg}
\end{algorithm}

In more detail,
\begin{itemize}
\item \textbf{Step 1 (Estimating $\beta^\dagger$):} We compute the usual OLS estimator
\begin{align*}
\widehat\beta^j_{\text{ind}}=(\mathbf{X}^{j\top}\mathbf{X}^j)^{-1}\mathbf{X}^{j\top}\vY^j
\end{align*}
for each task $j\in[N]$ independently. Then, we combine these estimates using the element-wise trimmed mean to estimate the shared parameter vector $\widehat\beta_{\text{RM}}^\dagger\approx\beta^\dagger$ --- i.e., for each $i\in[d]$,
\begin{align}\label{eq:tm_step1}
\widehat\beta_{\text{RM},(i)}^\dagger=\texttt{TrimmedMean}\left(\{\widehat\beta_{\text{ind},(i)}^j\}_{j\in[N]},\;\omega\right),
\end{align}
where $\omega > 0$ is the trimming hyperparameter that we specify later. (Recall that the subscript $(i)$ represents the $i^{\text{th}}$ entry of the vector.)
\item \textbf{Step 2 (Estimating $\beta^j):$} Next, we use LASSO regression to compute $\widehat\beta_{\text{RM}}^j$, leveraging our assumption that the instance-specific bias term $\beta^j-\beta^\dagger$ is sparse:
\begin{align*}
\widehat\beta_{\text{RM}}^j=\argmin_\beta\left\{\frac{1}{n_j}\|\vX^j\beta-\vY^j\|_2^2+\lambda_j\|\beta-\widehat\beta_{\text{RM}}^\dagger\|_1\right\}
\end{align*}
\end{itemize}

\subsection{Design Intuition} \label{sec:RM-intuition}

We now provide intuition for our design choices relative to alternative strategies; the corresponding error rates are summarized in Table~\ref{tab:rates} (see \S\ref{sec:comp_baselines} for precise definitions and more details).
\begin{table}[H]
\centering
\begin{tabular}{lccc}
\toprule
\textbf{Estimator} &
\multicolumn{2}{c}{\textbf{Estimation Error}} &
\textbf{Bound Type} \\
& \textit{Standard Regime} & \textit{Data-Poor Regime} & \\
\midrule
Independent $\widehat\beta_{\text{ind}}^j$ & $\frac{d}{\sqrt{n_j}}$ & $\frac{d}{\sqrt{n_j}}$ & Lower \\
Averaging $\widehat\beta_{\text{avg}}^j$ or Pooling $\widehat\beta_{\text{pool}}^j$ & $\|\delta^j\|_1+\frac{d}{\sqrt{Nn_j}}$ & $\|\delta^j\|_1+\frac{1}{\sqrt{Nn_j}}$ & Lower \\
Averaging Multitask $\widehat\beta_{\text{AM}}^j$ & $\frac{\min\{Ns,d\}}{\sqrt{n_j}} + \frac{d}{\sqrt{Nn_j}}$ & $\frac{\min\{Ns, d\}}{\sqrt{n_j}}$ & Lower \\
\midrule
Robust Multitask $\widehat\beta_{\text{RM}}^j$ & $\sqrt{\frac{sd}{n_j}}+\frac{d}{\sqrt{Nn_j}}$ & $\frac{s}{\sqrt{n_j}}$ & Upper \\
\bottomrule
\end{tabular}
\caption{Comparison of parameter estimation error $\sup_{\cG} \E\bb{\|\widehat\beta^j - \beta^j\|_1}$ (see \S\ref{sec:comp_baselines} for the precise definitions of these estimators); constants and logarithmic factors are omitted for clarity. The upper bound for our robust multitask estimator outperforms the worst-case lower bounds for intuitive baseline estimators under the same set of problem settings $\cG$; our improvement is largest for data-poor tasks.}
\label{tab:rates}
\end{table}
One strategy is to simply use the \emph{independent} OLS estimator $\widehat\beta_{\text{ind}}^j$ (from Step 1) to estimate $\beta^j$; this is an unbiased estimator, but has very high variance since it only uses the limited data observed in task $j$ and does not leverage shared structure across instances. As a result, it has high error when $n_j$ is small (see Table~\ref{tab:rates}).

An alternative strategy is to estimate the shared model $\beta^\dagger$ using data across instances, e.g., the \emph{averaging} estimator takes the model average of the independent estimators:
\begin{align*}
\widehat\beta_{\text{avg}}^j=\frac{1}{N}\sum_{i\in[N]}\widehat\beta_{\text{ind}}^i.
\end{align*}
This estimator has low variance since it leverages data across tasks, but it is biased since it does not account for the task-specific idiosyncratic bias term $\delta^j=\beta^j-\beta^\dagger$. Similarly, estimating the shared model $\beta^\dagger$ through OLS on data \emph{pooled} across instances suffers the same drawbacks. As shown in Table~\ref{tab:rates}, the error of such estimators never approaches zero due to the bias term $\delta^j$.

Thus, a natural two-step strategy to achieve low variance and low bias is to first compute an estimate $\widehat\beta^\dagger$ of the shared parameter, and then try to debias it to estimate $\beta^j$. Since the bias $\beta^j - \beta^\dagger$ is $s$-sparse by assumption, it should intuitively be easier to debias $\widehat\beta^\dagger$ than to directly estimate $\beta^j$. 

Along these lines, consider the following \emph{averaging multitask} estimator, denoted by the subscript $\text{AM}$. Here, we estimate the shared parameter via model averaging, $\widehat\beta_{\text{AM}}^\dagger = \widehat\beta_{\text{avg}}^j$. Then, we use an $\ell_1$ penalty on $\beta - \widehat\beta_{\text{AM}}^\dagger$ (i.e., LASSO regression) on data from instance $j$ to debias $\widehat\beta_{\text{AM}}^\dagger$:
\begin{align}\label{eq:prob_avglasso}
\widehat\beta_{\text{AM}}^j=\argmin_\beta\left\{\frac{1}{n_j}\|\vX^j\beta-Y^j\|_2^2+\lambda_j\|\beta-\widehat\beta_{\text{AM}}^\dagger\|_1\right\}.
\end{align}
(Note that this strategy is identical to Algorithm \ref{alg:tmean_reg}, except it uses the traditional mean instead of the trimmed mean in Step 1.) To see why equation~\eqref{eq:prob_avglasso} helps, suppose we had a perfect estimate of the shared model $\widehat\beta_{\text{AM}}^\dagger = \beta^\dagger$; then, $\beta^j - \widehat\beta_{\text{AM}}^\dagger$ would be $s$-sparse, in which case LASSO requires exponentially fewer observations for recovering $\beta^j$ (relative to $\widehat\beta_{\text{AM}}^\dagger$) than traditional OLS.

The issue with the approach outlined above is that $\beta^j - \widehat\beta_{\text{AM}}^\dagger$ is \textit{not} $s$-sparse, or even ``close'' to being $s$-sparse. To illustrate, we can decompose
\begin{align}\label{eq:dcmp_avglasso}
\beta^j-\widehat\beta_{\text{AM}}^\dagger=\underbrace{\beta^j-\beta^\dagger}_{\text{$s$-sparse}}+\underbrace{\beta^\dagger-\widetilde\beta_{\text{AM}}^\dagger}_{(Ns)\text{-sparse}}+\underbrace{\widetilde\beta_{\text{AM}}^\dagger-\widehat\beta_{\text{AM}}^\dagger}_{\text{not sparse but small}},
\quad\text{where}~
\widetilde\beta_{\text{AM}}^\dagger=\frac{1}{N}\sum_{j\in[N]}\beta^j
\end{align}
Here, $\widetilde\beta_{\text{AM}}^\dagger$ is the value that $\widehat\beta_{\text{AM}}^\dagger$ converges to as $n_j\to\infty$ for all $j\in[N]$. Note that $\widehat\beta_{\text{AM}}^\dagger$ does not converge to $\beta^\dagger$; in fact, as noted in the problem formulation, $\beta^\dagger$ is not identifiable.
The first term in the decomposition is sparse, and the third term becomes small as $n=\sum_{j\in[N]}n_j$ becomes large (since $\widehat\beta_{\text{AM}}^\dagger$ effectively uses all $n$ samples to estimate $\widetilde\beta_{\text{AM}}^\dagger$); since LASSO can effectively recover parameters that are approximately sparse, these two terms are not problematic. The key issue is the second term:
\begin{align*}
\widetilde\delta_{\text{AM}}^\dagger
= \widetilde\beta_{\text{AM}}^\dagger - \beta^\dagger
= \frac{1}{N}\sum_{j\in[N]}(\beta^j - \beta^\dagger)
= \frac{1}{N}\sum_{j\in[N]}\delta^j,
\end{align*}
which is neither sparse nor small. This is illustrated in Figure~\ref{fig:robust}: since the support of the different bias terms $\{\delta_i\}_{i\in[N]}$ can be ``poorly-aligned'' (i.e., the idiosyncrasies for each task affect a different subset of features), the average across instances can result in $\widetilde\delta_{\text{AM}}^\dagger$ having as many as $\min\{Ns,d\}$ nonzero components (even as $n_j\rightarrow\infty$ for all $j\in[N]$). This in turn implies that $\beta^j-\widehat\beta_{\text{AM}}^\dagger$ is not sparse even for moderate values of $N$ such as $N=\Omega(d/s)$; thus, we cannot use LASSO to efficiently debias $\widehat{\beta}^\dagger_{\text{AM}}$.  Other classical estimators of the shared parameter (e.g., data pooling) suffer the same issue.

\begin{remark}
    Recall Figure~\ref{fig:heatmap_deltas}, mapping the estimated bias terms $\{\widehat{\delta}_i\}_{i\in[N]}$ from electronic medical record data for diabetes prediction (Example~\ref{ex:hosp}). Indeed, we observe that $\{\widehat\delta_i\}_{i\in[N]}$ are ``poorly-aligned'' (i.e., share support on different subsets of features), similar to the illustration in Figure~\ref{fig:robust}.
\end{remark}

\begin{figure}
\centering
\includegraphics[width=.7\textwidth]{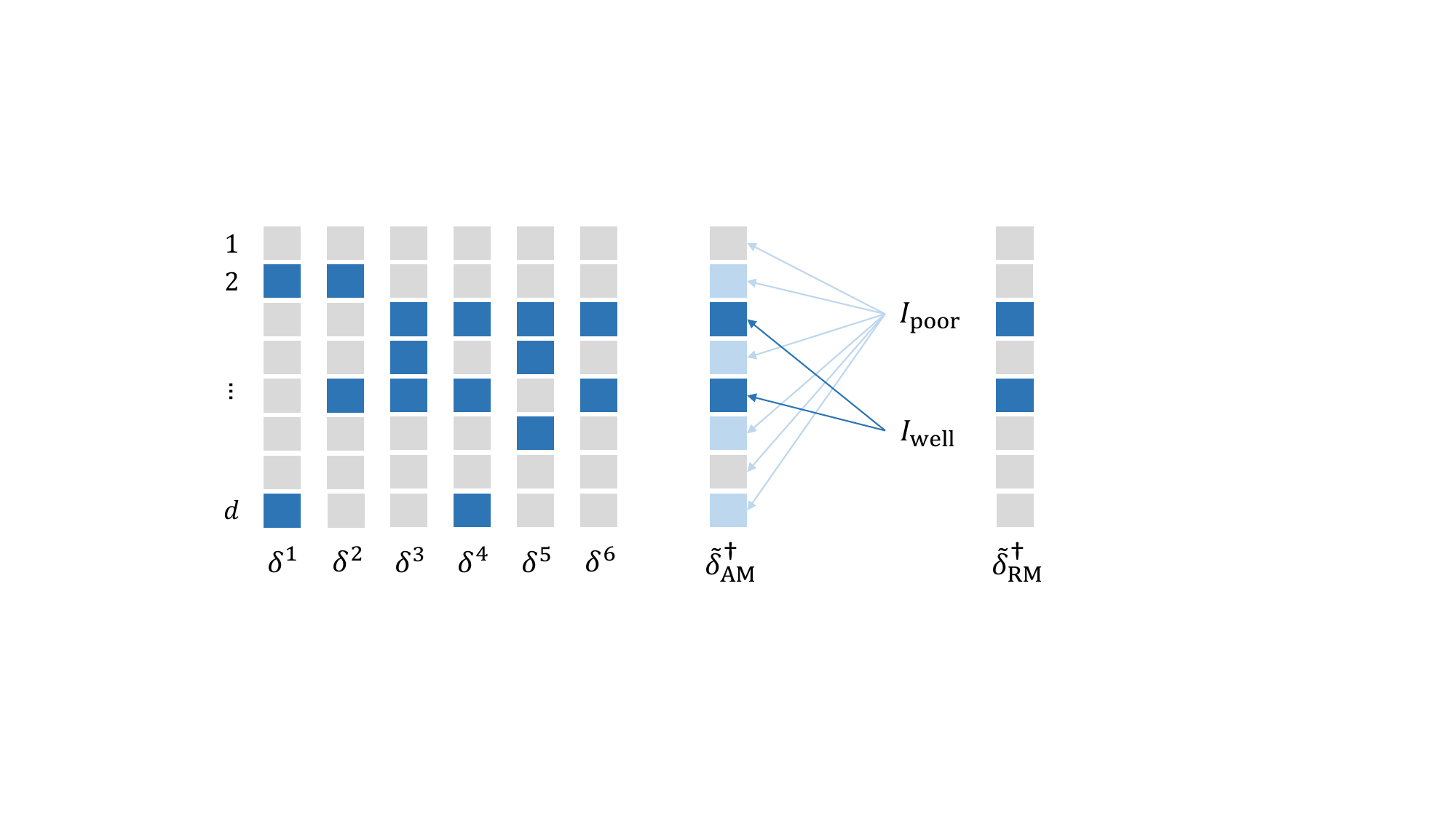}
\caption{Illustration of Step 1 of our robust multitask estimator for debiasing data collected from multiple instances. Blue squares depict the support; the shade of blue depicts the magnitude. $\cI_{\text{poor}}$ represents the index set which can be debiased using the trimmed mean across instances, while $\cI_{\text{well}}$ represents the index set which can be debiased using a subsequent LASSO regression for the target instance.} \label{fig:robust}
\end{figure}

Our \emph{robust multitask} estimator addresses this issue by using the trimmed mean $\widehat\beta_{\text{RM}}^\dagger$ in Step~1; we will show that this converges to a value $\widetilde\beta_{\text{RM}}^\dagger$ (as $n_j \rightarrow \infty$ for all $j\in[N]$) such that 
\begin{align*}
\widetilde\delta_{\text{RM}}^\dagger
= \widetilde\beta_{\text{RM}}^\dagger-\beta^\dagger
= \texttt{TrimmedMean}\bp{\{\beta^j\}_{j\in[N]} - \beta^\dagger,\;\omega}
= \texttt{TrimmedMean}\bp{\{\delta^j\}_{j\in[N]},\;\omega}
\end{align*}
is $\cO(s)$-sparse. In particular, we have the following decomposition:
\begin{align}
\label{eqn:rmdecomp}
\beta^j-\widehat\beta_{\text{RM}}^\dagger=\underbrace{\beta^j-\beta^\dagger}_{\text{$s$-sparse}}+\underbrace{\beta^\dagger-\widetilde\beta_{\text{RM}}^\dagger}_{O(s)\text{-sparse}}+\underbrace{\widetilde\beta_{\text{RM}}^\dagger-\widehat\beta_{\text{RM}}^\dagger}_{\text{not sparse but small}}.
\end{align}
As discussed above, the third term becomes small as $n$ becomes large. Since the second term $\widetilde\delta_{\text{RM}}^\dagger$ is $\cO(s)$-sparse, $\beta^j-\widehat\beta_{\text{RM}}^\dagger$ is approximately $\cO(s)$-sparse; thus, LASSO can efficiently debias $\widehat{\beta}^\dagger_{\text{RM}}$.

We now use a counting argument to illustrate why $\widetilde\delta_{\text{RM}}^\dagger$ is $\cO(s)$-sparse. As Figure~\ref{fig:robust} illustrates, we can separate the components $i\in[d]$ into two groups: ones that are ``well-aligned'' ($i\in \cI_{\text{well}}$) and ones that are ``poorly-aligned'' ($i\in \cI_{\text{poor}}$); see Definition~\ref{def:alignment} below. A poorly-aligned component $i$ is one where very few tasks $j\in[N]$ are biased in this component, i.e., $\beta_{(i)}^j\neq\beta_{(i)}^\dagger$. Intuitively, for each such component, the trimmed mean estimator treats these biased tasks as ``corruptions'' to our samples $\{\beta_{(i)}^j\}_{j\in[N]}$, and trims them (with high probability) when computing the average to obtain an unbiased estimate of $\beta_{(i)}^\dagger$. On the other hand, well-aligned components may remain arbitrarily biased. However, the pigeonhole principle implies that there cannot be many well-aligned components; thus, these components (in addition to the components affected by the sparse task-specific bias term) can be efficiently debiased by LASSO in Step 2. We now formalize this.
\begin{definition}[Well- and poorly-aligned components]
\label{def:alignment}
Given a constant $\zeta\in[0,1]$, a component $i\in[d]$ is $\zeta$-\emph{poorly-aligned} (denoted $i\in \cI_{\text{poor}}^\zeta$) if
\begin{align*}
\frac{|\{j\in[N]\mid\beta_{(i)}^j\neq\beta_{(i)}^\dagger\}|}{N}<\zeta.
\end{align*}
Otherwise, it is \emph{$\zeta$-well-aligned} (denoted $i\in \cI_{\text{well}}^\zeta$).
\end{definition}

In other words, a component $i$ is $\zeta$-poorly-aligned if at most $\zeta$ fraction of the $N$ tasks satisfy $\beta_{(i)}^j\neq\beta_{(i)}^\dagger$. Now, Step 1 constructs an estimator $\widehat\beta_{\text{RM}}^\dagger$ of $\beta^\dagger$ that converges to
\begin{align*}
\widetilde\beta_{\text{RM},(i)}^\dagger=
\begin{cases}
\beta^\dagger_{(i)}& \text{if}~i\in \cI_{\text{poor}}^\zeta \\
\text{unspecified}& \text{if}~i\in \cI_{\text{well}}^\zeta
\end{cases}
\end{align*}
as the sample sizes $\{n_j\}_{j\in[N]}$ become large. That is, we aim to correctly estimate all the poorly-aligned components, but the well-aligned components can be anything. We estimate each component $\beta^\dagger_{(i)}$ using the trimmed mean, which is robust to a small fraction $\zeta$ of arbitrarily corrupted samples.
For a given component $i$, let the corresponding corrupted tasks be
\begin{align*}
\cJ_i=\{j\in[N]\mid\beta^j_{(i)}\neq\beta^\dagger_{(i)}\}.
\end{align*}
By definition, for $i\in \cI_{\text{poor}}^\zeta$, we have $|\cJ_i|<N\zeta$. Thus, we can use the trimmed mean estimator to estimate $\beta_{(i)}^\dagger$:
\begin{align*}
\widehat\beta_{\text{RM},(i)}^\dagger=
\texttt{TrimmedMean}\bp{\{\widehat\beta_{\text{ind},(i)}^j\}_{j\in[N]},\, \omega}
\end{align*}
for some $\omega>\zeta$. This strategy ensures that $\widehat\beta_{\text{RM},(i)}^\dagger\approx\beta^\dagger_{(i)}$ for each poorly-aligned component as desired.
Now, note that there can only be a few well-aligned components. In particular, out of the $Nd$ total components in $\{\beta_{(i)}^j\}_{j \in [N]}$, there are at most $Ns$ components where $\beta_{(i)}^j\neq\beta_{(i)}^\dagger$ as a consequence of our sparse heterogeneity assumption. Then, by the pigeonhole principle, we have
\begin{align}\label{eq:wellaligned}
|\cI_{\text{well}}^\zeta|\le\frac{Ns}{N\zeta}=\frac{s}{\zeta}.
\end{align}
In other words, there are at most $s/\zeta$ well-aligned components, so $\widetilde\delta_{\text{RM}}^\dagger$ is $\cO(s)$-sparse as desired (for a constant choice of $\zeta$). Thus, we can efficiently debias our estimate using LASSO in Step 2. For simplicity, we will use $\cI_{\text{pool}}$ and $\cI_{\text{well}}$ to represent $\cI_{\text{pool}}^\zeta$ and $\cI_{\text{well}}^\zeta$ in the following whenever no ambiguity is raised.

\subsection{Comparison with Baselines}\label{sec:comp_baselines}

With the intuition from \S\ref{sec:RM-intuition} in hand, we now formalize the results for the baseline estimators in Table~\ref{tab:rates}. In particular, we contrast the \textit{upper bound} of our estimator (that we will derive in \S\ref{ssec:rme_theory}) with \textit{lower bounds} of these baselines. The proofs are provided in Appendix~\ref{app:comp_baselines}. 

We characterize the estimation error of an estimator $\widehat\beta^j$ through the following loss function:
\begin{align}\label{eq:eval_lwrbd}
\ell(\widehat\beta^j, \beta^j) = \sup_{\cG} \E\bb{\|\widehat\beta^j - \beta^j\|_1},
\end{align}
where $\cG = \{\{\vX^j\}_{j\in[N]}, \{\beta^j\}_{j\in[N]}, \{\cP_\epsilon^j\}_{j\in[N]}\}$ is the set of problem instances that satisfies our fixed design formulation in \eqref{eq:x_stand} and has positive-definite sample covariance matrices $\{\widehat{\Sigma}^j\}_{j \in [N]}$ (see Assumption~\ref{ass:posdef_off} in \S\ref{ssec:rme_theory}), ensuring that the independent OLS estimators are well-defined. $\cP_\epsilon^j$ is the distribution of the subgaussian noise $\vE^j$; the expectation is taken with respect to the distribution $\cP_\epsilon^j$. This choice of $\cG$ ensures that our upper and lower bounds are with respect to the same class of problem instances. 
We consider the following estimators:
\begin{itemize}
\item \textbf{Independent OLS:} This is the OLS estimator $\widehat{\beta}_{\text{ind}}^j=(\vX^{j\top}\vX^j)^{-1}\vX^\top \vY^j$ trained on data only from task $j$; it does not transfer information or perform any learning across tasks.
\begin{proposition}\label{prop:est_ind}
The estimation error of the independent estimator in both the standard and data-poor regimes satisfies
\begin{align*}
\ell(\widehat{\beta}^j_{\text{ind}}, \beta^j)
= \Omega\bp{\frac{d}{\sqrt{n_j}}}.
\end{align*}
\end{proposition}
\item \textbf{Averaging \& Pooling:} The averaging estimator $\widehat\beta_{\text{avg}}^j=\frac{1}{N}\sum_{i\in[N]}\widehat\beta_{\text{ind}}^i$ is a common approach that averages the independent OLS estimates across tasks to reduce variance (see, e.g., \citealp{dobriban2021distributed}); the pooling estimator $\widehat{\beta}^j_{\text{pool}} = \bp{\sum_{i \in [N]} \vX^{i\top}\vX^i}^{-1}\bp{\sum_{i \in [N]} \vX^{i\top} \vY^i}$ pools data across tasks and then train a single OLS estimator (see, e.g., \citealp{crammer2008learning,ben2010theory}). 
Note that both two estimators do not vary across instances $j$'s. The pooling estimator, different from the averaging estimator, accounts for differences in the sample covariance matrices $\widehat\Sigma^j=\vX^{j\top}\vX^j/n_j$ across instances. 
\begin{proposition}\label{prop:est_avg}
The estimation error of the averaging (pooling) estimator in the standard regime satisfies
\begin{align*}
\ell(\widehat\beta^j_{\text{avg}}, \beta^j) 
& = \Omega\bp{\left\|\delta^j\right\|_1 + \frac{d}{\sqrt{Nn_j}}},
\end{align*}
and in the data-poor regime satisfies
\begin{align*}
\ell(\widehat\beta^j_{\text{avg}}, \beta^j) 
& = \Omega\bp{\left\|\delta^j\right\|_1 + \frac{1}{\sqrt{Nn_j}}}.
\end{align*}
\end{proposition}
\item \textbf{Averaging Multitask:} This two-step estimator $\widehat\beta_{\text{AM}}^j$ is described in detail in \S\ref{sec:RM-intuition}. It is an ablation of our robust multitask estimator that uses the traditional mean rather than the trimmed mean in Step 1. The proof follows a LASSO lower bound argument as in Theorem 7.1 of \cite{lounici2011oracle}, where $\lambda_j$ takes the value chosen to upper bound the estimation error. The lower bound on the error of this estimator demonstrates the importance of robustness. 
\begin{proposition}\label{prop:est_avglasso}
Let the hyperparameter $\lambda_j=\sqrt{\frac{32(1+D_0)\sigma_j^2 \log(4d)}{n_j}}$ for some constant $D_0\ge3$. Then, the estimation error of the averaging multitask estimator in the standard regime satisfies
\begin{align*}
\ell(\widehat{\beta}^j_{\text{AM}}, \beta^j) 
= \tilde\Omega\bp{\frac{\min\{Ns,d\}}{\sqrt{n_j}} + \frac{d}{\sqrt{Nn_j}}},
\end{align*}
and in the data-poor regime satisfies
\begin{align*}
\ell(\widehat{\beta}^j_{\text{AM}}, \beta^j) 
= \tilde\Omega\bp{\frac{\min\{Ns,d\}}{\sqrt{n_j}}}.
\end{align*}
\end{proposition}
\end{itemize}

As shown in Table~\ref{tab:rates} (and proven in the next section), the upper bound of our \textsf{RMEstimator} outperforms the lower bounds of these estimators by leveraging sparse heterogeneity.

\section{Theoretical Guarantees for Multitask Learning} \label{ssec:estimator-theory}

Next, we prove performance guarantees for the \textsf{RMEstimator} under both fixed (\S\ref{ssec:rme_theory}) and random (\S\ref{sec:robmulti_est_random}) design; the statistical benefits of our approach are magnified for data-poor tasks (\S\ref{sec:robmulti_est_datapoor}). We extend the \textsf{RMEstimator} to several settings of interest, including robustly learning in the presence of some ``outlier'' tasks (\S\ref{sec:rbs_outlier}), under generalized linear models (\S\ref{ssec:glm}), and when we must choose a subset of similar tasks on which to share learning (\S\ref{sec:netstructure}). 

\subsection{Main Result}\label{ssec:rme_theory}

Our fixed design result holds under the standard assumption that the sample covariance matrices $\{\widehat{\Sigma}^j\}_{j\in[N]}$ are positive-definite~\citep{hastie2009elements}, i.e., the OLS estimator is well-defined for each task. (We will subsequently relax this assumption for random designs and for data-poor tasks.)

\begin{assumption}[Positive Definiteness] \label{ass:posdef_off}
There exists a constant $\psi>0$ such that, for any instance $j \in [N]$, we have $\lambda_{\min}(\widehat{\Sigma}^j) \ge \psi$.
\end{assumption}

We first show an intermediate result on the error of the trimmed mean estimator (Step 1 of our \textsf{RMEstimator}) for $N$ data samples, of which $\zeta$ fraction may be arbitrarily corrupted. 
\begin{proposition} \label{prop:tmean_iid}
Suppose we are given $N$ samples $\{Z_j\}_{j\in[N]}$ and a subset $\cJ\subseteq[N]$ of size $|\cJ|<\zeta N$ with $0<\zeta<1/2$, such that $\{Z_j\}_{j\in \cJ^c}$ are independent $\sigma_j$-subgaussian random variables with equal means $\mu=\mathbb{E}[Z_j]$ and $\{Z_j\}_{j\in \cJ}$ can be arbitrarily corrupted. Then, letting $\widehat\mu=\texttt{TrimmedMean}(\{Z_j\}_{j\in[N]},\;\omega)$ with $\omega=\zeta+\eta$, we have
\begin{align*}
\P\bb{|\widehat\mu - \mu| \ge C_0\max_{j\in\cJ^c}\sigma_j\bp{3\zeta + 4\eta}\sqrt{\log(\frac{3}{\eta})}} \le  3\exp\bp{-\frac{N\eta^2}{9}},
\end{align*}
for any $0<\eta\le1/2-1/C_0-\zeta$ with some constant $C_0>2$.
\end{proposition}
The proof is provided in Appendix~\ref{app:tmean_concen}. We use this result to show that $\widehat\beta_{\text{RM},(i)}^\dagger$ is close to the true mean $\beta^\dagger_{(i)}$ for poorly-aligned components $i\in \cI_{\text{poor}}^\zeta$. This result is similar to classical results from robust statistics \citep[see, e.g., ][]{jerryli}, but existing results typically assume that the uncorrupted samples are i.i.d. \citep[see, e.g., ][]{jerryli,lugosi2021robust}, whereas we only require independence (since we wish to apply it to $\{\widehat\beta_{\text{ind},(i)}^j\}_{j\in[N]}$, which might not be identically distributed).

As described in \eqref{eqn:rmdecomp} in \S\ref{sec:RM-intuition}, the remainder term $\beta^j - \widehat\beta_{\text{RM}}^\dagger$ is approximately sparse, allowing us to use the LASSO estimator (Step 2 of the \textsf{RMEstimator}) to efficiently recover each task-specific $\beta^j$. This yields our main error bound for the \textsf{RMEstimator}:
\begin{proposition}\label{prop:tmean_reg_hpb_arb}
Under Assumption~\ref{ass:posdef_off}, the estimator $\widehat{\beta}_{\text{RM}}^j$ satisfies
\begin{align*}
\|\widehat{\beta}_{\text{RM}}^j - \beta^j\|_1 \le \frac{6\lambda_js}{\zeta\psi} + C_0 d\bp{3\zeta + 4\eta}\max_{i \in [N]}\sqrt{\frac{\sigma_i^2}{n_i \psi}\log(\frac{3}{\eta})}
\end{align*}
with at least probability $1 - \bp{3d\exp\bp{-\frac{N\eta^2}{9}} + 2d \exp\bp{-\frac{\lambda_j^2 n_j}{32\sigma_j^2}}}$, for any $\lambda_j>0$, $0<\eta\le1/2-1/C_0-\zeta$ and $0<\zeta<1/2$ with some constant $C_0>2$.
\end{proposition}
We provide a proof in Appendix~\ref{app:tailadpt_multitask_std}. Proposition~\ref{prop:tmean_reg_hpb_arb} holds generally for any choice of regularization parameter $\lambda_j$, but we can choose it to minimize the error bound.

To better interpret the resulting implications, as discussed in \S\ref{sec:formulation}, we consider a ``standard'' regime, where all tasks have roughly similar sample sizes, i.e., $n_j = \Theta(n_{j'})=\Theta(n/N)$ for any $j, j'\in[N]$, where $n=\sum_{j\in[N]}n_j$. Under this regime, 
we have the following theorem for any instance $j$:
\begin{theorem}\label{thm:tmean_reg_hpb}
Under Assumption~\ref{ass:posdef_off}, the estimator $\widehat{\beta}_{\text{RM}}^j$ satisfies
\begin{align*} 
\|\widehat{\beta}_{\text{RM}}^j - \beta^j\|_1 
= & \tilde{\cO}\bp{\sqrt{\frac{sd}{n_j}} + \frac{d}{\sqrt{N n_j}}},
\end{align*}
with at least a probability of $1-\delta$ for any
$\delta\ge\exp\bp{-\frac{N}{9}(\frac{C_0-2}{4C_0})^2+\log(6d)}$ with some constant $C_0>2$, for appropriate choices of hyperparameters $\zeta$, $\eta$, and $\lambda$ provided in Appendix~\ref{app:tailadpt_multitask_std}.
\end{theorem}
The proof is provided in Appendix~\ref{app:tailadpt_multitask_std}, and follows essentially from Proposition~\ref{prop:tmean_reg_hpb_arb}. It is useful to compare the bound above with that of a single task $j$ in the same setting, but which does \textit{not} leverage knowledge sharing with other simultaneous tasks. Recall that the independent OLS estimator $\widehat\beta_{\text{ind}}^j$ on task $j$ yields an estimation error of $\cO(d/\sqrt{n_j})$. In contrast, if the number of tasks is at least $N=\Omega(d/s)$, our robust multitask estimator has an estimation error of at most $\tilde{\cO}\bp{\sqrt{sd/n_j}}$ with high probability, i.e., it improves the upper bound by a factor of $\sqrt{d}$. This can be a substantial improvement in high dimension (large $d$) and underscores the value of learning across tasks. When we have very few tasks from which to share knowledge (i.e., $N = o(d)$), multitask learning is still effective; in the worst case, we obtain the same estimation error as the independent OLS estimator. As we show in \S \ref{sec:experiments}, we obtain improved empirical results in practice even for modest values of $N$. As we will discuss in \S\ref{sec:robmulti_est_datapoor}, our improvement is much larger in the data-poor regime.

\begin{remark}
In Appendix~\ref{app:minimax_rmb}, we show minimax lower bounds that are tight for the data-poor regime, but slightly loose (by a factor of $\cO(\sqrt{d/s})$) in the standard regime.  
\end{remark}

\subsection{Data-Poor Regime} \label{sec:robmulti_est_datapoor}

As discussed earlier, multitask learning is especially effective for data-poor tasks, since shared learning can substantially reduce variance in estimation. To illustrate, we consider a task $j$ which has roughly a factor of $d^2$ fewer observations --- i.e., $n_j = \Theta(n_{j'}/d^2)$ for any $j' \ne j$.
Note that the data-poor task resides in a \textit{high-dimensional} setting (since $n_j \ll d$), and therefore is unlikely to satisfy Assumption~\ref{ass:posdef_off} on positive-definiteness. Thus, 
we replace it with the weaker compatibility condition --- a standard assumption in the high dimensional literature \citep{buhlmann2011statistics} --- which only imposes positive-definiteness in a restricted subspace containing $\beta^j$. 
\begin{definition}
Define the set of matrices for a set $\cS$ and a parameter $\psi > 0$ as
\begin{align*}
\cC(\cS, \psi) = \{\Sigma \in \R^{d \times d} \mid |\cS| v^\top \Sigma v \ge \psi \|v_\cS\|_1^2, \forall \|v_{\cS^c}\|_1 \le 7 \|v_\cS\|_1\}.
\end{align*}
\end{definition}
\begin{assumption}[Compatibility Condition] \label{ass:compcon_off}
There exists a constant $\psi>0$ such that, for the data-poor task $j$, we have $\widehat{\Sigma}^j \in \cC(\bar{\cI}_j, \psi)$, where $\bar{\cI}_j = \cI_{\text{well}} \cup \cI_j$ with $\cI_j=\{i\in[d]\mid\beta_{(i)}^j\neq\beta_{(i)}^\dagger\}$.
\end{assumption}

\begin{remark}\label{rmk:dp_alg}
We also make a minor modification to Algorithm~\ref{alg:tmean_reg} in the data-poor regime: we omit the data-poor task $j$ from the trimmed mean in Step 1 since $\widehat\beta^j_{\text{ind}}$ has very high variance. This outcome can also be achieved by increasing the value of the trimming hyperparameter $\omega$.
\end{remark}

With this modification, we have the following theorem for a data-poor task $j$:
\begin{theorem}\label{thm:tmean_reg_hpb_dp}
Under Assumption~\ref{ass:compcon_off} for the data-poor task $j$ and Assumption~\ref{ass:posdef_off} for all other tasks $i\neq j$, the estimator $\widehat{\beta}_{\text{RM}}^j$ satisfies
\begin{align*} 
\|\widehat{\beta}_{\text{RM}}^j - \beta^j\|_1 
= & \tilde{\cO}\bp{\frac{s}{\sqrt{n_j}}},
\end{align*}
with at least a probability of $1-\delta$ for any
$\delta\ge\exp\bp{-\frac{N-1}{9}(\frac{C_0-2}{4C_0})^2 + \log(6d)}$ with some constant $C_0>2$, for appropriate choices of hyperparameters $\zeta$, $\eta$, and $\lambda$ provided in Appendix~\ref{app:tailadpt_multitask_poor}.
\end{theorem}
The proof is provided in Appendix \ref{app:tailadpt_multitask_poor}, and is similar to that of Theorem~\ref{thm:tmean_reg_hpb}.
For the data-poor task, our estimation error depends only \textit{logarithmically} on the feature dimension $d$ (as opposed to linearly for independent OLS). In other words, the \textsf{RMEstimator} \textit{exponentially} reduces the estimation error (see Table~\ref{tab:rates}), which is especially valuable in data-poor problems.

It is worth noting that the error of task $j$ scales \textit{as if} the parameter $\beta^j$ is $s$-sparse. However, our parameters are \textit{not} sparse, i.e., $\|\beta^j\|_0 = d$. Rather, we achieve this scaling as a consequence of multitask learning. When related tasks are data-rich, they provide a good estimate of the shared model $\beta^\dagger$, which allows us to substantially reduce the dimensionality of our estimation problem by focusing on learning only the task-specific bias term (which is sparse) rather than $\beta^j$ (which is dense). This intuition aligns with similar settings considered in \cite{bastani2020predicting,xu2021group}.

\subsection{Random Design}\label{sec:robmulti_est_random}

Our main result in \S\ref{ssec:rme_theory} holds under a fixed design --- i.e., the $\{\vX^j\}_{j\in[N]}$ are observed. As discussed in \S \ref{sec:formulation}, our results straightforwardly extend to random designs, where each row of $\vX^j$ is randomly drawn from a distribution $\cP_X^j$. Now, instead of assuming positive-definiteness for each sample covariance matrix (Assumption~\ref{ass:posdef_off}), it suffices to only assume it for each true covariance matrix.

\begin{assumption}[Positive Definiteness] \label{ass:posdef_off_true}
There exists a constant $\tilde\psi>0$ such that for any $j \in [N]$ we have $\lambda_{\min}(\Sigma^j) \ge \tilde\psi$.
\end{assumption}

Note that, under a random design, we can no longer standardize the features as in \eqref{eq:x_stand}. Thus, we also assume that the observed features are bounded.
\begin{assumption}[Boundedness] \label{ass:bound_off}
There exists a constant $x_{\max}>0$ such that $\|X_i\|_\infty \le x_{\max}$.
\end{assumption}
\begin{remark}
The literature on OLS/ridge regression typically assumes a bound on the $\ell_2$-norm of the covariate $X_i$ \citep[see, e.g.,][]{hsu2011analysis,hsu2012random,zhang2005learning,smale2007learning}. This kind of assumption can be relaxed using moment or subgaussian conditions (see, e.g., Condition 3 in \citealp{hsu2011analysis}) at the cost of an extra term in the resulting high probability bound.
\end{remark}

Now we introduce a variant of Proposition~\ref{prop:tmean_reg_hpb_arb} below for the random design setting. 
\begin{proposition}\label{prop:tmean_reg_hpb_cov}
Under Assumptions~\ref{ass:posdef_off_true} and \ref{ass:bound_off}, the estimator $\widehat{\beta}_{\text{RM}}^j$ satisfies
\begin{align*}
\|\widehat{\beta}_{\text{RM}}^j - \beta^j\|_1 \le \frac{12\lambda_j s}{\zeta\tilde\psi} + C_0 d\bp{3\zeta + 4\eta}\max_{i \in [N]}\sqrt{\frac{2\sigma_i^2}{n_i \tilde\psi}\log(\frac{3}{\eta})}\,,
\end{align*}
with at least probability
$1 - \bp{3d\exp(-\frac{N\eta^2}{9}) + 2d \exp(-\frac{\lambda_j^2 n_j}{32\sigma_j^2 x_{\max}^2})+ \sum_{i \in [N]}d\exp(-\frac{\tilde\psi n_i}{8dx_{\max}^2})}$,
for any $\lambda_j>0$, $0<\eta\le1/2-1/C_0-\zeta$ and $0<\zeta<1/2$ with some constant $C_0>2$.
\end{proposition}

The proof is provided in Appendix \ref{app:rand_design}. Note that this result is nearly identical to Proposition~\ref{prop:tmean_reg_hpb_arb} in the fixed design setting, but it holds with a slightly smaller probability to account for the (unlikely) event that the random design matrices $\{\widehat\Sigma^j\}_{j\in[N]}$ are nearly singular. We primarily use the random design result to analyze the regret of the \textsf{RMBandit} in \S\ref{sec:rmbandit}.

\subsection{Robustness Against Outlier Tasks}\label{sec:rbs_outlier}

Thus far, we have assumed that all $N$ tasks are similar; yet, in practice, an unknown fraction $\varepsilon$ of these tasks may be \textit{outlier} tasks that do not actually contribute to shared learning. We now show that the \textsf{RMEstimator} is robust to such outlier tasks, and its error degrades gracefully in $\varepsilon$.

In particular, consider a subset $\bar\cJ \subseteq [N]$ of outlier tasks that do not satisfy our sparse heterogeneity assumption, with $|\bar\cJ| \leq \varepsilon N$. Then, for all $j\in\bar\cJ^c$, we still have $\beta^j = \beta^\dagger + \delta^j$ with $\|\delta^j\|_0 \le s$.\footnote{This setup is similar to the task relatedness environment in Assumption 4.3 of \cite{duan2022adaptive}, except that we measure the difference between model parameters in $\ell_1$ norm instead of $\ell_2$ norm. However, their proofs rely heavily on specific properties of the $\ell_2$ norm, and do not carry over to our setting with the $\ell_1$ norm.}

Following our analysis in \S\ref{sec:RM-intuition}, the trimmed mean estimator consistently estimates $\beta^\dagger_{(i)}$ for any components $i\in\cI_{\text{poor}}$ as long as $\zeta > \varepsilon$ (see Definition~\ref{def:alignment}). However, we may now have more well-aligned components due to the presence of $\varepsilon N$ outlier tasks. Specifically, out of the $Nd$ total components, there can be at most $Ns+\varepsilon Nd$ components with $\beta_{(i)}^j\neq\beta_{(i)}^\dagger$ for any non-outlier task $j\in\bar\cJ^c$. By the pigeonhole principle, we then have at most
\begin{align*}
|\cI_{\text{well}}|\le\frac{N(s+\varepsilon d)}{N\zeta}=\frac{s+\varepsilon d}{\zeta} \,.
\end{align*}
As a result, we have the following corollary for task $j$ in the presence of outlier tasks:
\begin{corollary}
\label{cor:tmean_reg_rbs}
Under Assumption~\ref{ass:posdef_off}, the estimator $\widehat{\beta}_{\text{RM}}^j$ satisfies
\begin{align*} 
\|\widehat{\beta}_{\text{RM}}^j - \beta^j\|_1 
=
\begin{cases}
\tilde{\cO}\bp{\sqrt{\frac{(s+\varepsilon d)d}{n_j}} + \frac{d}{\sqrt{N n_j}}}, & \text{for}~j\in\bar\cJ^c \\
\tilde{\cO}\bp{\frac{d}{\sqrt{n_j}}}, & \text{for}~j\in\bar\cJ,
\end{cases}
\end{align*}
with at least probability $1-\delta$ for any
$\delta\ge\exp\bp{-\frac{N}{9}(\frac{C_0-2}{2C_0}(1-\frac{1}{2}\sqrt{\frac{s+2\varepsilon d/3}{d}}))^2 + \log(6d)}$ and $\varepsilon \le \frac{1+8C_0\sqrt{s}/((C_0-2)\sqrt{d})}{(4C_0/(C_0-2))^2}$ for some constant $C_0>2$ and appropriate choices of hyperparameters $\zeta$, $\eta$, and $\lambda$ provided in Appendix~\ref{app:tailadpt_multitask_rbstout}.
\end{corollary}
The proof is provided in Appendix~\ref{app:tailadpt_multitask_rbstout}. Consistent with our prior results, we obtain an improvement in the context dimension $d$ for non-outlier tasks $j\in\bar\cJ^c$. Specifically, the estimation error of \textsf{RMEstimator} scales as $\tilde\cO(\sqrt{(s+\varepsilon d)d/n_j})$ for a sufficiently large number of tasks $N$, which is still smaller than the error of the independent OLS estimator. Yet, we have an additional $\tilde\cO(\sqrt{\epsilon d^2/n_j})$ term, which is slightly weaker compared to Theorem~\ref{thm:tmean_reg_hpb}, since the presence of outlier tasks adds noise. Indeed, when $\varepsilon \rightarrow 0$ (i.e., there are no outlier tasks) we recover our bound in Theorem~\ref{thm:tmean_reg_hpb}; when $\varepsilon \rightarrow 1$ (i.e., no tasks can share knowledge), our improvement disappears and the estimation error converges to that of independent OLS. Lastly, outlier tasks $j\in\bar\cJ$ do not share knowledge with other tasks, and hence also have the same estimation error as independent OLS.

\subsection{Generalized Linear Models}\label{ssec:glm}

Next, we generalize the \textsf{RMEstimator} to generalized linear models (GLMs), which may be more suitable for classification problems. We show that a natural analog of the \textsf{RMEstimator} that uses a maximum likelihood estimator (MLE) achieves a similar error bound as Theorem~\ref{thm:tmean_reg_hpb}.

In GLM, for task $j\in[N]$ with parameter $\beta^j$, the density of $Y_i$ conditioned on features $X_i$ satisfies
\begin{align*}
p(Y_i \mid X_i, \beta^j) \propto \exp\bp{Y_i X_i^\top\beta^j - A(X_i^\top\beta^j)},
\end{align*}
where the function $A:\R \rightarrow\R$ is known \citep{mccullagh1989generalized}. From standard properties of exponential families, $A$ is infinitely differentiable and strictly convex \citep{brown1986fundamentals}; without loss of generality, we assume $A'' > \phi_m$ for some constant $\phi_m>0$. Under this model, we have $\E[Y_i \mid X_i] = A'(X_i^\top\beta^j)$ and $\var[Y_i \mid X_i] = A''(X_i^\top \beta^j)$, where $A'$ and $A''$ are the first and second derivatives of $A$, and $A'$ is called the inverse link function. For example, $A'(x) = 1/(1+\exp(-x))$ and $Y_i$ is binary for logistic regression, and $A'(x)=x$ and $Y_i$ is continuous in linear regression. 

We consider the natural generalization of the \textsf{RMEstimator} from the linear setting in \S\ref{sec:robmulti_est_overview}, replacing both the OLS estimator in Step 1 and the LASSO regression in Step 2 with MLE. In detail,
\begin{itemize}
\item \textbf{Step 1 (Estimating $\beta^\dagger$):} For each task $j\in[N]$, we obtain the MLE through 
\begin{align}\label{eq:loss_glm}
\widehat{\beta}_{\text{ind}}^j = \argmin_{\beta} \cL(\beta \mid \vX^j, \vY^j), 
\quad \text{where}~ \cL(\beta \mid \vX^j, \vY^j) = \frac{1}{n_j}\sum_{i\in[n_j]} - Y_i X_i^\top\beta + A(X_i^\top\beta).
\end{align}
We then estimate the shared parameter $\beta^\dagger$ via the trimmed mean $\widehat\beta_{\text{RM},(i)}^\dagger$ as in \eqref{eq:tm_step1}.
\item \textbf{Step 2 (Estimating $\beta^j):$} Then, we apply a LASSO penalty to the MLE to compute $\widehat\beta_{\text{RM}}^j$:
\begin{align*}
\widehat\beta_{\text{RM}}^j = \argmin_\beta \cL(\beta \mid \vX^j, \vY^j)
+\lambda_j\|\beta-\widehat\beta_{\text{RM}}^\dagger\|_1.
\end{align*}
\end{itemize}

We impose a standard regularity condition on the link function \cite[see, e.g.,][]{negahban2010unified} to provide a guarantee on the convergence of the MLE.
\begin{assumption}[Bounded Hessian]\label{ass:hessbound_off}
There exists a constant $\phi_M>0$ such that $\|A''\|_\infty \le \phi_M$.
\end{assumption}
Then, we obtain the following error bound for the \textsf{RMEstimator} in the GLM setting: 
\begin{corollary}\label{cor:tmean_glm_hpb_arb}
Under Assumptions~\ref{ass:posdef_off}, \ref{ass:bound_off} and \ref{ass:hessbound_off}, the estimator $\widehat{\beta}_{\text{RM}}^j$ of a GLM satisfies
\begin{align*} 
\|\widehat{\beta}_{\text{RM}}^j - \beta^j\|_1 
= & \tilde{\cO}\bp{\sqrt{\frac{sd}{n_j}} + \frac{d}{\sqrt{N n_j}}}\,,
\end{align*}
for sufficiently large $n_j$ with at least probability $1-\delta$ for any
$\delta\ge\exp\bp{-\frac{N}{9}(\frac{C_0-2}{4C_0})^2+\log(6d)}$ for some constant $C_0>2$, and appropriate choices of hyperparameters $\zeta$, $\eta$, and $\lambda$ provided in Appendix~\ref{app:glm}.
\end{corollary}
We provide a proof in Appendix~\ref{app:glm}, which closely follows the linear case, except for the use of MLE to establish consistency of the trimmed mean and LASSO. The result mirrors Theorem~\ref{thm:tmean_reg_hpb}.

\subsection{Network Structure}\label{sec:netstructure}

In practice, one may have a large number of tasks, and may wish to a choose a subset of size $\widetilde{N} \le N$ within which to perform multitask learning. On one hand, increasing the number of tasks $\widetilde{N}$ improves estimation of the shared parameter; on the other hand, an increase in $\widetilde{N}$ may imply an increase in the task-specific sparsity parameter $s$ as tasks become increasingly dissimilar. 

In Appendix \ref{app:regret_network}, we assume knowledge of a network that captures the similarity between any pair of tasks; this can be inferred based on observed covariates (e.g., geographic distance between hospitals/stores or socio-economic indices of neighborhoods served) or data from past decision-making problems \citep[see, e.g., the disparity matrix in][]{crammer2008learning}. Then, for any given task, we can optimize the ``similarity radius'' of learning problems from which to transfer knowledge.
We consider a simple power scaling for this network, where the effective sparsity parameter of the selected subset of tasks varies as $\tilde{s} = \min(\widetilde{N}^\alpha, d)$ for some $\alpha \geq 0$. In this setting, we find it is optimal to choose the $\widetilde{N} = d^\frac{1}{\alpha+1}$ most similar tasks; this results in error bounds (see Corollary~\ref{cor:tmean_bdt_rgt_sglnet}) that scale with the network density $\alpha$. Further details and results are provided in Appendix \ref{app:regret_network}.

\section{Robust Multitask Contextual Bandits} \label{sec:rmbandit}

We now illustrate the utility of the \textsf{RMEstimator} for online learning via simultaneous multitask contextual bandits. This section describes the modified model setup and assumptions for the bandit setting (\S \ref{sec:rmbandit-model} - \ref{sec:rmbandit-assmp}), overviews our proposed \textsf{RMBandit} algorithm (\S\ref{sec:rmbandit-description}), and demonstrates improved total and task-specific regret bounds (\S\ref{sec:rmbandit_mainresult}), including in the data-poor regime (\S\ref{sec:rmbandit-datapoor}).

\subsection{Model}\label{sec:rmbandit-model}

Analogous to the formulation in \S\ref{sec:formulation}, we consider $N$ distinct linear contextual bandit instances. The decision-maker at each instance has access to $K$ potential arms (decisions) with uncertain and context-dependent rewards.

\paragraph{Arrivals.} Let $T$ be the overall time horizon across all bandit instances. At each time step $t$, a new observation arrives at one of the $N$ instances, given by the random variable $Z_t \in [N]$ --- i.e., instance $j$ receives an arrival with probability $p_j$, where $\sum_{i=1}^N p_i = 1$; thus, $Z_t$ follows a categorical distribution $\texttt{CG}(\textbf{p})$ with $\textbf{p}=\begin{bmatrix} p_1 & \cdots & p_N \end{bmatrix}$. In expectation, instance $j$ will have $p_j T$ observations. Again, we consider two regimes: (i) the standard regime where $p_j = \Theta(1/N)$ for all $j\in[N]$, and (ii) the data-poor regime where a single instance $j$ has $p_j = \Theta(p_{j'}/d^2)$ for $j' \neq j$. Each observation has a context $X_t \in \R^d$; if $Z_t=j$, then $X_t$ is drawn i.i.d. from an unknown distribution $\cP_X^j$.

\paragraph{Rewards.} The reward for pulling arm $k$ for an observation with context vector $X_t$ at instance $j$ is $Y_t = X_t^{\top} \beta_k^j + \epsilon_t$. Here, each arm $k$ at instance $j$ is parameterized by an unknown arm parameter $\beta_k^j \in \R^{d}$, and the corresponding noise $\epsilon_t$ is an i.i.d. $\sigma_j$-subgaussian random variable given $Z_t=j$.
We assume that a given arm satisfies sparse heterogeneity across instances --- i.e., there exists $\beta_k^\dagger\in\mathbb{R}^d$ such that $\beta_k^j = \beta_k^\dagger + \delta_k^j$, where $\delta_k^j$ is sparse (i.e., $\|\delta_k^j\|_0 \le s$ for some $s\in\mathbb{N}$) for all $k \in [K], j \in [N]$. 

\paragraph{Performance.} We want a sequential policy $\pi$ that learns the arm parameters $\{\beta_k^j\}_{k\in[K],j\in[N]}$ over time, in order to maximize expected reward for each arrival. The overall policy $\pi$ is composed of sub-policies $\pi^j_t: \cX^j \rightarrow [K]$ at each instance $j$; we use $\pi_t$ to represent the arm played at time $t$.
We measure the performance of $\pi$ by its cumulative expected regret \citep{lai1985asymptotically}, modified to extend across multiple heterogeneous bandit instances. In particular, when $Z_t = j$ (an observation arrives at instance $j$), we compare ourselves to the oracle policy $\pi_*^j$ at instance $j$, which knows the arm parameters $\{\beta_k^j\}_{k \in [K]}$ in advance. Naturally, $\pi_*^j$ chooses the arm with the best expected reward, i.e. $\pi_*^j(X_t) = \argmax_{k \in [K]} X_t^{\top} \beta_k^j$ for any $X_t$. The expected regret incurred by pulling arm $\pi_t = k$ at time $t$ given an arrival at instance $j$ is thus
\begin{align*}
r_t^j = \E\left[\max_{k' \in [K]} (X_t^{\top} \beta_{k'}^j) - X_t^{\top} \beta_k^j \,\middle|\, Z_t = j \right],
\end{align*}
which is simply the difference between the expected reward of using $\pi_*^j$ and $\pi_t^j$. Further taking the expectation over the randomness where the observation arrives, the expected regret of an overall policy composed of sub-policies $\{\pi_t^j\}_{j \in [N]}$ at time $t$ is $r_t = \sum_{j \in [N]} \P\bb{Z_t = j}r_t^j = \sum_{j \in [N]} p_j r_t^j$.
Our goal is to derive a policy that minimizes the cumulative expected regret $R_T = \sum_{t=1}^T r_t$ across all instances; we also study the instance-specific cumulative regret $R_T^j = \sum_{t=1}^T p_j r_t^j$ for each $j\in[N]$.

\subsection{Assumptions}\label{sec:rmbandit-assmp}

As discussed earlier, we embed our robust multitask estimator into the linear contextual bandit setting studied in \cite{goldenshluger2013linear,bastani2020online}; therefore, our next four assumptions are directly adapted from this literature. 

First, we have a standard assumption that the features and arm parameters are bounded. 

\begin{assumption}[Boundedness] \label{ass:bounded}
There exist constants $x_{\max}>0$ and $b>0$ such that $\|X\|_{\infty} \le x_{\max}$ for any $X \in \cX^j, j \in [N]$ and $\|\beta_k^j\|_1 \le b$ for any $k\in[K],j\in[N]$.
\end{assumption}

Our second assumption is that, for each instance $j \in [N]$, the $K$ arms can be split into two mutually exclusive sets:
(i) optimal arms $k \in \cK_{\text{opt}}^j$ that are \textit{strictly} optimal in expected reward (by at least $h$) for any contexts drawn from a set $U_k^j \subseteq \cX^j$ with positive support on $\cP_X^j$, or
(ii) suboptimal arms $k \in \cK_{\text{sub}}$ that are \textit{strictly} suboptimal in expected reward (by at least $h$) for all contexts in $\cX^j$.
In other words, we assume that every arm is either optimal for at least \textit{some} individuals, or suboptimal for \textit{all} individuals. This assumption ensures that every arm in $\cK_{\text{opt}}^j$ will roughly receive at least $p_*p_jT$ samples and quickly learn accurate parameters under a regret-minimizing policy. 

\begin{assumption}[Arm Optimality] \label{ass:armopt} 
All $K$ arms at any given instance $j$ belong to one of two mutually exclusive sets: $\cK_{\text{opt}}^j$ of optimal arms or $\cK_{\text{sub}}^j$ of suboptimal arms. There exists some $h > 0$ such that: (i) each $k \in \cK_{\text{sub}}^j$ satisfies $X^\top \beta_{k}^j < \max_{k' \ne k} X^\top \beta_{k'}^j - h$ for any $X \in \cX^j$, and (ii) each $k \in \cK_{\text{opt}}^j$ is optimal on a set
\begin{align*}
U_k^j = \{X \in \cX^j \mid X^\top\beta_k^j > \max_{k' \ne k} X^\top\beta_{k'}^j + h\},
\end{align*}
with positive measure, i.e., $\P\bb{X \in U_k^j \mid Z=j} \ge p_*$ for some constant $p_* > 0$.
\end{assumption}

Our third assumption ensures that linear regression is feasible within the set $U_k^j$; this is a mild assumption since it is with respect to the \textit{true} covariance matrix, which only requires that no features are perfectly collinear in this set. 
In contrast, the \textit{sample} covariance matrix may not be positive-definite at time $t$, if we have observed too few samples from that instance.

\begin{assumption}[Positive Definiteness] \label{ass:posdef}
For any arm $k \in [K]$ and instance $j \in [N]$, the true covariance matrix $\Sigma_k^j = \E[X X^\top | X \in U_k^j, Z = j]$ satisfies $\lambda_{\min}\bp{\Sigma_k^j} \ge \psi$ for some constant $\psi > 0$.
\end{assumption}

Our fourth assumption is a margin condition that ensures that the density of the context distribution $\cP_X^j$ for each instance $j$ is bounded near a decision boundary (i.e., the hyperplane given by $\{X\in\cX^j \mid X^\top \beta_{k'}^j = X^\top \beta_k^j\}$ for any pair $k' \neq k$). It allows for any bounded, continuous features, as well as any discrete features with a finite number of values.

\begin{assumption}[Margin Condition] \label{ass:marcond}
For any arms $k$ and $k'$ of any instance $j \in [N]$, there exists a constant $L>0$ such that $\P\bb{|X^{\top}(\beta_k^j - \beta_{k'}^j)| \le \kappa \mid Z = j} \le L \kappa$ for any $\kappa > 0$.
\end{assumption}

\begin{remark}
Our regret bounds straightforwardly extend under a more general $\alpha$-margin condition (see, e.g., \citealp{bastani2021mostly}), as shown in Appendix~\ref{sec:margin_cond_dis}.
\end{remark}

The assumptions thus far are standard and have been adapted directly from the bandit literature. We now introduce a new assumption motivated by our multitask setting. In general, an arm $k$ can be optimal (belong to $\cK_{\text{opt}}^j$) at one bandit instance $j$ and be suboptimal (belong to $\cK_{\text{sub}}^{j'}$) at a neighboring instance $j'$. This implies that we will observe $\cO(p_jT)$ samples from arm $k$ at instance $j$ but only $\cO(\log(p_{j'}T))$ samples at instance $i$ under a regret-minimizing policy; in other words, instance $j$ cannot effectively transfer knowledge from instance $i$ about arm $k$. Thus, we impose that if an arm $k \in [K]$ is optimal for \textit{any} instance $j$, it is also optimal for at least \textit{some} subset of the other instances so that we have enough observations to enable multitask learning.

\begin{assumption}[Optimality Density] \label{ass:optdens}
The set of bandit instances where arm $k$ is an optimal arm, i.e., $\cW_k = \{j \in [N] \mid k \in \cK_{\text{opt}}^j\}$, has cardinality of either 0 or at least $\rho N$ for some constant $\rho > 0$ for any $k \in [K]$. 
\end{assumption}

\subsection{Algorithm Description} \label{sec:rmbandit-description}

Next, we propose the \textsf{RMBandit} algorithm (Algorithm~\ref{alg:tmean_bdt}), which leverages the \textsf{RMEstimator} to efficiently learn across $N$ simultaneous linear contextual bandit instances. In this section, we drop the subscript \text{RM} and denote our multitask estimator as simply $\widehat\beta^j_k$ for arm $k$ and instance $j$.

Following prior work, the \textsf{RMBandit} manages the exploration-exploitation tradeoff using a small amount ($\cO(\log(T))$) of forced random exploration in each instance $j\in [N]$. Furthermore, for each instance $j$ and arm $k\in [K]$, it trades off between (i) an unbiased \emph{forced-sample} estimator, which is trained only on forced random samples, and (ii) a potentially biased \emph{all-sample} estimator, which is trained on all observations for arm $k$. Instead of using LASSO \citep{bastani2020online} or OLS \citep{goldenshluger2013linear} for these estimators, we use our \textsf{RMEstimator}. 

This introduces two important challenges. First, our multitask estimator leverages data \textit{across} instances, which induces (previously absent) correlations among any arm parameter estimates for a fixed arm $k$. However, our error bound for the trimmed mean estimator (Proposition~\ref{prop:tmean_iid} - \ref{prop:tmean_reg_hpb_arb}) requires that our OLS inputs across instances into the trimmed mean be independent in order to recover a reasonable estimate of the shared model $\beta^\dagger$ (see Step 1 of our \textsf{RMEstimator} in \S\ref{sec:robmulti_est_overview}). Thus, we introduce a new \textit{batching} strategy, where we only perform parameter updates in batches rather than after every time step. This ensures that our arm parameter estimates in the current batch are independent conditioned on the observations from previous batches.\footnote{A formal statement of this claim is provided in Lemma~\ref{lem:tmean_bdt_ase_ind} in Appendix~\ref{sec:rmbandit-proofstrategy}.} Importantly, this batching strategy does not change the convergence rates (and therefore regret), and has the added advantage of being far more computationally tractable. 

Second, our robust multitask estimator requires two hyperparameters: the trimming hyperparameter $\omega$ and the LASSO regularization parameter $\lambda$ (see Algorithm~\ref{alg:tmean_reg}). We specify a \textit{trimming path} for $\omega_t$ to dynamically trade off bias and variance over time, in order to control the convergence of our robust multitask estimators. Intuitively, we trim less for small $t$ (when we have little data) to reduce variance at the cost of admitting ``small'' corruptions; as $t$ increases (when we have collected more data), we trim more aggressively to eliminate even small corruptions that can bias our estimates. For $\lambda_t$, the path follows that derived in \cite{bastani2020online} correspondingly.

In more detail, we split the time horizon $T$ into disjoint sequential batches $\bigcup_{m\ge 0}\cB_m$. The initial batch $\cB_0$ has size $q\log(T)$ for some tuning parameter $q$, and the following batches iteratively double in length (i.e., the $m^{\text{th}}$ batch $\cB_m$ has size $|\cB_m| = 2^{m-1} |\cB_0|$),
which yields a total of $M$ batches with
\begin{align}\label{eq:defM}
M = \left\lceil\log_2\bp{\frac{T}{q\log(T)}}\right\rceil.
\end{align}
We define $\cB_{\bar{m}}=\bigcup_{l=0}^m\cB_l$ as the union of all the batches up to batch $m$. We denote our \textsf{RMEstimator} (Algorithm~\ref{alg:tmean_reg}) at instance $j$ for arm $k$ as
$\widehat\beta^j_k(\cB,\, \lambda,\, \omega)$,
where the first argument indicates the training sample, i.e., all observations where we pulled arm $k$ in batch $\cB$; the remaining arguments are hyperparameters, i.e., the LASSO regularization parameter $\lambda$ and the trimmed mean parameter $\omega$.

\begin{algorithm}[htbp]
\begin{algorithmic}
\State \textbf{Inputs:} $\omega_0,\lambda_0$, $\zeta_{1,0}, \eta_{1,0}, \lambda_{1, 0}$, $q$, $h$
\State Set $M=\left\lceil\log_2\bp{\frac{T}{q\log(T)}}\right\rceil$, $\cB_0 = [q\log(T)]$, and $\cB_m = \{ t \in [T] \mid 2^{m-1}|\cB_0| < t \le 2^m|\cB_0|\}$ for $m\in[M]$
\State Initialize $\lambda_{0,j}=\lambda_0 /\sqrt{|\cB_0^j|}$
\For {$t \in [T]$}
\State Observe an arrival at instance $j = Z_t \sim \texttt{CG}(\textbf{p})$, and the corresponding context $X_t \sim \cP_X^j$
\If {$t \in \cB_0$}
\State Pull arm $\pi_t = \bp{(\sum_{r\in[t]} \mathbbm{1}(Z_r = j)-1) \bmod K} + 1$
\ElsIf {$t \in \cB_m$}
\State Let $\cK = \{k \in [K] \mid X_t^{\top} \widehat{\beta}_k^j(\cB_0, \lambda_{0,j}, \omega_0) \ge \max_{i \in [K]} X_t^{\top} \widehat{\beta}_i^j(\cB_0, \lambda_{0,j}, \omega_0) - h/2\}$ 
\State Pull arm $\pi_t = \argmax_{k \in \cK} X_t^{\top} \widehat{\beta}_k^j(\cB_{\bar{m-1}}, \lambda_{1,j,\bar{m-1}}, \omega_{1, m-1})$
\EndIf
\State Observe reward $Y_t = X_t^\top \beta_{\pi_t}^j + \epsilon_t$
\If {$t = |\cB_{\bar{m}}|$ for $m \in [M]$ (i.e., when $\cB_m$ ends)}
\State Update $\zeta_{1,m} = \zeta_{1,0}$, $\eta_{1,m} = \eta_{1, 0} \sqrt{\log(d\min_{i \in [N], |\cB_m^i| > 0}|\cB_m^i|)}$, and $\omega_{1,m} = \zeta_{1,m} + \eta_{1,m}$
\State Update $\lambda_{1,j,\bar{m}} = \lambda_{1,0} \sqrt{\log(d|\cB_{\bar{m}}^j|)/|\cB_{\bar{m}}^j|}$ for each $j \in [N]$
\EndIf
\EndFor
\end{algorithmic}
\caption{Robust Multitask Bandit (\textsf{RMBandit})}
\label{alg:tmean_bdt}
\end{algorithm}

\paragraph{Strategy.} In our initial batch $\cB_0$, we deterministically forced-sample each arm $k$ of instance $j$ when an individual is observed at instance $j$ (i.e., when $Z_t = j$). At the end of this initial batch, we obtain a forced-sample estimator $\widehat\beta^j_k(\cB_0, \lambda_{0,j}, \omega_{0})$ for each $j\in[N]$ and $k\in[K]$; this forced-sample estimator remains fixed for the entire time horizon $T$. On the other hand, we also maintain an all-sample estimator $\widehat\beta^j_k(\cB_{\bar{m}}, \lambda_{1,j,\bar{m}}, \omega_{1,m})$ for each $j\in[N]$ and $k\in[K]$; this estimator is periodically re-trained with updated hyperparameters at the end of each batch $m$ and thereby fixed for the following batch $m+1$. One distinction of this estimator from Algorithm~\ref{alg:tmean_reg} is that the trimmed mean estimator in Step 1 is built upon only data from $\cB_m$ such that arm parameter estimates are independent. However, in Step 2 we still use all the data of the target instance from $\cB_{\bar{m}}$, of which the samples might be correlated across instances. 

The algorithm is executed as follows. If $t \in \cB_m$ and a new arrival is observed at instance $j$, we first use the forced-sample estimators to find the highest estimated reward achievable among the $K$ arms at instance $j$. These estimates allow us to identify a subset of arms $\cK \subseteq K$ whose rewards are within $h/2$ of the estimated optimal reward. Then, within this set, we pull the arm $k \in \cK$ that has the highest estimated reward according to the all-sample estimators.

\subsection{Regret Analysis} \label{sec:rmbandit_mainresult}

First, Proposition~\ref{prop:tmean_bdt_rgt_all} below bounds the total regret across all $N$ bandit instances.

\begin{proposition}\label{prop:tmean_bdt_rgt_all}
When $N = \Omega(\log(d)\log(T))$, 
the cumulative expected regret of all instances up to time $T$ of \textsf{RMBandit} satisfies
\begin{align*}
R_T=\cO\bp{Kd(sN+d)\log(N)\log^2(dT/N)}\,,
\end{align*}
for appropriate choices of hyperparameters $\omega_0$, $\zeta_{1,0}$, $\eta_{1,0}$, $\lambda_{0}$, $\lambda_{1,0}$, and $q$ provided in Appendix~\ref{app:regret_sfixed_std}.
\end{proposition}
We provide a proof strategy in Appendix~\ref{sec:rmbandit-proofstrategy}, with details relegated to Appendix~\ref{app:regret_sfixed_std}. Note that we assume that $N$ is not \textit{too} small; our approach is designed for problems where there are at least a few distinct problem instances to enable multitask learning. In \S\ref{sec:experiments}, we show improved empirical results even for modest values of $N$.

\begin{remark}
The above result is based on our margin condition in Assumption \ref{ass:marcond}. As noted in \S\ref{sec:rmbandit-assmp}, we consider a more general $\alpha$-margin condition in Appendix~\ref{sec:margin_cond_dis}, and derive corresponding regret bounds that may scale polynomially in $T$ depending on the strength of the assumption.  
\end{remark}

Next, we consider the regret of \textsf{RMBandit} for a single instance $j$. To make a direct comparison to existing regret bounds with time horizon $T$, we rescale the expected time horizon\footnote{Note that, given a fixed time horizon across all $N$ instances, the time horizon (i.e., number of observations) for a single instance $j$ is a random variable since the distribution of arrivals across instances ($\{Z_t\}_{t=1}^T$) is a random process.} for instance $j$ to be $T$ as well, i.e., since we expect $p_j = \Theta(1/N)$ fraction of the total arrivals at instance $j$, we scale our total horizon as $T/p_j = \Theta(NT)$.
\begin{theorem}\label{thm:tmean_bdt_rgt_sgl}
Consider a total time horizon of $T/p_j = \Theta(NT)$, so instance $j$ has expected time horizon $T$. When $N = \Omega(\log(d)\log(T))$, the cumulative expected regret of instance $j$ of \textsf{RMBandit} satisfies
\begin{align*}
R_T^j = \cO\bp{Kd\bp{s+d/N}\log(N)\log^2(dT)},
\end{align*}
for appropriate choices of hyperparameters $\omega_0$, $\zeta_{1,0}$, $\eta_{1,0}$, $\lambda_{0}$, $\lambda_{1,0}$, and $q$ provided in Appendix~\ref{app:regret_sfixed_std}.
\end{theorem}
The proof is provided in Appendix~\ref{app:regret_sfixed_std}. Prior literature shows that such an instance would achieve regret that scales as $\cO(d^2\log^{\frac{3}{2}}(d)\log(T))$~\citep{bastani2020online}. In contrast, our upper bound on the regret for instance $j$ using \textsf{RMBandit} is smaller by a factor of $d$, but larger by a factor of $\log(T)$;\footnote{The extra factor of $\log(T)$ is likely an analytical limitation that arises because the \textsf{RMBandit} uses LASSO, e.g., the high-dimensional contextual bandit also attains a regret that scales as $\log^2(T)$ \citep{bastani2020online}.} this is a substantial improvement for even a moderate context dimension $d$.
In other words, by appropriately managing the bias-variance tradeoff over time, the \textsf{RMBandit} achieves analogous improvements to online learning as we observed with the \textsf{RMEstimator} in offline learning.

\subsection{Data-Poor Regime} \label{sec:rmbandit-datapoor}

We now turn to the data-poor regime where we expect magnified benefits to multitask learning (matching our offline results in \S\ref{sec:robmulti_est_datapoor}). Let the target instance $j$ be data-poor, i.e., $p_j = \Theta(p_{j'}/d^2)=\Theta(1/(d^2N))$. Again, to make a direct comparison to existing regret bounds, we rescale our total horizon as $T/p_j = \Theta(d^2NT)$, implying an expected time horizon of $T$ for instance $j$.
\begin{theorem}\label{thm:bdt_rgt_sgl_dp}
Consider a total time horizon of $T/p_j = \Theta(d^2NT)$, so the data-poor instance $j$ has expected time horizon $T$. When $N = \Omega(\log(d)\log(T))$, the cumulative expected regret of instance $j$ of \textsf{RMBandit} satisfies
\begin{align*}
R_T^j = \cO\bp{Ks^2\log(N)\log^2(dT)},
\end{align*}
for appropriate choices of hyperparameters $\omega_0$, $\zeta_{1,0}$, $\eta_{1,0}$, $\lambda_{0}$, $\lambda_{1,0}$, and $q$ provided in Appendix~\ref{app:regret_datapoor}.
\end{theorem}
We provide a proof in Appendix~\ref{app:regret_datapoor}. The above result shows that \textsf{RMBandit} attains a regret bound that only scales \textit{logarithmically} in the context dimension $d$ --- i.e., we obtain an \textit{exponential} reduction in regret, which is especially valuable in data-poor problems.
Note that the regret of instance $j$ scales as if the arm parameters $\{\beta_k^j\}$ are $s$-sparse (see, e.g., \cite{bastani2020online}) while our arm parameters are dense; this intuition aligns with our offline result in Theorem~\ref{thm:tmean_reg_hpb_dp}.

\section{Experiments} \label{sec:experiments}

We now illustrate the value of our approach on synthetic and real datasets in both offline and online learning settings. For experiments on real data, we analyze a patient health risk prediction task in the offline setting; in the online setting, we focus on two well-studied applications of bandits --- i.e., personalized patient interventions, and dynamic pricing.

\subsection{Multitask Learning}\label{sec:exp_off}

We first analyze the efficiency of our \textsf{RMEstimator} in the offline setting. 
More specifically, we simulate the following linear regression algorithms: (i) without multitask learning: OLS \citep{hastie2009elements}, LASSO \citep{tibshirani}, and (ii) with multitask learning: group LASSO \citep{yuan2006model,lounici2009taking}, nuclear-norm regularization \citep{negahban2011estimation,pontil2013excess}, pooling \citep{crammer2008learning,ben2010theory} and \textsf{RMEstimator}.
The first two approaches treat each task independently via OLS and LASSO respectively, while the rest share knowledge across $N$ tasks to learn for the target task through group LASSO regularization, nuclear-norm regularization, pooling then OLS, and our \textsf{RMestimator} respectively. 

\textbf{Synthetic.} Figure~\ref{fig:synthetic_off} compares the prediction error across three different settings of our formulation in \S\ref{sec:formulation} with varying $N, d$ and $s$. Appendix \ref{app:exp-synth} provides additional details. 

\begin{figure}[htbp]
\centering
\begin{subfigure}[b]{0.32\textwidth}
  \centering
  \includegraphics[width=\textwidth]{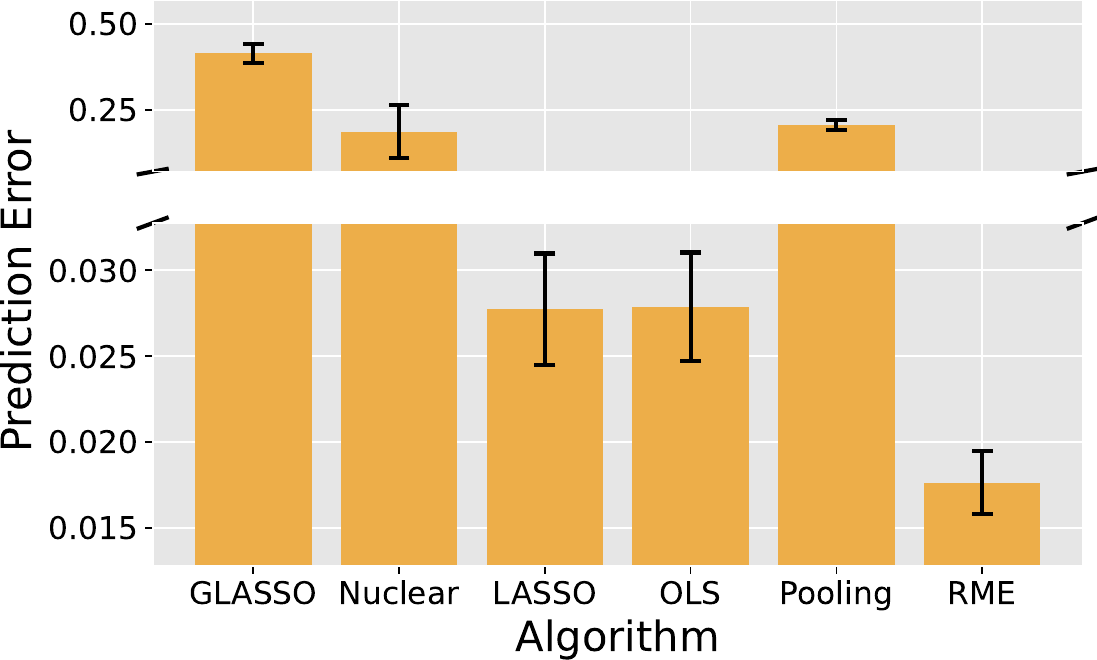}
  \caption{$N=30, d=20, s=2$}
\end{subfigure}
\begin{subfigure}[b]{0.32\textwidth}
  \centering
  \includegraphics[width=\textwidth]{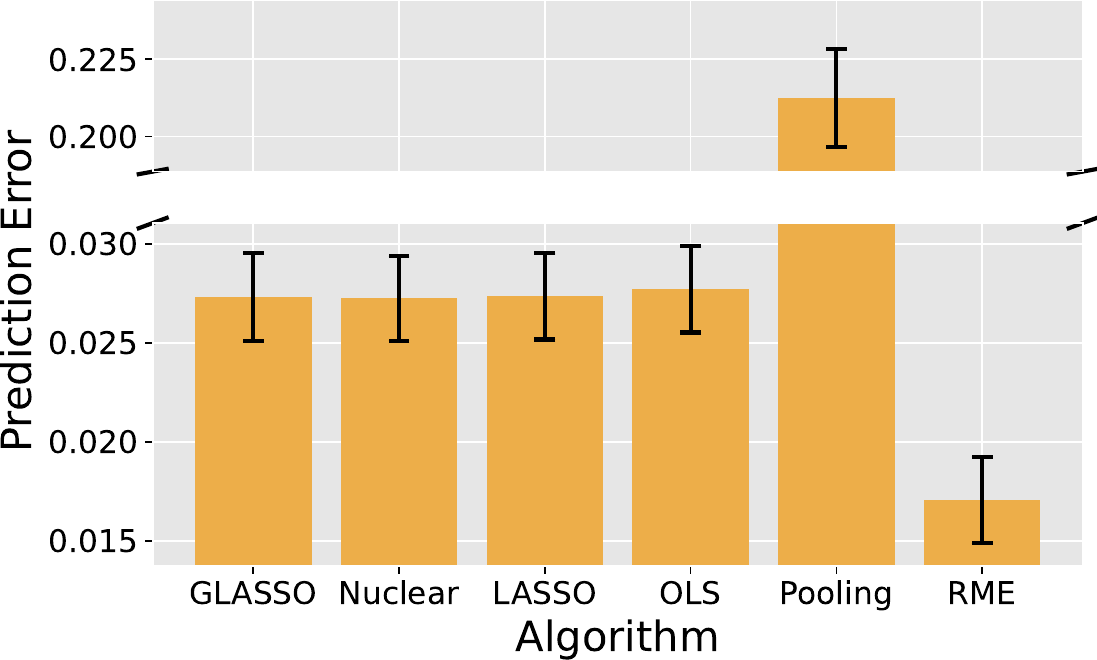}
  \caption{$N=10, d=20, s=2$}
\end{subfigure}
\begin{subfigure}[b]{0.32\textwidth}
  \centering
  \includegraphics[width=\textwidth]{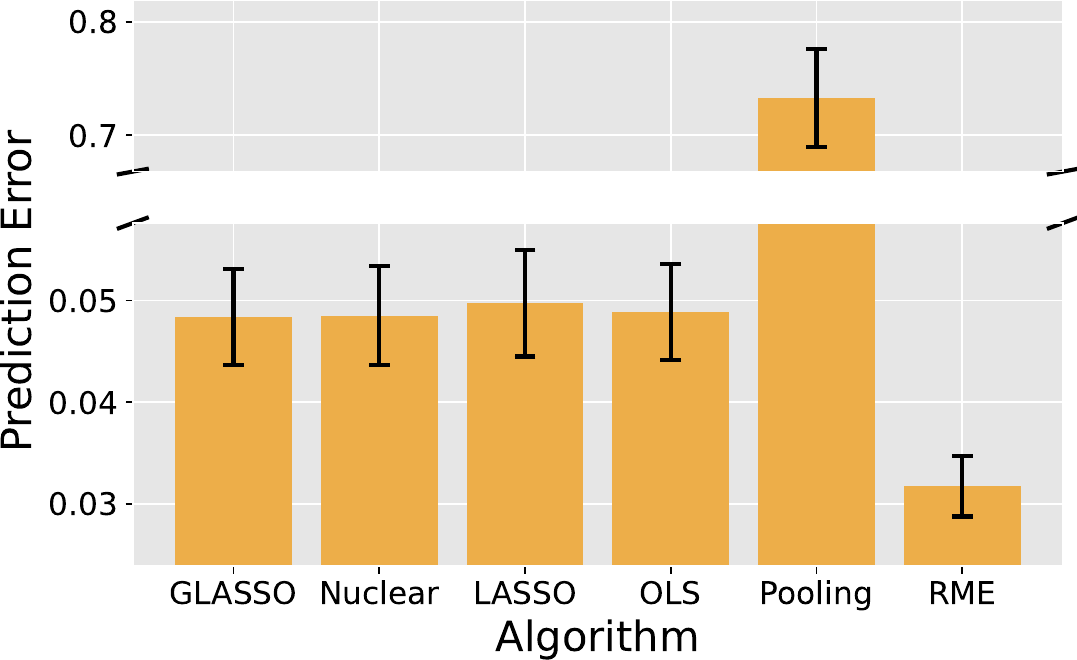}
  \caption{$N=15, d=40, s=5$}
\end{subfigure}
\caption{Bars depict prediction error of one task averaged over 20 trials, with corresponding 95\% confidence intervals. `GLASSO' is group LASSO, `Nuclear' is nuclear-norm regularization, and `RME' is \textsf{RMEstimator}.}
\label{fig:synthetic_off}
\end{figure}

We observe that the \textsf{RMEstimator} substantially reduces prediction error across the board under varying conditions. OLS/LASSO miss out on the opportunity to perform shared learning. Classical estimators that do perform multitask learning impose structural assumptions that are not met by our synthetic data (and more importantly, appear unwarranted in the real dataset we study in the next subsection) and therefore also perform poorly --- e.g., pooling ignores task-specific bias, group LASSO assumes predictive models across tasks are sparse with shared support, and nuclear-norm regularization postulates a shared latent low-rank structure across tasks. 

These results are unsurprising since our synthetic data satisfies sparse heterogeneity, so we expect \textsf{RMEstimator} to outperform other estimators that do not explicitly leverage this structure.
The next experiment examines real electronic medical record data, which may not satisfy our assumptions.

In Appendix~\ref{app:add_exp_deppara}, we comprehensively explore the robustness of the \textsf{RMEstimator} against a range of model parameters (i.e., $N$, $s$, $n_i/n_j$) and assumptions on the task specific bias terms $\{\delta_j\}_{j\in[N]}$ (their magnitude, alignment of support across tasks, approximate sparsity) through additional simulations. We find the predictive accuracy of \textsf{RMEstimator} improves with large $N$, small $s$, large $n_i/n_j$, and when the support of the task-specific bias terms are either very well- or poorly-aligned. In Appendix~\ref{app:scadmcp}, we further explore variants of the LASSO penalty (e.g., SCAD or MCP) in Step 2 of the \textsf{RMEstimator} and find similar results.

\textbf{Risk Prediction in Health Data.}
Diabetes is a leading cause of severe health complications such as cardiovascular disease, stroke, and chronic kidney disease \citep{ismail2021association}. Thus, there is significant interest in leveraging machine learning for early detection of (Type II) diabetes, in order to improve treatment outcomes \citep{zhang2020machine}. However, significant evidence shows that machine learning models trained on one health system can perform poorly on a different health system \citep{quinonero2008dataset, subbaswamy2020development}; this can be due to dataset shifts such as changes in patient demographics, disease prevalence, measurement timing, equipment, and treatment patterns. Thus, it is important to train provider-specific risk models.

In this experiment, we use electronic medical record data across $N=13$ healthcare providers to learn a good diabetes risk prediction model for a single provider. After basic preprocessing, we have approximately 80 patient-specific features constructed from information available before the most recent visit (e.g., past diagnoses, procedures and medications); our outcome is an indicator variable for whether the patient was diagnosed with diabetes during the most recent visit. We aim to learn the best linear classifier, and evaluate different methods based on the diagnosis accuracy over time. Appendix~\ref{app:exp-diabetes} provides additional details on the setup. We compare the prediction accuracy over 
two of the hospitals in Figure~\ref{fig:diabetes_off}, one with 301 unique patients observed during the sample period and the other with 246 unique patients.

\begin{figure}[htbp]
\centering
\begin{subfigure}[b]{0.42\textwidth}
  \centering
  \includegraphics[width=\textwidth]{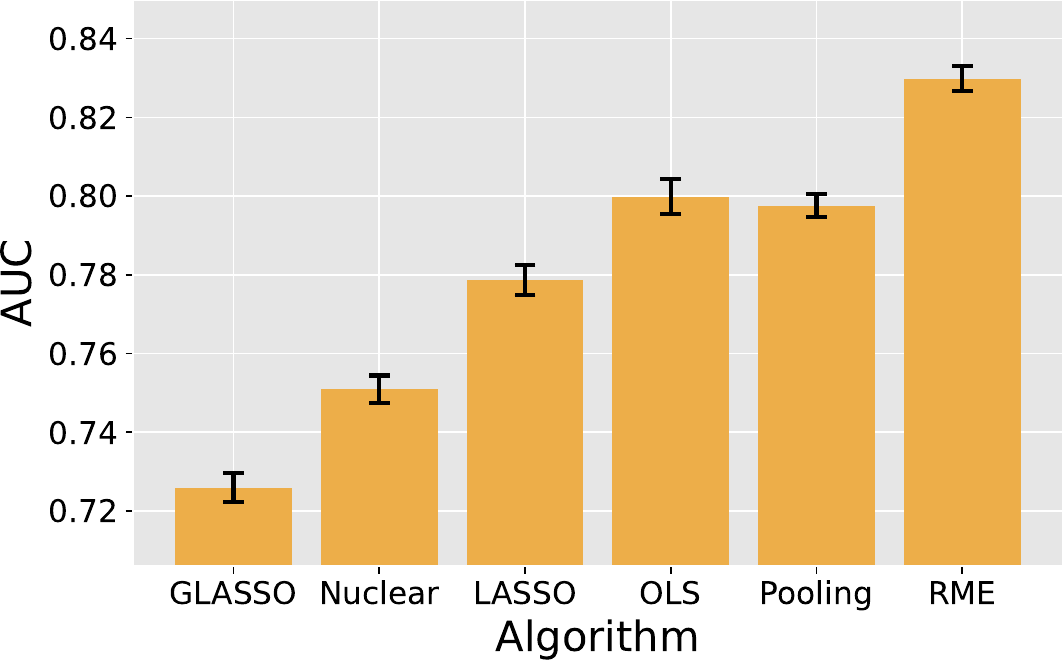}
  \caption{Hospital A}
\end{subfigure}
\begin{subfigure}[b]{0.42\textwidth}
  \centering
  \includegraphics[width=\textwidth]{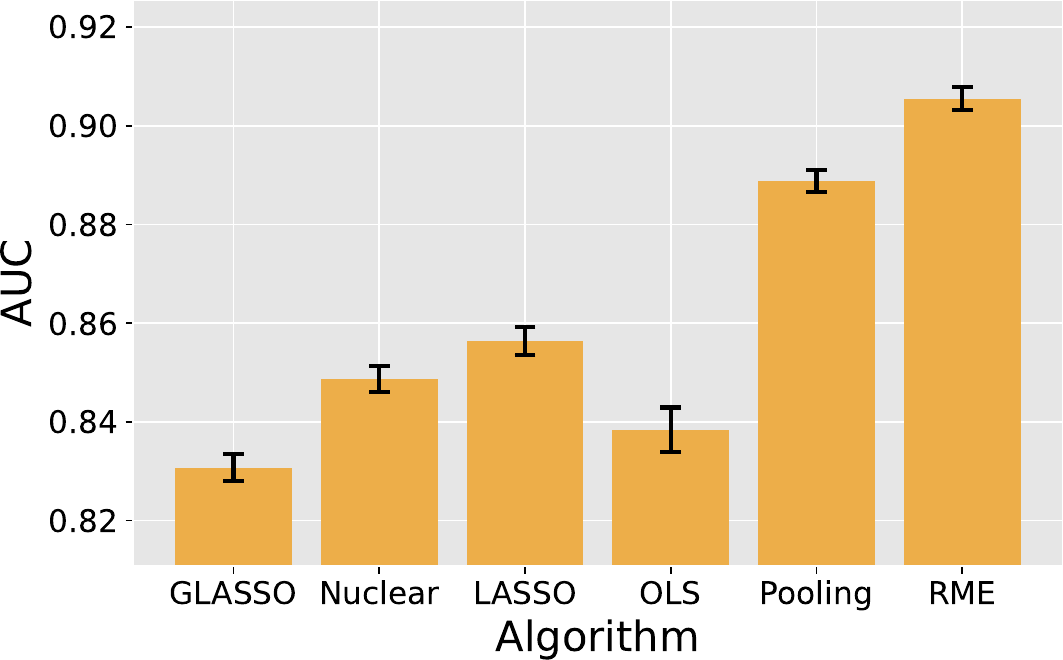}
  \caption{Hospital B}
\end{subfigure}
\caption{Bars depict out-of-sample performance measured by AUC at one hospital (averaged over 1,000 trials), with 95\% confidence intervals. Hospitals A and B have 301 and 246 unique patients respectively. `GLASSO' refers to group LASSO, `Nuclear' nuclear-norm regularization, and `RME' our robust multitask estimator.}
\label{fig:diabetes_off}
\end{figure}

Figure~\ref{fig:diabetes_off} shows the out-of-sample predictive accuracy (AUC). Once again, we can find that \textsf{RMEstimator} achieves the best performance, improving over the second best algorithm (independent OLS) by 3.7\% and 8\% at Hospitals A and B respectively. The pooling algorithm outperforms OLS by leveraging more data but is still not competitive with \textsf{RMEstimator} ignoring hospital-specific heterogeneity. Group LASSO and nuclear-norm regularization perform poorly, suggesting that the assumptions imposed may be unwarranted in our dataset. 

Since this is a classification problem with binary outcomes, in Appendix~\ref{app:loghealth}, we also compare the logistic regression analog of the \textsf{RMEstimator} with the logistic regression analogs of all the baseline algorithms. Once again, we find \textsf{RMEstimator} achieves the best performance; yet, in this specific task, logistic regression does not perform as well as its linear counterpart for all methods.\footnote{Linear and logistic regression estimates are often highly correlated even when the outcomes are binary, and produce nearly identical decisions~\citep[see, e.g.,][]{pohlman2003comparison}; however, linear models are unbiased in small samples, enabling faster convergence and improved multitask learning in the low-data regime, which may explain our improved performance with linear classifiers. In practice, one should choose the best predictor based on out-of-sample AUC.}

\subsection{Multitask Bandits}

Next, we illustrate the performance of our \textsf{RMBandit} algorithm in the online setting. More specifically, we simulate the following linear contextual bandit algorithms: (i) without multitask learning: OLS Bandit \citep{goldenshluger2013linear}, LASSO Bandit \citep{bastani2020online}, and (ii) with multitask learning: \textsf{GOBLin} \citep{cesa2013gang}, Trace-norm Bandit \citep{cella2022multi}, a pooling bandit algorithm, and \textsf{RMBandit}.
The first two approaches operate $N$ independent bandit instances via either OLS or LASSO. The \textsf{GOBLin} algorithm is a state-of-the-art multitask bandit algorithm that uses a Laplacian matrix and ridge regression to jointly learn the instances, thereby $\ell_2$-regularizing both the parameters and their pairwise differences. It builds upon the OFUL algorithm \citep{abbasi2011improved}, which leverages UCB for linear contextual bandits. Besides, we also implement other multitask bandit algorithms such as Trace-norm Bandit using trace-norm regularization (i.e., nuclear-norm regularization), a pooling bandit algorithm that pools observations and then uses OLS Bandit, and our \textsf{RMBandit} algorithm.\footnote{We numerically find that our algorithm is more stable if we use the forced samples to initialize the all-sample estimator; therefore, we make this practical modification in our experiments.}

\textbf{Synthetic.}
Figure~\ref{fig:synthetic_on} shows the expected cumulative regret over time for a single contextual linear bandit instance under varying $N, K, d$ and $s$. Appendix \ref{app:exp-synth} provides additional details.

\begin{figure}[htbp]
\centering
\begin{subfigure}[b]{0.32\textwidth}
  \centering
  \includegraphics[width=\textwidth]{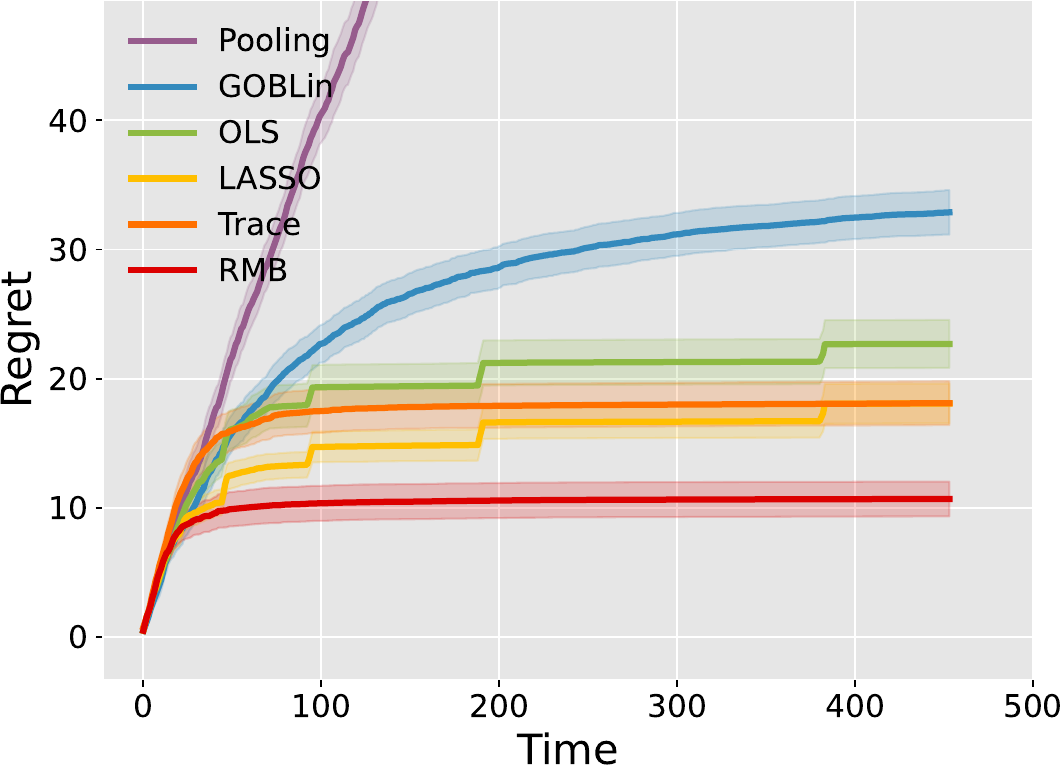}
  \caption{$N=30, K=3, d=20, s=2$}
\end{subfigure}
\begin{subfigure}[b]{0.32\textwidth}
  \centering
  \includegraphics[width=\textwidth]{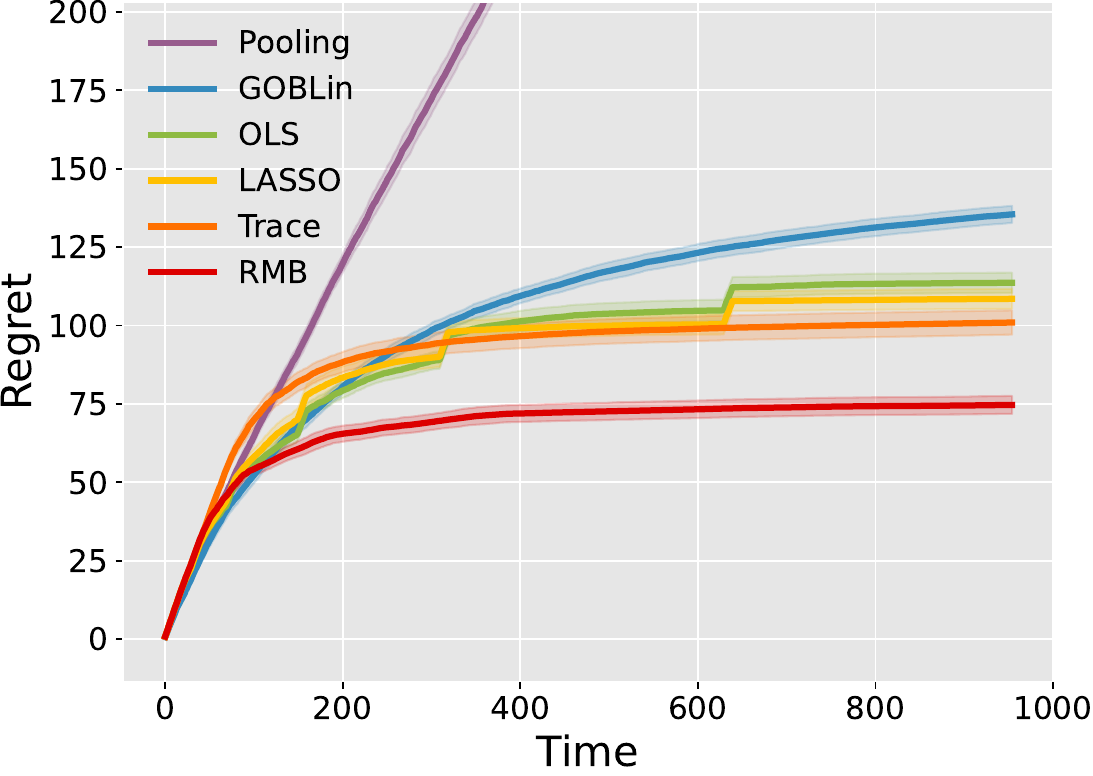}
  \caption{$N=10, K=10, d=20, s=2$}
\end{subfigure}
\begin{subfigure}[b]{0.32\textwidth}
  \centering
  \includegraphics[width=\textwidth]{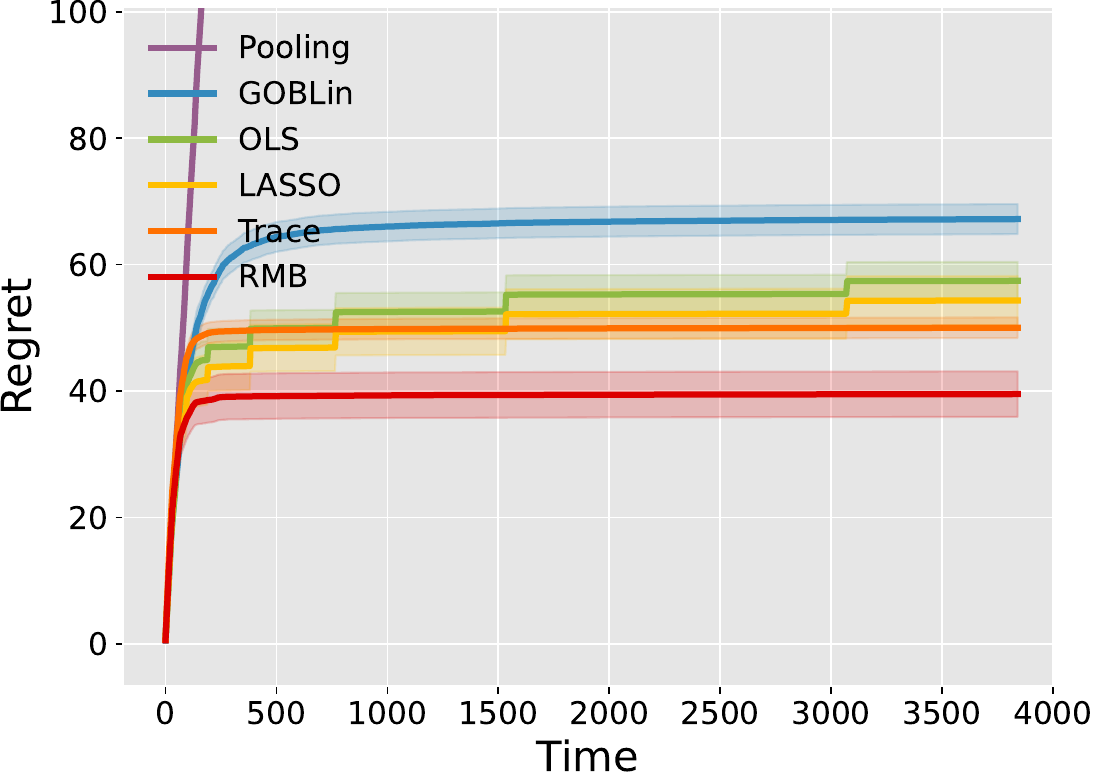}
  \caption{$N=15, K=3, d=40, s=5$}
\end{subfigure}
\caption{Cumulative regret of a single linear contextual bandit (averaged over 20 trials); shaded regions show 95\% confidence intervals. `Trace' refers to Trace-norm Bandit, and `RMB' our \textsf{RMBandit} algorithm.}
\label{fig:synthetic_on}
\end{figure}

Similar to the offline setting, we find that \textsf{RMBandit} significantly outperforms other algorithms under sparse heterogeneity. The pooling algorithm ignores task-specific heterogeneity, and the Trace-norm Bandit and \textsf{GOBLin} (which assumes that the arm parameters across instances are close in $\ell_2$ norm) impose structure that are not met by our synthetic data. \textsf{GOBLin} performs particularly poorly, possibly because they also build on the UCB algorithm, which is known to over-explore compared to the OLS Bandit \citep[see, e.g.,][]{russo2017tutorial,bastani2021mostly}; an interesting future direction is to adapt multitask learning to other bandit algorithms such as Thompson Sampling.

Appendix~\ref{app:add_exp_robbandit} explores the robustness of \textsf{RMBandit} against hyperparameter choices $\omega_0$, $\zeta_{1,0}$, $\eta_{1,0}$, and $q$. We find that our algorithm's cumulative regret is robust across specifications, which is important in practice where these hyperparameters might not be well-specified.

The following two experiments examine real datasets, which may not satisfy our assumptions.

\textbf{Personalized Intervention in Health Data.}
Using the same medical data as in \S\ref{sec:exp_off}, we aim to learn a good diabetes intervention model for each single hospital in an online manner. We consider a simple binary reward that directly evaluates the accuracy of our intervention or classification of patients, i.e., the reward is 1 if the prediction is correct and 0 otherwise; thus, we have two arms, i.e., $K=2$, in our bandit model that are either to intervene or not intervene on the patient. We aim to learn the best linear classifier online as patient observations accrue, and evaluate different methods based on the classification accuracy over time. We compare the fraction of incorrectly intervened patients over time and show the results of two of the hospitals in Figure~\ref{fig:diabetes_on}, one with 355 unique patients observed during the sample period and the other with 176 unique patients. We additionally fit a linear oracle, which leverages all observed data from the provider in hindsight using a leave-one-out approach, representing the best achievable performance within a linear model family. Appendix \ref{app:exp-diabetes} provides additional details on the setup. 

\begin{figure}[htbp]
\centering
\begin{subfigure}[b]{0.42\textwidth}
  \centering
  \includegraphics[width=\textwidth]{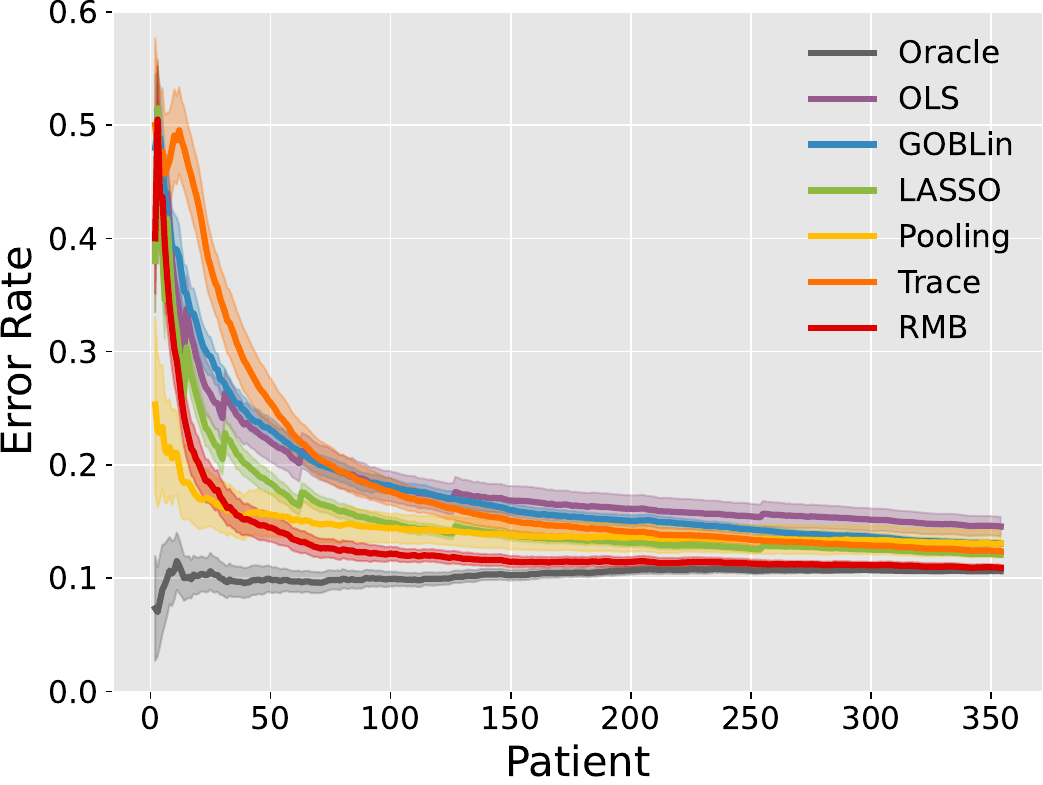}
  \caption{Hospital A}
\end{subfigure}
\begin{subfigure}[b]{0.42\textwidth}
  \centering
  \includegraphics[width=\textwidth]{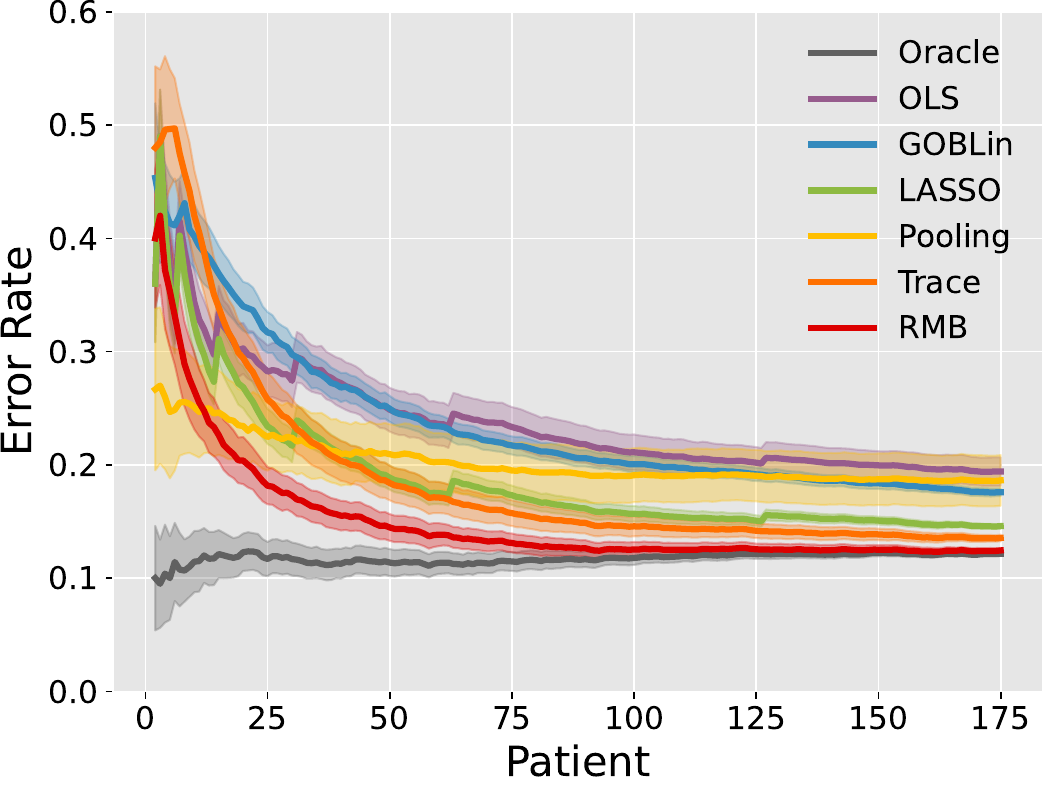}
  \caption{Hospital B}
\end{subfigure}
\caption{Lines depict the fraction of incorrect interventions averaged over 50 trials of a linear contextual bandit, with shaded regions the corresponding 95\% confidence intervals. Hospitals A and B have 355 and 176 unique patients respectively. `Trace' refers to Trace-norm Bandit, and `RMB' our \textsf{RMBandit} algorithm.}
\label{fig:diabetes_on}
\end{figure}

We observe that \textsf{RMBandit} performs favorably compared to the other baseline algorithms, and converges to the oracle's classification accuracy much faster. In particularly, our algorithm is insignificantly different from the oracle after observing around 150 and 75 patients respectively, while the other algorithms are all significantly worse than the oracle within the sample period. Notably, the pooling algorithm has good performance on average in the very early stage, likely since our algorithm requires at least some initial data to learn heterogeneous structure across hospitals. 

\begin{remark}
Note that any use of randomization in a medical decision-making task typically requires obtaining patient consent, careful assessment of patient safety under all arms, and other contextual considerations. Within this framework, patient-facing experiments are regularly conducted to assess whether innovative interventions improve patient outcomes in a cost-effective way~\citep[see, e.g.,][]{volpp2017effect,nahum2024optimizing}. Bandit/adaptive designs can improve the efficiency~\citep{bastani2021efficient} and safety of these experiments~\citep{pallmann2018adaptive}.
\end{remark}

\textbf{Dynamic Pricing in Retail Data.}
Contextual bandit algorithms can also naturally be extended to solve dynamic pricing problems with unknown demand \citep{besbes2009dynamic}. We consider such a demand forecasting and price optimization task for food distributors; to this end, we use a publicly available dataset of orders from a meal delivery company.\footnote{\url{https://datahack.analyticsvidhya.com/contest/genpact-machine-learning-hackathon-1/}} 

In this experiment, we use data across $N=20$ fulfillment centers, serving between three to eight thousand orders each during the sample period. There are 19 features including the intercept, the category and cuisine pertaining to the order, as well as associated promotions. The decision variable is the (continuous) price for the order; rather than arm parameters, there is a single set of unknown parameters (per instance) that aims to predict demand/revenue as a function of price and the observed features. Following the approach of \cite{ban2021personalized}, we model the price elasticity of demand as a linear function of the observed features:
\begin{align*}
Y_t = X_t^\top\beta_0^j + p_t \cdot (X_t^\top\beta_1^j) + \epsilon_t.
\end{align*}
Here, $\{\beta_0^j, \beta_1^j\}$ are the unknown parameters corresponding to instance $j$; conditioned on an arrival with context $X_t$ at instance $Z_t=j$, $Y_t$ is the observed revenue for the chosen price $p_t$ and noise $\epsilon_t$. Regret is measured with respect to an oracle that knows $\{\beta_0^j, \beta_1^j\}_{j=1}^N$.

We straightforwardly extend \textsf{RMBandit} and the other baseline bandit algorithms to the dynamic pricing setting using a batched explore-then-commit strategy employed by \cite{ban2021personalized}. Figure~\ref{fig:pricing} shows the cumulative regret of \textsf{RMX} (our dynamic pricing analog of \textsf{RMBandit}) compared to other benchmarks including \textsf{ILQX} and \textsf{ILSX} \citep[the LASSO- and OLS-based pricing algorithms introduced in][]{ban2021personalized}, our dynamic pricing analog of \textsf{GOBLin} and Trace-norm Bandit. Appendix \ref{app:exp-pricing} provides further details and pseudocode for our dynamic pricing algorithms.

\begin{figure}[htbp]
\centering
\begin{subfigure}[b]{0.42\textwidth}
  \centering
  \includegraphics[width=\textwidth]{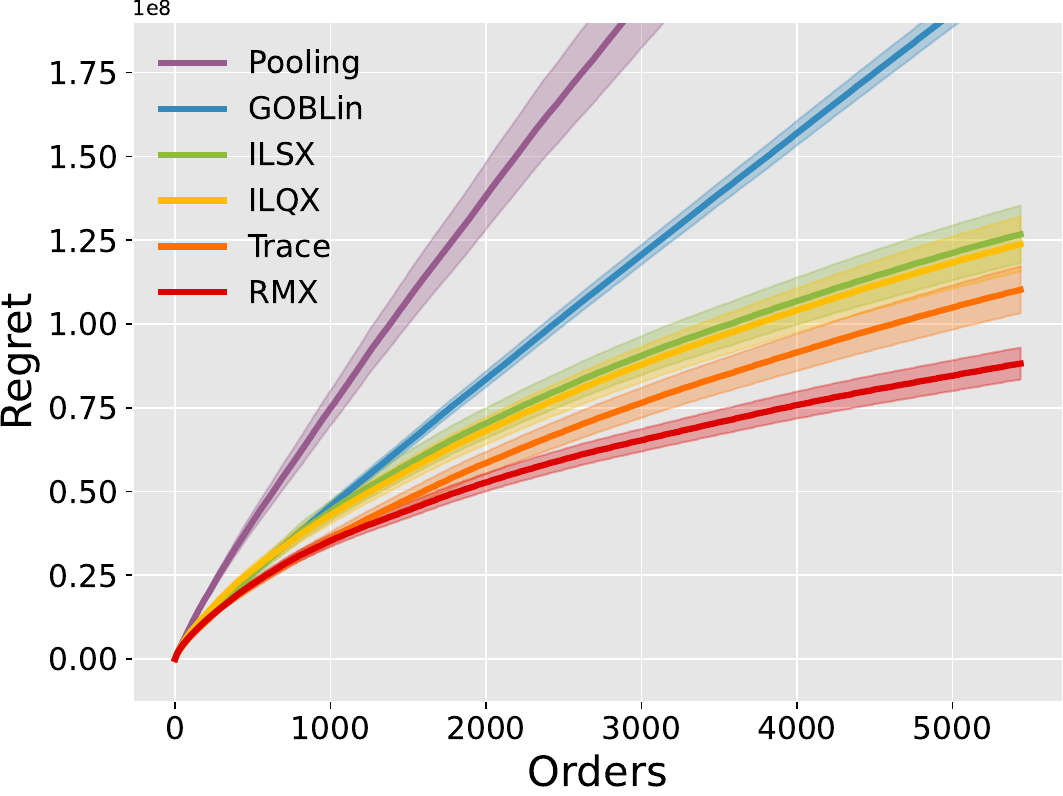}
  \caption{Fulfillment Center A}
\end{subfigure}
\begin{subfigure}[b]{0.42\textwidth}
  \centering
  \includegraphics[width=\textwidth]{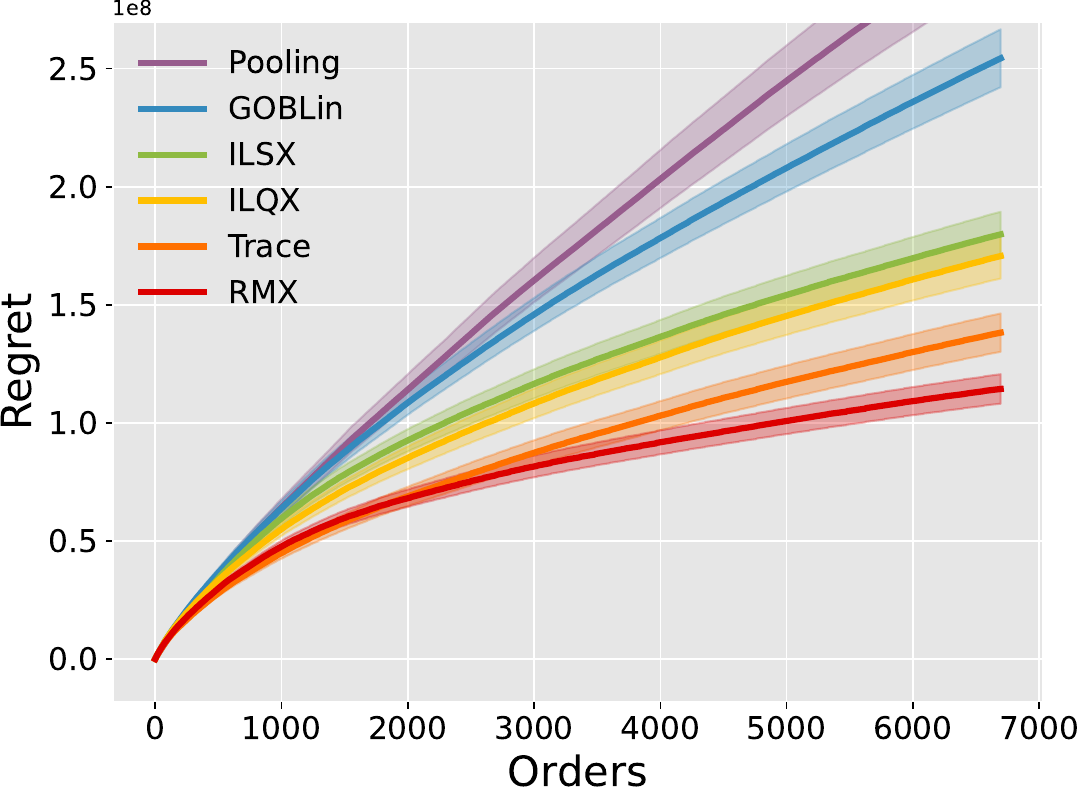}
  \caption{Fulfillment Center B}
\end{subfigure}
\caption{Lines depict the cumulative regret of a linear contextual bandit averaged over 100 trials, with shaded regions the corresponding 95\% confidence intervals. `Trace' refers to Trace-norm Bandit, and `RMX' our pricing analog of \textsf{RMBandit}.}
\label{fig:pricing}
\end{figure}

We observe that \textsf{RMX} performs favorably compared to the dynamic pricing analogs of the other baseline bandit algorithms. Thus, our insights on multitask learning carry over analogously to the dynamic pricing context.

\section{Discussion and Conclusions}

Decision-makers frequently want to learn heterogeneous treatment effects across many simultaneous experiments. Examples range from learning patient risk across hospitals for personalized interventions \citep{bastani2020predicting, mullainathan2017does}, learning drug effectiveness across combination therapies for clinical trial decisions \citep{bertsimas2016analytics}, learning COVID-19 risk across travelers for targeting tests \citep{bastani2021efficient}, and learning demand across stores for promotion targeting \citep{baardman2020detecting, cohen2018promotion} or dynamic pricing \citep{bastani2021meta}.
We propose a novel \textsf{RMEstimator} that improves the efficacy of downstream decisions by learning better predictive models with lower sample complexity in the context dimension $d$. To the best of our knowledge, our work proposes the first combination of robust statistics (to learn across similar instances) and LASSO regression (to debias the results) to yield improved bounds for multitask learning.
In the online learning setting, these problems translate to running simultaneous contextual bandit algorithms. To this end, we propose the \textsf{RMBandit} algorithm to effectively navigate the exploration-exploitation tradeoff across bandit instances, thereby improving regret bounds in $d$.

We highlight several features of our proposed approach that make it a particularly attractive solution. First, it is well known that data limitations result in worse model performance, which in turn can imply \textit{unfair} decisions, e.g., in healthcare, such biases disproportionately affect protected groups or minorities due to limited representative data \citep{rajkomar2018ensuring}. A natural approach to alleviating unfairness is to improve the performance of our models for data-poor instances \citep[see, e.g., discussion in][]{hardt2016equality}. We show that multitask learning can be especially valuable in such settings --- our approach leverages data from data-rich instances to provide an exponential improvement in performance for data-poor instances. Thus, we provide one additional tool (among others) for improving fairness in decision-making.

Second, privacy and regulatory constraints prevent granular data sharing in many applications. A growing literature on \textit{federated learning} studies training statistical models over siloed datasets, while keeping data localized \citep{li2020federated}. While our focus is on multitask learning, our approach satisfies the constraints of federated learning, since we only require sharing aggregate statistics (in this case, OLS regression parameters) across instances. All model training is performed locally at the instance-level and does not require any raw data from other instances.

Third, practical deployment of bandits often precludes real-time updates to the model. For instance, many individuals may appear for service simultaneously \citep{schwartz2017customer} and there may be operational constraints or concerns over model reliability \citep{bastani2021efficient}. Our \textsf{RMBandit} algorithm employs a batching strategy that only requires a logarithmic number (in the time horizon $T$) of model updates, while preserving convergence rates (and therefore
regret). Furthermore, it has the added advantage of being far more computationally tractable.


{\SingleSpacedXI
\bibliographystyle{ormsv080} 
\bibliography{refs} 
}

%
%
%
\newpage
\renewcommand{\theHsection}{\Alph{section}}
\begin{APPENDICES}

\section{Multitask Learning}
In this section, we provide the proofs of Proposition~\ref{prop:tmean_iid} in \S\ref{app:tmean_concen}, Proposition~\ref{prop:tmean_reg_hpb_arb} and Theorem~\ref{thm:tmean_reg_hpb} in \S\ref{app:tailadpt_multitask_std}, Theorem~\ref{thm:tmean_reg_hpb_dp} in \S\ref{app:tailadpt_multitask_poor}, Proposition~\ref{prop:tmean_reg_hpb_cov} in \S\ref{app:rand_design}, Corollary~\ref{cor:tmean_reg_rbs} in \S\ref{app:tailadpt_multitask_rbstout}, and Corollary~\ref{cor:tmean_glm_hpb_arb} in \S\ref{app:glm}.

\subsection{Trimmed Mean Estimator}\label{app:tmean_concen}

\begin{proof}{Proof of Proposition~\ref{prop:tmean_iid}}
Recall that the indices of the corrupted samples are denoted by $\cJ\subseteq[N]$ (so the rest are $\cJ^c=[N]\setminus\cJ$) with $|\cJ|<N\zeta$ and $\zeta < 1/2$. By assumption, $\{Z_j\}_{j\in\cJ^c}$ are independent and $\sigma_j$-subgaussian with mean $\mu$ respectively. Then, using a Chernoff bound, any uncorrupted sample $j\in \cJ^c$ satisfies
\begin{align*}
\P\bb{|Z_j - \mu| \ge t} \le 2\exp\bp{-\frac{t^2}{2\sigma_j^2}}
\end{align*}
for any $t > 0$. Letting $t = \sqrt{2\sigma_j^2\log(\frac{3}{\eta})}$, it follows that
\begin{align*}
Z_j \not \in \bb{\mu - \sqrt{2\sigma_j^2\log(\frac{3}{\eta})},\; \mu + \sqrt{2\sigma_j^2\log(\frac{3}{\eta})}}
\end{align*}
with a probability of at most $2\eta/3$. By Hoeffding's inequality, we have
\begin{align*}
\P\bb{\sum_{j \in \cJ^c} \mathbbm{1}\bp{Z_j \not \in I} \ge t }
\le \exp\bp{-\frac{2 (t - \sum_{j\in \cJ^c} p_j)^2}{|\cJ^c|}},
\end{align*}
where $p_j =\P\bb{Z_j \not \in I} \le 2\eta/3$ and 
$I = \left[\mu - \max_{j \in \cJ^c}\sqrt{2\sigma_j^2\log(\frac{3}{\eta})}, \mu + \max_{j \in \cJ^c}\sqrt{2\sigma_j^2\log(\frac{3}{\eta})}\right]$.
Taking $t = \eta |\cJ^c|$, we have
\begin{align*}
\P\bb{\sum_{j \in \cJ^c} \mathbbm{1}\bp{Z_j \not \in I} \ge \eta |\cJ^c|} \le \exp\bp{-\frac{2\eta^2|\cJ^c|}{9}};
\end{align*}
in other words, with a high probability, at most $\eta$ fraction of $\cJ^c$ are outside a neighborhood $I$ around the mean $\mu$. As a consequence, on the event 
\begin{align*}
V=\{\sum_{j\in \cJ^c} \mathbbm{1}\bp{Z_j \not \in I} \le \eta |\cJ^c|\},
\end{align*}
at most $\zeta+\eta$ fraction of the $N$ samples are outside the interval $I$ (recall that $\zeta N$ samples are corrupted). As a subgaussian distribution might not be symmetric, by trimming the upper and lower $\omega$-quantile of samples with $\omega=\zeta+\eta$, the remaining ones are guaranteed to fall into $I$.

Suppose the event $V$ holds. Let $\{Z_{j_\iota}\}_{\iota=N\omega+1}^{N(1-\omega)}$ denote the samples after trimming and $\cT=\{j_\iota\}_{\iota=N\omega+1}^{N(1-\omega)}$ the corresponding index set ($j_\iota$ defined in \S\ref{ssec:estimator-prelim}). And let $\cU = \{j \in [N] \mid Z_j \in I\}$ denote the set of samples that lie in $I$. Note that $\cT \subseteq \cU$ on the event of $V$ from our argument above. Then, we have
\begin{align}\label{eq:decomp_trimset}
\left|\sum_{j \in \cT} (Z_j - \mu)\right| = \left|\sum_{j \in \cT \cap \cU} (Z_j - \mu)\right| \le \left|\sum_{j \in \cT \cap \cU \cap \cJ^c} (Z_j - \mu)\right| + \left|\sum_{i \in \cT \cap \cU \cap \cJ} (Z_j - \mu)\right|.
\end{align}
The first term on the RHS of (\ref{eq:decomp_trimset}) is upper bounded by
\begin{align*}
\left|\sum_{j \in \cT \cap \cU \cap \cJ^c} (Z_j-\mu)\right| \le \left|\sum_{j \in \cT^c \cap \cU \cap \cJ^c} (Z_j-\mu)\right| + \left|\sum_{j \in \cU \cap \cJ^c} (Z_j-\mu)\right|.
\end{align*}
As we remove $2(\zeta+\eta)$ fraction of the samples, we have
\begin{align}\label{eq:outobs_bound}
\left|\sum_{j \in \cT^c \cap \cU \cap \cJ^c} (Z_j-\mu)\right| \le 2(\zeta+\eta) N \max_{j \in \cJ^c}\sqrt{2\sigma_j^2\log(\frac{3}{\eta})}.
\end{align}
Since those samples in $\cJ^c$ that lie inside the interval $I$ are independent and bounded, we can apply Hoeffding's inequality
\begin{align}\label{eq:probbound_gi}
\P\bb{\left|\frac{1}{|\cJ^c \cap \cU|} \sum_{j \in \cJ^c \cap \cU} \bp{Z_j - \E[Z_j \mid \cJ^c \cap \cU]}\right| \ge \chi\cdot\max_{j \in \cJ^c} \sqrt{\sigma_j^2\log(\frac{3}{\eta})} } \le 2\exp\bp{-\frac{|\cJ^c\cap \cU|\chi^2}{4}},
\end{align}
for any $\chi>0$. The truncation on these samples introduces a bias of at most
\begin{align}\label{eq:biascap_gi}
\left|\E[Z_j \mid \cJ^c \cap \cU]-\mu\right| = \left|\frac{\E[(Z_j-\mu) \mathbbm{1}(Z_j \not \in I) \mid \cJ^c]}{\P(Z_j \in I \mid \cJ^c)}\right| \le \frac{\E[|Z_j-\mu|^k \mid \cJ^c]^{1/k} \P[Z_j \not \in I \mid \cJ^c]^{1/q}}{\P(Z_j \in I \mid \cJ^c)},
\end{align}
where the last inequality uses H\"older's inequality and $k,q$ are such that $1/k+1/q=1$. Recall that $\P[Z_j \not \in I \mid \cJ^c] \le 2\eta/3$. In addition, $\E[|Z_j-\mu|^k \mid \cJ^c]^{1/k} \le e^{1/e} \sigma_j \sqrt{k}$ for $k \ge 2$ by the property of subgaussian \citep{rigollet2015high}. 
Taking $k=\log(\frac{3}{2\eta})$, inequality~(\ref{eq:biascap_gi}) gives
\begin{align*}
|\E[Z_j \mid \cJ^c \cap U]-\mu| 
 \le \frac{8\sigma_j\eta\sqrt{\log(\frac{3}{2\eta})}}{3-2\eta}.
\end{align*}
Then, the high probability bound in (\ref{eq:probbound_gi}) implies
\begin{align*}
\P\bb{|\sum_{j \in \cJ^c \cap \cU} (Z_j - \mu)| \ge |\cJ^c \cap \cU|\bp{\chi\cdot\max_{j \in \cJ^c} \sqrt{\sigma_j^2\log(\frac{3}{\eta})} + \max_{j \in \cJ^c}\frac{8\sigma_j\eta\sqrt{\log(\frac{3}{2\eta})}}{3-2\eta}}} \le 2\exp\bp{-\frac{|\cJ^c \cap \cU|\chi^2}{4}}.
\end{align*}
Further setting $\chi = \eta$ and by our assumption $\eta<1/2$, we have
\begin{align}\label{eq:probbound_gi_drv}
\P\bb{\left|\sum_{j \in \cJ^c \cap \cU} (Z_j - \mu)\right| \ge 5|\cJ^c \cap \cU| \eta \max_{j \in \cJ^c}\sqrt{\sigma_j^2\log(\frac{3}{\eta})}} \le 2\exp\bp{-\frac{|\cJ^c \cap \cU|\eta^2}{4}}.
\end{align}
Combining (\ref{eq:outobs_bound}) and (\ref{eq:probbound_gi_drv}), it holds with a high probability that
\begin{align*}
\left|\sum_{j \in \cT \cap \cU \cap \cJ^c} (Z_j-\mu)\right| \le \bp{2\sqrt{2}(\zeta+\eta)N + 5\eta|\cJ^c \cap \cU|} \max_{j \in \cJ^c}\sqrt{\sigma_j^2\log(\frac{3}{\eta})}
\end{align*}
Additionally, the second term on the RHS of (\ref{eq:decomp_trimset}) has
\begin{align*}
\left|\sum_{j\in\cT \cap \cU \cap \cJ} (Z_j-\mu)\right| \le \zeta N \max_{j \in \cJ^c}\sqrt{2\sigma_j^2\log(\frac{3}{\eta})}.
\end{align*}
Thus, we can write 
\begin{align*}
\left|\frac{1}{|\cT|}\sum_{i \in \cT} Z_j - \mu\right| & \le \frac{\max_{j \in \cJ^c}\sigma_j}{(1-2(\zeta+\eta))N} \bp{\sqrt{2}(3\zeta+2\eta)N + 5\eta|\cJ^c \cap \cU|} \sqrt{\log(\frac{3}{\eta})} \\
& \le C_0 \max_{j \in \cJ^c}\sigma_j\bp{3\zeta + 4\eta}\sqrt{\log(\frac{3}{\eta})}
\end{align*}
with a high probability, where we use $|\cT| = (1-2(\zeta+\eta))N$, $|\cJ^c \cap \cU| \le N$ and $\eta \le 1/2 -1/C_0-\zeta$. Since $|\cJ^c \cap \cU| \ge (1 - \eta)|\cJ^c|$ on the event $V$, we have 
\begin{align*}
\P\bb{\left|\frac{1}{|\cT|}\sum_{j \in \cT} Z_j - \mu\right| \ge C_0 \max_{j \in \cJ^c}\sigma_j\bp{3\zeta + 4\eta}\sqrt{\log(\frac{3}{\eta})}} \le  2\exp\bp{-\frac{N\eta^2}{8}},
\end{align*}
where we use $|\cJ^c| \ge (1-\zeta)N$ and $\eta + \zeta < 1/2$. Together with a union bound on the event $V$, we have
\begin{align*}
\P\bb{\left|\frac{1}{|\cT|}\sum_{j \in \cT} Z_j - \mu\right| \ge C_0 \max_{j \in \cJ^c}\sigma_j \bp{3\zeta + 4\eta}\sqrt{\log(\frac{3}{\eta})}} \le  3\exp\bp{-\frac{N\eta^2}{9}}. \halmos
\end{align*}
\end{proof}

\subsection{Standard Regime}
\label{app:tailadpt_multitask_std}

The hyperparameters are
\begin{gather*}
\lambda=\max_{i\in[N]}\sqrt{32\sigma_i^2 \log(\frac{4d}{\delta})}, \quad 
\zeta=\frac{C_0-2}{4C_0}\sqrt{\frac{s}{d}}, \quad
\eta=\sqrt{\frac{9\log(\frac{6d}{\delta})}{N}}.
\end{gather*}
Note that $\lambda_j = \lambda/\sqrt{n_j}$ as in Algorithm~\ref{alg:tmean_reg}.

We first provide an estimation error bound of our estimator given arbitrary choices of hyperparameters (i.e., Proposition~\ref{prop:tmean_reg_hpb_arb}). We begin with the following lemma.
\begin{lemma}\label{def:eventH}
Define the event
\begin{align}\label{eq:eventH}
\cH^j = \bc{\frac{2}{n_j}\|\vX^{j\top}\vE^j\|_\infty \le \frac{\lambda_j}{2}}. 
\end{align}
Then, we have
\begin{align*}
\P\bb{\cH^j} \ge 1 - 2d \exp\bp{-\frac{\lambda_j^2 n_j}{32\sigma_j^2 }}.
\end{align*}
\end{lemma}
\begin{proof}{Proof of Lemma~\ref{def:eventH}}
For any column $i$ of the design matrix $\vX^j$, i.e., $\vX^j_{(\cdot, i)}$, we have $\|\frac{1}{\sqrt{n_j}}\vX^j_{(\cdot, i)}\|_2 \le 1$. Then, by Lemma~\ref{lem:subgauvec_conc}, we have
\begin{align*}
\P\bb{(\cH^j)^c} 
& = \P\bb{\max_{i \in [d]} \frac{1}{\sqrt{n_j}}|\vX_{(\cdot, i)}^{j\top}\vE^j| \ge \frac{\lambda_j\sqrt{n_j}}{4}}
\le d \max_{i \in [d]} \P\bb{ \frac{1}{\sqrt{n_j}}|\vX_{(\cdot, i)}^{j\top}\vE^j| \ge \frac{\lambda_j\sqrt{n_j}}{4}}
\le 2d \exp\bp{-\frac{\lambda_j^2 n_j}{32\sigma_j^2}}. \Halmos
\end{align*}
\end{proof}

\begin{proof}{Proof of Proposition~\ref{prop:tmean_reg_hpb_arb}}
At a high level, our model has
\begin{align*}
\vY^j = \vX^j (\beta^\dagger + \delta^j) + \vE^j = \vX^j \bp{(\beta^\dagger_{\cI_{\text{poor}}} + \widehat{\beta}_{\text{RM},\cI_{\text{well}}}^\dagger) + (\beta^\dagger_{\cI_{\text{well}}} - \widehat{\beta}^\dagger_{\text{RM},\cI_{\text{well}}} + \delta^j)} + \vE^j, 
\end{align*}
where $\beta_{\cI}$ given an index set $\cI$ is defined at the beginning of \S\ref{sec:formulation}, and $\beta^\dagger_{\cI_{\text{well}}} - \widehat{\beta}^\dagger_{\text{RM},\cI_{\text{well}}} + \delta^j$ is $((1/\zeta+1)s)$-sparse --- in particular, letting 
\begin{align*}
\bar{\cI}_j = \cI_{\text{well}} \cup \cI_j,
\end{align*}
where $\cI_j=\{i\in[d]\mid\beta_{(i)}^j\neq\beta_{(i)}^\dagger\}$ are the components of $\beta^j$ that do not equal $\beta^\dagger$, then we have $|\bar{\cI}_j| \le (1/\zeta + 1)s$. Intuitively, we can show that $\beta^\dagger_{\cI_{\text{poor}}} + \widehat{\beta}^\dagger_{\text{RM}, \cI_{\text{well}}}$ is closely approximated by $\widehat{\beta}_{\text{RM}}^\dagger$ in Step 1 (i.e., $(\widehat{\beta}_{\text{RM}}^\dagger - \beta^\dagger)_{\cI_{\text{poor}}}$ is small). Then, we can use LASSO to recover the rest of the parameters, i.e., the sparse vector $\beta^\dagger_{\cI_{\text{well}}} - \widehat{\beta}_{\text{RM},\cI_{\text{well}}}^\dagger + \delta^j$, to efficiently estimate $\beta^j$ in Step 2.

First, we show that $\widehat{\beta}_{\text{RM},\cI_{\text{poor}}}^\dagger$ approximates $\beta_{\cI_{\text{poor}}}^\dagger$ well. We start by noticing that each OLS estimator $\widehat{\beta}_{\text{ind}}^j = (\vX^{j\top}\vX^j)^{-1}\vX^{j\top} \vY^j$ is a subgaussian random vector with mean $\beta^j$ --- particularly, the $i^{\text{th}}$ component $\widehat\beta_{\text{ind},(i)}^j$ of $\widehat\beta_{\text{ind}}^j$ is $(\sqrt{\sigma_j^2/(n_j \psi)})$-subgaussian. This is because
\begin{align*}
\E[\exp(\lambda(\widehat{\beta}_{\text{ind},(i)}^j-\beta_{(i)}^j))] & = \E[\exp(\lambda (\vX^{j\top}\vX^j)_{(i, \cdot)}^{-1}\vX^{j\top} \vE^j)]
\le \exp\bp{\frac{\lambda^2\sigma_j^2\|(\vX^{j\top}\vX^j)_{(i, \cdot)}^{-1}\vX^{j\top}\|_2^2}{2}}
\le \exp\bp{\frac{\lambda^2\sigma_j^2}{2n_j \psi}},
\end{align*}
where the last inequality uses Assumption~\ref{ass:posdef_off} and follows
\begin{align*}
\|(\vX^{j\top}\vX^j)_{(i, \cdot)}^{-1}\vX^{j\top}\|_2^2 = (\vX^{j\top}\vX^j)_{(i, i)}^{-1} \le \lambda_{\max}\bp{(\vX^{j\top}\vX^j)^{-1}}=\frac{1}{n_j \lambda_{\min}(\widehat{\Sigma}^j)} \le \frac{1}{n_j \psi}.
\end{align*}

Now consider our robust multitask estimator $\{\widehat{\beta}_{\text{RM}}^j\}_{j\in[N]}$ computed by Algorithm~\ref{alg:tmean_reg}. Recall that for any poorly-aligned component $i \in \cI_{\text{poor}}$, the corresponding corrupted subset of instances is $\cJ_i=\{j \in [N] \mid \beta_{(i)}^j \neq \beta_{(i)}^\dagger \}$. By definition of $\cI_{\text{poor}}$, we have $|\cJ_i| < \zeta N < N/2$. As the data from different instances are mutually independent, the vectors $\{\widehat{\beta}_{\text{ind}}^j \}_{j \in [N]}$ are independent. Thus, we can apply Proposition~\ref{prop:tmean_iid} to the trimmed mean of $\{\widehat{\beta}_{\text{ind}}^j\}_{j \in [N]}$, where we use the fact that $\widehat\beta_{\text{ind}}^j$ is $(\sqrt{\sigma_j^2/(n_j \psi)})$-subgaussian:
\begin{align*}
\P\bb{\left|\widehat{\beta}_{\text{RM},(i)}^\dagger - \beta_{(i)}^\dagger\right| \ge C_0 \bp{3\zeta + 4\eta}\max_{\iota \in [N]}\sqrt{\frac{\sigma_\iota^2}{n_\iota \psi}\log(\frac{3}{\eta})}} \le 3\exp\bp{-\frac{N\eta^2}{9}}.
\end{align*}
Using a union bound over all $i \in \cI_{\text{poor}}$, we have
\begin{align}\label{eq:tmean_reg_cmm}
\P\bb{\left\|(\widehat{\beta}_{\text{RM}}^\dagger - \beta^\dagger)_{\cI_{\text{poor}}}\right\|_1 \ge C_0 d\bp{3\zeta + 4\eta}\max_{\iota \in [N]}\sqrt{\frac{\sigma_\iota^2}{n_\iota \psi}\log(\frac{3}{\eta})}} \le 3d\exp\bp{-\frac{N\eta^2}{9}}.
\end{align}
Next, we show that LASSO can recover $\beta^j$ efficiently in Step 2. The rest of the proof is adapted from the LASSO analysis in Chapter 6.2 of \cite{buhlmann2011statistics}. Applying the basic inequality of LASSO to Step 2 of our algorithm gives
\begin{align*}
\frac{1}{n_j}\|\vX^j\widehat{\beta}_{\text{RM}}^j - \vY^j\|_2^2 + \lambda_j \|\widehat{\beta}_{\text{RM}}^j - \widehat{\beta}_{\text{RM}}^\dagger\|_1 \le \frac{1}{n_j}\|\vX^j\beta^j - \vY^j\|_2^2 + \lambda_j \|\beta^j - \widehat{\beta}_{\text{RM}}^\dagger\|_1.
\end{align*}
Plugging in $\vY^j = \vX^j \beta^j + \vE^j$, and conditioned on $\cH^j$ in (\ref{eq:eventH}), we have
\begin{align}
\label{eq:basicineq}
\frac{1}{n_j}\|\vX^j(\widehat{\beta}_{\text{RM}}^j - \beta^j)\|_2^2 + \lambda_j \|\widehat{\beta}_{\text{RM}}^j - \widehat{\beta}_{\text{RM}}^\dagger\|_1
&\le \frac{2}{n_j}\vE^{j \top}\vX^j(\widehat{\beta}_{\text{RM}}^j - \beta^j) + \lambda_j \|\beta^j - \widehat{\beta}_{\text{RM}}^\dagger\|_1 \nonumber \\
& \le \frac{2}{n_j} \|\vX^{j \top}\vE^j\|_{\infty} \|\widehat{\beta}_{\text{RM}}^j - \beta^j\|_1 + \lambda_j \|\beta^j - \widehat{\beta}_{\text{RM}}^\dagger\|_1 \nonumber \\
& \le \frac{\lambda_j}{2} \|\widehat{\beta}_{\text{RM}}^j - \beta^j\|_1 + \lambda_j \|\beta^j - \widehat{\beta}_{\text{RM}}^\dagger\|_1.
\end{align}
Note that the second term on the LHS has
\begin{align*}
\|\widehat{\beta}_{\text{RM}}^j - \widehat{\beta}_{\text{RM}}^\dagger\|_1 
& = \|(\widehat{\beta}_{\text{RM}}^j - \widehat{\beta}_{\text{RM}}^\dagger)_{\bar{\cI}_j^c}\|_1
+ \|(\widehat{\beta}_{\text{RM}}^j - \widehat{\beta}_{\text{RM}}^\dagger)_{\bar{\cI}_j}\|_1 \nonumber \\
& \ge \|(\widehat{\beta}_{\text{RM}}^j - \beta^j)_{\bar{\cI}_j^c}\|_1
- \|(\beta^j - \widehat{\beta}_{\text{RM}}^\dagger)_{\bar{\cI}_j^c}\|_1 
+ \|(\beta^j - \widehat{\beta}_{\text{RM}}^\dagger)_{\bar{\cI}_j}\|_1
- \|(\widehat{\beta}_{\text{RM}}^j - \beta^j)_{\bar{\cI}_j}\|_1.
\end{align*}
Plugging it into \eqref{eq:basicineq} and decomposing the two terms on the RHS based on $\bar{\cI}_j$ and $\bar{\cI}_j^c$, we have
\begin{align}\label{eq:basicineq_arrb}
\frac{1}{n_j}\|\vX^j(\widehat{\beta}_{\text{RM}}^j - \beta^j)\|_2^2 + \frac{\lambda_j}{2} \|(\widehat{\beta}_{\text{RM}}^j - \beta^j)_{\bar{\cI}_j^c}\|_1 
& \le \frac{3\lambda_j}{2} \|(\widehat{\beta}_{\text{RM}}^j - \beta^j)_{\bar{\cI}_j}\|_1 + 2\lambda_j \|(\widehat{\beta}_{\text{RM}}^\dagger - \beta^j)_{\bar{\cI}_j^c}\|_1.
\end{align}
Since $\beta^j - \widehat{\beta}_{\text{RM}}^\dagger = (\beta^\dagger - \widehat{\beta}_{\text{RM}}^\dagger)_{\cI_{\text{pool}}} + (\beta^\dagger_{\cI_{\text{well}}} - \widehat{\beta}^\dagger_{\text{RM},\cI_{\text{well}}} + \delta^j)$, it holds that
\begin{align*}
\|(\widehat{\beta}_{\text{RM}}^\dagger - \beta^j)_{\bar{\cI}_j^c}\|_1
& = \|((\beta^\dagger - \widehat{\beta}_{\text{RM}}^\dagger)_{\cI_{\text{pool}}})_{\bar{\cI}_j^c}\|_1 
\le \|(\beta^\dagger - \widehat{\beta}_{\text{RM}}^\dagger)_{\cI_{\text{pool}}}\|_1,
\end{align*}
where the first equality follows $(\beta^\dagger_{\cI_{\text{well}}} - \widehat{\beta}^\dagger_{\text{RM},\cI_{\text{well}}} + \delta^j)_{\bar{\cI}_j^c} = \zero$, and the second inequality follows $\bar{\cI}_j^c \subseteq \cI_{\text{pool}}$. 
With the above result and adding $\frac{\lambda_j}{2} \|(\widehat{\beta}_{\text{RM}}^j - \beta^j)_{\bar{\cI}_j}\|_1$ on both sides of \eqref{eq:basicineq_arrb}, we obtain
\begin{align}\label{eq:basicineq_arr}
\frac{1}{n_j}\|\vX^j(\widehat{\beta}_{\text{RM}}^j - \beta^j)\|_2^2 + \frac{\lambda_j}{2} \|\widehat{\beta}_{\text{RM}}^j - \beta^j\|_1
& \le 2\lambda_j \|(\widehat{\beta}_{\text{RM}}^j - \beta^j)_{\bar{\cI}_j}\|_1 + 2\lambda_j \|(\widehat{\beta}_{\text{RM}}^\dagger - \beta^\dagger)_{\cI_{\text{poor}}}\|_1.
\end{align}
By Assumption~\ref{ass:posdef_off}, we have
\begin{align}\label{eq:basicineq_sprs}
\|(\widehat{\beta}_{\text{RM}}^j - \beta^j)_{\bar{\cI}_j}\|_1 
& \le \sqrt{(1/\zeta + 1)s} \|\widehat{\beta}_{\text{RM}}^j - \beta^j\|_2
\le \sqrt{\frac{(1/\zeta + 1)s}{\psi} \frac{1}{n_j}\|\vX^j(\widehat{\beta}_{\text{RM}}^j - \beta^j)\|_2^2}.
\end{align}
Therefore, we derive from inequality~(\ref{eq:basicineq_arr}) that
\begin{align*}
\frac{1}{2n_j}\|\vX^j(\widehat{\beta}_{\text{RM}}^j - \beta^j)\|_2^2 + \frac{\lambda_j}{2} \|\widehat{\beta}_{\text{RM}}^j - \beta^j\|_1
& \le \frac{2(1/\zeta + 1)s\lambda_j^2}{\psi} + 2\lambda_j \|(\widehat{\beta}_{\text{RM}}^\dagger - \beta^\dagger)_{\cI_{\text{poor}}}\|_1,
\end{align*}
where we use the fact that $2ab \le a^2 + b^2$. Since $\zeta < 1/2$, we further get
\begin{align*}
\|\widehat{\beta}_{\text{RM}}^j - \beta^j\|_1 \le \frac{6s\lambda_j}{\zeta\psi} + 4 \|(\widehat{\beta}_{\text{RM}}^\dagger - \beta^\dagger)_{\cI_{\text{poor}}}\|_1.
\end{align*}
Combining the above with inequality~(\ref{eq:tmean_reg_cmm}) and Lemma~\ref{def:eventH}, our result then follows. \Halmos
\end{proof}

Now, we prove Theorem~\ref{thm:tmean_reg_hpb} by applying Proposition~\ref{prop:tmean_reg_hpb_arb}. 
\begin{proof}{Proof of Theorem~\ref{thm:tmean_reg_hpb}}
Since $\delta \le 1$ and $d \ge 1$, we have
\begin{align*}
\sqrt{\log(\frac{3}{\eta})} = \sqrt{\frac{1}{2}\log\bp{\frac{N}{\log(\frac{6d}{\delta})}}} \le \sqrt{\log(N)}.
\end{align*}
Plugging our choice of hyperparameters in Proposition~\ref{prop:tmean_reg_hpb_arb}, we have
\begin{align*}
\|\widehat{\beta}_{\text{RM}}^j - \beta^j\|_1 
& \le \frac{6\lambda_js}{\zeta\psi} + C_0 d\bp{3\zeta + 4\eta}\max_{i \in [N]}\sqrt{\frac{\sigma_i^2}{n_i \psi}\log(\frac{3}{\eta})} \\
& \le \frac{24s}{\zeta\psi}\sqrt{\frac{2\sigma_j^2 \log(\frac{4d}{\delta})}{n_j}} + 3C_0\zeta d\max_{i \in [N]}\sqrt{\frac{\sigma_i^2\log(N)}{n_i \psi}} + 12C_0d\max_{i \in [N]}\sqrt{\frac{\sigma_i^2\log(N)\log(\frac{6d}{\delta})}{Nn_i \psi}},
\end{align*}
with probability at least $1 - \delta$. As $n_i$'s are similar in their scales, it suffices to take $\zeta=\frac{C_0-2}{4C_0}\sqrt{\frac{s}{d}}$ to minimize the RHS. 
Thus, we have with high probability
\begin{align*}
\|\widehat{\beta}_{\text{RM}}^j - \beta^j\|_1 \le \frac{96C_0}{(C_0-2)\psi}\sqrt{\frac{2\sigma_j^2 sd\log(\frac{4d}{\delta})}{n_j}} + \frac{3C_0}{4}\max_{i \in [N]}\sqrt{\frac{\sigma_i^2sd\log(N)}{\psi n_i}} + 12C_0d\max_{i \in [N]}\sqrt{\frac{\sigma_i^2\log(N)\log(\frac{6d}{\delta})}{\psi Nn_i}}.
\end{align*}
This holds for any $\delta\ge\exp\bp{-\frac{N}{9}(\frac{C_0-2}{2C_0}(1-\frac{1}{2}\sqrt{\frac{s}{d}}))^2 + \log(6d)}$, as we require $\eta \le 1/2-1/C_0-\zeta$. \Halmos
\end{proof}

\subsection{Data-Poor Regime}
\label{app:tailadpt_multitask_poor}

The hyperparameters are
\begin{gather*}
\lambda=\max_{i\ne j}\sqrt{32\sigma_i^2 \log(\frac{4d}{\delta})}, \quad 
\zeta=\frac{C_0-2}{4C_0}, \quad
\eta=\sqrt{\frac{9\log(\frac{6d}{\delta})}{N-1}}.
\end{gather*}
Note that $\lambda_j = \lambda/\sqrt{n_j}$ as in Algorithm~\ref{alg:tmean_reg}. 

We first provide an estimation error bound for a data-poor task given arbitrary choices of hyperparameters. 
\begin{proposition}\label{prop:tmean_reg_hpb_arb_dp}
The estimator $\widehat{\beta}_{\text{RM}}^j$ of a data-poor task $j$ satisfies
\begin{align*}
\|\widehat{\beta}_{\text{RM}}^j - \beta^j\|_1 \le \frac{6\lambda_js}{\zeta\psi} + C_0 d\bp{3\zeta + 4\eta}\max_{i\ne j}\sqrt{\frac{\sigma_i^2}{n_i \psi}\log(\frac{3}{\eta})}
\end{align*}
with at least probability $1 - \bp{3d\exp(-\frac{N\eta^2}{9}) + 2d\exp(-\frac{\lambda_j^2 n_j}{32\sigma_j^2})}$, for any $\lambda_j>0$, $0<\eta\le1/2-1/C_0-\zeta$ and $0<\zeta<1/2$ with some constant $C_0>2$.
\end{proposition}

\begin{proof}{Proof of Proposition~\ref{prop:tmean_reg_hpb_arb_dp}}
The proof is similar to that of Proposition~\ref{prop:tmean_reg_hpb_arb}. We list the details different from Proposition~\ref{prop:tmean_reg_hpb_arb} below.

Applying Proposition~\ref{prop:tmean_iid} to the trimmed mean of $\{\widehat{\beta}_{\text{ind}}^\iota\}_{\iota \ne j}$, we obtain
\begin{align}\label{eq:tmean_reg_cmm_dp}
\P\bb{\left\|(\widehat{\beta}_{\text{RM}}^\dagger - \beta^\dagger)_{\cI_{\text{poor}}}\right\|_1 \ge C_0 d\bp{3\zeta + 4\eta}\max_{\iota \ne j}\sqrt{\frac{\sigma_\iota^2}{n_\iota \psi}\log(\frac{3}{\eta})}} \le 3d\exp\bp{-\frac{N\eta^2}{9}}.
\end{align}
Then, following the proof steps of Proposition~\ref{prop:tmean_reg_hpb_arb} until \eqref{eq:basicineq_arrb}, we have
\begin{align}\label{eq:basicineq_arr_dpp}
\frac{1}{n_j}\|\vX^j(\widehat{\beta}_{\text{RM}}^j - \beta^j)\|_2^2 + \frac{\lambda_j}{2} \|(\widehat{\beta}_{\text{RM}}^j - \beta^j)_{\bar{\cI}_j^c}\|_1 
& \le \frac{3\lambda_j}{2} \|(\widehat{\beta}_{\text{RM}}^j - \beta^j)_{\bar{\cI}_j}\|_1 + 2\lambda_j \|(\widehat{\beta}_{\text{RM}}^\dagger - \beta^\dagger)_{\cI_{\text{poor}}}\|_1.
\end{align}
Consider the following two cases separately: (i) $\|(\widehat{\beta}_{\text{RM}}^j - \beta^j)_{\bar{\cI}_j}\|_1 \le \|(\widehat{\beta}_{\text{RM}}^\dagger - \beta^\dagger)_{\cI_{\text{poor}}}\|_1$, and (ii) $\|(\widehat{\beta}_{\text{RM}}^j - \beta^j)_{\bar{\cI}_j}\|_1 > \|(\widehat{\beta}_{\text{RM}}^\dagger - \beta^\dagger)_{\cI_{\text{poor}}}\|_1$.
In the first case, we can obtain directly from inequality~\eqref{eq:basicineq_arr_dpp} that
\begin{align*}
 \|\widehat{\beta}_{\text{RM}}^j - \beta^j\|_1 \le 8 \|(\widehat{\beta}_{\text{RM}}^\dagger - \beta^\dagger)_{\cI_{\text{poor}}}\|_1.
\end{align*}
In the second case, note that we have $\|(\widehat{\beta}_{\text{RM}}^j - \beta^j)_{\bar{\cI}_j^c}\|_1
\le 7 \|(\widehat{\beta}_{\text{RM}}^j - \beta^j)_{\bar{\cI}_j}\|_1$; thus, given Assumption~\ref{ass:compcon_off} and $|\bar{\cI}_j| \le (1/\zeta + 1)s$, we have
\begin{align*}
\|(\widehat{\beta}_{\text{RM}}^j - \beta^j)_{\bar{\cI}_j}\|_1 
&\le \sqrt{\frac{(1/\zeta + 1)s}{\psi} \frac{1}{n_j}\|\vX^j(\widehat{\beta}_{\text{RM}}^j - \beta^j)\|_2^2}.
\end{align*}
Then, we can derive from inequality \eqref{eq:basicineq_arr_dpp} that
\begin{align*}
\frac{\psi}{(1/\zeta+1)s}\|(\widehat{\beta}_{\text{RM}}^j - \beta^j)_{\bar{\cI}_j}\|_1^2 & \le \frac{1}{n_j}\|\vX^j(\widehat{\beta}_{\text{RM}}^j - \beta^j)\|_2^2
\le \frac{3\lambda_j}{2} \|(\widehat{\beta}_{\text{RM}}^j - \beta^j)_{\bar{\cI}_j}\|_1,
\end{align*}
that is,
\begin{align*}
\|(\widehat{\beta}_{\text{RM}}^j - \beta^j)_{\bar{\cI}_j}\|_1 & \le \frac{(1/\zeta + 1)s}{\psi}\frac{3\lambda_j}{2} \le \frac{9s\lambda_j}{4\zeta\psi},
\end{align*}
where the last inequality uses $\zeta < 1/2$. Therefore, we have
\begin{align*}
\|(\widehat{\beta}_{\text{RM}}^j - \beta^j)_{\bar{\cI}_j}\|_1 & \le \frac{9s\lambda_j}{4\zeta\psi} + 8 \|(\widehat{\beta}_{\text{RM}}^\dagger - \beta^\dagger)_{\cI_{\text{poor}}}\|_1.
\end{align*}
Combined the above with \eqref{eq:tmean_reg_cmm_dp} and Lemma~\ref{def:eventH}, our result then follows. \Halmos
\end{proof}

Now, we prove Theorem~\ref{thm:tmean_reg_hpb_dp} by applying Proposition~\ref{prop:tmean_reg_hpb_arb_dp}.
\begin{proof}{Proof of Theorem~\ref{thm:tmean_reg_hpb_dp}}
Plugging our choice of hyperparameters in Proposition~\ref{prop:tmean_reg_hpb_arb_dp}, we have
\begin{align*}
\|\widehat{\beta}_{\text{RM}}^j - \beta^j\|_1 \le \frac{24s}{\zeta\psi}\sqrt{\frac{2\sigma_j^2 \log(\frac{4d}{\delta})}{n_j}} + 3C_0\zeta\max_{i \ne j}\sqrt{\frac{d^2\sigma_i^2\log(N)}{n_i \psi}} + 12C_0\max_{i \ne j}\sqrt{\frac{d^2\sigma_i^2\log(N)\log(\frac{6d}{\delta})}{Nn_i \psi}},
\end{align*}
with probability at least $1 - \delta$. Since $n_j$ is assumed to be similar to $n_i/d^2$ for any $i\ne j$ in magnitude, choosing $\zeta$ to be any constant smaller than $1/2-1/C_0$ suffices to minimize the RHS. Thus, we take $\zeta=\frac{C_0-2}{4C_0}$. Then, with probability at least $1-\delta$, we have
\begin{align*}
\|\widehat{\beta}^j - \beta^j\|_1 & \le \frac{96C_0}{(C_0-2)\psi}\sqrt{\frac{2\sigma_j^2 s^2\log(\frac{4d}{\delta})}{n_j}} + \frac{3C_0}{4}\max_{i \ne j}\sqrt{\frac{d^2\sigma_i^2\log(N)}{\psi n_i}} + 12C_0\max_{i \ne j}\sqrt{\frac{d^2\sigma_i^2\log(N)\log(\frac{6d}{\delta})}{Nn_i\psi}}.
\end{align*}
This holds for any $\delta\ge\exp\bp{-\frac{N-1}{9}(\frac{C_0-2}{4C_0})^2 + \log(6d)}$, as we require $\eta \le 1/2-1/C_0-\zeta$. \Halmos
\end{proof}

\subsection{Random Design}\label{app:rand_design}

The proof of Proposition~\ref{prop:tmean_reg_hpb_cov} straightforwardly follows that of Proposition~\ref{prop:tmean_reg_hpb_arb}, once accounting for the event that Assumption~\ref{ass:posdef_off} holds. To that end, we provide the following lemma under Assumption~\ref{ass:posdef_off_true} and \ref{ass:bound_off} that Assumption~\ref{ass:posdef_off} holds with high probability.
\begin{lemma}\label{lem:tmean_off_cov}
We have
\begin{align*}
\P\bb{\lambda_{\min}(\widehat{\Sigma}^j) \ge \frac{\tilde\psi}{2}} \ge 1 - d\exp\bp{-\frac{\tilde\psi n_j}{8dx_{\max}^2}}.
\end{align*}
\end{lemma}
\begin{proof}{Proof of Lemma~\ref{lem:tmean_off_cov}}
The proof follows directly Lemma~\ref{lem:mat_cherbound} by noting that 
$\lambda_{\max}(X_t X_t^\top) \le \|X_t\|_2^2 \le d x_{\max}^2$ and setting $t = 1/2$ and $L = d x_{\max}^2$ in Lemma~\ref{lem:mat_cherbound}. \Halmos
\end{proof}

We provide a similar result of random design for a data-poor task $j$ under Assumption~\ref{ass:posdef_off_true} and \ref{ass:bound_off} (which is useful for the coming bandit analysis). First, we show that the compatibility condition in Assumption~\ref{ass:compcon_off} also holds with high probability under Assumption~\ref{ass:posdef_off_true}, a result comparable to Lemma~\ref{lem:tmean_off_cov}. 
\begin{lemma}\label{lem:tmean_off_cmpcd_dp}
When $n_j \ge 3\log(d)/D_1^2$, we have
\begin{align*}
\P\bb{\widehat{\Sigma}^j \in \cC(\bar{\cI}_j, \frac{\tilde\psi}{2})} \ge 1 - \exp\bp{-D_1^2n_j},
\end{align*}
where $D_1=\min\bc{\frac{1}{2},\frac{\zeta\tilde{\psi}}{768sx_{\max}^2}}$.
\end{lemma}
\begin{proof}{Proof of Lemma~\ref{lem:tmean_off_cmpcd_dp}}
First, note that since $\lambda_{\min}(\Sigma^j)>\tilde\psi$, we have $\Sigma^j \in \cC(\bar{\cI}_j, \tilde\psi)$. This is because, for any $v \in \R^d$ and $\cS\subseteq[d]$, we have $\|v_{\cS}\|_1 \le \sqrt{|\cS|}\|v_{\cS}\|_2$.
Therefore, 
\begin{align*}
|\cS|v^T\Sigma v \ge |\cS|\tilde\psi\|v\|_2^2 \ge |\cS|\tilde\psi\|v_{\cS}\|_2^2 \ge \tilde\psi\|v_{\cS}\|_1^2.
\end{align*}
Then, the result follows by applying Lemma EC.6 in \cite{bastani2020online}. \Halmos
\end{proof}

Next, we introduce a variant of Proposition~\ref{prop:tmean_reg_hpb_cov} for a data-poor task. The proof straightforwardly follows by combining Lemma~\ref{lem:tmean_off_cmpcd_dp} and Proposition~\ref{prop:tmean_reg_hpb_cov}.
\begin{proposition}\label{prop:tmean_reg_hpb_cov_dp}
When $n_j \ge 3\log(d)/D_1^2$, the estimator $\widehat{\beta}_{\text{RM}}^j$ of data-poor task $j$ satisfies
\begin{align*}
\|\widehat{\beta}_{\text{RM}}^j - \beta^j\|_1 \le \frac{12\lambda_j s}{\zeta\tilde\psi} + C_0 d\bp{3\zeta + 4\eta}\max_{i\ne j}\sqrt{\frac{2\sigma_i^2}{n_i\tilde\psi}\log(\frac{3}{\eta})}
\end{align*}
with probability at least
$1 - \bp{3d\exp(-\frac{N\eta^2}{9}) + 2d \exp(-\frac{\lambda_j^2 n_j}{32\sigma_j^2 x_{\max}^2}) + \exp(-D_1^2n_j) + \sum_{i\ne j}d\exp(-\frac{\tilde\psi n_i}{8dx_{\max}^2})}$, for any $\lambda_j>0$, $0<\eta\le1/2-1/C_0-\zeta$ and $0<\zeta<1/2$ with some constant $C_0>2$. 
\end{proposition}

\subsection{Robustness Against Outlier Tasks}
\label{app:tailadpt_multitask_rbstout}

The hyperparameters are
\begin{gather*}
\lambda=\max_{i\in[N]}\sqrt{32\sigma_i^2 \log(\frac{4d}{\delta})}, \quad 
\zeta=\frac{C_0-2}{4C_0}\sqrt{\frac{s+2\varepsilon d/3}{d}}, \quad
\eta=\sqrt{\frac{9\log(\frac{6d}{\delta})}{N}}.
\end{gather*}
Note that $\lambda_j = \lambda/\sqrt{n_j}$ as in Algorithm~\ref{alg:tmean_reg}.

\begin{proof}{Proof of Corollary~\ref{cor:tmean_reg_rbs}}
The proof follows those of Proposition~\ref{prop:tmean_reg_hpb_arb} and Theorem~\ref{thm:tmean_reg_hpb} closely. We list the details that differ below. 

Remember $\bar{\cI}_j = \cI_{\text{well}} \cup \cI_j$, where $\cI_j=\{i\in[d]\mid\beta_{(i)}^j\neq\beta_{(i)}^\dagger\}$. However, now we have $|\bar{\cI}_j| \le (1/\zeta+1)s+\varepsilon d/\zeta$ for $j\in\bar{\cJ}^c$ due to $\varepsilon N$ outlier tasks, and $|\bar{\cI}_j| \le d$ for any outlier $j\in\bar{\cJ}$. The proof for any task $j\in\bar{\cJ}^c$ is the same as that of Proposition~\ref{prop:tmean_reg_hpb_arb} except for \eqref{eq:basicineq_sprs}. Now by Assumption~\ref{ass:posdef_off} and $|\bar{\cI}_j| \le (1/\zeta+1)s+\varepsilon d/\zeta$, we have
\begin{align*}
\|(\widehat{\beta}_{\text{RM}}^j - \beta^j)_{\bar{\cI}_j}\|_1 
& \le \sqrt{(1/\zeta+1)s+\varepsilon d/\zeta} \|\widehat{\beta}_{\text{RM}}^j - \beta^j\|_2
\le \sqrt{\frac{(1/\zeta+1)s+\varepsilon d/\zeta}{\psi} \frac{1}{n_j}\|\vX^j(\widehat{\beta}_{\text{RM}}^j - \beta^j)\|_2^2}.
\end{align*}
Alternatively, for any $j\in\bar{\cJ}$, we have $\bar{\cI}_j^c=\emptyset$ since $|\bar{\cI}_j|\le d$. Therefore, following the proof steps of Proposition~\ref{prop:tmean_reg_hpb_arb} until \eqref{eq:basicineq_arrb}, we have
\begin{align*}
\frac{1}{n_j}\|\vX^j(\widehat{\beta}_{\text{RM}}^j - \beta^j)\|_2^2 + \frac{\lambda_j}{2} \|(\widehat{\beta}_{\text{RM}}^j - \beta^j)_{\bar{\cI}_j^c}\|_1 
& \le \frac{3\lambda_j}{2} \|(\widehat{\beta}_{\text{RM}}^j - \beta^j)_{\bar{\cI}_j}\|_1.
\end{align*}
Then, by Assumption~\ref{ass:posdef_off} and $|\bar{\cI}_j| \le d$, we have instead
\begin{align*}
\|(\widehat{\beta}_{\text{RM}}^j - \beta^j)_{\bar{\cI}_j}\|_1 
& \le \sqrt{d} \|\widehat{\beta}_{\text{RM}}^j - \beta^j\|_2
\le \sqrt{\frac{d}{\psi} \frac{1}{n_j}\|\vX^j(\widehat{\beta}_{\text{RM}}^j - \beta^j)\|_2^2}.
\end{align*}
Correspondingly, we get
\begin{align*}
\|\widehat{\beta}_{\text{RM}}^j - \beta^j\|_1 \le 
\begin{cases}
\frac{(6s+4\varepsilon d)\lambda_j}{\zeta\psi} + 4 \|(\widehat{\beta}_{\text{RM}}^\dagger - \beta^\dagger)_{\cI_{\text{poor}}}\|_1, & \text{for}~j\in\bar\cJ^c, \\
\frac{4d\lambda_j}{\psi}, & \text{for}~j\in\bar\cJ.
\end{cases}
\end{align*}
Again combining the above with inequality~(\ref{eq:tmean_reg_cmm}) and Lemma~\ref{def:eventH}, we get a high probability bound of $\beta^j$.
Finally, given our choice of the hyperparameters and following the proof of Theorem~\ref{thm:tmean_reg_hpb}, it holds for any $j\in\bar\cJ^c$ that
\begin{multline*}
\|\widehat{\beta}_{\text{RM}}^j - \beta^j\|_1 \le \frac{96C_0}{(C_0-2)\psi}\sqrt{\frac{2\sigma_j^2 (s+2\varepsilon d/3)d\log(\frac{4d}{\delta})}{n_j}} + \frac{3C_0}{4}\max_{i \in [N]}\sqrt{\frac{\sigma_i^2(s+2\varepsilon d/3)d\log(N)}{\psi n_i}} \\
+ 12C_0d\max_{i \in [N]}\sqrt{\frac{\sigma_i^2\log(N)\log(\frac{6d}{\delta})}{\psi Nn_i}},
\end{multline*}
and for any $j\in\bar\cJ$ that
\begin{align*}
\|\widehat{\beta}_{\text{RM}}^j - \beta^j\|_1 \le \frac{16d}{\psi}\sqrt{\frac{2\sigma_j^2 \log(\frac{4d}{\delta})}{n_j}},
\end{align*}
with at least a probability of $1-\delta$. This holds for any $\delta\ge\exp\bp{-\frac{N}{9}(\frac{C_0-2}{2C_0}(1-\frac{1}{2}\sqrt{\frac{s+2\varepsilon d/3}{d}}))^2 + \log(6d)}$ and $\varepsilon \le \frac{1+8C_0\sqrt{s}/((C_0-2)\sqrt{d})}{(4C_0/(C_0-2))^2}$, as we require $\eta \le 1/2-1/C_0-\zeta$ and $\zeta \ge \varepsilon$.
\Halmos
\end{proof}

\subsection{Generalized Linear Model}\label{app:glm}

The hyperparameters are
\begin{gather*}
\lambda=\max_{i\in[N]}\sqrt{2\phi_Mx_{\max}^2 \log(\frac{4d}{\delta})}, \quad 
\zeta=\frac{C_0-2}{4C_0}\sqrt{\frac{s}{d}}, \quad
\eta=\sqrt{\frac{9\log(\frac{6d}{\delta})}{N}}.
\end{gather*}
Note that $\lambda_j = \lambda/\sqrt{n_j}$ as in Algorithm~\ref{alg:tmean_reg}.

\begin{proof}{Proof of Corollary~\ref{cor:tmean_glm_hpb_arb}}
Our proof follows closely those of Proposition~\ref{prop:tmean_reg_hpb_arb} and Theorem~\ref{thm:tmean_reg_hpb}. Differently, we replace the linear regression with the maximum likelihood estimation. For simplicity, we will use $\cL(\beta)$ to represent $\cL(\beta \mid \vX^j, \vY^j)$ in this proof. We list the details different from Proposition~\ref{prop:tmean_reg_hpb_arb} as follows.

Denote the Hessian of our our loss function in \eqref{eq:loss_glm} as $\nabla^2\cL(\beta)$. 
Given the property of GLM, the asymptotics of maximum likelihood estimation holds with
\begin{align*}
\sqrt{n_j}(\nabla^2\cL(\widehat\beta_{\text{ind}}^j))^{1/2} (\widehat{\beta}^j_{\text{ind}} - \beta^j) \xrightarrow{d} \cN(0, \vI),
\end{align*}
according to \cite{van2000asymptotic}. Therefore, we have 
\begin{align}\label{eq:glm_conc}
\P\bb{|\widehat{\beta}^j_{\text{ind},(i)} - \beta_{(i)}^j| > t} \le 2\P\bb{|\cN(0, (\nabla^2\cL(\widehat\beta_{\text{ind}}^j))^{-1}_{(i, i)}/n_j)| > t},
\end{align}
for sufficiently large $n_j$ by the definition of convergence in distribution. For a fixed design, the Hessian of our loss function satisfies
\begin{align}\label{eq:str_conv_glm}
\nabla^2\cL(\beta) = \frac{1}{n_j}\sum_{i\in[n_j]}A''(X_i^\top\beta)X_iX_i^\top \ge \phi_m \widehat{\Sigma}^j
\end{align}
for any $\beta$, where we use the strong convexity of $A''$ in the last inequality. Then, we can obtain
\begin{align*}
(\nabla^2\cL(\widehat\beta_{\text{ind}}^j))^{-1}_{(i, i)} \le \lambda_{\max}((\nabla^2\cL(\widehat\beta_{\text{ind}}^j))^{-1}) = \frac{1}{\lambda_{\min}(\nabla^2\cL(\widehat\beta_{\text{ind}}^j))} \le \frac{1}{\phi_m\psi},
\end{align*}
where the last inequality uses Assumption~\ref{ass:posdef_off}. Combined with \eqref{eq:glm_conc} and a Chernoff bound for gaussian distribution \citep{rigollet2015high}, we have 
\begin{align*}
\P\bb{|\widehat{\beta}^j_{\text{ind},(i)} - \beta_{(i)}^j| > t} \le 4\exp\bp{-\frac{n_j\phi_m\psi t^2}{2}},
\end{align*}
that is, $\widehat{\beta}^j_{\text{ind},(i)}$ is $(\sqrt{1/(n_j\phi_m\psi)})$-subgaussian with mean $\beta_{(i)}^j$.

Then, we follow the same proof steps of Proposition~\ref{prop:tmean_reg_hpb_arb} and derive
\begin{align}\label{eq:tmean_glm_cmm}
\P\bb{\left\|(\widehat{\beta}_{\text{RM}}^\dagger - \beta^\dagger)_{\cI_{\text{poor}}}\right\|_1 \ge C_0 d\bp{3\zeta + 4\eta}\max_{\iota \in [N]}\sqrt{\frac{1}{n_\iota\phi_m\psi}\log(\frac{3}{\eta})}} \le 3d\exp\bp{-\frac{N\eta^2}{9}}.
\end{align}
Now we apply the LASSO proof technique to the GLM setting as follows. The basic inequality of LASSO is
\begin{align}\label{eq:tmean_glm_step1}
\cL(\widehat{\beta}_{\text{RM}}^j) + \lambda_j \|\widehat{\beta}_{\text{RM}}^j - \widehat{\beta}_{\text{RM}}^\dagger\|_1 
\le \cL(\beta^j) + \lambda_j \|\beta^j - \widehat{\beta}_{\text{RM}}^\dagger\|_1.
\end{align}
Using a second-order Taylor expansion and combined with \eqref{eq:str_conv_glm}, we obtain
\begin{align*}
\cL(\widehat{\beta}_{\text{RM}}^j) - \cL(\beta^j) - \nabla \cL(\beta^j)^\top (\widehat{\beta}_{\text{RM}}^j - \beta^j) \ge \phi_m (\widehat{\beta}_{\text{RM}}^j - \beta^j)^\top \widehat{\Sigma}^j (\widehat{\beta}_{\text{RM}}^j - \beta^j),
\end{align*}
where $\nabla \cL(\beta^j)$ is the gradient of the loss function. Then, following \eqref{eq:tmean_glm_step1}, we have
\begin{align*}
\phi_m (\widehat{\beta}_{\text{RM}}^j - \beta^j)^\top \widehat{\Sigma}^j (\widehat{\beta}_{\text{RM}}^j - \beta^j) + \lambda_j \|\widehat{\beta}_{\text{RM}}^j - \widehat{\beta}_{\text{RM}}^\dagger\|_1
& \le \|\nabla \cL(\beta^j)\|_{\infty} \|\widehat{\beta}_{\text{RM}}^j - \beta^j\|_1 + \lambda_j \|\beta^j - \widehat{\beta}_{\text{RM}}^\dagger\|_1.
\end{align*}
To bound $\|\nabla \cL(\beta^j)\|_{\infty}$, we apply Lemma 6 in \cite{negahban2010unified} and obtain
\begin{align}\label{eq:eventH_glm}
\P\bb{\|\nabla \cL(\beta^j)\|_{\infty} \le \frac{\lambda_j}{2}} \ge 1 - 2d \exp\bp{-\frac{\lambda_j^2 n_j}{2\phi_M x_{\max}^2}},
\end{align}
under Assumption~\ref{ass:bound_off} and \ref{ass:hessbound_off}. 
The rest of the proof again follows that of Proposition~\ref{prop:tmean_reg_hpb_arb} by noticing 
\begin{align*}
\|(\widehat{\beta}_{\text{RM}}^j - \beta^j)_{\bar{\cI}_j}\|_1 
& \le \sqrt{(1/\zeta + 1)s} \|\widehat{\beta}_{\text{RM}}^j - \beta^j\|_2
\le \sqrt{\frac{(1/\zeta + 1)s}{\psi}(\widehat{\beta}_{\text{RM}}^j - \beta^j)^\top\widehat{\Sigma}^j(\widehat{\beta}_{\text{RM}}^j - \beta^j)}
\end{align*}
under Assumption~\ref{ass:posdef_off}.
Therefore, we finally get
\begin{align*}
\|\widehat{\beta}_{\text{RM}}^j - \beta^j\|_1 \le \frac{6s\lambda_j}{\zeta\psi\phi_m} + 4 \|(\widehat{\beta}_{\text{RM}}^\dagger - \beta^\dagger)_{\cI_{\text{poor}}}\|_1.
\end{align*}
Combining the above with \eqref{eq:tmean_glm_cmm} and \eqref{eq:eventH_glm}, we get a high probability bound of $\beta^j$. 
Finally, given our choice of the hyperparameters and following the proof of Theorem~\ref{thm:tmean_reg_hpb}, it holds that
\begin{align*}
\|\widehat{\beta}_{\text{RM}}^j - \beta^j\|_1 \le \frac{24C_0}{(C_0-2)\psi\phi_m}\sqrt{\frac{2\phi_M x_{\max}^2sd\log(\frac{4d}{\delta})}{n_j}} + \frac{3C_0}{4}\max_{i \in [N]}\sqrt{\frac{sd\log(N)}{\psi\phi_m n_i}} + 12C_0d\max_{i \in [N]}\sqrt{\frac{\log(N)\log(\frac{6d}{\delta})}{\psi\phi_m Nn_i}},
\end{align*}
with at least a probability of $1-\delta$. This holds for any $\delta\ge\exp\bp{-\frac{N}{9}(\frac{C_0-2}{2C_0}(1-\frac{1}{2}\sqrt{\frac{s}{d}}))^2 + \log(6d)}$, as we require $\eta \le 1/2-1/C_0-\zeta$.
\Halmos
\end{proof}

\subsection{Network Structure}\label{app:regret_network}

We analyze the impact of an underlying network structure on the estimation error of $\beta^j$'s. We assume knowledge of a network that captures the similarity between any pair of tasks; this can be inferred based on observed covariates (e.g., geographic distance between hospitals/stores or socio-economic indices of neighborhoods served) or data from past decision-making problems \citep[see, e.g., the disparity matrix in][]{crammer2008learning}. Then, for any given task, we can optimize the ``similarity radius'' of learning problems from which to transfer knowledge, resulting in error bounds that scale with the underlying network density.
Specifically, we examine error bound as a function of the underlying network structure, where vertices represent tasks and edges capture their pairwise similarity. This sheds light on choosing the number of tasks $\widetilde{N} \le N$ for estimation of each task $j$, minimizing a bias-variance tradeoff. In other words, as we incorporate more tasks, we reduce variance (since we have more data) but we increase bias (since we are incorporating observations from more disparate sources). 

Formally, consider the dependence of the parameter estimation error on the underlying network structure of tasks, when available. Particularly, we consider a fully-connected network with $N$ vertices (each representing a task) and edge weights $s_{i,j}$ capturing the pairwise similarities between any two tasks $(i,j)\in[N]\times[N]$ as measured by our sparse difference metric, i.e., $\|\beta^j-\beta^i\|_0 \le s_{i,j}$. Note that this graph is undirected since $s_{i,j}=s_{j,i}$; furthermore, if two tasks $i$ and $j$ are unrelated, then they trivially satisfy $s_{i,j}=d$. 
Then, for any given task $j$, we can optimize the subset of tasks $\cQ_j\subseteq[N]$ from which to transfer knowledge. For simplicity, we assume a strategy where we fix a threshold $\tilde s$, and keep all tasks with sparse disparity at most $\tilde{s}$ --- i.e.,
\begin{align*}
\cQ_j = \{ i \in [N] \mid s_{i,j} \le \tilde{s} \}.
\end{align*}
We denote the effective number of tasks by $\widetilde{N} = |\cQ_j|$. Under this assumption, there is a tradeoff between choosing smaller $\tilde{s}$, which yields smaller $\widetilde{N}$ (resulting in lower bias but larger variance), and larger $\tilde{s}$, which yields larger $\widetilde{N}$ (resulting in higher bias but smaller variance). The optimal choice of $\tilde{s}$ (and correspondingly, $\widetilde{N}$) depends on the relationship between $\tilde{s}$ and $\widetilde{N}$. Here we consider a natural power law scaling --- i.e.,
\begin{align}\label{eq:spars_netstrc}
\tilde{s} = \min(\widetilde{N}^\alpha, d),
\end{align}
for some $\alpha \geq 0$. In other words, as we increase the number of neighbouring tasks we include, our sparsity parameter increases by some power law $\tilde{s}^\alpha$ until it eventually hits the maximum possible value $d$. Our main result allows us to easily compute the optimal choice of $\tilde{s}$ (and $\widetilde{N}$), resulting in estimation errors that scale with the network density $\alpha$.
\begin{corollary}\label{cor:tmean_bdt_rgt_sglnet}
Under the network structure in \eqref{eq:spars_netstrc} and when there are sufficient tasks $N = \Omega(d^\frac{1}{\alpha+1})$, the optimal estimation error of $\widehat{\beta}_{\text{RM}}^j$ is
\begin{align*} 
\|\widehat{\beta}_{\text{RM}}^j - \beta^j\|_1 
= \tilde{\cO}\bp{\frac{d^{\frac{2\alpha+1}{2(\alpha+1)}}}{\sqrt{n_j}}}
\end{align*}
by choosing $\widetilde{N} = \Theta(d^\frac{1}{\alpha+1})$ with at least a probability of $1-\delta$ for any
$\delta\ge\exp\bp{-\frac{N}{9}(\frac{C_0-2}{4C_0})^2+\log(6d)}$ with some constant $C_0>2$, for appropriate choices of hyperparameters $\zeta$, $\eta$, and $\lambda$ provided in Appendix~\ref{app:tailadpt_multitask_std}. 
\end{corollary}
\begin{proof}{Proof of Corollary~\ref{cor:tmean_bdt_rgt_sglnet}}
Our network structure is exogenous regarding the sparsity threshold $\widetilde{s}_j$; thus, we can follow the proof steps of Theorem~\ref{thm:tmean_reg_hpb} but using a selected number of related tasks $\widetilde{N}$ for task $j$. 

Replacing $s$ and $N$ by $\widetilde{s}_j$ and $\widetilde{N}$, plugging $\widetilde{s}_j = \widetilde{N}^\alpha$ in the error bound derived in Theorem~\ref{thm:tmean_reg_hpb}, and optimizing over $\widetilde{N}$, we can derive the optimal choice of $\widetilde{N} = \Theta\bp{d^{\frac{1}{\alpha+1}}}$. The result follows by noticing the constraint on $\delta$ becomes $\delta\ge\exp\bp{-\frac{d^{\frac{1}{\alpha+1}}}{9}(\frac{C_0-2}{4C_0})^2+\log(6d)}$ given the current choice of $\tilde{s}$ and $\widetilde{N}$. \Halmos
\end{proof}

Again, we obtain an improvement in the context dimension $d$; in particular, the estimation error of \textsf{RMEstimator} scales in proportion to $d^{\frac{2\alpha+1}{2(\alpha+1)}}$, which is always smaller than the $d$-scaling of OLS where we do not learn from other tasks. The extent of this estimation error scales with the network density $\alpha$. When $\alpha \rightarrow 0$ (i.e., there are many tasks with high similarity to the target task), we eliminate a factor of $\sqrt{d}$, which can be substantial in high dimension and is aligned with Theorem~\ref{thm:tmean_reg_hpb}; when $\alpha \rightarrow \infty$ (i.e., there are essentially no tasks with high similarity to the target task), our improvement disappears and the error converges to that of OLS.

\subsection{Minimax Lower Bound}\label{app:minimax_rmb}

In this section, we provide a minimax lower bound for our multitask learning problem. This shows that our \textsf{RMEstimator} is minimax optimal and matches the lower bound in the data-poor regime. 

Define
\begin{align*}
\cQ(s, \beta^\dagger) = \{\beta\in\R^d \mid \|\beta - \beta^\dagger\|_0 \le s\}.
\end{align*}
Then, the minimax risk of the estimation error in our multitask learning setting is defined as
\begin{align*}
\tilde\ell(\widehat{\beta}^j, \beta^j) = \inf_{\widehat{\beta}^j} \sup_{\substack{\beta^\dagger\in\R^d; \{\beta^j\}_{j\in[N]}\subseteq\cQ(s, \beta^\dagger) \\
\{\vX^j\}_{j\in[N]}, \{\cP_\epsilon^j\}_{j\in[N]}}}
\E[\|\widehat{\beta}^j - \beta^j\|_1].
\end{align*}
Since we do not know $\beta^\dagger$, we take $\beta^\dagger$ as a parameter as well and consider the worst-case loss over $\beta^\dagger$ together with $\{\beta^j\}_{j\in[N]}$. We have the following result for the minimax risk (defined above) of our multitask problem:
\begin{proposition}\label{prop:minimax_rmb}
The minimax risk of the estimation error in our multitask learning setting in the standard regime satisfies 
\begin{align*}
\tilde\ell(\widehat{\beta}^j, \beta^j)
= \tilde\Omega\bp{\frac{s}{\sqrt{n_j}}+\frac{d}{\sqrt{Nn_j}}},
\end{align*}
and in the data-poor regime satisfies
\begin{align*}
\tilde\ell(\widehat{\beta}^j, \beta^j)
= \Omega\bp{\frac{s}{\sqrt{n_j}}}.
\end{align*}
\end{proposition}
The minimax lower bound in the proposition above matches our upper bound in the data-poor regime in Theorem~\ref{thm:tmean_reg_hpb_dp}. However, in the standard regime, when $d\gg s$, there's a $\cO(\sqrt{d/s})$ mismatch between this lower bound and our upper bound in Theorem~\ref{thm:tmean_reg_hpb}; closing this gap could be an interesting direction of future work. 

\begin{proof}{Proof of Proposition~\ref{prop:minimax_rmb}}
Considering a worst-case scenario over $\vX^j$'s and $\cP^j_{\epsilon}$'s, it suffices to assume $\vE^j \sim \cN(\zero, \sigma_j^2 \vI)$ and $\widehat{\Sigma}^j = \vI$ for $j\in[N]$ in the following.

First, we have
\begin{align*}
\tilde\ell(\widehat{\beta}^j, \beta^j)
\ge \inf_{\widehat{\beta}^j} \sup_{\beta^\dagger=\zero; \{\beta^j\}_{j\in[N]}\subseteq\cQ(s, \zero)} \E[\|\widehat{\beta}^j - \beta^j\|_1]
\ge \inf_{\widehat{\beta}^j} \sup_{\beta^j\in\subseteq\cQ(s, \zero)}\E[\|\widehat{\beta}^j - \beta^j\|_1] = \tilde{\Omega}(\frac{s}{\sqrt{n_j}}).
\end{align*}
The last inequality holds since no knowledge is shared in this case (i.e., $\beta^\dagger=\zero$) and we can treat each task separately. As we have a sparse $\beta^j$, the minimax lower bound coincides with that of a high dimensional linear regression problem. The proof follows a local Fano method \citep[see, e.g., Section 8.4.1 in][]{duchi2023lecture} and a local packing construction \citep[see Theorem 1 (b) in][]{raskutti2011minimax}. We list the details below for completeness. Let $\cV$ be a packing of the set $\{v\in\{-1, 0, 1\}^d \mid \|v\|_0=s\}$ and $V$ be a random variable that takes values uniformly on $\cV$. Then, by Lemma 5 of \cite{raskutti2011minimax}, there exists a $\cV$ such that the cardinality $|\cV| \ge \exp(s\log(\frac{d-s}{s/2})/2)$ and $\|v-v'\|_0\ge s/2$ for any $v, v'\in\cV$. Define $\beta^j_v = \delta v$, and we have 
\begin{align*}
\|\beta^j_v - \beta^j_{v'}\|_1 = \delta \|v - v'\|_1 =  \delta \sum_{i\in[d]} |v_{(i)} - v_{(i)}'| \ge s\delta/2,
\end{align*}
where the last inequality holds since $v$ consists of only values $-1$, $0$ or $1$.
Next, given our assumption on $\vE^j$ and $\widehat{\Sigma}^j$, the KL divergence has the following explicit form and satisfies
\begin{align*}
D_{kl}(\beta^j_v \mid \beta^j_{v'}) = \frac{1}{2\sigma_j^2}\|\vX^j(\beta^j_v-\beta^j_{v'})\|_2^2 \le \frac{n_j\delta^2}{2\sigma_j^2} \|v-v'\|_2^2 \le \frac{4s\delta^2n_j}{\sigma_j^2},
\end{align*}
where the last inequality again uses the fact that each element of $v$ takes only values $-1$, $0$ or $1$.
Using the Fano's inequality \citep[see, e.g., Proposition 8.4.3 in][]{duchi2023lecture} yields the lower bound 
\begin{align*}
\tilde\ell(\widehat{\beta}^j, \beta^j) \ge \frac{s\delta}{2}\bp{1 - \frac{I(V; \vY^j)+\log2}{\log|\cV|}} \ge \frac{s\delta}{2}\bp{1 - \frac{4s\delta^2n_j/\sigma_j^2+\log2}{s\log(\frac{d-s}{s/2})/2}},
\end{align*}
where $I(V; \vY^j)$ is the mutual information between $V$ and $\vY^j$, and the second inequality uses $I(V; \vY^j)\le \frac{1}{|\cV|^2}\sum_{v, v'}D_{kl}(\beta^j_v \mid \beta^j_{v'})$ \citep[see (8.4.5) in][]{duchi2023lecture}. Taking $\delta = \sqrt{\frac{\sigma_j^2\log(\frac{d-s}{s/2})}{32n_j}}$, then we have $\tilde\ell(\widehat{\beta}^j, \beta^j)=\tilde\Omega(\frac{s}{\sqrt{n_j}}).$\footnote{Note that \cite{raskutti2011minimax} provides a loose lower bound of $\tilde{\Omega}(\sqrt{\frac{s}{n_j}})$; here we improve their bound by using the local Fano method.}

Moreover, we also have
\begin{align*}
\tilde\ell(\widehat{\beta}^j, \beta^j)
\ge \inf_{\widehat{\beta}^j} \sup_{\beta^\dagger\in\R^d; \beta^j=\beta^\dagger, \forall j\in[N]} \E[\|\widehat{\beta}^j - \beta^j\|_1] = \tilde{\Omega}(\frac{d}{\sqrt{n}}),
\end{align*}
where $n=\sum_{j\in[N]}n_j$. As $\beta^j=\beta^\dagger$ for any $j\in[N]$ represents the homogeneous case where all tasks are the same, it reduces to a linear regression problem with $n=\sum_{j\in[N]}n_j$ samples. Thus, we can directly apply the minimax lower bound of linear regression in Example 8.4.5 of Section 8.4.1 of \cite{duchi2023lecture}, where we use a similar proof strategy as above. 

Combining the above two results, we have
\begin{align*}
\tilde\ell(\widehat{\beta}^j, \beta^j)
=\tilde{\Omega}(\frac{s}{\sqrt{n_j}}+\frac{d}{\sqrt{n}}).
\end{align*}
In the standard regime where $n_j = \Theta(n/N)$, we have $\tilde\ell(\widehat{\beta}^j, \beta^j)=\tilde{\Omega}(\frac{s}{\sqrt{n_j}}+\frac{d}{\sqrt{Nn_j}})$; in the data-poor regime where $n_j = \Theta(n_{j'}/d^2)$, we have $\tilde\ell(\widehat{\beta}^j, \beta^j)=\tilde{\Omega}(\frac{s}{\sqrt{n_j}}+\frac{1}{\sqrt{Nn_j}})=\tilde{\Omega}(\frac{s}{\sqrt{n_j}})$, since $\frac{1}{\sqrt{Nn_j}} \ll \frac{s}{\sqrt{n_j}}$. \Halmos
\end{proof}

\section{Lower Bounds for Baselines}
\label{app:comp_baselines}

In this section, we provide detailed statements and proofs for the lower bounds discussed in \S\ref{sec:comp_baselines} and Table~\ref{tab:rates}. 

It suffices to derive a lower bound for a concrete instantiation of $\cG = \{\{\vX^j\}_{j\in[N]}, \{\beta^j\}_{j\in[N]}, \{\cP_\epsilon^j\}_{j\in[N]}\}$ since the error measure in \eqref{eq:eval_lwrbd} takes a worst-case scenario over $\cG$.
Therefore, for the remainder of this section, we assume $\vE^j \sim \cN(\zero, \sigma_j^2 \vI)$, and $\widehat{\Sigma}^j = \vI$ for $j \in [N]$. Our choices of errors $\vE^j$ are all gaussian, which ensures the parameter estimates are gaussian as well, thereby enabling us to obtain lower bounds by applying the following lemma:
\begin{lemma}\label{lem:lwrbdl1_gaussian}
Consider a multivariate gaussian random variable $X \sim \cN(\mu, \Sigma) \in \R^d$. We have
\begin{align*}
\E\bb{\|X\|_1} \ge \frac{1}{2}\|\mu\|_1 + \frac{1}{\sqrt{2\pi}}\tr(\Sigma^{\frac{1}{2}}).
\end{align*}
\end{lemma}
\begin{proof}{Proof of Lemma~\ref{lem:lwrbdl1_gaussian}}
Consider the $i^{\text{th}}$ component of $X$, i.e., $X_{(i)}$. Let $\sigma_i^2 = \Sigma_{(i, i)}$. We have $X_{(i)} \sim \cN(\mu_{(i)}, \sigma_i^2)$. Without loss of generality, assume $\mu_{(i)} \ge 0$; otherwise, we can consider $-X_{(i)}$ instead and its $\ell_1$ norm stays the same. By our gaussian assumption, it holds that
\begin{align*}
\E[|X_{(i)}|] & = \int_{-\infty}^\infty |x + \mu_{(i)}| \frac{1}{\sqrt{2\pi \sigma_i^2}}e^{-\frac{x^2}{2\sigma_i^2}} dx 
\ge \int_{0}^\infty (x + \mu_{(i)}) \frac{1}{\sqrt{2\pi \sigma_i^2}}e^{-\frac{x^2}{2\sigma_i^2}} dx
= \frac{1}{2}\mu_{(i)} + \frac{1}{\sqrt{2\pi}}\sigma_i. 
\end{align*}
Then, we further have
\begin{align*}
\E[\|X\|_1] = \sum_{i\in[d]}\E[|X_{(i)}|] & = \frac{1}{2}\|\mu\|_1 + \frac{1}{\sqrt{2\pi}}\sum_{i\in[d]} \sqrt{\Sigma_{(i, i)}} \ge \frac{1}{2}\|\mu\|_1 + \frac{1}{\sqrt{2\pi}}\tr(\Sigma^{\frac{1}{2}}),
\end{align*}
where the last step uses $\sqrt{\Sigma_{(i, i)}} = \|\Sigma^{\frac{1}{2}}_{(i, \cdot)}\|_2 \ge \Sigma^{\frac{1}{2}}_{(i, i)}$. \Halmos
\end{proof}

\subsection{Independent Estimator}
\label{sec:comp_baselines_ind}

First, we provide a proof of lower bound for the independent estimator $\widehat\beta_{\text{ind}}^j$.

\begin{proof}{Proof of Proposition~\ref{prop:est_ind}}
For our choice of $\vX^j$ and $\vE^j$, the estimation error follows a gaussian distribution:
\begin{align*}
\widehat\beta^j_{\text{ind}} - \beta^j \sim \cN\bp{\zero, \frac{\sigma_j^2}{n_j}\vI}.
\end{align*}
Therefore, using Lemma~\ref{lem:lwrbdl1_gaussian}, we have
\begin{align*}
\E\bb{\|\widehat{\beta}^j_{\text{ind}} - \beta^j\|_1} \ge \frac{d\sigma_j}{\sqrt{2\pi n_j}}. \Halmos
\end{align*}
\end{proof}

\subsection{Averaging/Pooling Estimator}
\label{sec:comp_baselines_avg}

Next, we provide a proof of lower bound for the averaging estimator $\widehat\beta_{\text{avg}}^j$ and the pooling estimator $\widehat\beta_{\text{pool}}^j$.

\begin{proof}{Proof of Proposition~\ref{prop:est_avg}}
For our choice of $\vX^j$'s and $\vE^j$'s, the estimation error of the averaging estiamtor follows a gaussian distribution:
\begin{align*}
\widehat{\beta}^j_{\text{avg}} - \beta^j = \frac{1}{N} \sum_{i \in [N]} (\widehat{\beta}_{\text{ind}}^i - \beta^i) + \frac{1}{N} \sum_{i \in [N]} (\delta^i - \delta^j) \sim \cN\bp{\frac{1}{N} \sum_{i \in [N]} (\delta^i - \delta^j), \frac{1}{N^2} \sum_{i \in [N]} \frac{\sigma_i^2}{n_i}\vI}.
\end{align*}
Therefore, by Lemma~\ref{lem:lwrbdl1_gaussian}, we have
\begin{align*}
\E\bb{\|\widehat{\beta}^j_{\text{avg}} - \beta^j\|_1} \ge \frac{1}{2}\left\|\frac{1}{N} \sum_{i \in [N]} (\delta^i - \delta^j)\right\|_1 + \frac{1}{\sqrt{2\pi}} \sqrt{\frac{1}{N}\sum_{i \in [N]}\frac{\sigma_i^2n_j}{n_i}} \frac{d}{\sqrt{Nn_j}}. 
\end{align*}

Similarly, the estimation error of the pooling estimator also follows a gaussian distribution:
\begin{align*}
\widehat{\beta}^j_{\text{pool}} - \beta^j
&= \bp{\sum_{i \in [N]} \vX^{i\top}\vX^i}^{-1}\bp{\sum_{i \in [N]} \vX^{i\top} \vX^i(\delta^i - \delta^j)} + \bp{\sum_{i \in [N]} \vX^{i\top}\vX^i}^{-1}\bp{\sum_{i \in [N]} \vX^{i\top} \vE^i} \\
&\sim \cN\bp{\frac{\sum_{i \in [N]} n_i(\delta^i - \delta^j)}{\sum_{i \in [N]} n_i}, \frac{\sum_{i \in [N]} \sigma_i^2n_i}{(\sum_{i \in [N]} n_i)^2}\vI}.
\end{align*}
Therefore, Lemma~\ref{lem:lwrbdl1_gaussian} implies
\begin{align*}
\E\bb{\|\widehat{\beta}^j_{\text{pool}} - \beta^j\|_1} \ge \frac{1}{2}\left\|\frac{\sum_{i \in [N]} n_i(\delta^i - \delta^j)}{\sum_{i \in [N]} n_i}\right\|_1 + \frac{1}{\sqrt{2\pi}} \sqrt{\frac{(\sum_{i \in [N]}\sigma_i^2n_i)Nn_j}{(\sum_{i \in [N]}n_i)^2}} \frac{d}{\sqrt{Nn_j}}. 
\end{align*}

In data-poor regime, we use all the instances except $j$ to calculate $\widehat{\beta}^j_{\text{avg}}$ and $\widehat{\beta}^j_{\text{pool}}$. The proof strategy is similar. \Halmos
\end{proof}

\subsection{Averaging Multitask Estimator}
\label{sec:comp_baselines_am}

Finally, we provide a proof of lower bound for the averaging multitask estimator. Following the proof of the LASSO lower bound in Theorem 7.1 of \cite{lounici2011oracle}, we assume that $\lambda_j$ is chosen based on the corresponding upper bound analysis; thus, we let $\lambda_j=\sqrt{\frac{32\sigma_j^2}{n_j}\log(\frac{4d}{\delta})}$ and $\delta=d^{-D_0}$ with a constant $D_0\ge3$ through a similar argument as Lemma~\ref{def:eventH}.

\begin{proof}{Proof of Proposition~\ref{prop:est_avglasso}}
The proof strategy is adapted from that of Theorem 7.1 in \cite{lounici2011oracle}. The first order condition of problem~(\ref{eq:prob_avglasso}) is, 
\begin{align}\label{eq:foc_avglasso}
\begin{cases}
\frac{1}{n_j} \bp{\vX^{j\top}(\vY^j - \vX^j \widehat{\beta}^j_{\text{AM}})}_{(i)} = \lambda_j \text{sign}(\widehat{\beta}_{\text{AM},(i)}^j - \widehat\beta_{\text{AM},(i)}^\dagger) & \text{if}~\widehat{\beta}_{\text{AM},(i)}^j \ne \widehat\beta_{\text{AM},(i)}^\dagger \\
\left|\frac{1}{n_j} \bp{\vX^{j\top}(\vY^j - \vX^j \widehat{\beta}^j_{\text{AM}})}_{(i)}\right| \le \lambda_j & \text{if}~\widehat{\beta}_{\text{AM},(i)}^j = \widehat\beta_{\text{AM},(i)}^\dagger.
\end{cases}
\end{align}
Note that on the event $\cH^j$ in (\ref{eq:eventH}), it holds that
$\frac{2}{n_j}|(\vX^{j\top}\vE^j)_{(i)}| \le \frac{\lambda_j}{2}$.
Combining it with (\ref{eq:foc_avglasso}), we have
\begin{align*}
\frac{3\lambda_j}{4}
\le \left|\frac{1}{n_j} \bp{\vX^{j\top}(\vX^j\beta^j - \vX^j \widehat{\beta}^j_{\text{AM}})}_{(i)}\right|
= |(\widehat{\beta}^j_{\text{AM}} - \beta^j)_{(i)}|
\end{align*}
for each $i$ such that $\widehat{\beta}_{\text{AM},(i)}^j \ne \widehat\beta_{\text{AM},(i)}^\dagger$, where the last equality is from our assumption $\widehat{\Sigma}^j = \vI$. Note that for the rest of the components in $[d]$, we have $\widehat{\beta}_{\text{AM},(i)}^j = \widehat\beta_{\text{AM},(i)}^\dagger$. Summing over all $i \in [d]$, we get
\begin{align}\label{eq:lwrbd_avglasso_f1}
\|\widehat{\beta}^j_{\text{AM}}-\beta^j\|_1 \ge \frac{3|\cV|\lambda_j}{4}
\end{align}
with at least a probability of $\P\bb{\cH^j} \ge 1-d^{-C}/2$ given our choice of $\lambda_j$, where $\cV = \{i\in[d] \mid \widehat{\beta}_{\text{AM},(i)}^j \ne \widehat\beta_{\text{AM},(i)}^\dagger\}$. 
Define $\widehat\delta_{\text{AM}}^j = \widehat\beta_{\text{AM}}^j - \widehat\beta_{\text{AM}}^\dagger$ and $\widetilde\delta_{\text{AM}}^j = \beta^j - \widetilde\beta_{\text{AM}}^\dagger$, where $\widetilde\beta_{\text{AM}}^\dagger$ is defined in (\ref{eq:dcmp_avglasso}). Note that $|\cV| = \|\widehat\delta^j_{\text{AM}}\|_0$. 

Next, we prove by contradiction that $\widehat\delta_{\text{AM},(i)}^j = 0$ implies $\widetilde\delta_{\text{AM},(i)}^j = 0$ with high probability for any $i\in[d]$. Suppose this is not true and there exists $i\in[d]$ such that $\widehat\delta_{\text{AM},(i)}^j = 0$ but $\widetilde\delta_{\text{AM},(i)}^j \ne 0$. Again, by the first order condition~(\ref{eq:foc_avglasso}) and on the event $\cH^j$, we have
\begin{align*}
|(\widehat{\beta}^\dagger_{\text{AM}} - \beta^j)_{(i)}| = |(\widehat{\beta}^j_{\text{AM}} - \beta^j)_{(i)}| \le \frac{5\lambda_j}{4},
\end{align*}
and hence
\begin{align*}
|(\beta^j - \widetilde\beta_{\text{AM}}^\dagger)_{(i)}| \le \frac{5\lambda_j}{4} + |(\widetilde\beta_{\text{AM}}^\dagger - \widehat{\beta}^\dagger_{\text{AM}})_{(i)}|.
\end{align*}
Note that 
\begin{align*}
\widehat{\beta}^\dagger_{\text{AM}} - \widetilde\beta_{\text{AM}}^\dagger = \frac{1}{N} \sum_{i \in [N]} (\widehat{\beta}_{\text{ind}}^i - \beta^i) \sim \cN\bp{\zero, \frac{1}{N^2} \sum_{i \in [N]} \frac{\sigma_i^2}{n_i}\vI}.
\end{align*}
Using a Chernoff bound, we have for any $t > 0$ and $i \in [d]$
\begin{align*}
\P\bb{|(\widehat{\beta}^\dagger_{\text{AM}} - \widetilde\beta_{\text{AM}}^\dagger)_{(i)}| \ge t} \le 2\exp\bp{-\frac{t^2}{\frac{2}{N^2}\sum_{m \in [N]}\frac{\sigma_m^2}{n_m}}}.
\end{align*}
Take $t=\sqrt{\frac{2D_0\log(4d)}{N^2}\sum_{m\in[N]}\frac{\sigma_m^2}{n_m}}$, where $D_0$ is any constant with $D_0\ge3$. Then, using a union bound over $\cH^j$, the true parameter $\beta^j$ should satisfy
\begin{align}\label{eq:negamlower}
|(\beta^j - \widetilde\beta_{\text{AM}}^\dagger)_{(i)}| \le \sqrt{\frac{50(1+D_0)\sigma_j^2}{n_j}\log(4d)} + \sqrt{\frac{2D_0}{N^2}\sum_{m\in[N]}\frac{\sigma_m^2}{n_m}\log(4d)}
\end{align}
with at least a probability of $1-d^{-D_0}$.

Note that the inequality \eqref{eq:negamlower} should hold for any parameters $\{\beta^j\}_{j\in[N]}$ if our statement holds true --- i.e., there exists $i\in[d]$ such that $\widehat\delta_{\text{AM},(i)}^j = 0$ but $\widetilde\delta_{\text{AM},(i)}^j \ne 0$. However, we can find $\beta^j$'s that do not satisfy \eqref{eq:negamlower}. Consider the following two situations respectively: (i) $Ns \le d$ and (ii) $Ns > d$. 
In particular, 
\begin{enumerate}[(i).]
\item when $Ns \le d$, let (a) $|\delta^k_{(i)}| > \sqrt{\frac{50(1+D_0)\sigma_j^2N^2}{n_j}\log(4d)} + \sqrt{2D_0\sum_{m\in[N]}\frac{\sigma_m^2}{n_m}\log(4d)}$ when $\delta^k_{(i)} \ne 0$ for any $i\in[d],~k\in[N]$, and (b) $|\{k\in[N]\mid \delta_{(i)}^k\ne 0\}| \le 1$ for any $i\in[d]$; 
\item when $Ns > d$, let (a) $\delta^j_{(i)} > \sqrt{\frac{50(1+D_0)\sigma_j^2d^2}{s^2n_j}\log(4d)} + \sqrt{\frac{2D_0d^2}{s^2N^2}\sum_{m\in[N]}\frac{\sigma_m^2}{n_m}\log(4d)}$ and $-\delta^k_{(i)} > \sqrt{\frac{50(1+D_0)\sigma_j^2d^2}{s^2n_j}\log(4d)} + \sqrt{\frac{2D_0d^2}{s^2N^2}\sum_{m\in[N]}\frac{\sigma_m^2}{n_m}\log(4d)}$ when $\delta^j_{(i)}, \delta^k_{(i)} \ne 0$ for any $k\ne j, ~i\in[d]$, and (b) $|\{k\in[N]\mid \delta_{(i)}^k\ne 0\}| = Ns/d$ for any $i\in[d]$. 
\end{enumerate}
In both cases, we find specific $\beta^j$'s that raise a contradiction to \eqref{eq:negamlower}. As a consequence, whenever $\widehat\delta_{\text{AM},(i)}^j = 0$, it holds that  $\widetilde\delta_{\text{AM},(i)}^j = 0$ with probability at least $1 - d^{-C}$, and hence $\|\widehat\delta_{\text{AM}}^j\|_0 \ge \|\widetilde\delta_{\text{AM}}^j\|_0$. Correspondingly, note that $|\cV| \ge \min\{Ns, d\}$ given $\|\widetilde\delta_{\text{AM}}^j\|_0 = \min\{Ns, d\}$ in our design above. 

Additionally, it always holds true that
\begin{align*}
\|\widehat{\beta}^j_{\text{AM}}-\beta^j\|_1 \ge \|(\widehat\beta^\dagger_{\text{AM}} - \beta^j)_{\cV^c}\|_1.
\end{align*}
By Lemma~\ref{lem:lwrbdl1_gaussian} and given the set $\cV$, we have 
\begin{align}\label{eq:lwbd_avg_var}
\E\bb{\|(\widehat\beta^\dagger_{\text{AM}} - \beta^j)_{\cV^c}\|_1} 
& \ge \frac{1}{2}\left\|\frac{1}{N} \sum_{i \in [N]} (\delta^i - \delta^j)_{\cV^c}\right\|_1 + \frac{1}{\sqrt{2\pi}} \sqrt{\frac{1}{N}\sum_{i \in [N]}\frac{\sigma_i^2n_j}{n_i}} \frac{|\cV^c|}{\sqrt{Nn_j}} \nonumber \\
& = \frac{1}{\sqrt{2\pi}} \sqrt{\frac{1}{N}\sum_{i \in [N]}\frac{\sigma_i^2n_j}{n_i}} \frac{|\cV^c|}{\sqrt{Nn_j}},
\end{align}
where the last equality holds because the support of $\widehat\delta_{\text{AM}}^j$ includes that of $\widetilde\delta_{\text{AM}}^j$ as shown in the last paragraph. 
Given $|\cV| \le d$ and the fact that $|\cV| \ge \min\{Ns, d\}$ holds with probability at least $1-d^{-C}$, we derive from (\ref{eq:lwbd_avg_var}) that
\begin{align}\label{eq:lwbd_avg_var1}
\E\bb{\|\widehat\beta^j_{\text{AM}} - \beta^j\|_1} 
& \ge \E\bb{\frac{1}{\sqrt{2\pi}} \sqrt{\frac{1}{N}\sum_{i \in [N]}\frac{\sigma_i^2n_j}{n_i}} \frac{d-|\cV|}{\sqrt{Nn_j}} \,\middle|\, d \ge |\cV| \ge \min\{Ns, d\}} (1-d^{-D_0}).
\end{align}
Further, from (\ref{eq:lwrbd_avglasso_f1}), we also have
\begin{align}\label{eq:lwbd_avg_bias}
\E\bb{\|\widehat{\beta}^j_{\text{AM}}-\beta^j\|_1} \ge \E\bb{\frac{3|\cV|\lambda_j}{4}}\P\bb{\cH^j} \ge \E\bb{\frac{3(1-d^{-D_0}/2)|\cV|\lambda_j}{4} \,\middle|\, d \ge |\cV| \ge \min\{Ns, d\}}(1-d^{-D_0}).
\end{align}
Combining (\ref{eq:lwbd_avg_var1}) and (\ref{eq:lwbd_avg_bias}), we have
\begin{align*}
\E\bb{\|\widehat\beta^j_{\text{AM}} - \beta^j\|_1} 
& \ge \frac{1-d^{-D_0}}{2}\E\bb{\frac{3(1-d^{-D_0}/2)|\cV|\lambda_j}{4}+\frac{1}{\sqrt{2\pi}} \sqrt{\frac{1}{N}\sum_{i \in [N]}\frac{\sigma_i^2n_j}{n_i}} \frac{d-|\cV|}{\sqrt{Nn_j}} \,\middle|\, d \ge |\cV| \ge \min\{Ns, d\}},
\end{align*}
where we use $\max\{a,b\} \ge (a+b)/2$. As the above lower bound is linear in $|\cV|$, its minimum value is taken at either end of the interval $[\min\{Ns, d\}, d]$. Therefore, we can derive that
\begin{multline*}
\E\bb{\|\widehat\beta^j_{\text{AM}} - \beta^j\|_1} \ge (1-d^{-D_0})\min\left\{\frac{3(1-d^{-D_0}/2)d}{2}\sqrt{\frac{2(1+D_0)\sigma_j^2}{n_j}\log(4d)}, \right.\\ \left.\frac{3(1-d^{-D_0}/2)\min\{Ns, d\}}{2}\sqrt{\frac{2(1+D_0)\sigma_j^2}{n_j}\log(4d)} + \frac{1}{2\sqrt{2\pi}} \sqrt{\frac{1}{N}\sum_{i \in [N]}\frac{\sigma_i^2n_j}{n_i}} \frac{\max\{d-Ns, 0\}}{\sqrt{Nn_j}}\right\}.
\end{multline*}
Therefore, when $Ns = o(d)$, we can write
\begin{align*}
\E\bb{\|\widehat\beta^j_{\text{AM}} - \beta^j\|_1} 
& = \tilde\Omega(\frac{Ns}{\sqrt{n_j}} + \frac{d}{\sqrt{Nn_j}});
\end{align*}
when $Ns = \Omega(d)$, we get
\begin{align*}
\E\bb{\|\widehat\beta^j_{\text{AM}} - \beta^j\|_1} & 
= \tilde\Omega(\frac{d}{\sqrt{n_j}}).
\end{align*}
The proof for the data-poor regime is similar, where in the first stage we use all the instances except $j$ to calculate $\widehat{\beta}^\dagger_{\text{AM}}$. \Halmos
\end{proof}

\section{Proof Strategy for \textsf{RMBandit}} \label{sec:rmbandit-proofstrategy}

In this section, we sketch the proof of our regret bound in Proposition~\ref{prop:tmean_bdt_rgt_all} (and hence Theorem~\ref{thm:tmean_bdt_rgt_sgl}). The proof builds on the regret analysis of LASSO Bandit~\citep{bastani2020online}, but with the confidence intervals afforded by our \textsf{RMEstimator} in the random design. As noted earlier, one key challenge is the requirement that the OLS estimators $\{\widehat\beta^j_{\text{ind}}\}_{j\in[N]}$ across different instances be independent in order to invoke our \textsf{RMEstimator}; \textsf{RMBandit} achieves this goal using a batching strategy, as highlighted in Lemma~\ref{lem:tmean_bdt_ase_ind}.
 
\textbf{Forced-Sample Estimator.}
Our algorithm uses a separate forced-sample estimator, which we can guarantee is close to the true parameter with high probability. 
Intuitively, the forced-sample estimator is sufficiently accurate to exclude arms in $\cK_{sub}^j$ from consideration, and thus the all-sample estimator only needs to identify the optimal arms among $\cK_{opt}^j$, which can be proved guaranteed with high probability.
\begin{proposition}\label{prop:tmean_bdt_fse}
When $N = \Omega\bp{\log(d)\log(T)}$, the forced-sample estimator $\widehat{\beta}_{k,0}^j = \widehat{\beta}_k^j(\cB_0,\lambda_{0,j},\omega_0)$ satisfies
\begin{align*}
\P\bb{\|\widehat{\beta}_{k,0}^j - \beta_k^j\|_1 \ge \frac{h}{4x_{\max}}} \le \frac{8}{T},
\end{align*}
for the choices of hyperparameter $\zeta_{0}$, $\eta_0$, $\lambda_{0}$, and $q$ specified in Appendix~\ref{app:regret_sfixed_std}.
\end{proposition}
We give a proof in Appendix~\ref{app:regret_sfixed_std}. At a high level, this result follows directly from our tail inequality in a random design (i.e., Proposition~\ref{prop:tmean_reg_hpb_cov}), since the forced samples are i.i.d. random variables. 

\textbf{All-Sample Estimator.}
Next, we provide a tail inequality for our all-sample estimator for all arms that belong to $\cK_{opt}^j$. In contrast to the forced-sample estimator, which is based on $\cO(\log(T))$ samples, the all-sample estimator is based on $\cO(T)$ samples (since we will show that all optimal arms receive a linear number of samples with high probability). Therefore, the all-sample estimator has smaller error than the forced-sample estimator (the tradeoff is that these samples are adaptively assigned to arms, so they may be collected from biased regions of the covariate space; thus, the i.i.d. samples generated when using the forced-sample estimator are needed to ensure that the all-sample estimator converges). In particular, define the following event, which says that all the forced-sample estimators have small error:
\begin{align}\label{eq:eventA}
\cA = \bc{\|\widehat{\beta}_k^j(\cB_0,\lambda_{0,j},\omega_0) - \beta_k^j\|_1 \le \frac{h}{4x_{\max}}, \forall j \in [N], k \in [K]}. 
\end{align}
This event holds with high probability by Proposition~\ref{prop:tmean_bdt_fse}. Our next result shows that our all-sample estimator satisfies the following tail inequality conditional on the event $\cA$.
\begin{proposition}\label{prop:tmean_bdt_ase}
When the event $\cA$ holds and $N = \Omega\bp{\log(d)\log(T)}$,
the all-sample estimator $\widehat\beta^j_{k,\bar{m}}=\widehat\beta^j_k(\cB_{\bar{m}}, \lambda_{1,j,\bar{m}}, \omega_{1,m})$ of optimal arm $k \in \cK_{\text{opt}}^j$ satisfies
\begin{align*}
\|\widehat\beta^j_{k,\bar{m}} - \beta_k^j\|_1 \le C_1 \sqrt{\frac{sd\log(d p_j|\cB_{\bar{m}}|)}{p_j|\cB_{\bar{m}}|}} + C_2 \sqrt{\frac{sd\log(\rho N)}{p_j|\cB_m|}} + C_3 d\sqrt{\frac{\log(dp_j|\cB_m|)\log(\rho N)}{N p_j |\cB_m|}}
\end{align*}
with probability at least $1 - \bp{\frac{6}{\min_{i\in\cW_k}p_i|\cB_m|} + \frac{4}{p_j|\cB_m|}
+ \sum_{i \in \cW_k}7d\exp(-\frac{p_*p_i\psi|\cB_m|}{32dx_{\max}^2})}$
for the hyperparameter choices $\zeta_{1,0}$, $\eta_{1,0}$, and $\lambda_{1,0}$, and the constants $C_1$, $C_2$ and $C_3$ specified in Appendix~\ref{app:regret_sfixed_std}.
\end{proposition}
We give a proof in Appendix~\ref{app:regret_sfixed_std}. As previously discussed, since our all-sample estimators are constructed using all available samples, they may not be independent across instances; however, the trimmed mean estimator in Step 1 of our \textsf{RMEstimator} (described in \S\ref{sec:robmulti_est_overview}) requires that the OLS inputs are independent across instances. By using a batching strategy, we ensure that the samples from the same batch $\cB_m$ that we use to calculate the OLS inputs are (conditionally) independent across bandit instances (note that in Step 2 we still use all the data of the target instance from all batches $\cB_{\bar{m}}$).  
In particular, we have the following lemma:
\begin{lemma}\label{lem:tmean_bdt_ase_ind}
The samples assigned to arm $k$ in batch $\cB_m$ (for any $m \ge 1$) are independent across bandit instances conditioned on $\cF_{\bar{m-1}} = \sigma(\{X_t, Z_t, Y_t\}_{t \in \cB_{\bar{m-1}}})$, the $\sigma$-algebra generated by the samples from $\cB_{\bar{m-1}}$.
\end{lemma}
We give a proof in Appendix~\ref{app:regret_sfixed_std}.
Given this lemma, Proposition~\ref{prop:tmean_bdt_ase} follows instantly by applying Proposition~\ref{prop:tmean_reg_hpb_cov}.

\textbf{Regret Analysis.}
Finally, we describe how the above results enable us to prove Proposition~\ref{prop:tmean_bdt_rgt_all}. For this regret analysis, we decompose time steps $t\in[T]$ into three cases, and bound the regret across time steps in each case separately:
\begin{enumerate}[(i).]
\item \label{case:regret1} when $T \le N$, or the forced-sample batch ($t \in \cB_0$) or the first all-sample batch ($t \in \cB_1$);
\item \label{case:regret2} when $T \ge N$ and $\cA$ does not hold, all the remaining batches ($t \in \cB_m$ for $m > 1$);
\item \label{case:regret3} when $T \ge N$ and $\cA$ holds, all the remaining batches ($t \in \cB_m$ for $m > 1$).
\end{enumerate}
For case (\ref{case:regret1}), note that the sizes of the first two batches $\cB_0$ and $\cB_1$ are both $q\log(T)$ and $q$ scales as $\tilde\cO(Kd(sN+d))$. In the worst case, the regret for one time step is at most $2bx_{\max}$, so the regret in this case is bounded.
For case (\ref{case:regret2}), we have shown that the event $\cA$ holds with high probability. Similar to case (\ref{case:regret1}), in the worst case, the regret for one time step is at most $2bx_{\max}$, so the regret in this case is also bounded with high probability.
Finally, for case (\ref{case:regret3})
when $\cA$ holds, Proposition~\ref{prop:tmean_bdt_ase} guarantees that the all-sample estimator has small error with high probability, again ensuring that the regret is bounded with high probability. We provide the details in Appendix~\ref{app:regret_sfixed_std}.

\section{Multitask Bandits}\label{app:regret_sfixed}

In this section, we provide the proofs for Proposition~\ref{prop:tmean_bdt_rgt_all} and Theorem~\ref{thm:tmean_bdt_rgt_sgl} in \S\ref{app:regret_sfixed_std}, and Theorem~\ref{thm:bdt_rgt_sgl_dp} in \S\ref{app:regret_datapoor}.

\subsection{Standard Regime} \label{app:regret_sfixed_std}

The hyperparameters are
\begin{gather*}
\omega_0=\zeta_0+\eta_0, \quad 
\zeta_0 = \zeta_{1,0} = \frac{C_0-2}{4C_0}\sqrt{\frac{s}{d}}, \quad
\eta_0 = \sqrt{\frac{27\log(d)|\cB_0|}{qN}}, \quad
\eta_{1,0} = \sqrt{\frac{9}{\rho N}}, \\
\lambda_{0} = \max_{i\in[N]}\sqrt{\frac{96\sigma_i^2x_{\max}^2K\log(d)|\cB_0|}{q}}, \quad
\lambda_{1,0} = \max_{i\in[N]}\sqrt{\frac{384\sigma_i^2x_{\max}^2}{p_*}},
\end{gather*}
and 
\begin{multline*}
q = \max\left\{\frac{2\cdot384^3C_0^2 x_{\max}^4 (\max_{i \in [N]}\sigma_i^2/p_i) Ksd\log(d)}{(C_0-2)^2h^2 p_*^2 \psi^2}, \frac{576C_0^2x_{\max}^2(\max_{i \in [N]}\sigma_i^2/p_i)Ksd\log(N)}{h^2p_*\psi}, \right.\\
\left. \frac{6\cdot192^2C_0^2x_{\max}^2(\max_{i \in [N]}\sigma_i^2/p_i)Kd^2\log(d)\log(N)}{h^2p_*\psi N}, \frac{96x_{\max}^2Kd\log(dN)}{p_*\psi(\min_{i \in [N]} p_i)} \right\}.
\end{multline*}
Note that
\begin{align*}
\lambda_{0,j} = \lambda_{0}/\sqrt{|\cB_0^j|}, \quad
\zeta_{1,m}=\zeta_{1,0}, \quad
\eta_{1,m} = \eta_{1, 0} \sqrt{\log(d\min_{i \in [N], |\cB_m^i| > 0}|\cB_m^i|)}, \quad
\lambda_{1,j,\bar{m}} = \lambda_{1,0} \sqrt{\frac{\log(d|\cB_{\bar{m}}^j|)}{|\cB_{\bar{m}}^j|}}
\end{align*}
as in Algorithm~\ref{alg:tmean_bdt}. 

The constants in Proposition~\ref{prop:tmean_bdt_ase} are
\begin{align*}
C_1 = \frac{96^2C_0\sigma_jx_{\max}}{(C_0-2)p_*^{3/2}\psi}, \quad 
C_2 = \frac{3C_0}{2p_*}(\max_{i \in [N]}\sqrt{\frac{2\sigma_i^2p_j}{\psi p_i}}), \quad
C_3 = \frac{24C_0}{p_*}(\max_{i \in [N]}\sqrt{\frac{2\sigma_i^2p_j}{\psi\rho p_i}}).
\end{align*}

\textbf{Forced-Sample Estimator.}
First, we provide an estimation error bound for our forced-sample estimators (i.e., Proposition~\ref{prop:tmean_bdt_fse}). 
For simplicity, we use $\zeta$, $\eta$ and $\lambda_j$ to represent $\zeta_0$, $\eta_0$ and $\lambda_{0,j}$ (described at the beginning of \S\ref{app:regret_sfixed_std}) respectively in the following.

\paragraph{Additional Notation.} Let $\cB_0^j$ be the index set of observations at instance $j$, $\cB_{0, k}^j \subseteq \cB_0^j$ be the subset forced sampled at arm $k$, and $\bar{\cB}_{0, k}^j\subseteq \cB_{0, k}^j$ be the subset where $X_t \in U_k^j$; particularly, 
\begin{gather*}
\cB_0^j = \{t \in \cB_0 \mid Z_t = j\}, 
\quad \cB_{0, k}^j = \{t \in \cB_0 \mid Z_t = j, \, (k - 1)\equiv(\sum_{r\in[t]} \mathbbm{1}(Z_r = j) - 1) \bmod K\}, \\
\bar{\cB}_{0, k}^j = \{t \in \cB_0 \mid Z_t = j, \, X_t \in U_k^{Z_t}, \, (k - 1)\equiv(\sum_{r\in[t]} \mathbbm{1}(Z_r = j) - 1) \bmod K\}.
\end{gather*}
Let $\widehat{\Sigma}(\cB)$ be the sample covariance matrix calculated using the samples $\{X_t\}_{t \in \cB}$, i.e., $\widehat{\Sigma}(\cB) = \sum_{t\in\cB}X_tX_t^\top/|\cB|$.

\begin{lemma}\label{lem:tmean_bdt_fse_ind}
The forced samples at arm $k$ are independent across bandit instances.
\end{lemma}
\begin{proof}{Proof of Lemma~\ref{lem:tmean_bdt_fse_ind}}
The forced samples of arm $k$ at instance $j$ are $\{(X_t, Y_t)\}_{t \in \cB_{0, k}^j}$, where the set of covariates is
\begin{align*}
\{X_t \mid t \in \cB_{0, k}^j, \, Z_t = j, \, (k - 1)\equiv(\sum_{r\in[t]} \mathbbm{1}(Z_r = j) - 1) \bmod K\}.
\end{align*}
Since $\{(X_t, Z_t)\}_{t \in \cB_0}$ are independent, $X_t$ is independent of $Z_{t'}$ and $X_{t'}$ conditional on $Z_t$ for any $t' \ne t$ and $t' \in \cB_0$. Therefore, given $Z_t=j$, we derive that $\sum_{r\in[t]} \mathbbm{1}(Z_r = j) -1 = \sum_{r\in[t-1]} \mathbbm{1}(Z_r = j)$ is also independent of $X_t$. This implies $X_t$'s observed at arm $k$ are independent across bandit instances. Similarly, $Y_t$'s are also independent across different instances noting that the noises $\epsilon_t$ only depends on $Z_t$ by design. The result then follows. \Halmos
\end{proof}

\begin{lemma}\label{lem:tmean_bdt_fse_iid}
The samples $\{X_t\}_{t \in \cB_{0, k}^j}$ are i.i.d. with distribution $\cP_{X}^j$, and its subset $\{X_t\}_{t \in \bar{\cB}_{0, k}^j}$ are i.i.d. with distribution $\cP_{X \mid X \in U_k^j}^j$.
\end{lemma}
\begin{proof}{Proof of Lemma~\ref{lem:tmean_bdt_fse_iid}}
Similar to the proof of Lemma~\ref{lem:tmean_bdt_fse_ind}, we can show that $\{X_t\}_{t \in \cB_{0, k}^j}$ are independent. As $\sum_{r\in[t]} \mathbbm{1}(Z_r = j)$ is independent of $X_t$ given $Z_t = j$, $X_t$ follows the distribution $\cP_{X}^j$. On the other hand, the subsamples $\{X_t\}_{t \in \bar{\cB}_{0, k}^j}$ form the set
\begin{align*}
\{X_t \mid t \in \bar{\cB}_{0, k}^j, \, X_t \in U_k^{Z_t}, \, Z_t = j, \, (k - 1)\equiv(\sum_{r\in[t]} \mathbbm{1}(Z_r = j) - 1) \bmod K\}.
\end{align*}
Since $\{(X_t, Z_t)\}_{t \in \cB_0}$ are independent, $X_t$ is independent of $Z_{t'}$ and $X_{t'}$ conditional on $\bc{X_t \in U_k^{Z_t}, Z_t=j}$ for any $t' \ne t$ and $t' \in \cB_0$. Correspondingly, we can conclude that $\{X_t\}_{t \in \bar{\cB}_{0, k}^j}$ are i.i.d. drawn from $\cP_{X \mid X \in U_k^j}^j$. \Halmos
\end{proof}

\begin{lemma}\label{lem:tmean_bdt_fse_armsmp}
Define the events
\begin{gather}\label{eq:eventDM}
\bar{\cD}_{0,k}^j = \bc{|\bar{\cB}_{0, k}^j| \ge \frac{p_*}{2}|\cB_{0, k}^j|}, 
\quad \cM_0^j = \bc{|\cB_0^j| \ge \frac{p_j}{2}|\cB_0|}.
\end{gather}
Then, it holds that
\begin{align*}
\P\bb{\bar{\cD}_{0,k}^j} \ge 1 - 2\exp\bp{-\frac{p_* |\cB_{0, k}^j|}{10}}, 
\quad \P\bb{\cM_0^j} \ge 1 - 2\exp\bp{-\frac{p_j|\cB_0|}{10}},
\end{align*}
given $|\cB_{0, k}^j|$ and $|\cB_0|$ respectively. 
\end{lemma}
\begin{proof}{Proof of Lemma~\ref{lem:tmean_bdt_fse_armsmp}}
Applying Lemma~\ref{lem:ind_cherbound} to the indicator random variables $\mathbbm{1}\bp{t \in \bar{\cB}_{0, k}^j}$ for all $t \in \cB_{0, k}^j$ with $\mu = \E\bb{\sum_{t \in \cB_{0, k}^j}\mathbbm{1}\bp{t \in \bar{\cB}_{0, k}^j}} = \sum_{t \in \cB_{0, k}^j} \P\bb{X_t \in U_k^j \mid Z_t = j}$,
we have
\begin{align*}
\P\bb{||\bar{\cB}_{0, k}^j| - \mu| \ge \frac{\mu}{2}} \le 2\exp\bp{-\frac{\mu}{10}}.
\end{align*}
Noting $\mu \ge p_* |\cB_{0, k}^j|$ by Assumption~\ref{ass:armopt}, our first result then follows. 

Similarly, applying Lemma~\ref{lem:ind_cherbound} to the indicator random variables $\mathbbm{1}\bp{Z_t=j}$ for all $t \in \cB_0$ with $\mu = \E\bb{\sum_{t \in \cB_0}\mathbbm{1}\bp{Z_t=j}} = \sum_{t \in \cB_0} \P\bb{Z_t = j} = p_j|\cB_0|$, we have
\begin{align*}
\P\bb{||\cB_0^j| - p_j|\cB_0|| \ge \frac{p_j|\cB_0|}{2}} \le 2\exp\bp{-\frac{p_j|\cB_0|}{10}},
\end{align*}
which implies our second result.\Halmos
\end{proof}

\begin{lemma}\label{lem:tmean_bdt_fse_cov}
Define the event 
\begin{align}\label{eq:eventE0}
\bar{\cE}_{0,k}^j =\bc{\lambda_{\min}(\widehat{\Sigma}(\bar{\cB}_{0, k}^j)) \ge \frac{\psi}{2}}.
\end{align}
Then, we have 
\begin{gather*}
\P\bb{\bar{\cE}_{0,k}^j} \ge 1 - d\exp\bp{-\frac{\psi|\bar{\cB}_{0, k}^j|}{8dx_{\max}^2}}, \\
\P\bb{\lambda_{\min}(\widehat\Sigma(\cB_{0, k}^j)) \ge \frac{p_*\psi}{4}}
\ge 1 - \bp{d\exp\bp{-\frac{p_*\psi|\cB_{0, k}^j|}{16dx_{\max}^2}} + 2\exp\bp{-\frac{p_* |\cB_{0, k}^j|}{10}}},
\end{gather*} 
given $|\bar{\cB}_{0, k}^j|$ and $|\cB_{0, k}^j|$ respectively.
\end{lemma}
\begin{proof}{Proof of Lemma~\ref{lem:tmean_bdt_fse_cov}}
Note that $\{X_tX_t^\top\}_{t\in\bar{\cB}_{0,k}^j}$ are i.i.d. according to Lemma~\ref{lem:tmean_bdt_fse_iid}. Our first result follows then by applying Lemma~\ref{lem:tmean_off_cov}. 

Conditioned on $\bar{\cD}_{0,k}^j$ in (\ref{eq:eventDM}) and $\bar{\cE}_{0,k}^j$ in (\ref{eq:eventE0}), Lemma~\ref{lem:bdt_fse_covsubsmp} implies $\lambda_{\min}(\widehat{\Sigma}(\cB_{0, k}^j)) \ge p_*\psi/4$.
Therefore, we have
\begin{align*}
\P\bb{\lambda_{\min}(\widehat\Sigma(\cB_{0, k}^j)) \le \frac{p_*\psi}{4}}
& \le \P\bb{(\bar{\cD}_{0,k}^j)^c\cup(\bar{\cE}_{0,k}^j)^c} \le \P\bb{(\bar{\cE}_{0,k}^j)^c \,\middle|\, \bar{\cD}_{0,k}^j} + \P\bb{(\bar{\cD}_{0,k}^j)^c}.
\end{align*}
Then, applying Lemma~\ref{lem:tmean_bdt_fse_armsmp}, we have
\begin{gather*}
\P\bb{(\bar{\cD}_{0,k}^j)^c} \le 2\exp\bp{-\frac{p_* |\cB_{0, k}^j|}{10}}, \\
\P\bb{(\bar{\cE}_{0,k}^j)^c \,\middle|\, \bar{\cD}_{0,k}^j} \le \E\bb{d\exp\bp{-\frac{\psi|\bar{\cB}_{0, k}^j|}{8dx_{\max}^2}} \,\middle|\, \bar{\cD}_{0,k}^j} \le d\exp\bp{-\frac{p_*\psi|\cB_{0, k}^j|}{16dx_{\max}^2}}. 
\end{gather*}
Our second result then follows. \Halmos
\end{proof}

Now, we prove Proposition~\ref{prop:tmean_bdt_fse} by combining Proposition~\ref{prop:tmean_reg_hpb_cov} and all the previous results.
\begin{proof}{Proof of Proposition~\ref{prop:tmean_bdt_fse}}
Lemma~\ref{lem:tmean_bdt_fse_ind} and \ref{lem:tmean_bdt_fse_iid} imply that the forced-sample OLS estimators are subgaussian and independent across task $j\in[N]$; thus, the conditions of Proposition~\ref{prop:tmean_reg_hpb_cov} are satisfied. Applying Proposition~\ref{prop:tmean_reg_hpb_cov} by setting $\tilde\psi = p_*\psi/2$, we have 
\begin{multline*}
\P\bb{\|\widehat{\beta}_{k,0}^j - \beta_k^j\|_1 \ge \frac{24\lambda_j s}{p_*\zeta\psi} + 2C_0d(3\zeta+4\eta)\max_{i \in [N]}\sqrt{\frac{\sigma_i^2}{p_*\psi|\cB_{0, k}^i|}\log(\frac{3}{\eta})}} \\
\le 3d\exp\bp{-\frac{N\eta^2}{9}} + 2d \exp\bp{-\frac{\lambda_j^2 |\cB_{0, k}^j|}{32\sigma_j^2 x_{\max}^2}} + \sum_{i \in [N]}\P\bb{\lambda_{\min}(\widehat\Sigma(\cB_{0, k}^i)) \le \frac{p_*\psi}{4}} \\
\le 3d\exp\bp{-\frac{N\eta^2}{9}} + 2d \exp\bp{-\frac{\lambda_j^2 |\cB_{0, k}^j|}{32\sigma_j^2 x_{\max}^2}} \\
 + \sum_{i \in [N]}d\exp\bp{-\frac{p_*\psi|\cB_{0, k}^i|}{16dx_{\max}^2}} + \sum_{i \in [N]}2\exp\bp{-\frac{p_* |\cB_{0, k}^i|}{10}},
\end{multline*}
given $\{|\cB_{0,k}^i|\}_{i\in[N]}$, where the second inequality uses Lemma~\ref{lem:tmean_bdt_fse_cov}.
Since $|\cB_{0, k}^j| = |\cB_{0}^j|/K$ by our forced sampling design, we have for given $\{|\cB_{0}^i|\}_{i\in[N]}$
\begin{multline}\label{eq:fse_esterrb_armsmp1}
\P\bb{\|\widehat{\beta}_{k,0}^j - \beta_k^j\|_1 \ge \frac{24\lambda_j s}{p_*\zeta\psi} + 2C_0d(3\zeta+4\eta)\max_{i \in [N]}\sqrt{\frac{K\sigma_i^2}{p_*\psi|\cB_{0}^i|}\log(\frac{3}{\eta})}} \\
\le 3d\exp\bp{-\frac{N\eta^2}{9}} + 2d \exp\bp{-\frac{\lambda_j^2 |\cB_{0}^j|}{32K\sigma_j^2 x_{\max}^2}} \\
+ \sum_{i \in [N]}d\exp\bp{-\frac{p_*\psi|\cB_{0}^i|}{16Kdx_{\max}^2}} + \sum_{i \in [N]}2\exp\bp{-\frac{p_* |\cB_{0}^i|}{10K}}.
\end{multline}

Now we configure the hyperparameters $\zeta$, $\lambda_j$, $\eta$ and $q$ (recall that $|\cB_0|=q\log(T)$) such that our forced-sample estimator has estimation error smaller than $h/(4x_{\max})$ with high probability. Similar to the proof of Theorem~\ref{thm:tmean_reg_hpb}, we take $\zeta = \frac{C_0-2}{4C_0}\sqrt{\frac{s}{d}}$; we further set
\begin{align*}
\eta = \sqrt{\frac{27\log(d)|\cB_0|}{qN}},
\quad
\lambda_j = \sqrt{\frac{96\sigma_j^2x_{\max}^2K\log(d)|\cB_0|}{q|\cB_0^j|}}.
\end{align*}
In the following, we will frequently use the inequality
\begin{align}\label{eq:log_sum2mul}
3\log(T)\log(x) \ge \log(Tx)
\end{align} 
for $T > 1$ and any $x > 1$. Given our choices of $\eta$ and $\lambda_j$ and inequality~(\ref{eq:log_sum2mul}), the sum of the first two probabilities on the RHS of (\ref{eq:fse_esterrb_armsmp1}) is upper bounded by $5/T$. To bound the sizes of the batches $\{|\cB_0^i|\}_{i\in[N]}$, we apply a union bound over $\bigcap_{i\in[N]} \cM_0^i$ defined in (\ref{eq:eventDM}) using Lemma~\ref{lem:tmean_bdt_fse_armsmp}. This yields 
\begin{multline}\label{eq:fse_esterrb_armsmp2}
\P\bb{\|\widehat{\beta}_{k,0}^j - \beta_k^j\|_1 \ge \frac{768C_0\sigma_jx_{\max}}{(C_0-2)p_*\psi}\sqrt{\frac{3Ksd\log(d)}{qp_j}} + C_0\bp{\frac{3\sqrt{sd}}{2}+24d\sqrt{\frac{3\log(d)|\cB_0|}{qN}}}\max_{i\in[N]}\sqrt{\frac{K\sigma_i^2\log(N)}{2p_*\psi p_i|\cB_0|}}} \\
\le \frac{5}{T} + \sum_{i\in[N]}d\exp\bp{-\frac{p_*\psi p_i|\cB_0|}{32Kdx_{\max}^2}}
+ \sum_{i\in[N]}4\exp\bp{-\frac{p_* p_i|\cB_0|}{20K}} \\
\le \frac{5}{T} + \sum_{i\in[N]}3d\exp\bp{-\frac{p_*\psi p_i|\cB_0|}{32Kdx_{\max}^2}},
\end{multline}
where we use $\log(\frac{3}{\eta}) \le \frac{\log(N)}{2}$, the first inequality uses
$\P\bb{(\cM_0^i)^c} \le 2\exp\bp{-\frac{p_*p_i|\cB_0|}{20K}}$, and the last inequality uses $\psi \le \lambda_{\min}\bp{\Sigma_k^j} \le \lambda_{\max}\bp{\Sigma_k^j} \le \|X_t\|_2^2 \le d x_{\max}^2$. Next, we choose a sufficiently large $q$ such that 
\begin{gather*}
\frac{h}{8x_{\max}} \ge \frac{768C_0\sigma_jx_{\max}}{(C_0-2)p_*\psi}\sqrt{\frac{3Ksd\log(d)}{qp_j}}, \quad
\frac{h}{16x_{\max}} \ge C_0\frac{3\sqrt{sd}}{2}\max_{i\in[N]}\sqrt{\frac{K\sigma_i^2\log(N)}{2p_*\psi p_i|\cB_0|}}, \\
\frac{h}{16x_{\max}} \ge  24C_0d\sqrt{\frac{3\log(d)|\cB_0|}{qN}}\max_{i\in[N]}\sqrt{\frac{K\sigma_i^2\log(N)}{2p_*\psi p_i|\cB_0|}}, \quad \frac{3}{T} \ge \sum_{i\in[N]}3d\exp\bp{-\frac{p_*\psi p_i|\cB_0|}{32Kdx_{\max}^2}}.
\end{gather*}
With a choice of $q$ satisfying all the above constraints, the result then follows \eqref{eq:fse_esterrb_armsmp2}. It suffices to set
\begin{multline*}
q = \max\left\{\frac{2\cdot384^3C_0^2 x_{\max}^4 (\max_{i \in [N]}\sigma_i^2/p_i) Ksd\log(d)}{(C_0-2)^2h^2 p_*^2 \psi^2}, \frac{576C_0^2x_{\max}^2(\max_{i \in [N]}\sigma_i^2/p_i)Ksd\log(N)}{h^2p_*\psi}, \right.\\
\left. \frac{6\cdot192^2C_0^2x_{\max}^2(\max_{i \in [N]}\sigma_i^2/p_i)Kd^2\log(d)\log(N)}{h^2p_*\psi N}, \frac{96x_{\max}^2Kd\log(dN)}{p_*\psi(\min_{i \in [N]} p_i)} \right\}.
\end{multline*}

Remember we also require $\eta \le 1/2-1/C_0-\zeta$ according to Proposition~\ref{prop:tmean_reg_hpb_cov}, which is satisfied as long as 
\begin{align}\label{eq:fse_constraint}
\log(T) \le \bp{\frac{C_0-2}{2C_0}\bp{1-\frac{1}{2}\sqrt{\frac{s}{d}}}}^2 \frac{N}{27\log(d)}. \Halmos
\end{align}
\end{proof}

\textbf{All-Sample Estimator.}
Next, we prove Proposition~\ref{prop:tmean_bdt_ase}, which shows our all-sample estimators have small estimation errors with high probability. 
For simplicity, we use $\zeta$, $\eta$ and $\lambda_j$ to represent $\zeta_{1,m}$, $\eta_{1,m}$ and $\lambda_{1,j,\bar{m}}$ (described at the beginning of \S\ref{app:regret_sfixed_std}) respectively in the following.

\paragraph{Additional Notation.} Let $\cB_m^j$ be the index set of observations at instance $j$ in batch $m$, $\cB_{m, k}^j\subseteq\cB_m^j$ be the subset batch sampled at arm $k$, and $\bar{\cB}_{m, k}^j\subseteq\cB_{m, k}^j$ be the subset where $X_t \in U_k^j$ when $\cA$ holds; particularly, for any $m\ge1$,
\begin{gather*}
\cB_m^j = \bc{t \in \cB_m \mid Z_t=j}, 
\quad \cB_{m, k}^j = \bc{t \in \cB_m \mid Z_t=j,\, \pi_{\bar{m-1}}^{Z_t}(X_t)=k}, \\
\quad \bar{\cB}_{m, k}^j = \bc{t \in \cB_m \mid Z_t=j,\, X_t \in U_k^{Z_t},\, \cA},
\end{gather*}
where we use the notation $\bar{m}$ to represent the union $\bigcup_{l=0}^m$ and $\widetilde{m}$ to represent $\bigcup_{l=1}^m$, e.g., $\cB_{\widetilde{m}}=\bigcup_{l=1}^m\cB_l$, $\cB_{\bar{m}, k}^j=\bigcup_{l=0}^m \cB_{l,k}^j$, and, with a slight abuse of notation, we use $\pi_{\bar{m-1}}^{Z_t}$ to represent the sub-policy of instance $j$ estimated using the data from $\cB_{\bar{m-1}}$. 

First, we provide a proof for Lemma~\ref{lem:tmean_bdt_ase_ind} to show the conditional independence of our samples in each batch.
\begin{proof}{Proof of Lemma~\ref{lem:tmean_bdt_ase_ind}}
The collected samples of arm $k$ at instance $j$ in the batch $\cB_m$ are $\{(X_t, Y_t)\}_{t \in \cB_{m, k}^j}$, where the set of covariates is
\begin{align*}
\bc{X_t \,\middle|\, t \in \cB_{m, k}^j,\, Z_t = j,\, \pi_{\bar{m-1}}^{Z_t}(X_t) = k}.
\end{align*}
Note that our estimated policy $\pi_{\bar{m-1}}^{Z_t}$ depends on $Z_t$ and is constructed using samples from $\cB_{\bar{m-1}}$. Since $\{(X_t, Z_t, Y_t)\}_{t \in \cB_{\bar{m}}}$ are independent, $\{(X_t, Z_t, \pi_{\bar{m-1}}^{Z_t}(X_t))\}_{t \in \cB_m}$ are independent conditional on $\cF_{\bar{m-1}}$. 
Thus, for any $t' \ne t$ and $t' \in \cB_m$, $X_t$ is independent of $Z_{t'}$, $X_{t'}$ and $\pi_{\bar{m-1}}^{Z_{t'}}(X_{t'})$ conditional on $\{Z_t=j, \pi_{\bar{m-1}}^{Z_t}(X_t)=k, \cF_{\bar{m-1}}\}$. This implies $X_t$'s of arm $k$ in batch $m$ are conditionally independent across instances. Moreover, since the noises $\epsilon_t$'s are independent of $X_t$'s and only depends on $Z_t$'s by design, the result then follows. \Halmos
\end{proof}
\begin{remark}
Note that the samples across instances are also independent given $\{\cA, \cF_{\bar{m-1}}\}$ since $\cA \in \cF_{\bar{m-1}}$.
\end{remark}

\begin{lemma}\label{lem:tmean_bdt_ase_iid}
(i) $\bar{\cB}_{m, k}^j \subseteq \cB_{m, k}^j$, (ii) $\{X_t\}_{t \in \cB_{m, k}^{j}}$ are i.i.d. from $\cP_{X \mid \pi_{\bar{m-1}}^j(X)=k}^j$ conditioned on $\cF_{\bar{m-1}}$, and its subset $\{X_t\}_{t \in \bar{\cB}_{m, k}^{j}}$ are i.i.d. from $\cP_{X \mid X \in U_k^j}^j$, and (iii) $\{X_t\}_{t \in \bar{\cB}_{\widetilde{m}, k}^{j}}$ are i.i.d. from $\cP_{X \mid X \in U_k^j}^j$.
\end{lemma}
\begin{proof}{Proof of Lemma~\ref{lem:tmean_bdt_ase_iid}}
The first claim follows Lemma EC.11 in \cite{bastani2020online}. If $Z_t = j$, $X_t \in U_k^j$ and the event $\cA$ holds, then $\pi_{\bar{m-1}}^{Z_t}(X_t)=k$ and hence $t \in \cB_{m, k}^j$. 

Similar to the proof of Lemma~\ref{lem:tmean_bdt_ase_ind}, we can show that $\{X_t\}_{t \in \cB_{m, k}^{j}}$ are i.i.d. from distribution $\cP_{X \mid \pi_{\bar{m-1}}^j(X)=k}^j$ conditioned on $\cF_{\bar{m-1}}$. Additionally,
note that the event $\cA$ only depends on forced samples from $\cB_0$ and is hence independent of $\{(X_t, Z_t)\}_{t\in\cB_m}$ for any $m \ge 1$. Therefore, $X_t$ for any $t \in \bar{\cB}_{m, k}^j$ follows distribution $\cP_{X \mid X \in U_k^j}^j$. Since $\{(X_t, Z_t)\}_{t \in \cB_m}$ are independent, $X_t$ is independent of $Z_{t'}$ and $X_{t'}$ given $\{Z_t=j, X_t \in U_k^{Z_t}, \cA\}$ for any $t'\ne t$ and $t' \in \cB_m$. Thus, $\{X_t\}_{t \in \bar{\cB}_{m, k}^j}$ are also independent. 

Correspondingly, we can show $\{X_t\}_{t \in \bar{\cB}_{\widetilde{m}, k}^{j}}$ are also i.i.d. from $\cP_{X \mid X \in U_k^j}^j$, where
$\bar{\cB}_{\widetilde{m}, k}^j = \bigcup_{l=1}^m \bar{\cB}_{l, k}^j = \bc{t \in \bigcup_{l=1}^m\cB_l \,\middle|\, Z_t=j,\, X_t \in U_k^{Z_t},\, \cA}$. \Halmos
\end{proof}
\begin{remark}
Note that (ii) and (iii) both hold further conditioned on $\cA$. The first statement in (ii) still holds as $\cA\in\cF_{\bar{m-1}}$ while the second statement in (ii) and the statement in (iii) hold as $\{(X_t, Z_t, Y_t)\}_{\cB_0}$ are independent of $\{(X_t, Z_t)\}_{\cB_{\widetilde{m}}}$.
\end{remark}

\begin{lemma}\label{lem:tmean_bdt_ase_armsmp}
Define the events
\begin{align}\label{eq:eventDMm}
\bar{\cD}_{m,k}^j = \bc{|\bar{\cB}_{m, k}^j| \ge \frac{p_*}{2}|\cB_m^j|},
\quad \cM_m^j = \bc{|\cB_m^j| \ge \frac{p_j}{2}|\cB_m|}, 
\quad \cM_{\bar{m}}^j = \bc{\frac{3p_j}{2}|\cB_{\bar{m}}| \ge |\cB_{\bar{m}}^j| \ge \frac{p_j}{2}|\cB_{\bar{m}}|}.
\end{align}
Then, we have
\begin{gather*}
\P\bb{\bar{\cD}_{m,k}^j \,\middle|\, \cA} \ge 1 - 2\exp\bp{-\frac{p_*|\cB_m^j|}{10}},
\quad \P\bb{\cM_m^j} \ge 1 - 2\exp\bp{-\frac{p_j|\cB_m|}{10}}, 
\quad \P\bb{\cM_{\bar{m}}^j} \ge 1 - 2\exp\bp{-\frac{p_j|\cB_{\bar{m}}|}{10}},
\end{gather*}
given $|\cB_m^j|$, $|\cB_m|$ and $|\cB_{\bar{m}}|$ respectively.
\end{lemma}
\begin{proof}{Proof of Lemma~\ref{lem:tmean_bdt_ase_armsmp}}
By definition of $\bar{\cB}_{m, k}^{j}$, we have
\begin{align*}
|\bar{\cB}_{m, k}^{j}| 
= \sum_{t \in \cB_m^j} \mathbbm{1}\bp{t \in \bar{\cB}_{m, k}^{j}} = \sum_{t \in \cB_m^j} \mathbbm{1}\bp{Z_t = j, X_t \in U_k^{Z_t}} \mathbbm{1}\bp{\cA}.
\end{align*}
Take $\mu = \E\bb{|\bar{\cB}_{m, k}^{j}| \,\middle|\, \cA} = \E\bb{\sum_{t \in \cB_m^j} \mathbbm{1}\bp{Z_t = j, X_t \in U_k^{Z_t}}} = \sum_{t \in \cB_m^j} \P\bb{X_t \in U_k^j \mid Z_t = j}$,
where the second equality is from the fact that $\{(X_t, Z_t)\}_{t \in \cB_m}$ are independent of $\{(X_t, Z_t, Y_t)\}_{t \in \cB_0}$. Then, the first bound about $\bar{\cD}_{m,k}^j$ follows using Lemma~\ref{lem:ind_cherbound} and the fact that $\mu \ge p_*|\cB_m^j|$ given by Assumption~\ref{ass:armopt}. The rest of the proof is similar to that of Lemma~\ref{lem:tmean_bdt_fse_armsmp}. \Halmos
\end{proof}
\begin{remark}
Note that the high probability bound of $|\cB_m^j|$ also holds conditioned on $\cA$ as $\{Z_t\}_{t\in\cB_m}$ are independent of $\{(X_t, Z_t, Y_t)\}_{t \in \cB_0}$. 
\end{remark}

\begin{lemma}\label{lem:tmean_bdt_ase_mineig}
Given $|\cB_m^j|$, we have
\begin{align*}
\P\bb{\lambda_{\min}(\widehat\Sigma(\cB_{m, k}^j)) \ge \frac{p_*\psi}{4} \,\middle|\, \cA} \ge 1 - \bp{d\exp\bp{-\frac{p_*\psi|\cB_m^j|}{16dx_{\max}^2}} + 2\exp\bp{-\frac{p_*|\cB_m^j|}{10}}}.
\end{align*}
Further on the event $\cM_m^j$ and $\cM_{\bar{m}}^j$ in \eqref{eq:eventDMm}, it holds that
\begin{align*}
\P\bb{\lambda_{\min}(\widehat\Sigma(\cB_{\bar{m}, k}^j)) \ge \frac{p_*\psi}{24} \,\middle|\, \cA}
\ge 1 - \bp{d\exp\bp{-\frac{p_*\psi|\cB_m^j|}{16dx_{\max}^2}} + 2\exp\bp{-\frac{p_*|\cB_m^j|}{10}}}.
\end{align*}
\end{lemma}
\begin{proof}{Proof of Lemma~\ref{lem:tmean_bdt_ase_mineig}}
Define the following event analogous to (\ref{eq:eventE0})
\begin{align*}
\bar{\cE}_{m,k}^j =\bc{\lambda_{\min}(\widehat\Sigma(\bar{\cB}_{m, k}^j)) \ge \frac{\psi}{2}}.
\end{align*}
By Lemma~\ref{lem:bdt_fse_covsubsmp} and the fact that $|\cB_m^j| \ge |\cB_{m, k}^j|$, it holds that $\lambda_{\min}(\widehat\Sigma(\cB_{m, k}^j)) \ge p_*\psi/4$ on the events $\bar{\cD}_{m,k}^j$ in (\ref{eq:eventDMm}) and $\bar{\cE}_{m,k}^j$ above. 
Thus, we have
\begin{align*}
\P\bb{\lambda_{\min}(\widehat\Sigma(\cB_{m, k}^j)) \le \frac{p_*\psi}{4} \,\middle|\, \cA}
& \le \P\bb{(\bar{\cD}_{m,k}^j)^c\cup(\bar{\cE}_{m,k}^j)^c \,\middle|\, \cA} 
\le \P\bb{(\bar{\cE}_{m,k}^j)^c \,\middle|\, \bar{\cD}_{m,k}^j, \cA} + \P\bb{(\bar{\cD}_{m,k}^j)^c \,\middle|\, \cA}.
\end{align*}
Our first result then follows by noting that
\begin{align*}
\P\bb{(\bar{\cE}_{m,k}^j)^c \,\middle|\, \bar{\cD}_{m,k}^j, \cA} \le d\exp\bp{-\frac{p_*\psi|\cB_m^j|}{16dx_{\max}^2}},
\quad \P\bb{(\bar{\cD}_{m,k}^j)^c \,\middle|\, \cA} \le 2\exp\bp{-\frac{p_*|\cB_m^j|}{10}},
\end{align*}
where the first bound can be derived similarly following the proof of Lemma~\ref{lem:tmean_bdt_fse_cov} given Lemma~\ref{lem:tmean_bdt_ase_iid},
and the second bound uses Lemma~\ref{lem:tmean_bdt_ase_armsmp}. 

Since $|\cB_m^j| \ge |\cB_{m, k}^j|$, the above also holds for $\lambda_{\min}(\widehat\Sigma(\cB_m^j))$. 
Then, on the event $\cM_m^j$ and $\cM_{\bar{m}}^j$, we have
\begin{align*}
\lambda_{\min}(\widehat\Sigma(\cB_{\bar{m}, k}^j)) \ge \frac{|\cB_m^j|}{|\cB_{\bar{m}, k}^j|} \lambda_{\min}(\widehat\Sigma(\cB_m^j)) 
\ge \frac{|\cB_m^j|}{|\cB_{\bar{m}}^j|} \frac{p_*\psi}{4} \ge \frac{p_*\psi}{24},
\end{align*}
where the last inequality uses Lemma~\ref{lem:bdt_fse_covsubsmp} and the fact that $|\cB_{\bar{m}}|=2|\cB_m|$. Our second result then follows. \Halmos 
\end{proof}

Now we provide the proof of Proposition~\ref{prop:tmean_bdt_ase}. 
\begin{proof}{Proof of Proposition~\ref{prop:tmean_bdt_ase}}
At a high level, our proof is adapted from Proposition~\ref{prop:tmean_reg_hpb_cov}, and similar to that of Proposition~\ref{prop:tmean_bdt_fse}, but now conditioned on $\cA$. However, it differs in the following two aspects. First, given arm $k$, we consider learning across a subset of instances of which arm $k$ is an optimal arm, i.e., $\cW_k \subseteq [N]$ defined in Assumption~\ref{ass:optdens}. This is because the suboptimal arms won't observe any users on the event of $\cA$. 
Moreover, since we use batch data from $\{\cB_{m,k}^i\}_{i\in[N]}$ to compute our trimmed mean estimator but now all data from $\cB_{\bar{m},k}^j$ to debias for instance $j$ using LASSO, we now bound the following events 
\begin{align*}
\cH^j(\cB_{\bar{m},k}^j) = \bc{\frac{2}{|\cB_{\bar{m}, k}^j|}\|\vX^{j\top}\vE^j\|_\infty \le \frac{\lambda_j}{2}},
\end{align*}
and $\{\lambda_{\min}(\widehat\Sigma(\cB_{\bar{m}, k}^j)) \ge p_*\psi/24\}$ instead (in contrast to $\cH^j(\cB_{m,k}^j)$ and $\{\lambda_{\min}(\widehat\Sigma(\cB_{m, k}^j))\ge p_*\psi/24\}$) respectively. 
In our definition of $\cH^j(\cB_{\bar{m},k}^j)$ above, for simplicity, we use $(\vX^j, \vY^j)$ to represent data collected at arm $k$ and instance $j$ up to batch $m$, i.e., $\{(X_t, Y_t)\}_{t\in\cB_{\bar{m}, k}^j}$; note that $\cB_{\bar{m},k}^j$ contains an \emph{adapted} sequence of observations so we will use a concentration inequality for martingale sequence to bound $\cH^j(\cB_{\bar{m},k}^j)$.

According to Lemma~\ref{lem:tmean_bdt_ase_ind} and \ref{lem:tmean_bdt_ase_iid}, our all-sample OLS estimators are subgaussian and independent across instances conditioned on $\{\cA, \cF_{\bar{m-1}}\}$; thus, the conditions of Proposition~\ref{prop:tmean_reg_hpb_cov} are satisfied. 
Applying Proposition~\ref{prop:tmean_reg_hpb_cov}, we can write
\begin{multline*}
\P\bb{\|\widehat\beta^j_{k,\bar{m}} - \beta_k^j\|_1 \ge \frac{144\lambda_j s}{p_*\zeta\psi} + 2C_0d(3\zeta+4\eta)\max_{i \in \cW_k}\sqrt{\frac{\sigma_i^2}{p_*\psi|\cB_{m, k}^i|}\log(\frac{3}{\eta})} \,\middle|\, \cA, \cF_{\bar{m-1}}} \\
\le 3d\exp\bp{-\frac{\rho N\eta^2}{9}} + \P\bb{(\cH^j(\cB_{\bar{m},k}^j))^c \,\middle|\, \cA, \cF_{\bar{m-1}}}
+ \P\bb{\lambda_{\min}(\widehat\Sigma(\cB_{\bar{m}, k}^j)) \le \frac{p_*\psi}{24} \,\middle|\, \cA, \cF_{\bar{m-1}}} \\
+ \sum_{i \in \cW_k}\P\bb{\lambda_{\min}(\widehat\Sigma(\cB_{m, k}^i)) \le \frac{p_*\psi}{4} \,\middle|\, \cA, \cF_{\bar{m-1}}},
\end{multline*}
given $\{|\cB_{m, k}^i|\}_{i\in\cW_k}$ and $|\cB_{\bar{m}, k}^j|$. 
Using a concentration inequality for martingale sequence as in Lemma EC.2 in \cite{bastani2020online}, we can prove that 
\begin{align}\label{def:eventHb}
\P\bb{\cH^j(\cB_{\bar{m},k}^j)} \ge 1 - 2d \exp\bp{-\frac{\lambda_j^2 |\cB_{\bar{m}, k}^j|}{32\sigma_j^2 x_{\max}^2}},
\end{align}
with a similar proof strategy to Lemma~\ref{def:eventH}.
Taking expectation over $\cF_{\bar{m-1}}$ on both sides, we obtain
\begin{multline*}
\P\bb{\|\widehat\beta^j_{k,\bar{m}} - \beta_k^j\|_1 \ge \frac{144\lambda_j s}{p_*\zeta\psi} + 2C_0d(3\zeta+4\eta)\max_{i \in \cW_k}\sqrt{\frac{\sigma_i^2}{p_*\psi|\cB_{m, k}^i|}\log(\frac{3}{\eta})} \,\middle|\, \cA} \\
\le 3d\exp\bp{-\frac{\rho N\eta^2}{9}} + 2d \exp\bp{-\frac{\lambda_j^2 |\cB_{m, k}^j|}{32\sigma_j^2 x_{\max}^2}}
+ \P\bb{\lambda_{\min}(\widehat\Sigma(\cB_{\bar{m}, k}^j)) \le \frac{p_*\psi}{24} \,\middle|\, \cA} \\
+ \sum_{i \in \cW_k}\P\bb{\lambda_{\min}(\widehat\Sigma(\cB_{m, k}^i)) \le \frac{p_*\psi}{4} \,\middle|\, \cA},
\end{multline*}
where we apply \eqref{def:eventHb} and use $|\cB_{\bar{m}, k}^j| \ge |\cB_{m, k}^j|$.
Note that on the event $\bar{\cD}_{m,k}^j$ in \eqref{eq:eventDMm}, we have $|\cB_{m, k}^j| \ge |\bar{\cB}_{m, k}^j| \ge p_*|\cB_{m}^j|/2$ using Lemma~\ref{lem:tmean_bdt_ase_iid}.
Thus, with a union bound over $\bigcap_{i \in \cW_k}\bar{\cD}_{m,k}^i$, we have
\begin{multline}\label{eq:ase_esterrb_armsmp1}
\P\bb{\|\widehat\beta^j_{k,\bar{m}} - \beta_k^j\|_1 \ge \frac{144\lambda_j s}{p_*\zeta\psi} + 2C_0d(3\zeta+4\eta)\max_{i \in \cW_k}\sqrt{\frac{2\sigma_i^2}{p_*^2\psi|\cB_m^i|}\log(\frac{3}{\eta})} \,\middle|\, \cA}  \\
\le 3d\exp\bp{-\frac{\rho N\eta^2}{9}} + 2d \exp\bp{-\frac{\lambda_j^2 p_*|\cB_m^j|}{64\sigma_j^2 x_{\max}^2}}
+ \P\bb{\lambda_{\min}(\widehat\Sigma(\cB_{\bar{m}, k}^j)) \le \frac{p_*\psi}{24} \,\middle|\, \cA} \\
+ \sum_{i \in \cW_k}\P\bb{\lambda_{\min}(\widehat\Sigma(\cB_{m, k}^i)) \le \frac{p_*\psi}{4} \,\middle|\, \cA} + \sum_{i \in \cW_k}2\exp\bp{-\frac{p_*|\cB_{m}^i|}{10}}
\end{multline}
given $\{|\cB_{m, k}^i|\}_{i\in\cW_k}$ and $|\cB_{\bar{m}, k}^j|$, where we use Lemma~\ref{lem:tmean_bdt_ase_armsmp}.
Similarly, we take $\zeta = \frac{C_0-2}{4C_0}\sqrt{\frac{s}{d}}$ and set
\begin{align*}
\eta = \sqrt{\frac{9\log(d\min_{i \in \cW_k}|\cB_m^i|)}{\rho N}}, 
\quad \lambda_j = \sqrt{\frac{384\sigma_j^2x_{\max}^2\log(d|\cB_{\bar{m}}^j|)}{p_*|\cB_{\bar{m}}^j|}}.
\end{align*}
Since $|\cB_m^i|=0$ for any $i \in [N] \setminus \cW_k$ conditioned on the event $\cA$, the value of $\eta$ is equivalent to
\begin{align*}
\eta = \sqrt{\frac{9\log(d\min_{i \in [N], |\cB_m^i| > 0}|\cB_m^i|)}{\rho N}}. \end{align*}
Further, as $\log(d\min_{i \in [N], |\cB_m^i| > 0}|\cB_m^i|) \ge \log(d) \ge 1$, it holds that $\sqrt{\log(\frac{3}{\eta})} \le \sqrt{\frac{\log(\rho N)}{2}}$.
Then, using a union bound over $\bigcap_{i \in \cW_k} (\cM_m^i \cap \cM_{\bar{m}}^j)$ and applying Lemma~\ref{lem:tmean_bdt_ase_mineig} and \ref{lem:tmean_bdt_ase_armsmp}, we obtain from inequality~(\ref{eq:ase_esterrb_armsmp1}) that
\begin{multline*}
\P\bb{\|\widehat\beta^j_{k,\bar{m}} - \beta_k^j\|_1 \ge C_1 \sqrt{\frac{sd\log(d p_j|\cB_{\bar{m}}|/2)}{p_j|\cB_{\bar{m}}|}} + C_2 \sqrt{\frac{sd\log(\rho N)}{p_j|\cB_m|}} + C_3 d\sqrt{\frac{\log(dp_j|\cB_m|/2)\log(\rho N)}{N p_j |\cB_m|}} \,\middle|\, \cA } \\
\le \frac{6}{\min_{i\in\cW_k}p_i|\cB_m|} + \frac{8}{p_j|\cB_{\bar{m}}|}
+ \sum_{i \in \cW_k}2d\exp\bp{-\frac{p_*p_i\psi|\cB_m|}{32dx_{\max}^2}} + \sum_{i \in \cW_k}10\exp\bp{-\frac{p_*p_i|\cB_m|}{20}} \\
\le \frac{6}{\min_{i\in\cW_k}p_i|\cB_m|} + \frac{4}{p_j|\cB_m|}
+ \sum_{i \in \cW_k}7d\exp\bp{-\frac{p_*p_i\psi|\cB_m|}{32dx_{\max}^2}},
\end{multline*}
where $C_1$, $C_2$, and $C_3$ are constants listed at the beginning of \S\ref{app:regret_sfixed_std}, we use $|\cB_m|/|\cB_{\bar{m}}|=1/2$ for $m\ge 1$, the first inequality uses $\P\bb{(\cM_m^i)^c\cup(\cM_{\bar{m}}^j)^c} \le 4\exp\bp{-\frac{p_*p_i|\cB_m|}{20}}$, and the last inequality uses $\psi \le d x_{\max}^2$. 

Finally, to satisfy $\eta \le 1/2 - 1/C_0 - \zeta$, we require
\begin{align*}
\log(d|\cB_m|) \le \bp{\frac{C_0-2}{2C_0}\bp{1-\frac{1}{2}\sqrt{\frac{s}{d}}}}^2 \frac{\rho N}{9}.
\end{align*}
Since $|\cB_m| \le T$, we conclude that it suffices to have $N = \Omega(\log(d)\log(T))$, combined with (\ref{eq:fse_constraint}). \Halmos
\end{proof}

\textbf{Regret Analysis.}
Finally, we prove Proposition~\ref{prop:tmean_bdt_rgt_all}, which provides an upper bound on the cumulative regret across all bandit instances, by bounding the regret of the three cases listed in Appendix~\ref{sec:rmbandit-proofstrategy}.

We begin by providing the following useful lemma, which shows a sufficiently large amount of data is used to train the all-sample estimators.
\begin{lemma}\label{lem:bdt_batchsize}
For any $t \in \cB_m$ with $m > 1$, we have
\begin{align*}
|\cB_{m-1}| \ge \frac{t}{4}, \quad |\cB_{\bar{m-1}}| \ge \frac{t}{2}.
\end{align*}
\end{lemma}
\begin{proof}{Proof of Lemma~\ref{lem:bdt_batchsize}}
By our design, $|\cB_m| = 2^{m-1}|\cB_0|$ for any $m \ge 1$, which implies for any $m > 1$
\begin{align*}
\frac{|\cB_{m-1}|}{t} \ge \frac{|\cB_{m-1}|}{\sum_{i=0}^m |\cB_i|} = \frac{1}{4}, \quad
\frac{|\cB_{\bar{m-1}}|}{t} \ge \frac{\sum_{i=0}^{m-1} |\cB_i|}{\sum_{i=0}^m |\cB_i|} = \frac{1}{2}. \Halmos
\end{align*}
\end{proof}

Next, we provide a per-period regret bound at time $t$ for instance $j$ in case~(\ref{case:regret3}) in Appendix~\ref{sec:rmbandit-proofstrategy}.
\begin{lemma}\label{lem:tmean_bdt_rgtper}
When $\cA$ holds, $N=\Omega\bp{\log(d)\log(T)}$ and $Z_t=j$, the expected regret at time $t \in \cB_m$ for $m > 1$ is upper bounded by
\begin{multline*}
r_t^j \le 24x_{\max}^2LK\bp{C_1^2 \frac{sd\log(dp_jt)}{p_jt} + 2C_2^2 \frac{sd\log(\rho N)}{p_jt} + 2C_3^2 \frac{d^2\log(\rho N)\log(dp_jt)}{N p_j t}} \\
+ 4bx_{\max}K \bp{\max_{i\in[N]}\frac{24}{p_it} + \frac{16}{p_jt} + 7dN\exp\bp{-\frac{p_*\psi (\min_{i\in[N]}p_i) t}{128dx_{\max}^2}}}.
\end{multline*}
\end{lemma}
\begin{proof}{Proof of Lemma~\ref{lem:tmean_bdt_rgtper}}
Without loss of generality, assume arm $1$ is optimal for $X_t$, i.e., $\argmax_{k \in [K]} X_t^\top \beta_k^j = 1$. Note that here the optimal arm is a function of $X_t$ and hence a random variable; for simplicity, we fix arm $1$ as the optimal arm in the following. 
Consider the conditional expected regret at time $t\in\cB_m$
\begin{align*}
r_t^j(X_t) = \E\bb{\sum_{k \in \cK} X_t^{\top} (\beta_1^j - \beta_k^j)\mathbbm{1}\bp{\pi_{\bar{m-1}}^j(X_t)=k} \,\middle|\, X_t, Z_t=j, \cA},
\end{align*}
where the set of arms $\cK$ is defined in Algorithm~\ref{alg:tmean_bdt}.
Since $\pi_{\bar{m-1}}^j(X_t)=k$ implies $X_t^\top \widehat{\beta}^j_{k,\bar{m-1}} \ge X_t^\top \widehat{\beta}^j_{1,\bar{m-1}}$, we have
\begin{align}\label{eq:tmean_bdt_upch}
r_t^j(X_t) \le \E\bb{\sum_{k \in \cK} X_t^{\top} (\beta_1^j - \beta_k^j)\mathbbm{1}\bp{X_t^{\top} \widehat{\beta}^j_{k,\bar{m-1}} \ge X_t^{\top} \widehat{\beta}^j_{1,\bar{m-1}}} \,\middle|\, X_t, Z_t=j, \cA}.
\end{align}
Define the event
\begin{align*}
\cL_k^j = \bc{2x_{\max}\delta \le X_t^{\top} (\beta_1^j - \beta_k^j)}.
\end{align*}
Then, we can decompose the upper bound of the regret in (\ref{eq:tmean_bdt_upch}) into two parts given $\cL_k^j$, that is, 
\begin{align}\label{eq:tmean_bdt_rgtdecomp}
r_t^j(X_t) \le \sum_{r=1,2} r_{t,r}^j(X_t),
\end{align}
where
\begin{align}
\label{eq:tmean_bdt_rgtdcmp1} r_{t,1}^j(X_t) = & \E\bb{\sum_{k \in \cK} X_t^{\top} (\beta_1^j - \beta_k^j)\mathbbm{1}\bp{\{X_t^\top \widehat{\beta}^j_{k,\bar{m-1}} \ge X_t^\top \widehat{\beta}^j_{1,\bar{m-1}}\} \cap \cL_k^j} \,\middle|\, X_t, Z_t=j, \cA}, \\
\label{eq:tmean_bdt_rgtdcmp2} r_{t,2}^j(X_t) = & \E\bb{\sum_{k \in \cK} X_t^{\top} (\beta_1^j - \beta_k^j)\mathbbm{1}\bp{\{X_t^\top \widehat{\beta}^j_{k,\bar{m-1}} \ge X_t^\top \widehat{\beta}^j_{1,\bar{m-1}}\} \cap (\cL_k^j)^c} \,\middle|\, X_t, Z_t=j, \cA}.
\end{align}
On one hand, the event $\bc{\{X_t^\top \widehat{\beta}^j_{k,\bar{m-1}} \ge X_t^\top \widehat{\beta}^j_{1,\bar{m-1}}\} \cap \cL_{k}^j}$ regarding $r_{t,1}^j(X_t)$ implies 
\begin{align*}
X_t^{\top} (\widehat{\beta}_{k,\bar{m-1}}^j - \beta_k^j) - X_t^{\top}(\widehat{\beta}_{1,\bar{m-1}}^j - \beta_1^j) \ge X_t^{\top}(\beta_1^j - \beta_k^j) \ge 2 x_{\max} \delta.
\end{align*}
Thus, at least one of $|X_t^{\top} (\widehat{\beta}_{\iota,\bar{m-1}}^j - \beta_\iota^j)|$ with $\iota \in \{1, k\}$ must be greater than $x_{\max} \delta$, which means
\begin{multline*}
\E\bb{\mathbbm{1}\bp{\{X_t^{\top} \widehat{\beta}^j_{k,\bar{m-1}} \ge X_t^{\top} \widehat{\beta}^j_{1,\bar{m-1}}\} \cap \cL_{k}^j} \mid X_t, Z_t=j, \cA} \\
\le \sum_{\iota \in \{1,k\}} \P\bb{|X_t^{\top} (\widehat{\beta}_{\iota, \bar{m-1}}^j - \beta_\iota^j)| \ge x_{\max} \delta \,\middle|\, X_t, Z_t=j, \cA} 
\le \sum_{\iota \in \{1,k\}} \P\bb{\|\widehat{\beta}_{\iota,\bar{m-1}}^j - \beta_\iota^j\|_1 \ge \delta \,\middle|\, \cA}.
\end{multline*}
According to Lemma EC.18 in \cite{bastani2020online}, the set of arms $\cK$ filtered by the forced-sample estimator contains the optimal arm $k = \argmax_{i \in [K]} X_t^\top \beta_i^j$ given $Z_t=j$ and no suboptimal arms in $\cK_{\text{sub}}^j$ when $\cA$ holds; thus, both arm $1$ and $k$ are not suboptimal. We further upper bound the above probability using Proposition~\ref{prop:tmean_bdt_ase} together with Lemma~\ref{lem:bdt_batchsize}:
\begin{align*}
\P\bb{\|\widehat{\beta}_{\iota,\bar{m-1}}^j - \beta_\iota^j\|_1 \ge \delta}
\le \max_{i\in[N]}\frac{24}{p_it} + \frac{16}{p_jt}
+ 7dN\exp\bp{-\frac{p_*\psi (\min_{i\in[N]}p_i) t}{128dx_{\max}^2}},
\end{align*}
for $\iota \in \{1, k\}$, where 
\begin{align*}
\delta = C_1 \sqrt{\frac{2sd\log(d p_jt)}{p_jt}} + C_2 \sqrt{\frac{4sd\log(\rho N)}{p_jt}} + C_3 d\sqrt{\frac{4\log(dp_jt)\log(\rho N)}{N p_j t}}.
\end{align*}
Then, we can obtain from (\ref{eq:tmean_bdt_rgtdcmp1}) that
\begin{multline}\label{eq:expperrgt_r1}
\E\bb{r_{t, 1}^j(X_t) \mid Z_t = j, \cA} \le 2bx_{\max}K \sum_{\iota \in \{1,k\}}\P\bb{\|\widehat{\beta}_{\iota,\bar{m-1}}^j - \beta_\iota^j\|_1 \ge \delta} \\
\le 4bx_{\max}K \bp{\max_{i\in[N]}\frac{24}{p_it} + \frac{16}{p_jt} + 7dN\exp\bp{-\frac{p_*\psi (\min_{i\in[N]}p_i) t}{128dx_{\max}^2}}}.
\end{multline}
On the other hand, by Assumption~\ref{ass:marcond}, we have for the term $r_{t,2}^j(X_t)$ that
\begin{align}\label{eq:expperrgt_r2}
\E\bb{r_{t,2}^j(X_t) \mid Z_t = j, \cA} \le 2x_{\max}\delta K \P\bb{(\cL_{k}^j)^c} \le 4x_{\max}^2LK\delta^2.
\end{align}
Combining (\ref{eq:expperrgt_r1}), (\ref{eq:expperrgt_r2}) and (\ref{eq:tmean_bdt_rgtdecomp}), our result follows by using the inequality $3(a^2+b^2+c^2) \ge (a+b+c)^2$. \Halmos
\end{proof}

Now we prove Proposition~\ref{prop:tmean_bdt_rgt_all} for the total cumulative regret across the three cases in Appendix~\ref{sec:rmbandit-proofstrategy}.
\begin{proof}{Proof of Proposition~\ref{prop:tmean_bdt_rgt_all}}
First, we bound the worst-case regret for case~\eqref{case:regret1}. By our design, we have $2q\log(T)$ time steps in total from $\cB_0\cup\cB_1$. The worst-case regret per time step is $2bx_{\max}$. Thus, the cumulative regret of case~\eqref{case:regret1} is at most $2bx_{\max}(2q\log(T)+N)$.

Next, we provide the worst-case regret bound for case~\eqref{case:regret2}. The probability that $\cA$ does not take place is at most $8KN/T$, with a union bound over all arms and bandit instances using Proposition~\ref{prop:tmean_bdt_fse}. Plus, the worst-case cumulative regret is at most $2bx_{\max}T$ throughout $\{\cB_m\}_{m > 1}$ when $\cA$ fails. Thus, the cumulative regret is at most $16bx_{\max}KN$ in such case.

Finally, we show the cumulative regret bound for case~\eqref{case:regret3} is upper bounded by
\begin{multline*}
\sum_{j \in [N]} \left[ 24x_{\max}^2LK\bp{C_1^2sd\log(dp_jT) + 2C_2^2 sd\log(\rho N) + 2C_3^2\frac{d^2\log(\rho N)\log(dp_jT)}{N}}\log(p_jT) \right. \\
\left. + 4bx_{\max}K \bp{(16+\max_{i \in [N]}\frac{24p_j}{p_i})\log(p_jT) + \max_{i \in [N]}\frac{896x_{\max}^2d^2p_j}{p_*\psi p_i N}} \right].
\end{multline*}
In detail, the cumulative expected regret from all instances over $\{\cB_m\}_{m > 1}$ is 
\begin{align}\label{eq:tmean_bdt_cumm3}
\E\bb{\sum_{t\in \bigcup_{m>1} \cB_m} r_t^{Z_t}(X_t) \,\middle|\, \cA} = \sum_{t\in \bigcup_{m>1} \cB_m} \E\bb{ \E\bb{r_t^{Z_t}(X_t) \,\middle|\, Z_t, \cA} \,\middle|\, \cA}
= \sum_{t =2q\log(T)+1}^{T} \sum_{j \in [N]} p_j r_t^j.
\end{align}
Note that we have
\begin{align*}
\int_{2q\log(T)}^T \frac{1}{p_jt} dt \le \frac{\log(p_jT)}{p_j}.
\end{align*}
Moreover, given our choice of $q$ (at the beginning of \S\ref{app:regret_sfixed_std}), we have
\begin{align*}
\int_{2q\log(T)}^T dN\exp\bp{-\frac{p_*\psi (\min_{i\in[N]}p_i) t}{128dx_{\max}^2}} dt \le & \frac{128d^2Nx_{\max}^2}{p_*\psi(\min_{i \in [N]}p_i)}\exp\bp{-\frac{p_*\psi(\min_{i \in [N]}p_i)q\log(T)}{64dx_{\max}^2}} \\
\le & \frac{128d^2Nx_{\max}^2}{p_*\psi(\min_{i \in [N]}p_i)T^{3K\log(dN)/2}}\\
\le & \frac{128x_{\max}^2d^2}{p_*\psi(\min_{i \in [N]}p_i)N},
\end{align*}
where the last inequality holds since $T^{3K\log(dN)/2} \ge N^2$ when $T \ge N$ and $K,d,N>1$. The regret bound follows by combining the above with Lemma~\ref{lem:tmean_bdt_rgtper}.

Summing up the cumulative expected regrets of the three cases, we have
\begin{multline*}
R_T \le 2bx_{\max}(2q\log(T)+N) + 16bx_{\max}KN \\
+ \sum_{j \in [N]} \left[ 24x_{\max}^2LK\bp{C_1^2sd\log(dp_jT) + 2C_2^2 sd\log(\rho N) + 2C_3^2\frac{d^2\log(\rho N)\log(dp_jT)}{N}}\log(p_jT) \right. \\
\left. + 4bx_{\max}K \bp{(16+\max_{i \in [N]}\frac{24p_j}{p_i})\log(p_jT) + \max_{i \in [N]}\frac{896x_{\max}^2d^2p_j}{p_*\psi p_i N}} \right].
\end{multline*}
In the standard case, $p_i = \Theta(1/N)$ for any $i\in[N]$ and thus $q = \cO(Kd(sN+d)\log(d)\log(N))$. The claim then follows. \Halmos
\end{proof}

Finally, we prove Theorem~\ref{thm:tmean_bdt_rgt_sgl} to show the regret upper bound of any single instance in the standard regime.
\begin{proof}{Proof of Theorem~\ref{thm:tmean_bdt_rgt_sgl}}
The proof is similar to that of Proposition~\ref{prop:tmean_bdt_rgt_all}, considering an expected time horizon of $T$ for instance $j$, i.e., a total time horizon of $T/p_j = \Theta(NT)$.

The cumulative expected regret of any instance $j$ in case (\ref{case:regret3}) is 
\begin{multline*}
\E\bb{\sum_{t\in \bigcup_{m>1} \cB_m} r_t^{Z_t}(X_t)\mathbbm{1}\bp{Z_t = j} \,\middle|\, \cA} 
= \sum_{t\in \bigcup_{m>1} \cB_m} p_j r_t^j \\
\le 24x_{\max}^2LK\bp{C_1^2sd\log(dT) + 2C_2^2 sd\log(\rho N) + 2C_3^2\frac{d^2\log(\rho N)\log(dT)}{N}}\log(T) \\
+ 4bx_{\max}K \bp{(16+\max_{i \in [N]}\frac{24p_j}{p_i})\log(T) + \max_{i \in [N]}\frac{896x_{\max}^2d^2p_j}{p_*\psi p_i N}}.
\end{multline*}
Besides, the cumulative expected regret of instance $j$ from case \eqref{case:regret1} and \eqref{case:regret2} is simply 
\begin{align*}
p_j \bp{2bx_{\max}(2q\log(T) + N) + 16bx_{\max}KN}.
\end{align*}
Combining all the above with $p_j=\Theta(1/N)$ in the standard case, our result then follows. \Halmos
\end{proof}

\subsection{Data-Poor Regime}\label{app:regret_datapoor}

The hyperparameters are
\begin{gather*}
\omega_0=\zeta_0+\eta_0, \quad \zeta_0 = \zeta_{1,0} = \frac{C_0-2}{4C_0}, \quad
\eta_0=\sqrt{\frac{27\log(d)|\cB_0|}{qN}}, \quad \eta_{1,0} = \sqrt{\frac{9}{\rho N}},\\
\lambda_{0} = \max_{i\in[N]}\sqrt{\frac{96\sigma_i^2x_{\max}^2K\log(d)|\cB_0|}{q}}, \quad \lambda_{1,j,0} = \max_{i\in[N]}\sqrt{\frac{384\sigma_j^2x_{\max}^2}{p_*}},
\end{gather*}
and
\begin{multline*}
q = \max\left\{\frac{2\cdot384^3C_0^2 x_{\max}^4 (\max\{\max_{i\ne j}\sigma_i^2sd/p_i, \sigma_j^2s^2/p_j\}) K\log(d)}{(C_0-2)^2h^2 p_*^2 \psi^2}, \frac{576C_0^2x_{\max}^2(\max_{i\ne j}d^2\sigma_i^2/p_i)Ks^2\log(N)}{h^2p_*\psi}, \right.\\
\left. \frac{6\cdot192^2C_0^2x_{\max}^2(\max_{i\ne j}d^2\sigma_i^2/p_i)K\log(d)\log(N)}{h^2p_*\psi N}, \frac{96x_{\max}^2Kd\log(dN)}{p_*\psi(\min_{i \ne j} p_i)}, \frac{20K}{p_*p_j}\right\}.
\end{multline*}
Note that 
\begin{align*}
\lambda_{0,j} = \lambda_{0}/\sqrt{|\cB_0^j|}, \quad
\zeta_{1,m}=\zeta_{1,0}, \quad
\eta_{1,m} = \eta_{1, 0} \sqrt{\log(d\min_{i\ne j, |\cB_m^i| > 0}|\cB_m^i|)}, \quad
\lambda_{1,j,\bar{m}} = \lambda_{1,0} \sqrt{\frac{\log(d|\cB_{\bar{m}}^j|)}{|\cB_{\bar{m}}^j|}}
\end{align*}
as in Algorithm~\ref{alg:tmean_bdt}. 

The constants stated in Proposition~\ref{prop:tmean_bdt_ase_dp} are
\begin{align*}
C_4 = \frac{96^2C_0\sigma_jx_{\max}}{(C_0-2)p_*^{3/2}\psi}, \quad 
C_5 = \frac{3C_0}{2p_*}(\max_{i\ne j}\sqrt{\frac{2\sigma_i^2d^2p_j}{\psi p_i}}), \quad
C_6 = \frac{24C_0}{p_*}(\max_{i\ne j}\sqrt{\frac{2\sigma_i^2d^2p_j}{\psi\rho p_i}}).
\end{align*}

The proof of Theorem~\ref{thm:bdt_rgt_sgl_dp} follows closely that of Proposition~\ref{prop:tmean_bdt_rgt_all} and Theorem~\ref{thm:tmean_bdt_rgt_sgl}.

\textbf{Forced-Sample Estimator.}
First, we provide an analog of Proposition~\ref{prop:tmean_bdt_fse} in the data-poor regime. For simplicity, we use $\zeta$, $\eta$ and $\lambda_j$ to represent $\zeta_0$, $\eta_0$ and $\lambda_{0,j}$ (described at the beginning of \S\ref{app:regret_datapoor}) respectively in the following.
\begin{proposition}\label{prop:tmean_bdt_fse_dp}
When $N=\Omega(\log(d)\log(T))$, the forced-sample estimator $\widehat{\beta}_{k,0}^j = \widehat{\beta}_k^j(\cB_0,\lambda_{0,j},\omega_0)$ of data-poor instance $j$ satisfies
\begin{align*}
\P\bb{\|\widehat{\beta}_{k,0}^j - \beta_k^j\|_1 \ge \frac{h}{4x_{\max}}} \le \frac{13}{T},
\end{align*}
for the hyperparameter choices $\zeta_{0}$, $\eta_0$, $\lambda_{0}$, and $q$ specified in Appendix~\ref{app:regret_datapoor}. 
\end{proposition}

Before we prove Proposition~\ref{prop:tmean_bdt_fse_dp}, we introduce the following lemma.
\begin{lemma}\label{lem:tmean_bdt_fse_cmpcd_dp}
When $|\bar{\cB}_{0, k}^j| \ge 3\log(d)/D_1^2$, we have
\begin{gather*}
\P\bb{\widehat{\Sigma}(\bar{\cB}_{0,k}^j) \in \cC(\bar{\cI}_j, \frac{\psi}{2})} \ge 1 - \exp\bp{-D_1^2|\bar{\cB}_{0, k}^j|}, \\
\P\bb{\widehat{\Sigma}(\cB_{0,k}^j) \in \cC(\bar{\cI}_j, \frac{p_*\psi}{4})} \ge 1 - \bp{\exp\bp{-\frac{p_*D_1^2|\cB_{0, k}^j|}{2}}+2\exp\bp{-\frac{p_* |\cB_{0, k}^j|}{10}}},
\end{gather*}
where $D_1 = \max\bc{\frac{1}{2},\frac{\zeta\psi}{768sx_{\max}^2}}$.
\end{lemma}
\begin{proof}{Proof of Lemma~\ref{lem:tmean_bdt_fse_cmpcd_dp}}
The proof of our first result is the same as Lemma~\ref{lem:tmean_off_cmpcd_dp}. The proof of the second result is similar to that of Lemma~\ref{lem:tmean_bdt_fse_cov}, by further applying Lemma~\ref{lem:bdt_fse_covsubsmp} and Lemma~\ref{lem:tmean_bdt_fse_armsmp}. \Halmos
\end{proof}

Now we prove Proposition~\ref{prop:tmean_bdt_fse_dp} by applying Proposition~\ref{prop:tmean_reg_hpb_cov_dp}.
\begin{proof}{Proof of Proposition~\ref{prop:tmean_bdt_fse_dp}}
The proof is analogous to that of Proposition~\ref{prop:tmean_bdt_fse}. Applying Proposition~\ref{prop:tmean_reg_hpb_cov_dp} with $\tilde\psi=p_*\psi/2$, we have
\begin{multline*}
\P\bb{\|\widehat{\beta}_{k,0}^j - \beta_k^j\|_1 \ge \frac{24\lambda_j s}{p_*\zeta\psi} + 2C_0d(3\zeta+4\eta)\max_{i\ne j}\sqrt{\frac{K\sigma_i^2}{p_*\psi|\cB_{0}^i|}\log(\frac{3}{\eta})}} \\
\le 3d\exp(-\frac{N\eta^2}{9}) + 2d \exp(-\frac{\lambda_j^2|\cB_{0,k}^j|}{32\sigma_j^2 x_{\max}^2})+\P\bb{\widehat\Sigma(\cB_{0, k}^j) \in \cC(\bar{\cI}_j, \frac{p_*\psi}{4})}+ \sum_{i\ne j}\P\bb{\lambda_{\min}(\widehat\Sigma(\cB_{0, k}^i)) \le \frac{p_*\psi}{4}}\\
\le 3d\exp\bp{-\frac{N\eta^2}{9}} + 2d \exp\bp{-\frac{\lambda_j^2 |\cB_{0}^j|}{32K\sigma_j^2 x_{\max}^2}}
 +\exp\bp{-\frac{p_*D_1^2|\cB_{0}^j|}{2K}} \\
 +\sum_{i\ne j}d\exp\bp{-\frac{p_*\psi|\cB_{0}^i|}{16Kdx_{\max}^2}}+\sum_{i \in [N]}2\exp\bp{-\frac{p_* |\cB_{0}^i|}{10K}},
\end{multline*}
given $\{|\cB_0^i|\}_{i\ne j}$, where the second inequality uses Lemma~\ref{lem:tmean_bdt_fse_cmpcd_dp}, Lemma~\ref{lem:tmean_bdt_fse_cov}, and $|\cB_{0, k}^j| = |\cB_{0}^j|/K$.

Correspondingly, take
\begin{align*}
\zeta = \frac{C_0-2}{4C_0}, \quad 
\eta = \sqrt{\frac{27\log(d)|\cB_0|}{qN}}, \quad
\lambda_j = \sqrt{\frac{96\sigma_j^2x_{\max}^2K\log(d)|\cB_0|}{q|\cB_0^j|}}.
\end{align*}
With a union bound over $\bigcap_{i\in[N]} \cM_0^i$ in (\ref{eq:eventDM}) by Lemma~\ref{lem:tmean_bdt_fse_armsmp}, we have
\begin{multline*}
\P\bb{\|\widehat{\beta}_{k,0}^j - \beta_k^j\|_1 \ge \frac{768C_0\sigma_jx_{\max}s}{(C_0-2)p_*\psi}\sqrt{\frac{3K\log(d)}{qp_j}} + C_0\bp{\frac{3s}{2}+24\sqrt{\frac{3\log(d)|\cB_0|}{qN}}}\max_{i\ne j}\sqrt{\frac{Kd^2\sigma_i^2\log(N)}{2p_*\psi p_i|\cB_0|}}} \\
\le \frac{5}{T}+\exp\bp{-\frac{p_*p_jD_1^2|\cB_{0}|}{4K}}
+\sum_{i\ne j}d\exp\bp{-\frac{p_*\psi p_i|\cB_0|}{32Kdx_{\max}^2}}
+\sum_{i\in[N]}4\exp\bp{-\frac{p_* p_i|\cB_0|}{20K}}\\
\le \frac{5}{T}+5\exp\bp{-\frac{p_*p_j|\cB_{0}|}{20K}}+\sum_{i\ne j}3d\exp\bp{-\frac{p_*\psi p_i|\cB_0|}{32Kdx_{\max}^2}},
\end{multline*}
where the last inequality uses $D_1\ge 1/2$. Next, we choose a sufficiently large $q$ such that 
\begin{gather*}
\frac{h}{8x_{\max}} \ge \frac{768C_0\sigma_jx_{\max}s}{(C_0-2)p_*\psi}\sqrt{\frac{3K\log(d)}{qp_j}}, \quad
\frac{h}{16x_{\max}} \ge C_0\frac{3s}{2}\max_{i\ne j}\sqrt{\frac{Kd^2\sigma_i^2\log(N)}{2p_*\psi p_i|\cB_0|}}, \\
\frac{h}{16x_{\max}} \ge  24C_0\sqrt{\frac{3\log(d)|\cB_0|}{qN}}\max_{i\ne j}\sqrt{\frac{Kd^2\sigma_i^2\log(N)}{2p_*\psi p_i|\cB_0|}}, \quad \frac{3}{T} \ge \sum_{i\ne j}3d\exp\bp{-\frac{p_*\psi p_i|\cB_0|}{32Kdx_{\max}^2}},\\
\frac{5}{T}\ge 5\exp\bp{-\frac{p_*p_j|\cB_{0}|}{20K}},
\end{gather*}
and it suffices to set
\begin{multline*}
q = \max\left\{\frac{2\cdot384^3C_0^2 x_{\max}^4 (\max\{\max_{i\ne j}\sigma_i^2sd/p_i, \sigma_j^2s^2/p_j\}) K\log(d)}{(C_0-2)^2h^2 p_*^2 \psi^2}, \frac{576C_0^2x_{\max}^2(\max_{i\ne j}d^2\sigma_i^2/p_i)Ks^2\log(N)}{h^2p_*\psi}, \right.\\
\left. \frac{6\cdot192^2C_0^2x_{\max}^2(\max_{i\ne j}d^2\sigma_i^2/p_i)K\log(d)\log(N)}{h^2p_*\psi N}, \frac{96x_{\max}^2Kd\log(dN)}{p_*\psi(\min_{i \ne j} p_i)}, \frac{20K}{p_*p_j}\right\}.
\end{multline*}

Additionally, we also require $\eta \le 1/2-1/C_0-\zeta$ according to Proposition~\ref{prop:tmean_reg_hpb_cov_dp}, which is satisfied as long as 
\begin{align*}
\log(T) \le \bp{\frac{C_0-2}{4C_0}}^2 \frac{N-1}{27\log(d)}. \Halmos
\end{align*}
\end{proof}

\textbf{All-Sample Estimator.}
Next, we provide a tail inequality of our all-sample estimator for the data-poor instance. 
For simplicity, we use $\zeta$, $\eta$ and $\lambda_j$ to represent $\zeta_{1,m}$, $\eta_{1,m}$ and $\lambda_{1,j,\bar{m}}$ (described at the beginning of \S\ref{app:regret_datapoor}) respectively in the following.

\begin{proposition}\label{prop:tmean_bdt_ase_dp}
When the event $\cA$ holds and $N = \Omega\bp{\log(d)\log(T)}$,
the all-sample estimator $\widehat\beta^j_{k,\bar{m}}=\widehat\beta^j_k(\cB_{\bar{m}}, \lambda_{1,j,\bar{m}}, \omega_{1,m})$ of data-poor instance $j$ and optimal arm $k \in \cK_{\text{opt}}^j$ satisfies
\begin{align*}
\|\widehat\beta^j_{k,\bar{m}} - \beta_k^j\|_1 \le C_4 \sqrt{\frac{s^2\log(d p_j|\cB_{\bar{m}}|)}{p_j|\cB_{\bar{m}}|}} + C_5 \sqrt{\frac{\log(\rho N)}{p_j|\cB_m|}} + C_6 \sqrt{\frac{\log(dp_j|\cB_m|)\log(\rho N)}{N p_j |\cB_m|}}
\end{align*}
with probability at least $1 - \bp{\frac{6}{\min_{i\in\cW_k, i\ne j}p_i|\cB_m|} + \frac{4}{p_j|\cB_m|}
+9\exp\bp{-\frac{p_*p_j|\cB_m|}{20}}+ \sum_{i \in \cW_k, i\ne j}5d\exp\bp{-\frac{p_*p_i\psi|\cB_m|}{32dx_{\max}^2}}}$
for the hyperparameter choices $\zeta_{1,0}$, $\eta_{1,0}$, and $\lambda_{1,0}$, and the constants $C_4$, $C_5$ and $C_6$ specified in Appendix~\ref{app:regret_datapoor}.
\end{proposition}

\begin{lemma}\label{lem:tmean_bdt_ase_cmpcd_dp}
On the event $\cM_m^j$ and $\cM_{\bar{m}}^j$ in \eqref{eq:eventDMm}, when $|\bar{\cB}_{m, k}^j| \ge 3\log(d)/D_1^2$, we have 
\begin{align*}
\P\bb{\widehat{\Sigma}(\cB_{\bar{m},k}^j) \in \cC(\bar{\cI}_j, \frac{p_*\psi}{24})} \ge 1 - \bp{\exp\bp{-\frac{D_1^2|\cB_{m}^j|}{2}} + 2\exp\bp{-\frac{p_*|\cB_m^j|}{10}}},
\end{align*}
where $D_1 = \max\bc{\frac{1}{2},\frac{\zeta\psi}{768sx_{\max}^2}}$.
\end{lemma}
\begin{proof}{Proof of Lemma~\ref{lem:tmean_bdt_ase_cmpcd_dp}}
The proof follows closely that of Lemma~\ref{lem:tmean_bdt_ase_mineig}, except that now we bound the event $\tilde{\cE}_{m,k}^j =\bc{\widehat\Sigma(\bar{\cB}_{m, k}^j)\in \cC(\bar\cI_j, \frac{\psi}{2})}$ similarly as Lemma~\ref{lem:tmean_bdt_fse_cmpcd_dp}. \Halmos 
\end{proof}

Now we are ready to prove Proposition~\ref{prop:tmean_bdt_ase_dp}.
\begin{proof}{Proof of Proposition~\ref{prop:tmean_bdt_ase_dp}}
The proof is similar to that of Proposition~\ref{prop:tmean_bdt_ase}. We list the details that differ from Proposition~\ref{prop:tmean_bdt_ase} as follows.

Applying Proposition~\ref{prop:tmean_reg_hpb_cov_dp}, together with a union bound over $\bigcap_{i \in \cW_k}\bar{\cD}_{m,k}^i$ in \eqref{eq:eventDMm}, we have
\begin{multline*}
\P\bb{\|\widehat\beta^j_{k,\bar{m}} - \beta_k^j\|_1 \ge \frac{144\lambda_j s}{p_*\zeta\psi} + 2C_0(3\zeta+4\eta)\max_{i \in \cW_k, i\ne j}\sqrt{\frac{2d^2\sigma_i^2}{p_*^2\psi|\cB_m^i|}\log(\frac{3}{\eta})} \,\middle|\, \cA}  \\
\le 3d\exp\bp{-\frac{\rho N\eta^2}{9}} + 2d \exp\bp{-\frac{\lambda_j^2 p_*|\cB_m^j|}{64\sigma_j^2 x_{\max}^2}}
+ \P\bb{\widehat\Sigma(\cB_{\bar{m}, k}^j) \not \in \cC(\bar\cI_j, \frac{p_*\psi}{24}) \,\middle|\, \cA} \\
+ \sum_{i \in \cW_k, i\ne j}\P\bb{\lambda_{\min}(\widehat\Sigma(\cB_{m, k}^i)) \le \frac{p_*\psi}{4} \,\middle|\, \cA} + \sum_{i \in \cW_k}2\exp\bp{-\frac{p_*|\cB_{m}^j|}{10}},
\end{multline*}
given $\{|\cB_{m, k}^i|\}_{i\in\cW_k}$ and $|\cB_{\bar{m}, k}^j|$. Correspondingly, take
\begin{align*}
\zeta = \frac{C_0-2}{4C_0}, \quad
\eta = \sqrt{\frac{9\log(d\min_{i\in\cW_k, i\ne j}|\cB_m^i|)}{\rho N}}, 
\quad \lambda_j = \sqrt{\frac{384\sigma_j^2x_{\max}^2\log(d|\cB_{\bar{m}}^j|)}{p_*|\cB_{\bar{m}}^j|}}.
\end{align*}
Note that the value of $\eta$ is equivalent to
\begin{align*}
\eta = \sqrt{\frac{9\log(d\min_{i\ne j, |\cB_m^i| > 0}|\cB_m^i|)}{\rho N}},
\end{align*}
since $|\cB_m^i|=0$ for any $i\in[N] \setminus \cW_k$ conditioned on the event $\cA$.
Similarly, using a union bound over $\bigcap_{i \in \cW_k} (\cM_m^i \cap \cM_{\bar{m}}^j)$ and applying Lemma~\ref{lem:tmean_bdt_ase_mineig}, \ref{lem:tmean_bdt_ase_cmpcd_dp} and \ref{lem:tmean_bdt_ase_armsmp}, we get
\begin{multline*}
\P\bb{\|\widehat\beta^j_{k,\bar{m}} - \beta_k^j\|_1 \ge C_4 \sqrt{\frac{s^2\log(d p_j|\cB_{\bar{m}}|/2)}{p_j|\cB_{\bar{m}}|}} + C_5 \sqrt{\frac{\log(\rho N)}{p_j|\cB_m|}} + C_6 \sqrt{\frac{\log(dp_j|\cB_m|/2)\log(\rho N)}{N p_j |\cB_m|}} \,\middle|\, \cA } \\
\le \frac{6}{\min_{i\in\cW_k, i\ne j}p_i|\cB_m|} + \frac{8}{p_j|\cB_{\bar{m}}|}
+\exp\bp{-\frac{D_1^2p_j|\cB_{m}|}{4}}+ \sum_{i \in \cW_k, i\ne j}d\exp\bp{-\frac{p_*p_i\psi|\cB_m|}{32dx_{\max}^2}} + \sum_{i \in \cW_k}8\exp\bp{-\frac{p_*p_i|\cB_m|}{20}} \\
\le \frac{6}{\min_{i\in\cW_k, i\ne j}p_i|\cB_m|} + \frac{4}{p_j|\cB_m|}
+9\exp\bp{-\frac{p_*p_j|\cB_m|}{20}}+ \sum_{i \in \cW_k, i\ne j}5d\exp\bp{-\frac{p_*p_i\psi|\cB_m|}{32dx_{\max}^2}},
\end{multline*}
where $C_4$, $C_5$, and $C_6$ are constants listed at the beginning of \S\ref{app:regret_datapoor} and we use $D_1\ge 1/2$. 

In addition, to satisfy $\eta \le 1/2 - 1/C_0 - \zeta$, we require
\begin{align*}
\log(d|\cB_m|) \le \bp{\frac{C_0-2}{2C_0}}^2 \frac{\rho N}{9}.\Halmos
\end{align*}
\end{proof}

\textbf{Regret Analysis.}
For the regret analysis in data-poor regime, we consider the same three cases in Appendix~\ref{sec:rmbandit-proofstrategy}.

We first provide a per-period regret bound at time $t$ for data-poor instance $j$ in case~\eqref{case:regret3} in Appendix~\ref{sec:rmbandit-proofstrategy}.
\begin{lemma}\label{lem:tmean_bdt_rgtper_dp}
When $\cA$ holds, $N=\Omega\bp{\log(d)\log(T)}$ and $Z_t=j$ for data-poor instance $j$, the expected regret at time $t \in \cB_m$ for $m > 1$ is upper bounded by
\begin{multline*}
r_t^j \le 24x_{\max}^2LK\bp{C_4^2 \frac{s^2\log(dp_jt)}{p_jt} + 2C_5^2 \frac{\log(\rho N)}{p_jt} + 2C_6^2 \frac{\log(\rho N)\log(dp_jt)}{N p_j t}} \\
+4bx_{\max}K \bp{\max_{i\ne j}\frac{24}{p_it} + \frac{16}{p_jt}+9\exp\bp{-\frac{p_*p_jt}{80}}
+5dN\exp\bp{-\frac{p_*\psi (\min_{i\ne j}p_i) t}{128dx_{\max}^2}}}.
\end{multline*}
\end{lemma}
\begin{proof}{Proof of Lemma~\ref{lem:tmean_bdt_rgtper_dp}}
The proof follows closely that of Lemma~\ref{lem:tmean_bdt_rgtper}. We list the details that differ from Lemma~\ref{lem:tmean_bdt_rgtper} as follows.

Applying Proposition~\ref{prop:tmean_bdt_ase_dp} and Lemma~\ref{lem:bdt_batchsize}, we can write
\begin{align*}
\P\bb{\|\widehat{\beta}_{\iota,\bar{m-1}}^j - \beta_\iota^j\|_1 \ge \delta}
\le \max_{i\ne j}\frac{24}{p_it} + \frac{16}{p_jt}
+9\exp\bp{-\frac{p_*p_jt}{80}}
+5dN\exp\bp{-\frac{p_*\psi (\min_{i\ne j}p_i) t}{128dx_{\max}^2}},
\end{align*}
for $\iota \in \{1, k\}$, where 
\begin{align*}
\delta = C_4 \sqrt{\frac{2s^2\log(d p_jt)}{p_jt}} + C_5 \sqrt{\frac{4\log(\rho N)}{p_jt}} + C_6 \sqrt{\frac{4\log(dp_jt)\log(\rho N)}{N p_j t}}.
\end{align*}
Then, we can obtain 
\begin{align*}
\E\bb{r_{t, 1}^j(X_t) \mid Z_t = j, \cA} 
\le 4bx_{\max}K \bp{\max_{i\ne j}\frac{24}{p_it} + \frac{16}{p_jt}
+9\exp\bp{-\frac{p_*p_jt}{80}}
+5dN\exp\bp{-\frac{p_*\psi (\min_{i\ne j}p_i) t}{128dx_{\max}^2}}}.
\end{align*}
The rest of the proof is the same as Lemma~\ref{lem:tmean_bdt_rgtper}. \Halmos
\end{proof}

Now we prove Theorem~\ref{thm:bdt_rgt_sgl_dp} to show the regret upper bound of the data-poor instance.
\begin{proof}{Proof of Theorem~\ref{thm:bdt_rgt_sgl_dp}}
The proof follows closely that of Theorem~\ref{thm:tmean_bdt_rgt_sgl}, considering an expected time horizon of $T$ for data-poor instance $j$, i.e., a total time horizon of $T/p_j = \Theta(d^2NT)$. 

Similarly, the cumulative expected regret of data-poor instance $j$ in case~\eqref{case:regret3} is 
\begin{multline*}
\E\bb{\sum_{t\in \bigcup_{m>1} \cB_m} r_t^{Z_t}(X_t)\mathbbm{1}\bp{Z_t = j} \,\middle|\, \cA} 
= \sum_{t\in \bigcup_{m>1} \cB_m} p_j r_t^j \\
\le 24x_{\max}^2LK\bp{C_4^2s^2\log(dT) + 2C_5^2 \log(\rho N) + 2C_6^2\frac{\log(\rho N)\log(dT)}{N}}\log(T) \\
+ 4bx_{\max}K \bp{\bp{16+\max_{i \ne j}\frac{24p_j}{p_i}}\log(T)+ \frac{720}{p_*} + \max_{i \in [N]}\frac{896x_{\max}^2d^2p_j}{p_*\psi p_i N}},
\end{multline*}
where we use Lemma~\ref{lem:tmean_bdt_rgtper_dp}.
Besides, the cumulative expected regret of data-poor instance $j$ in case \eqref{case:regret1} and \eqref{case:regret2} is simply 
\begin{align*}
p_j \bp{4bx_{\max}q\log(T) + 26bx_{\max}KN},
\end{align*}
where we use that the event $\cA$ holds with at least a probability of $1 - 13KN/T$.
Combining all the above with $p_j/p_{j'}=\Theta(1/d^2)$ for any $j'\ne j$ in the data-poor regime, the claim then follows. \Halmos
\end{proof}

\subsection{Margin Condition}\label{sec:margin_cond_dis}

In this section, we discuss how our regret bound in Proposition~\ref{prop:tmean_bdt_rgt_all} is affected by the margin condition (Assumption~\ref{ass:marcond}). We assume a more general margin condition \citep{bastani2021mostly} as follows, and show that our algorithm can still achieve an improvement in the context dimension.
\begin{assumption}[$\alpha$-Margin Condition] \label{ass:alpmarcond}
For any arms $k$ and $k'$ of any instance $j \in [N]$, there exists a constant $L>0$ such that $\P\bb{|X^{\top}(\beta_k^j - \beta_{k'}^j)| \le \kappa \mid Z = j} \le L \kappa^\alpha$ for any $\kappa > 0$ and for some $\alpha\ge 0$.
\end{assumption}
Throughout our regret analysis in Appendix~\ref{app:regret_sfixed_std}, we only use the margin condition in Lemma~\ref{lem:tmean_bdt_rgtper}. Thus, given Assumption~\ref{ass:alpmarcond}, we have the following analog of Lemma~\ref{lem:tmean_bdt_rgtper}:
\begin{lemma}\label{lem:tmean_bdt_rgtper_alpha}
When $\cA$ holds, $N=\Omega\bp{\log(d)\log(T)}$, $Z_t=j$ and $\alpha\ne 1$, the expected regret at time $t \in \cB_m$ for $m > 1$ is upper bounded by
\begin{multline*}
r_t^j \le 24x_{\max}^2LK\bp{C_1\sqrt{sd\log(dp_jT)} + C_2\sqrt{2sd\log(\rho N)} + C_3\sqrt{\frac{2d^2\log(\rho N)\log(dp_jT)}{N}}}^{\alpha+1}(p_jt)^{-\frac{\alpha+1}{2}} \\
+ 4bx_{\max}K \bp{\max_{i\in[N]}\frac{24}{p_it} + \frac{16}{p_jt} + 7dN\exp\bp{-\frac{p_*\psi (\min_{i\in[N]}p_i) t}{128dx_{\max}^2}}}.
\end{multline*}
\end{lemma}
\begin{proof}{Proof of Lemma~\ref{lem:tmean_bdt_rgtper_alpha}}
The proof is close to that of Lemma~\ref{lem:tmean_bdt_rgtper}. We can follow the same steps until \eqref{eq:expperrgt_r1}. Now instead we use Assumption~\ref{ass:alpmarcond}, and the term $r_{t,2}^j(X_t)$ has
\begin{align*}
\E\bb{r_{t,2}^j(X_t) \mid Z_t = j, \cA} \le 2x_{\max}\delta K \P\bb{(\cL_{k}^j)^c} \le 4x_{\max}^2LK\delta^{\alpha+1}.
\end{align*}
The claim then follows. \Halmos
\end{proof}

Now we are ready to state the following regret bound for all instances given our general margin condition. Intuitively, a larger value of $\alpha$ imposes stronger assumptions on the contextual distribution $\cP_X^j$ around the boundary --- i.e., it rules out distributions with high density around the boundary --- and hence makes the problem easier to learn. When $\alpha\rightarrow 0$, we obtain a $\cO(\sqrt{T})$ regret guarantee; however, when $\alpha\rightarrow 1$, we recover an optimal $\cO(\log(T))$ regret as stated in Proposition~\ref{prop:tmean_bdt_rgt_all}. The proof is similar to that of Proposition~\ref{prop:tmean_bdt_rgt_all}.
\begin{corollary}\label{cor:tmean_bdt_rgt_alpha}
When $N = \Omega(\log(d)\log(T))$, 
the total cumulative expected regret of all instances up to time $T$ of \textsf{RMBandit} satisfies
\begin{align*}
R_T=
\begin{cases}
\cO\bp{KN \bp{sd+d^2/N}^{\frac{\alpha+1}{2}}(\log(N)\log(dT/N))^{\frac{\alpha+1}{2}}(T/N)^{\frac{1-\alpha}{2}}}, & \text{where~}\alpha<1 \\
\cO\bp{KN \log(T/N)}, & \text{where~}\alpha>1
\end{cases}
\end{align*}
for appropriate choices of hyperparameters $\omega_0$, $\zeta_{1,0}$, $\eta_{1,0}$, $\lambda_{0}$, $\lambda_{1,0}$, and $q$ provided in Appendix~\ref{app:regret_sfixed_std}.
\end{corollary}

\section{Useful Lemmas}

This section collects useful results from the literature.

\begin{lemma} \label{lem:subgauvec_conc} 
Let $X=\begin{bmatrix} X_1 & \cdots & X_n \end{bmatrix}$ be a vector of $n$ independent $\sigma$-subgaussian random variables with mean $\mu$. Then, for any $a \in \R^n$ and $t \ge 0$, it holds that
\begin{align*}
\P\bb{|a^\top(X - \mu)| \ge t} \le 2\exp\bp{-\frac{t^2}{2\sigma^2\|a\|_2^2}}.
\end{align*}
\end{lemma}
\begin{proof}{Proof of Lemma~\ref{lem:subgauvec_conc}}
See Corollary 1.7 of \cite{rigollet2015high}.
\end{proof}

\begin{lemma} \label{lem:mat_cherbound}
Consider a sequence of independent random symmetric matrices $X_k \in \mathbb{R}^{d \times d},\, k\in[n]$ with $\lambda_{\min}(X_k) \ge 0$ and $\lambda_{\max}(X_k) \le L$ for any $k$. Let $\mu=\lambda_{\min}(\mathbb{E}[\sum_{k\in[n]} X_k])$.
We have for $0 < t < 1$ that
\begin{align*}
\mathbb{P}\bb{\lambda_{\min}(\sum_{k\in[n]} X_k) \ge t\mu} \ge 1 - d\exp\left(-\frac{(1-t)^2\mu}{2L}\right).
\end{align*}
\end{lemma}
\begin{proof}{Proof of Lemma~\ref{lem:mat_cherbound}}
See page 61 in \cite{tropp2015introduction}. 
\end{proof}

\begin{lemma} \label{lem:ind_cherbound}
Suppose $X_1, \cdots, X_n$ are $n$ independent Bernoulli random variables with mean $p_1, \cdots, p_n$ respectively. Let $\mu=\sum_{i\in[n]} p_i$. Then, we have
\begin{align*}
\P\bb{|\sum_{i\in[n]} X_i - \mu| \ge \frac{\mu}{2}} \le 2 \exp\bp{-\frac{\mu}{10}}.
\end{align*}
\end{lemma}
\begin{proof}{Proof of Lemma~\ref{lem:ind_cherbound}}
The result follows by taking $\epsilon=1/2$ in Corollary A.1.14 of \cite{alon2004probabilistic}.
\end{proof}

\begin{lemma}\label{lem:bdt_fse_covsubsmp}
For any sets $\cB,\cB'$ with $\cB' \subseteq \cB$, it holds that
$\lambda_{\min}(\widehat{\Sigma}(\cB))\ge\lambda_{\min}(\widehat{\Sigma}(\cB'))|\cB'|/|\cB|$. Besides, if $\widehat\Sigma(\cB')\in\cC(\cS, \psi)$, then $\widehat\Sigma(\cB)\in\cC(\cS, |\cB'|\psi/|\cB|)$.
\end{lemma}
\begin{proof}{Proof of Lemma~\ref{lem:bdt_fse_covsubsmp}}
See Lemma EC.23 and EC.7 in \cite{bastani2020online}. \Halmos
\end{proof}

\section{Experimental Details} \label{app:experiments}

\subsection{Synthetic Experiment Details} \label{app:exp-synth}

\paragraph{Offline.}
Each instance receives an equal $n_j=100$ observations (consider the standard data regime). We generate the shared parameters $\beta^\dagger$ by drawing each element independently from a uniform distribution $\texttt{Uniform}[0, 2]$, and normalizing them such that $\|\beta^\dagger\|_1=2$. We randomly draw $s$ entries out of the $d$ dimensions for each bias term $\delta^j$ independently, and then draw the values of the $s$ nonzero entries independently from a uniform distribution $\texttt{Uniform}[0, 1]$. 
Next, we draw the context vectors $X_t$ independently from a gaussian distribution $\cN(\zero, \vI)$, truncated at $-1$ and $1$ so that $\|X_t\|_\infty=1$. We draw the noise $\epsilon_t$ independently from a gaussian distribution $\cN(0, \sigma_j^2)$ with $\sigma_j=0.05$ for any instance $j\in[N]$. 
To test the performance of our algorithm, we leave 20\% of the data of the target instance as the test set, and use a $4$-fold cross validation to tune all the hyperparameters on the rest 80\%. 

\paragraph{Online.}
Our total time horizon across instances $T$ equals $15,000$, $10,000$ and $60,000$ respectively for the three settings in Figure~\ref{fig:synthetic_on} and the arrival probability $p_j = 1/N$ for all $j\in[N]$; thus, each instance will receive an expected $500$, $1,000$ and $4,000$ observations respectively. 
We generate the shared parameters $\{\beta^\dagger_k\}_{k\in[K]}$ by drawing each element independently from a uniform distribution $\texttt{Uniform}[0, 2]$, and normalizing them such that $\|\beta^\dagger_k\|_1=5$. We randomly draw $s$ entries out of the $d$ dimensions for each bias term $\delta_k^j$ independently, and then draw the values of the $s$ nonzero entries independently from a uniform distribution $\texttt{Uniform}[0, 1]$. 
Next, we draw the context vectors $X_t$ independently from a gaussian distribution $\cN(\zero, \vI)$, truncated at $-1$ and $1$ so that $\|X_t\|_\infty=1$. We draw the noise $\epsilon_t$ independently from a gaussian distribution $\cN(0, \sigma_j^2)$ with $\sigma_j=0.05$ for all instances $j\in[N]$. 

To ensure fair comparison, we tune the hyperparameters of all algorithms on a pre-specified grid. Define a hyperparameter $q_0$ to be such that $q=q_0 KN$ for $q$ in our Algorithm~\ref{alg:tmean_bdt}. We take $q=1$ for LASSO, OLS Bandit and the pooling algorithm, and $\lambda_1=\lambda_{2,0}=0.005$ for LASSO Bandit (note that the definition of $q$ in \cite{bastani2020online} is different from ours). We take $\alpha=0.5$ for \textsf{GOBLin}. We apply a trace-norm regularization on the parameters for each arm $k\in[K]$ respectively, and set the tuning constant in $\lambda_n$ to be $0.005$ in Trace-norm Bandit. For \textsf{RMBandit}, we take $q_0=0.5$; additionally, we set $\omega_0=\eta_{1,0}=0.2$, $\zeta_{1,0}=0.1$, and $\lambda_0=\lambda_{1,0}=0.005$ for the first two settings (a) and (b), and $\omega_0=\zeta_{1,0}=\eta_{1,0}=0.05$ and $\lambda_0=\lambda_{1,0}=0.01$ for the third setting (c). We take $h=5$ for the first two settings and $h=10$ for the third setting. For the robustness of our choices of the hyperparameters, please see additional experiments in Appendix~\ref{app:add_experiments}.

\subsection{Diabetes Experiment Details} \label{app:exp-diabetes}

Our original dataset consists of $9,948$ patients observed from 379 healthcare providers. However, many of these providers observe very few patients, so we restrict our experiment to the $N=13$ largest hospitals, of which each has at least 150 unique patients (mean 317, median 301) observed during the sample period. We perform standard variable selection as a pre-processing step in order to avoid overfitting when computing our linear oracles. In particular, we run a LASSO variable selection procedure by regressing diabetes outcomes against the 184 total features (note that we exclude the healthcare providers that we use in our experiment in this step to avoid overfitting), and we tune the hyperparameters using $10$-fold cross-validation. This leaves us with roughly $80$ commonly predictive features (the number of selected features depends on the randomness in the cross-validation procedure). Note that this is still a relatively large number of features compared to the number of observations, supporting our argument that arm parameters are likely dense. 

We fit a linear oracle to data from the target provider in hindsight; to avoid overfitting, we use a leave-one-out approach, i.e., for each patient, we train the best linear model on all data from the target provider excluding the current patient. Our oracle is constructed to provide the best achievable mean squared error within a linear model family. For the offline setting, we leave 50\% of the data as the test set, and use a $4$-fold cross validation to tune the hyperparameters on the rest data. For the online setting, to ensure fair comparison, we tune the hyperparameters of all algorithms, and we report the optimized results in Figure~\ref{fig:diabetes_on}.

\subsection{Pricing Experiment Details} \label{app:exp-pricing}

\paragraph{Data.} The original dataset covers 145 weeks of orders from 77 fulfillment centers across 51 cities. There are 14 different categories (e.g., beverages, snacks) and 4 different cuisines (e.g., Indian, Italian) for meals delivered by the company. We restrict our experiment to fulfillment centers in the three largest cities that have more than 2 fulfillment centers. Thus, we have $N=20$ centers, each processing an average (median) of 5,916 (6,118) orders during the sample period. One order arrives at each time step, and the chosen price is the checkout price, which includes discounts, taxes and delivery charges. The order price in our data ranges from \$45 to \$767; thus, we set $p_{\min}=40$ and $p_{\max}=800$. Following standard practice, we also normalize the price so it has a similar scale as the other features. Our outcome (demand) is given by the quantity in each order. The contexts are order-specific features including dummy variables capturing the category and cuisine, indicators of email or homepage promotions, and an intercept.  Overall, our $X_t$ has dimension 19, and therefore the dimensionality of the unknown parameters of the pricing model $d=38$. The oracle demand function of each center is estimated based on all the data from the corresponding center; we truncate the estimated coefficients with a maximum absolute value of $1,000$ (large values may occur due to multicollinearity of the features in certain centers).

\paragraph{Algorithm.} We now embed our robust multitask estimator within the \textsf{ILSX}/\textsf{ILQX} algorithmic framework proposed in \cite{ban2021personalized} to design our \textsf{RMX} algorithm; similarly, we embed the Laplacian estimator used by \textsf{GOBLin} \citep{cesa2013gang} to design the \textsf{GOBX} algorithm. 

Let $\beta^j=\begin{bmatrix} \beta_0^j \\ \beta_1^j \end{bmatrix}$ denote the unknown parameters for instance $j$. For our forced samples, we fix two experimental prices $p_1=200$ and $p_2=600$, which we charge in two sets of periods 
\begin{align*}
M_i^j=\bc{t \,\middle|\, \sum_{r\in[t]}\mathbbm{1}(Z_r=j)=E^2+i-1, \, E=1, 2, \cdots}
\end{align*}
for each experimental price $i\in[2]$ and each instance $j\in[N]$. Note that $M_i^j$ is a random set in the multitask setting, since it depends on the realization of arrivals $Z_t$. Let $M^j = M_1^j \cup M_2^j$ represent the forced price experimentation period, and let $M_t^j = \bc{r | r\in M^j, \, r<t}$ be the set of time periods when prices are forced at instance $j$ before time $t$. We update our \textsf{RMEstimator} of $\beta^j$'s at time periods 
\begin{align*}
M = \bc{t\,\middle|\,t=N(E^2+1), \, E=1, 2, \cdots},
\end{align*}
so that each instance obtains the same number of training observations in expectation as in the single-instance setting.\footnote{We initialize our algorithm with OLS until each instance has at least observed 2 orders.} Then, the samples used for estimating the optimal price at time $t$ are $\bigcup_{j\in[N]}M_{\gamma_t}^j$, where $\gamma_t = \max \bc{r \mid r \in M, r<t}$. 

Note that we now only maintain a single set of estimated parameters for instances (compared to both forced-sample and all-sample estimators in Algorithm~\ref{alg:tmean_bdt}). We denote our \textsf{RMEstimator} (Algorithm~\ref{alg:tmean_reg}) at instance $j$ at time $t$ as
\begin{align*}
\widehat\beta^j(\cup_{j\in[N]}M_{\gamma_t}^j, \lambda_{j,t}, \omega_t).
\end{align*}
The first argument indicates the training data, i.e., all observations with price experimentation before time $\gamma_t$ (recall that the robust multitask estimators are only updated at $t\in M$); the remaining arguments are hyperparameters. We denote the estimated optimal price of user $X_t$ at instance $j$ at time $t$ as
\begin{align*}
\widehat{p}^j(X_t, \widehat\beta^j) = \frac{X_t^\top\widehat\beta_0^j}{-2X_t^\top\widehat\beta_1^j},
\end{align*}
which is truncated at $p_{\min}$ and $p_{\max}$. We formalize our algorithm in Algorithm~\ref{alg:tmean_prc}.

\begin{algorithm}[htbp]
\SingleSpacedXI \small
\begin{algorithmic}
\State \textbf{Input parameters:} Initial hyperparameters $\zeta_0, \eta_0, \lambda$
\State Initialize $\{M_i^j\}_{i\in[2]}$, and $M$
\For {$t \in [T]$}
\State Observe an arrival at instance $j = Z_t$ and corresponding context $X_t$
\If {$t \in M_i^j$}
\State Charge price $p_t = p_i$
\Else 
\State Charge price $p_t = \widehat{p}^j(X_t, \widehat{\beta}^j(\cup_{j\in[N]}M_{\gamma_t}^j, \lambda_{j,t}, \omega_t))$
\EndIf
\State Observe demand $Y_t = X_t^\top\beta_0^j + p_t \cdot (X_t^\top\beta_1^j) + \epsilon_t$
\If {$t \in M$}
\State Update $\zeta_t = \zeta_0$, $\eta_t = \eta_0 \sqrt{\log(d\min_{j \in [N], |M_{\gamma_t}^j| > 0}|M_{\gamma_t}^j|)}$, and $\omega_t = \zeta_t + \eta_t$
\State Update $\lambda_{j,t} = \lambda_{j,0} |M_{\gamma_t}^j|^{\frac{1}{4}}\sqrt{\log(d|M_{\gamma_t}^j|)}$ for each $j\in[N]$
\EndIf
\EndFor
\end{algorithmic}
\caption{Robust Multitask Estimator with Price Experimentation (\textsf{RMX})}
\label{alg:tmean_prc}
\end{algorithm}

The \textsf{GOBX} algorithm follows exactly as in Algorithm \ref{alg:tmean_prc}, but uses the Laplacian-regularized estimator from \citep{cesa2013gang} instead of our robust multitask estimator. Once again, to ensure fair comparison, we tune the hyperparameters of all algorithms, and we report the optimized results in Figure \ref{fig:pricing}.

\section{Additional Experiments}\label{app:add_experiments}

\subsection{Dependence on Parameters in \textsf{RMEstimator}}\label{app:add_exp_deppara}

First, we study how the performance of \textsf{RMEstimator} scales with the parameters in our model setup. In the following, we consider the setting (c) in Figure~\ref{fig:synthetic_off}, where $N=15$, $d=40$, and $s=5$. The model setup is the same as the offline setting described in \S\ref{app:exp-synth}.

Aligned with our theory, Figure~\ref{fig:synthetic_vary_off} shows that the prediction error decreases with the number of instances $N$, increases with the sparsity level $s$ and again decreases with the sample size ratio $n_i/n_j$ for a target instance $j$ ($n_i=n_{i'}$ for $i\ne i'$ and $i, i' \ne j$). Intuitively, when more auxiliary information is available, e.g., a larger number of instances or higher arrival rate in neighboring instances, our estimation or prediction becomes more accurate and hence the prediction error declines. However, when the instances become more heterogeneous, e.g., a higher sparsity level $s$, less shared information can be transferred and the problem becomes harder to learn. We note that (c) in Figure~\ref{fig:synthetic_vary_off} is related to the data-poor regime; in particular, the prediction error in a data-poor instance (i.e., $n_i/n_j \gg 1$) is lower than the corresponding prediction error in a standard instance (i.e., $n_i/n_j \approx 1$) for \textsf{RMEstimator}, consistent with Theorem~\ref{thm:tmean_reg_hpb_dp}. 

\begin{figure}[htbp]
\centering
\begin{subfigure}[b]{0.32\textwidth}
  \centering
  \includegraphics[width=\textwidth]{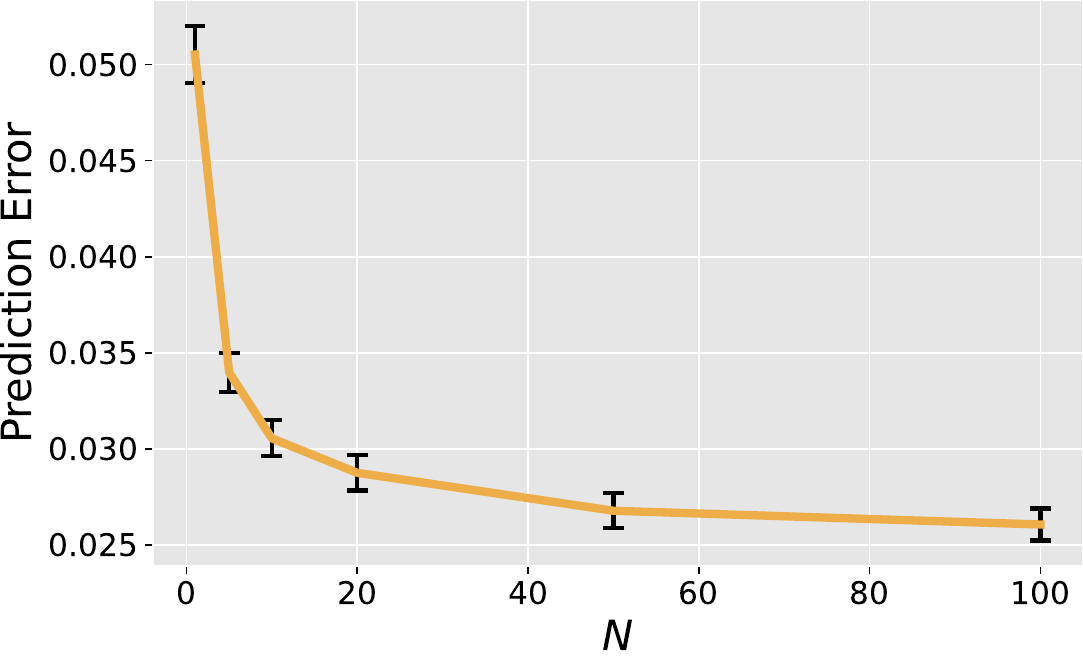}
  \caption{Varying $N$}
\end{subfigure}
\begin{subfigure}[b]{0.32\textwidth}
  \centering
  \includegraphics[width=\textwidth]{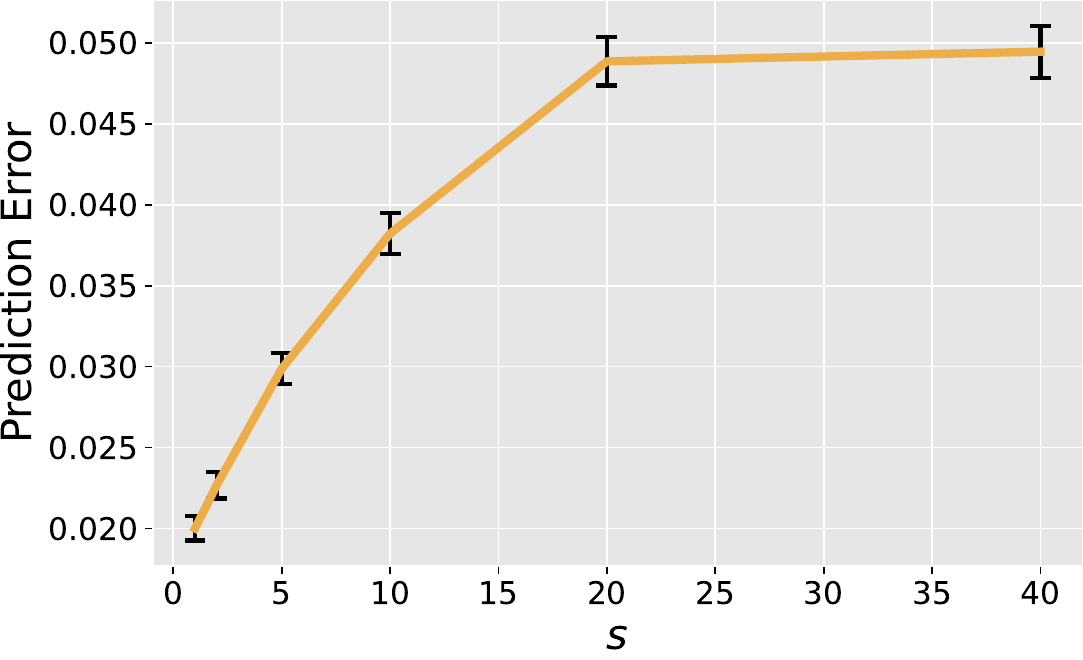}
  \caption{Varying $s$}
\end{subfigure}
\begin{subfigure}[b]{0.32\textwidth}
  \centering
  \includegraphics[width=\textwidth]{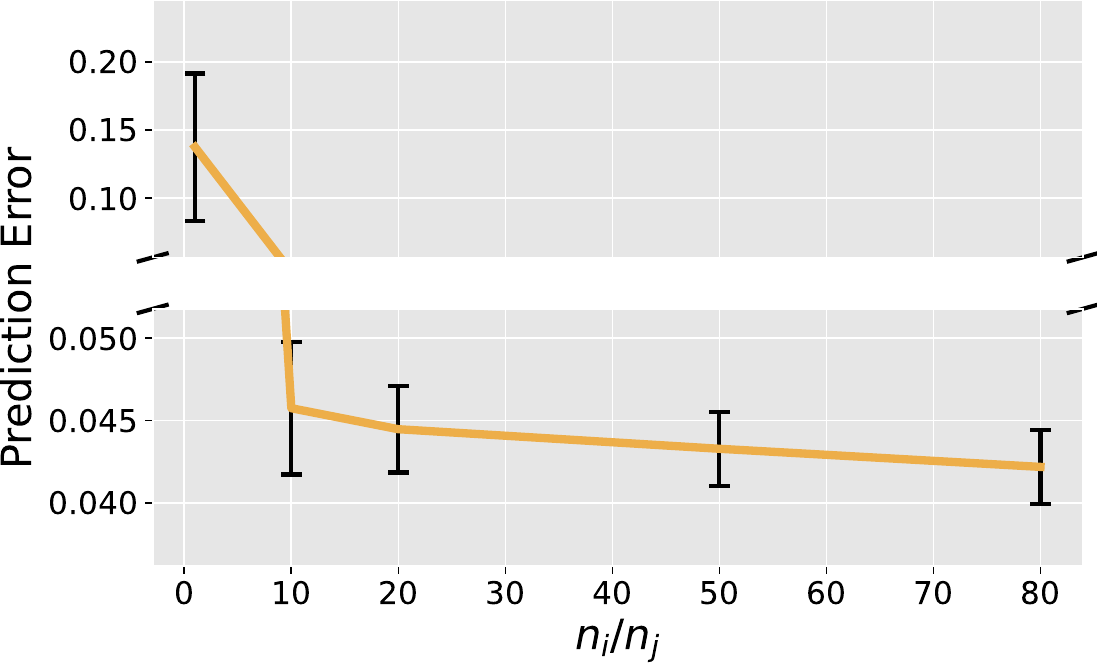}
  \caption{Varying $n_i/n_j$}
\end{subfigure}
\caption{Lines depict the prediction error averaged over 200 trials of a single linear instance, with error bars the corresponding 95\% confidence intervals.}
\label{fig:synthetic_vary_off}
\end{figure}

In addition, we also study how \textsf{RMEstimator} reacts to different values of the parameters $\beta^j$'s. Figure~\ref{fig:synthetic_other_off} (a) shows the prediction error versus the alignment of $\delta^j$'s. We randomly select $s$ out of the first $d_c$ dimensions, instead of out of all the $d$ dimensions, to create $\delta^j$'s. Thus, for smaller value of $d_c$, the supports of $\delta^j$'s are more aligned. We find a non-monotonic relation between the prediction error and the value of $d_c$. Intuitively, \textsf{RMEstimator} performs better when there are fewer well-aligned components to learn together with $\delta^j$ in Step 2 (i.e., $|\cI_{\text{well}}^\zeta|$ is smaller); that takes place either when the supports of $\delta^j$'s are very aligned and $|\cI_{\text{well}}^\zeta| \approx s$ or when the supports of $\delta^j$'s are not aligned at all so each dimension is approximately poorly-aligned and $|\cI_{\text{well}}^\zeta| \approx 0$. It is worth noting that even in the worst case, \textsf{RMEstimator} can still provide a reasonably good guarantee on the prediction error, which shows the robustness of this algorithm. 
Figure~\ref{fig:synthetic_other_off} (b) compares a pooling algorithm with \textsf{RMEstimator} given different magnitudes of $\delta^j$'s. More specifically, we draw the values of the $s$ nonzero entries of $\delta^j$ from a uniform distribution $\texttt{Uniform}[0, a]$. We find the pooling algorithm outperforms \textsf{RMEstimator} only when the value of $a$ is very small. Nevertheless, even in that case, our algorithm still transfers most of the information and provides a competitive prediction accuracy as the pooling algorithm. This suggests our algorithm can be useful in most of the settings empirically. Finally, we analyze the performance of our algorithm when $\delta^j$'s are approximately sparse in Figure~\ref{fig:synthetic_other_off} (c). In detail, we add a noise of $\texttt{Uniform}[-a_p, a_p]$ on each dimension of $\delta^j$. Therefore, $\delta^j$'s become less sparse when $a_p$ takes larger values. The result is consistent with Figure~\ref{fig:synthetic_vary_off} (b); that is, when $a_p$ is larger and there is less shared structure across instances, the problem becomes harder and our prediction accuracy declines. 

\begin{figure}[htbp]
\centering
\begin{subfigure}[b]{0.32\textwidth}
  \centering
  \includegraphics[width=\textwidth]{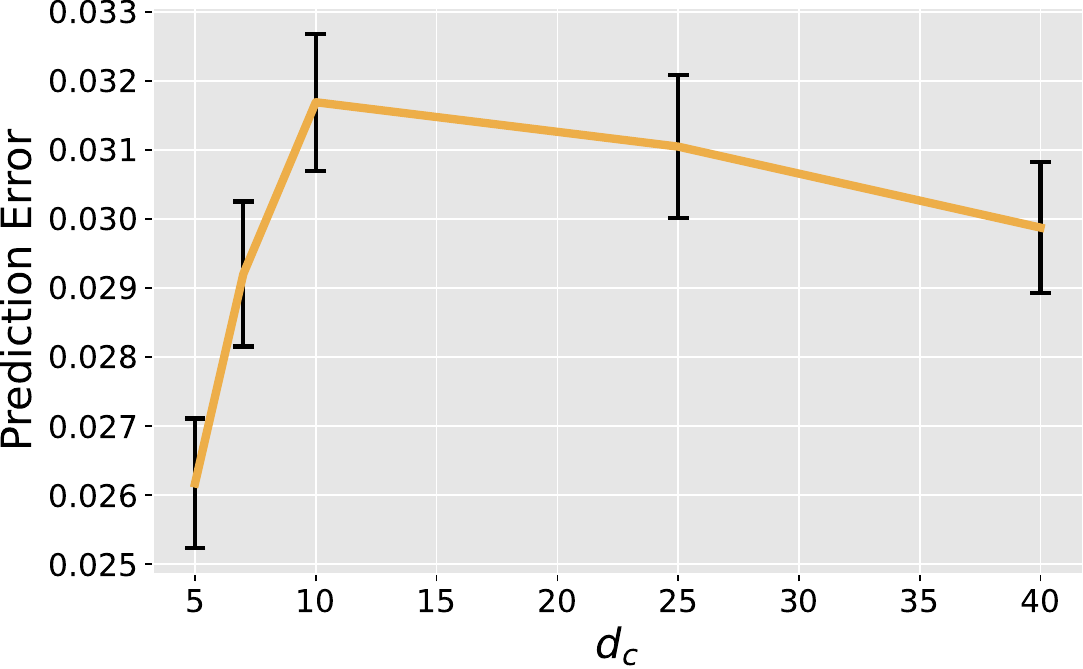}
  \caption{Alignments of $\delta^j$'s}
\end{subfigure}
\begin{subfigure}[b]{0.32\textwidth}
  \centering
  \includegraphics[width=\textwidth]{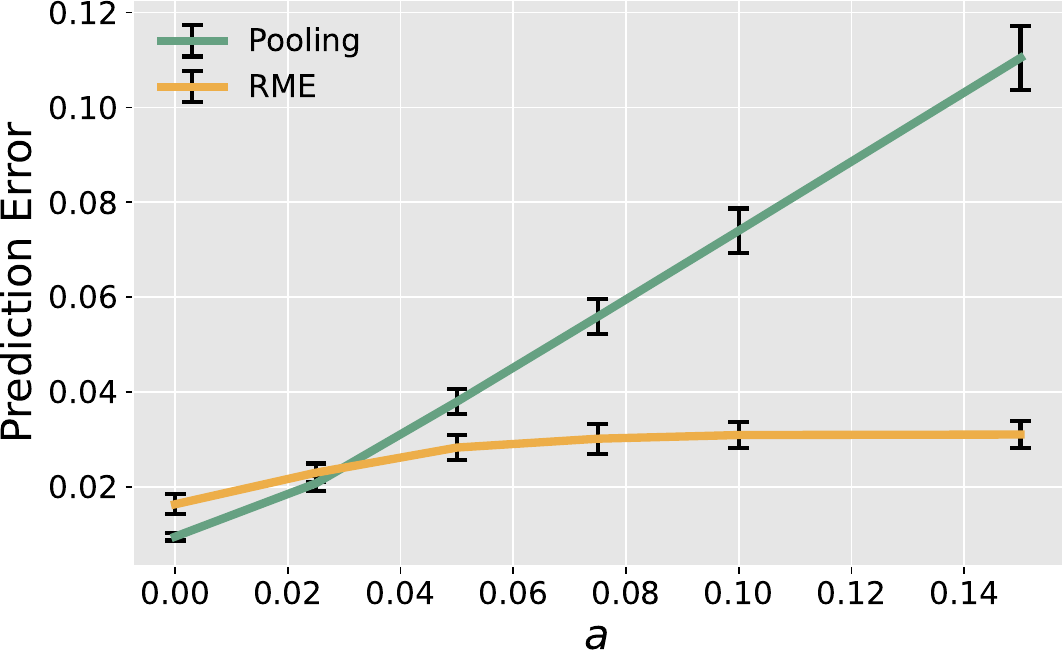}
  \caption{Magnitudes of $\delta^j$'s}
\end{subfigure}
\begin{subfigure}[b]{0.32\textwidth}
  \centering
  \includegraphics[width=\textwidth]{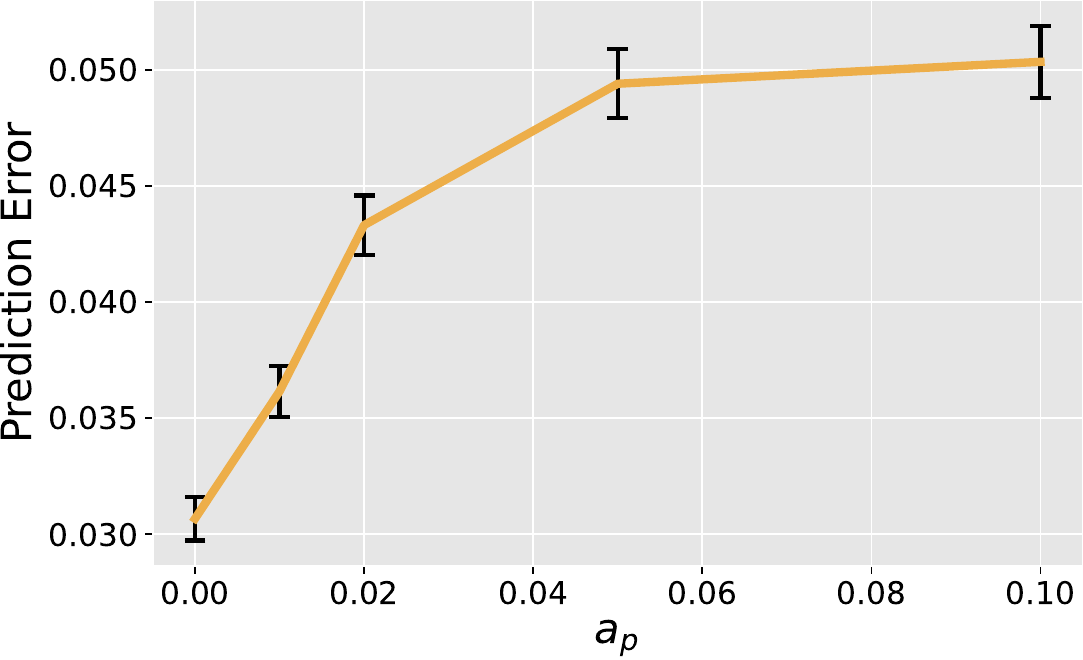}
  \caption{Approximate sparsity of $\delta^j$'s}
\end{subfigure}
\caption{Lines depict the prediction error averaged over 200 trials of a single linear instance, with error bars the corresponding 95\% confidence intervals.}
\label{fig:synthetic_other_off}
\end{figure}

\subsection{Comparison with SCAD \& MCP in \textsf{RMEstimator}} \label{app:scadmcp}

We now numerically explore alternative variants of the LASSO penalty, such as SCAD \citep{fan2001variable} and MCP \citep{zhang2010nearly}, in Step 2 of the \textsf{RMEstimator}. We consider the setting (b) in Figure~\ref{fig:synthetic_off}, where $N=10$, $d=20$, and $s=2$. The model setup is the same as the offline setting described in \S\ref{app:exp-synth}. In Figure~\ref{fig:scadmcp}, we find that the performance of \textsf{RMEstimator} based on SCAD or MCP is comparable or worse than that of our \textsf{RMEstimator} based on LASSO, which confirms the value of efficient knowledge transfer through robust statistics regardless of our choice of the sparse penalized algorithm.

\begin{figure}[htbp]
\centering
\includegraphics[width=.42\textwidth]{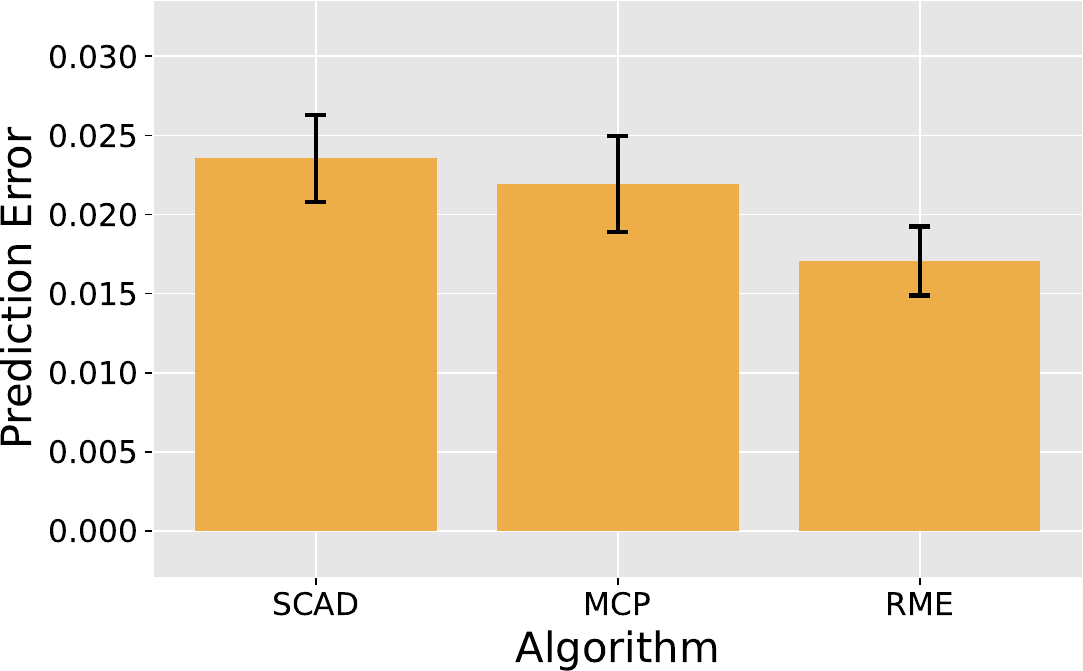}
\caption{Bars depict prediction error of one task averaged over 20 trials, with corresponding 95\% confidence intervals. `RME' is \textsf{RMEstimator} using the LASSO penalty, and `SCAD' and `MCP' are our \textsf{RMEstimator} using the SCAD~\citep{fan2001variable} and MCP~\citep{zhang2010nearly} penalties respectively.}
\label{fig:scadmcp}
\end{figure}

\subsection{Demonstration of Sparse Heterogeneity}\label{app:sprs_heatmap}

In Figure~\ref{fig:heatmap_deltas}, we used a hypothesis testing procedure based on the nonparametric bootstrap from the literature (see, e.g., \S 12.8.2 in \cite{wooldridge2010econometric}, or \S 3.4 in \cite{wasserman2006all}) to illustrate that our health risk prediction dataset supported our assumption of sparse heterogeneity. We provide additional details below.

We first compute separate linear estimators $\{\widehat{\beta}^j\}_{j\in[13]}$ for each of the 13 hospitals, then use the trimmed mean to estimate the shared parameter $\widehat{\beta}^\dagger=\texttt{TrimmedMean}(\{\widehat{\beta}^j\}_{j \in [13]},\;\omega)$, and finally compute the resulting task-specific parameters $\{\widehat{\delta}^j\}_{j\in[13]}$ for each hospital by subtracting the estimated shared parameter $\widehat{\beta}^\dagger$ from $\{\widehat{\beta}^j\}_{j=1}^{13}$, i.e., $\widehat{\delta}^j = \widehat{\beta}^j - \widehat{\beta}^\dagger$. Note that we do not use LASSO as in our algorithm design in \S\ref{sec:robmulti_est_overview}, i.e., this procedure does not impose any sparse heterogeneity structure on the model parameters $\{\beta^j\}_{j\in[13]}$. 

After we obtain the task-specific estimates $\{\widehat{\delta}^j\}_{j\in[13]}$ for each hospital, we run a hypothesis test on each entry of $\delta^j$ directly. We use a bootstrap hypothesis test across 500 random draws of the training data to determine whether the $i^{th}$ coefficient of $\delta^j$ (i.e., $\delta^j_{(i)}$) is statistically distinguishable from zero --- that is, the null hypothesis $\delta_{(i)}^j=0$ is not rejected at the 5\% significance level. We set coefficient $i$ of row $j$ (i.e., $\widehat\delta^j_{(i)}$) to zero if the null hypothesis is not rejected at a 5\% significance level, and otherwise maintain the estimate $\widehat\delta^j_{(i)}$. 

Given knowledge of the true trimmed mean hyperparameter $\omega$, this testing procedure directly follows the standard nonparametric bootstrap procedure from the literature (see, e.g., \S 12.8.2 in \cite{wooldridge2010econometric}, or \S 3.4 in \cite{wasserman2006all}). Particularly, for the hypothesis testing procedure, we can create a 95\% pivotal confidence interval \citep{wasserman2006all} for each entry of $\delta^j$ and check if it covers the value $0$; if it covers $0$, than the null is not rejected. We find that the sparse heterogeneity pattern in Figure 2 (which takes $\omega=0.1$) is robust against varying values of $\omega$, as shown in Figure~\ref{fig:heatmap_deltas_omgs} below.

\begin{figure}[htbp]
\centering
\begin{subfigure}[b]{0.48\textwidth}
  \centering
  \includegraphics[width=\textwidth]{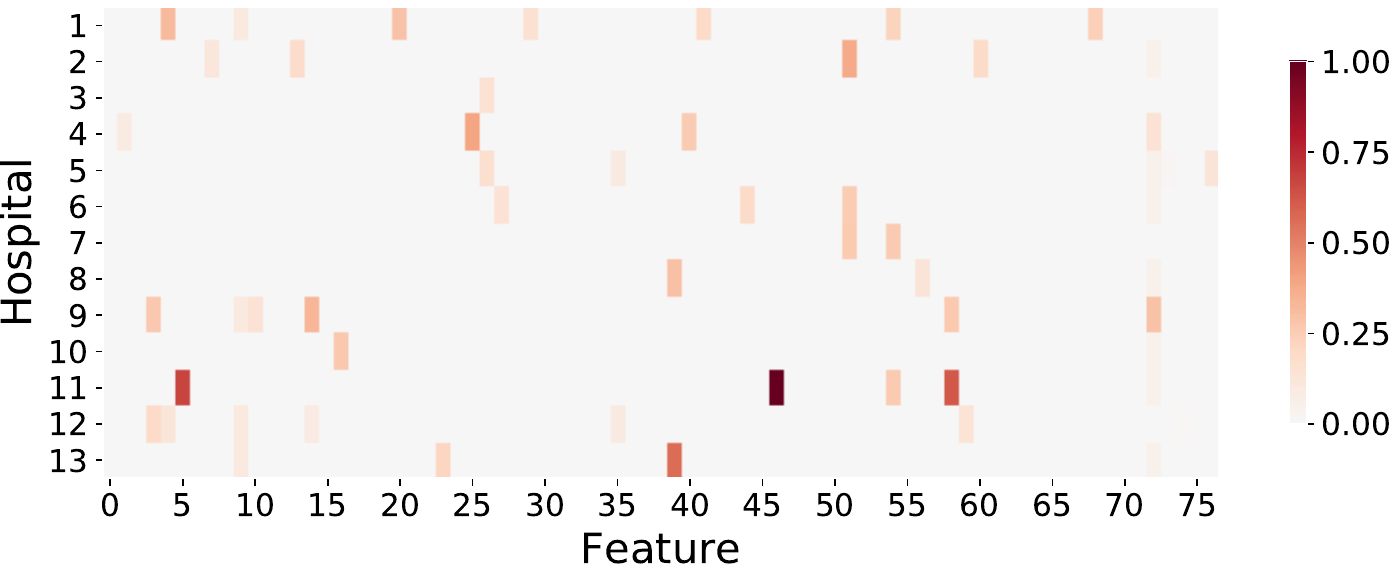}
  \caption{$\omega=0.05$}
\end{subfigure}
\begin{subfigure}[b]{0.48\textwidth}
  \centering
  \includegraphics[width=\textwidth]{Figures/fig_dbs_heatmap_coeftest_N13d77omg0.1.pdf}
  \caption{$\omega=0.1$}
\end{subfigure}\\
\begin{subfigure}[b]{0.48\textwidth}
  \centering
  \includegraphics[width=\textwidth]{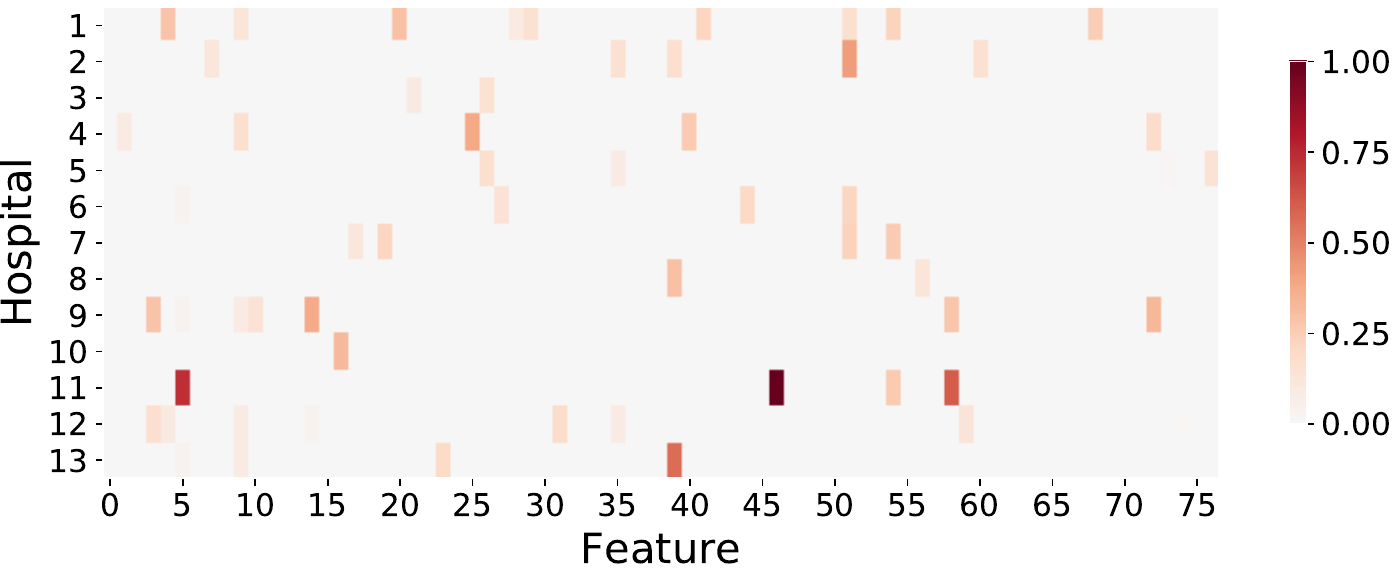}
  \caption{$\omega=0.2$}
\end{subfigure}
\caption{Heatmap of nonzero coefficients given by the estimated task-specific parameters $\{|\delta^j|\}_{j=1}^{13}$ for varying $\omega$. We set coefficient $i$ of row $j$ to 0 if the null hypothesis $\delta_{(i)}^j=0$ is not rejected at a 5\% significance level. }
\label{fig:heatmap_deltas_omgs}
\end{figure}

\subsection{Diabetes Risk Prediction using Logistic Regression}\label{app:loghealth}

Recall that the diabetes risk prediction task from \S\ref{sec:exp_off} is actually a classification problem with binary outcomes. Thus, we also compare the logistic regression analog of the \textsf{RMEstimator} (as described in \S\ref{ssec:glm}) with the logistic regression analogs of all the baseline algorithms (i.e., group LASSO, nuclear-norm regularization, LASSO, and regression, all estimated using maximum likelihood estimation with the logistic loss). The experimental setup is identical to the one described in Appendix~\ref{app:exp-diabetes}. Figure~\ref{fig:diabetes_off_lr} shows the resulting out-of-sample predictive accuracy for the \textsf{RMEstimator} and other baseline algorithms. Once again, we find that the \textsf{RMEstimator} achieves the best performance; yet, in this specific task, logistic regression does not perform as well as its linear counterpart for all methods (as can be seen by comparing the scale of the $y$-axis in Figure~\ref{fig:diabetes_off_lr} with that of Figure~\ref{fig:diabetes_off}). We hypothesize that this may be due to the small sample sizes involved. Linear and logistic regression estimates are often highly correlated even when the outcomes are binary, and produce nearly identical decisions~\citep[see, e.g.,][]{pohlman2003comparison}; however, linear models are unbiased in small samples, enabling faster convergence and improved multitask learning in the low-data regime, which may explain our improved performance with linear classifiers. In practice, one should choose the best predictor based on out-of-sample AUC; thus, we report results based on linear models in the main text.

\begin{figure}[htbp]
\centering
\begin{subfigure}[b]{0.42\textwidth}
  \centering
  \includegraphics[width=\textwidth]{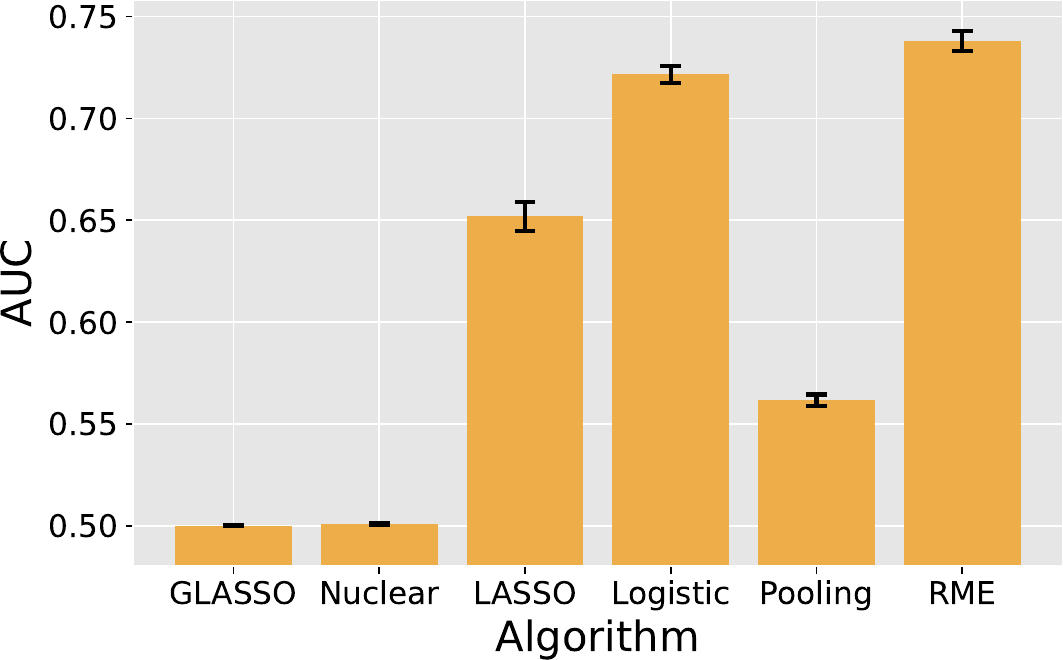}
  \caption{Hospital A}
\end{subfigure}
\begin{subfigure}[b]{0.42\textwidth}
  \centering
  \includegraphics[width=\textwidth]{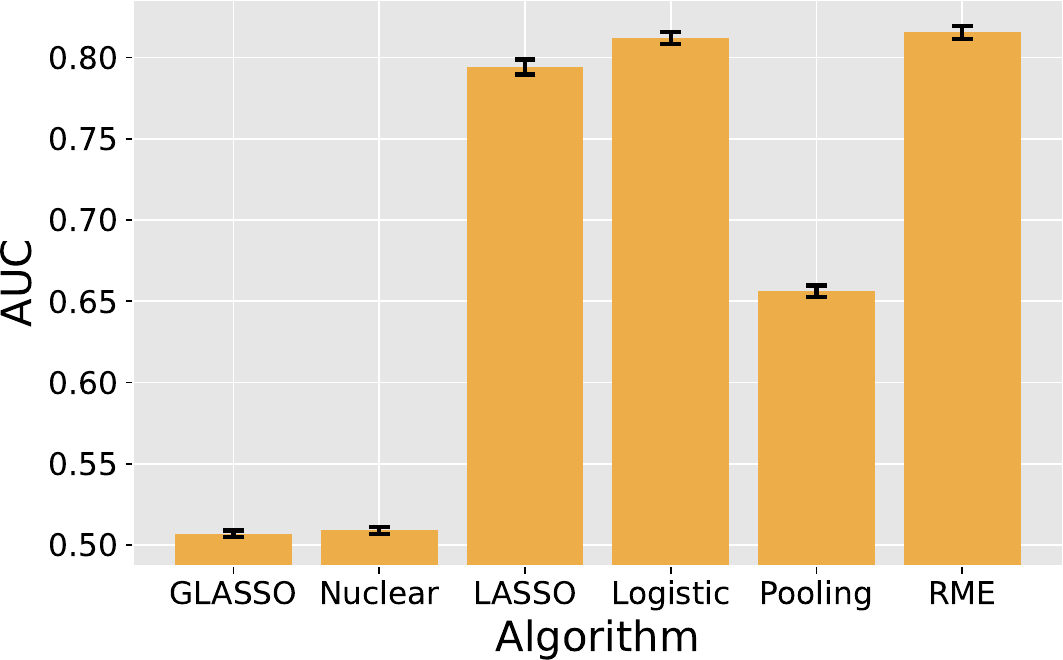}
  \caption{Hospital B}
\end{subfigure}
\caption{Bars depict out-of-sample preformance measured by AUC at one hospital (averaged over 1,000 trials), with 95\% confidence intervals. Hospitals A and B have 301 and 246 unique patients respectively. All the methods are based on logistic regression. `GLASSO' refers to group LASSO, `Nuclear' nuclear-norm regularization, and `RME' our robust multitask estimator.}
\label{fig:diabetes_off_lr}
\end{figure}

\subsection{Robustness to Hyperparameters in \textsf{RMBandit}}\label{app:add_exp_robbandit}

We now study the cumulative expected regret of \textsf{RMBandit} algorithm varying the hyperparameters specific to our algorithm: (i) $q_0\in\{0.2, 0.5, 0.7, 1\}$ (the hyperparameter $q_0$ is such that $q=q_0 KN$ for $q$ in Algorithm~\ref{alg:tmean_bdt}), (ii) $\omega_0, \zeta_{1, 0}, \eta_{1, 0}\in\{0.1, 0.2, 0.3\}$. In the following, we only focus on setting (b) in Figure~\ref{fig:synthetic_on}, i.e., $N=10$, $K=10$, $d=20$, and $s=2$.
The results in Figure~\ref{fig:synthetic_vary_on} are calculated over $T=10,000$ total time steps and averaged over 20 trials. We find that the cumulative regret is not substantially affected considering varying values of these hyperparameters; this suggests that our algorithm is robust, which is important especially in empirical applications where these hyperparameters might not be well specified.

\begin{figure}[htbp]
\centering
\begin{subfigure}[b]{0.32\textwidth}
  \centering
  \includegraphics[width=\textwidth]{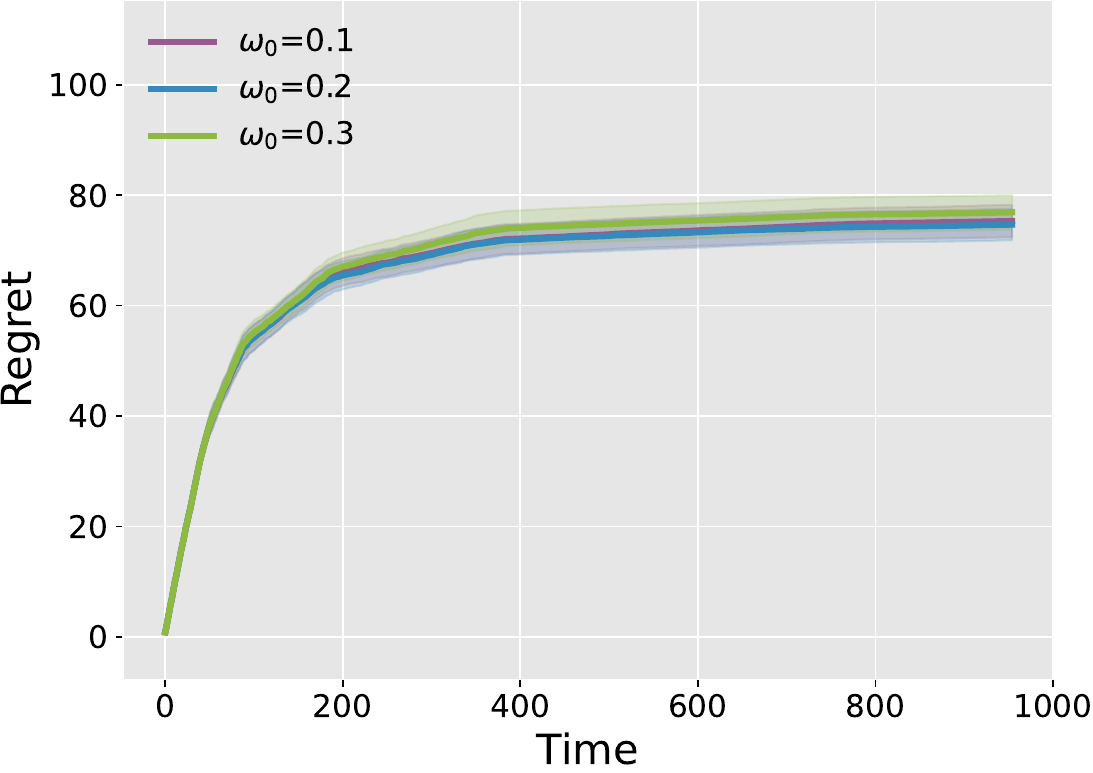}
  \caption{Varying $\omega_0$}
\end{subfigure}
\begin{subfigure}[b]{0.32\textwidth}
  \centering
  \includegraphics[width=\textwidth]{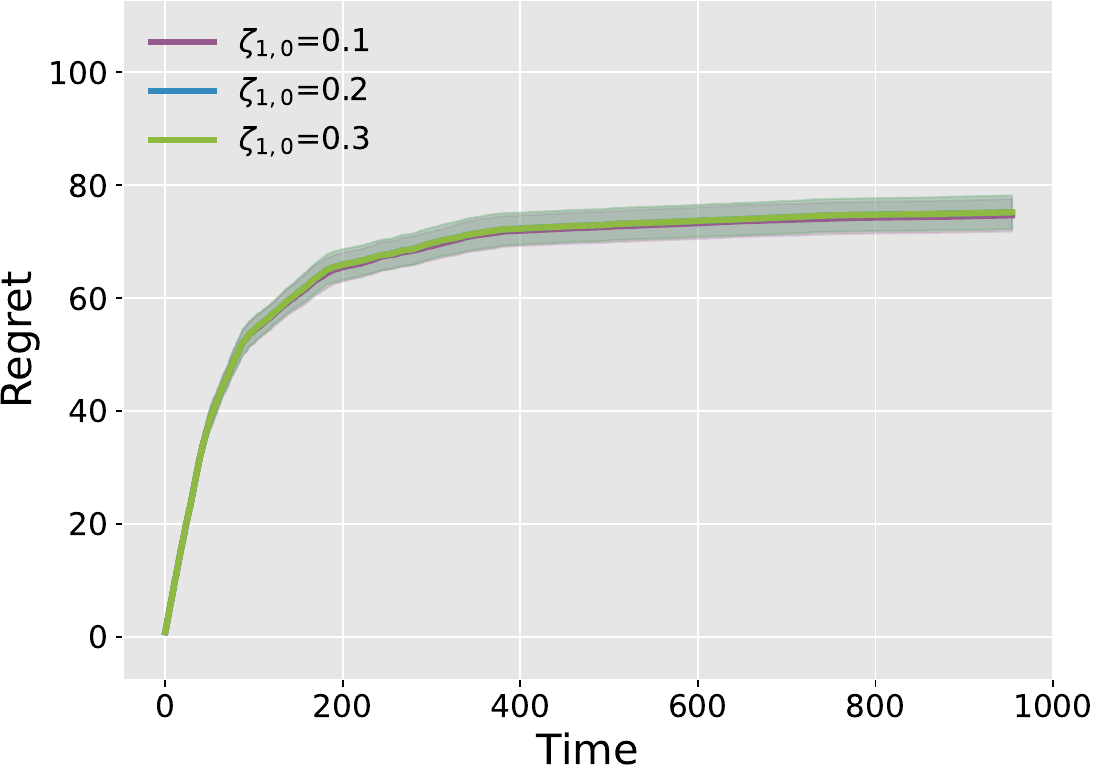}
  \caption{Varying $\zeta_{1, 0}$}
\end{subfigure}
\begin{subfigure}[b]{0.32\textwidth}
  \centering
  \includegraphics[width=\textwidth]{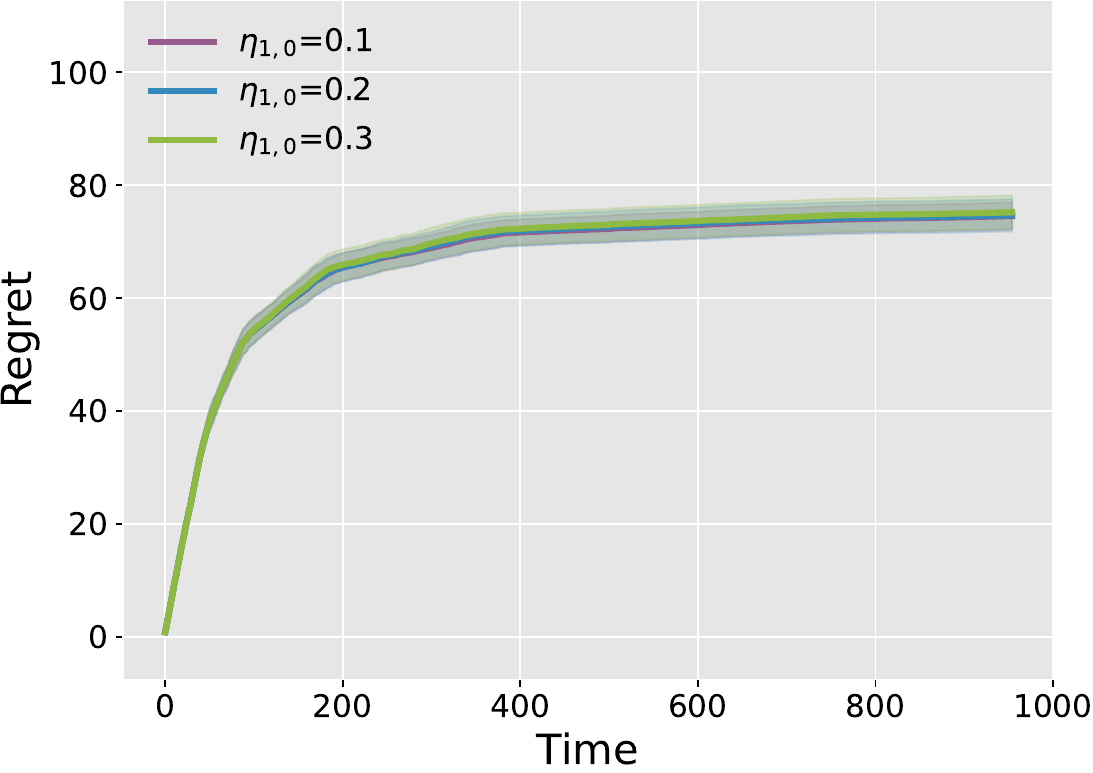}
  \caption{Varying $\eta_{1, 0}$}
\end{subfigure}\\
\begin{subfigure}[b]{0.32\textwidth}
  \centering
  \includegraphics[width=\textwidth]{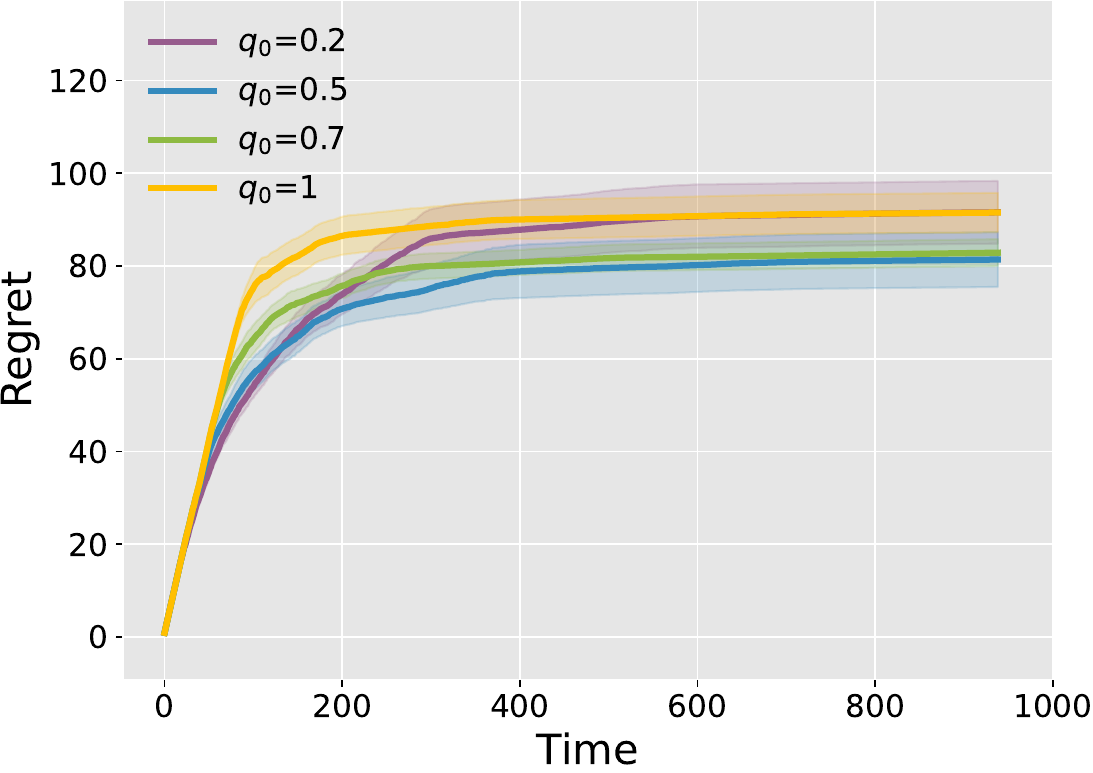}
  \caption{Varying $q_0$}
\end{subfigure}
\caption{Lines depict the cumulative regret averaged over 20 trials of a single linear contextual bandit, with shaded regions the corresponding 95\% confidence intervals.}
\label{fig:synthetic_vary_on}
\end{figure}

\end{APPENDICES}

\end{document}